\documentclass{article}
\usepackage[utf8]{inputenc}

\usepackage{geometry}
\usepackage{natbib}

\usepackage[unicode=true,pdfusetitle,
 bookmarks=true,bookmarksnumbered=false,bookmarksopen=false,
 breaklinks=true,pdfborder={0 0 0},pdfborderstyle={},backref=false,colorlinks=true]{hyperref}
\hypersetup{citecolor=blue}

\usepackage{amsmath}
\usepackage{amssymb}
\usepackage{amsthm}
\usepackage{amsfonts}
\usepackage{appendix}
\usepackage{booktabs}
\usepackage{multirow}

\makeatletter
\newcommand{\Biggg}{\bBigg@{3.5}}
\makeatother

\theoremstyle{plain}
\newtheorem{thm}{\protect\theoremname}
\theoremstyle{plain}
\newtheorem{lem}[thm]{\protect\lemmaname}
\theoremstyle{remark}

\theoremstyle{plain}
\newtheorem{cor}[thm]{\protect\corollaryname}
\theoremstyle{plain}

\providecommand{\corollaryname}{Corollary}
\providecommand{\lemmaname}{Lemma}
\providecommand{\remarkname}{Remark}
\providecommand{\theoremname}{Theorem}
\providecommand{\conjecturename}{Conjecture}

\usepackage{algorithm}
\usepackage{algpseudocode}
\usepackage{float}

\usepackage{enumitem}
\usepackage{caption} 

\newcommand{\R}{\mathbb{R}}

\renewcommand{\P}{\mathbb{P}}
\newcommand{\E}{\mathbb{E}}
\newcommand{\Var}{\mathbb{V}}

\newcommand{\Ehat}{\widehat{\mathbb{E}}}

\newcommand{\Phat}{\widehat{P}}
\newcommand{\Vhat}{\widehat{V}}
\newcommand{\Qhat}{\widehat{Q}}

\newcommand{\pistar}{\pi^\star}
\newcommand{\pihstar}{\widehat{\pi}^\star}

\newcommand{\one}{\mathbf{1}}

\newcommand{\ind}{\mathbb{I}}

\newcommand{\pihat}{\widehat{\pi}}
\newcommand{\Vhatstar}{\widehat{V}^\star}

\newcommand{\Vc}{\overline{V}}

\newcommand{\hc}{\overline{h}}

\newcommand{\rhohat}{\widehat{\rho}}
\newcommand{\hhat}{\widehat{h}}

\newcommand{\A}{\mathcal{A}}
\renewcommand{\S}{\mathcal{S}}

\newcommand{\gammared}{{\overline{\gamma}}}

\newcommand{\Otilde}{\widetilde{O}}

\newcommand{\tmix}{\tau_{\mathrm{unif}}}

\DeclareMathOperator*{\Clim}{\text{C-lim}}

\newcommand{\htilde}{\widetilde{h}}
\newcommand{\rhot}{\widetilde{\rho}}

\global\long\def\infnorm#1{\left\Vert #1\right\Vert _{\infty}}%
\global\long\def\infinfnorm#1{\left\Vert #1\right\Vert _{\infty \to \infty}}%
\global\long\def\onenorm#1{\left\Vert #1\right\Vert _{1}}%
\global\long\def\spannorm#1{\left\Vert #1\right\Vert _{\textnormal{span}}}%

\usepackage{scalerel,stackengine}
\newcommand\cig[1]{\scalerel*[5.5pt]{\Big#1}{%
  \ensurestackMath{\addstackgap[1.5pt]{\big#1}}}}
\newcommand\cigl[1]{\mathopen{\cig{#1}}}
\newcommand\cigr[1]{\mathclose{\cig{#1}}}

\global\long\def\infnorm#1{\left\Vert #1\right\Vert _{\infty}}%
\global\long\def\spannorm#1{\left\Vert #1\right\Vert _{\textnormal{span}}}%
\global\long\def\spannorm#1{\Vert #1\Vert _{\textnormal{span}}}%

\newcommand{\SolveAMDP}{\texttt{SolveAMDP}}
\newcommand{\SolveDMDP}{\texttt{SolveDMDP}}

\newcommand{\Topt}{\widetilde{\mathcal{T}}}

\newcommand{\Phatanc}{\underline{\Phat}}
\newcommand{\rhohatanc}{ \underline{\rhohat}}
\newcommand{\hhatanc}{\underline{\hhat}}
\newcommand{\Panc}{\underline{P}}
\newcommand{\rhoanc}{\underline{\rho}}
\newcommand{\hanc}{\underline{h}}
\newcommand{\Thatanc}{\underline{\widehat{\mathcal{T}}}}
\newcommand{\That}{\widehat{\mathcal{T}}}

\usepackage{leftindex}

\newcommand{\Vhatpert}{\leftindex_p{\Vhat}}
\newcommand{\Qhatpert}{\leftindex_p{\Qhat}}
\newcommand{\pihatpert}{\pihat^\star_p}
\newcommand{\Thatancpert}{\leftindex_p{\underline{\widehat{\mathcal{T}}}} }
\newcommand{\hhatancpert}{\leftindex_p{\hhatanc}}

\newcommand{\empbod}{\widehat{\gamma}^\star}
\newcommand{\empbodd}{\widehat{\gamma}^{\sharp}}
\newcommand{\ox}{\overline{x}}

\usepackage{xspace}
\newcommand{\plugin}{plug-in\xspace}
\newcommand{\Plugin}{Plug-in\xspace}

\usepackage{tikz}
\usetikzlibrary{shapes, arrows, positioning}

\definecolor{darkpastelgreen}{rgb}{0.01, 0.75, 0.24}
\usepackage[showdeletions]{color-edits}
\addauthor[Matthew]{mz}{red}
\addauthor[Yudong]{yc}{darkpastelgreen}
\usepackage{nicefrac}       % compact symbols for 1/2, etc.
\usepackage{microtype}      % microtypography

\renewcommand{\algorithmiccomment}[1]{\begin{flushright}$\triangleright$~#1\end{flushright}}

\date{}

\title{The \Plugin Approach for Average-Reward and Discounted MDPs: Optimal Sample Complexity Analysis}

\usepackage{authblk}
\author{Matthew Zurek}
\author{Yudong Chen}
\affil{Department of Computer Sciences, University of Wisconsin-Madison\\\texttt{\{matthew.zurek,yudong.chen\}@wisc.edu}}

\begin{document}

\maketitle

\begin{abstract}%
  We study the sample complexity of the \plugin approach for learning $\varepsilon$-optimal policies in average-reward Markov decision processes (MDPs) with a generative model. The \plugin approach constructs a model estimate then computes an average-reward optimal policy in the estimated model.  Despite representing arguably the simplest algorithm for this problem, the \plugin approach has never been theoretically analyzed. Unlike the more well-studied discounted MDP reduction method, the \plugin approach requires no prior problem information or parameter tuning. Our results fill this gap and address the limitations of prior approaches, as we show that the \plugin approach is optimal in several well-studied settings without using prior knowledge. Specifically it achieves the optimal diameter- and mixing-based sample complexities of $\widetilde{O}\left(SA \frac{D}{\varepsilon^2}\right)$ and $\widetilde{O}\left(SA \frac{\tau_{\mathrm{unif}}}{\varepsilon^2}\right)$, respectively, without knowledge of the diameter $D$ or uniform mixing time $\tau_{\mathrm{unif}}$.
  We also obtain span-based bounds for the \plugin approach, and complement them with algorithm-specific lower bounds suggesting that they are unimprovable. Our results require novel techniques for analyzing long-horizon problems which may be broadly useful and which also improve results for the discounted \plugin approach, removing effective-horizon-related sample size restrictions and obtaining the first optimal complexity bounds for the full range of sample sizes without reward perturbation.%
\end{abstract}

\section{Introduction}
Reinforcement learning (RL) has emerged as a powerful framework for sequential decision-making problems, where an agent learns to make decisions by interacting with an environment to maximize cumulative rewards. Average reward RL, in particular, focuses on optimizing the long-term average reward per time step, making it especially relevant in ongoing, infinite-horizon tasks where the goal is to maintain consistent performance over time.
In this paper, we study the foundational theoretical problem of determining the sample complexity required to learn a near-optimal policy in a Markov decision process (MDP) with access to a generative model. Although recent research has made significant strides in resolving the optimal sample complexity for this problem, a large amount of prior work (including all sample-optimal methods) relies on methods designed for discounted MDPs, where future rewards are multiplied by a discount factor to prioritize immediate rewards. This approach has several drawbacks: selecting the appropriate discount factor (or sequence of factors) is crucial and often requires prior knowledge about the problem, which may not be available in practice, potentially degrading performance. Even when the discount factor can be suitably tuned, it is still extrinsic to the average-reward problem, making it arguably unnatural to require its introduction. Technical challenges have hindered the analysis of more direct average-reward algorithms. 

Our study focuses on analyzing the average-reward \emph{\plugin} approach. This approach estimates the parameters of the MDP model and then uses any method to compute the optimal average-reward policy for the estimated model. In the context of discounted MDPs, this approach has been called model-based planning \citep{agarwal_model-based_2020, li_breaking_2020}, although we note that the \plugin approach is a particular ``model-based'' algorithm. We also note that the \plugin approach is a generic template for constructing estimators for a functional of an unknown distribution/model (by plugging the empirical distribution/model into the functional), which is broadly used beyond RL.
This is arguably the most natural model-based approach for solving average-reward MDPs, yet its finite-sample properties have never been theoretically examined. Not only does our analysis fill a major gap in our understanding of a basic algorithm for this problem, but also we show that this algorithm is optimal in several settings without requiring strong assumptions about prior knowledge of the problem, thus addressing many limitations of previous approaches. 

In particular, when combined with the stabilizing \textit{anchoring} technique which has previously appeared in the average-reward literature, we show that this algorithm can simultaneously achieve the optimal diameter- and mixing-based sample complexities of $\Otilde\left(SA \frac{D}{\varepsilon^2}\right)$ and $\Otilde\left(SA \frac{\tmix}{\varepsilon^2}\right)$ for learning an $\varepsilon$-optimal policy, respectively,
where $D$ is the diameter and $\tmix$ is the uniform mixing time, without needing to have prior knowledge of $D$ or $\tmix$ and without needing to tailor the algorithm to the particular situations. These results are corollaries of our bias-span-based complexity bounds for weakly communicating MDPs, for example $\Otilde\cig(SA \frac{\spannorm{h^\star} + \spannorm{\hhat^\star} + 1}{\varepsilon^2}\cig)$, where $\spannorm{h^\star}$ is the optimal bias span and $\spannorm{\hhat^\star}$ is the optimal bias span in a certain estimated MDP. We further show that the analysis behind this (and related) span-based bound is unimprovable, in the sense that the term $\spannorm{\hhat^\star}$ cannot be removed in general for the performance of the \plugin method.

While the average-reward \plugin approach can be seen as a large-discount-factor limit of the discounted-reward \plugin approach, previous discounted analyses are incapable of being adapted to this problem, requiring the development of novel techniques for analyzing the error of long-horizon problems which may be broadly useful.
In particular these techniques lead to improved results for the \plugin approach for discounted-reward MDPs, including removing effective-horizon-related sample size restrictions of previous results which achieve quadratic dependence on the effective horizon for the fixed MDP setting. 
We also obtain the first optimal complexity bounds for the full range of sample sizes without the need for reward perturbation.

\subsection{Related Work}

\setlength{\textfloatsep}{20pt plus 1.0pt minus 2.0pt}
\begin{table}[t]
{
\renewcommand{\arraystretch}{1.4} 
\centering
%\begin{tabular}{cccc}
\begin{tabular}{p{0.21\textwidth}ccc}
%\begin{tabular}{p{0.2\textwidth}|p{0.2\textwidth}|p{0.2\textwidth}|p{0.15\textwidth}|p{0.2\textwidth}} 
\toprule
Algorithm & Sample Complexity & Reference & \parbox[c]{1.6cm}{Prior\\Knowledge} \\ \midrule
%Primal-Dual SMD & $\Otilde\cigl( SA \frac{\tmix^2}{\varepsilon^2}\cigr)$ & \cite{jin_efficiently_2020} & Yes  \\ 
%DMDP Reduction & $\Otilde\left( SA \frac{\tmix}{\varepsilon^3}\right)$ & \cite{jin_towards_2021} & Yes  \\ 
%Policy Mirror Descent & $\Otilde\cigl( SA \frac{\tmix^3}{\varepsilon^2}\cigr)$ & \cite{li_stochastic_2024} & Yes \\ 
DMDP Reduction & $SA \frac{\tmix}{\varepsilon^2} $ & \cite{wang_optimal_2023} & Yes  \\ %\midrule
%DMDP Reduction & $\Otilde\left(SA\frac{\spannorm{h^\star}}{\varepsilon^3} \right)$ & \cite{wang_near_2022} & Yes  \\ 
%Refined Q-Learning & $\Otilde\cigl(SA\frac{\spannorm{h^\star}^2}{\varepsilon^2} \cigr)$ & \cite{zhang_sharper_2023} & Yes  \\ 
DMDP Reduction & $SA\frac{\spannorm{h^\star}+1}{\varepsilon^2}$ & \cite{zurek_span-based_2025} & Yes  \\ \midrule
Diameter Estimation + DMDP Reduction & $SA\frac{D}{\varepsilon^2} + S^2 A D^2 $ & \cite{tuynman_finding_2024} & No  \\ 
Dynamic Horizon Q-Learning & $SA\frac{\tmix^8}{\varepsilon^8} $ & \cite{jin_feasible_2024} & No  \\ 
Stochastic Saddle-Point Optimization & $S^2A^2\frac{\spannorm{h^{\pihat}}^4}{\varepsilon^2} $ & \cite{neu_dealing_2024} & No  \\ \midrule
\Plugin Approach & $SA\frac{\spannorm{h^\star} + \spannorm{\hhat^\star}+1}{\varepsilon^2}\log(\frac{1}{1-\empbod}) $ & Our Theorem \ref{thm:AMDP_plugin_thm} & No  \\ \midrule
%Anchored+Perturbed \Plugin Approach & $\Otilde\cigl(SA\frac{\spannorm{h^\star} + \min\{\spannorm{\hhatanc^\star}, \spannorm{\hanc^{\pihat}} \}}{\varepsilon^2} \cigr)$ & Our Theorem \ref{thm:AMDP_best_of_both} & No \\
\multirow{ 3}{0.2\textwidth}{Anchored+Perturbed \Plugin Approach} &  \parbox[c]{5cm}{$SA\frac{\spannorm{h^\star} + \min\{\spannorm{\hhatanc^\star}, \spannorm{\hanc^{\pihat}} \}+1}{\varepsilon^2} $} & Our Theorem \ref{thm:AMDP_best_of_both} & \multirow{ 3}{*}{No} \\
& $SA \frac{D}{\varepsilon^2} $ & Our Corollary \ref{thm:diameter_complexity} & \\
& $SA \frac{\tmix}{\varepsilon^2} $ & Our Corollary \ref{thm:mixing_complexity} & \\ \midrule
 \parbox[c]{3cm}{$\sqrt{n}$-Horizon\\DMDP Reduction}  & $SA\frac{\spannorm{h^\star}^2 + 1 }{\varepsilon^2} $ & Our Theorem \ref{thm:span_based_without_knowledge} & No \\ \bottomrule
\end{tabular}
\caption{\textbf{Algorithms and sample complexity bounds for average reward MDPs} with $S$ states and $A$ actions, for finding an $\varepsilon$-optimal policy under a generative model (up to $\log$ factors). All results assume at least that $P$ is weakly communicating. Furthermore all bounds involving $\tmix$ assume that $P$ is uniformly mixing, and all bounds involving $D$ assume that $P$ is communicating. $\spannorm{h^{\pihat}}$ is the bias span of the policy $\pihat$ returned by the algorithm. See Section \ref{sec:main_results} for other definitions. We always have $\spannorm{h^\star} \leq D$ and $\spannorm{h^\star}, \spannorm{h^{\pihat}}, \spannorm{\hanc^\pi} \leq 3\tmix$, and we have $\spannorm{\hhatanc^\star} \leq O(D)$ with high probability when $n \geq \widetilde{\Omega}(D)$. 
}
\label{table:AMDPs}
}
\end{table}

We summarize related work on learning optimal policies in average-reward MDPs (AMDPs) in Table \ref{table:AMDPs}.
There is a long history of work on this problem which we do not fully recount here (e.g. \cite{jin_efficiently_2020, jin_towards_2021, li_stochastic_2024, wang_near_2022, zhang_sharper_2023}), instead starting with the works \cite{wang_optimal_2023} and \cite{zurek_span-based_2025} which were the first to obtain optimal sample complexities in their respective settings (we refer to their references for more history of this problem). Each of these works use the DMDP reduction approach with a carefully chosen effective horizon, $\frac{\tmix}{\varepsilon}$ and $\frac{\spannorm{h^\star}}{\varepsilon}$, respectively, which requires prior knowledge of the values of these complexity parameters. The $\Otilde\left(SA \frac{\tmix}{\varepsilon^2} \right)$ complexity result of \cite{zurek_span-based_2025} implies a $\Otilde\left(SA\frac{D}{\varepsilon^2} \right)$ complexity for the finite diameter setting since $\spannorm{h^\star} \leq D$ \citep{bartlett_regal_2012, lattimore_bandit_2020}, and it also implies the $\Otilde\left(SA \frac{\tmix}{\varepsilon^2} \right)$ complexity obtained by \cite{wang_optimal_2023} since $\spannorm{h^\star} \leq 3 \tmix$ (Lemma \ref{lem:mixing_param_relationships}, also see \cite{wang_near_2022}). These results match minimax lower bounds of $\widetilde{\Omega}\left(SA \frac{\tmix}{\varepsilon^2} \right)$ \citep{jin_towards_2021} and $\widetilde{\Omega}\left(SA \frac{D}{\varepsilon^2} \right)$ (by the relationships between $\tmix, D,$ and $\spannorm{h^\star}$, these both imply a $\widetilde{\Omega}\left(SA \frac{\spannorm{h^\star}}{\varepsilon^2} \right)$ lower bound).

Recently there has been significant interest in removing the need for prior knowledge of complexity parameters. \cite{tuynman_finding_2024} show that an upper bound for the diameter can be estimated and then used within the approach of \cite{zurek_span-based_2025} to circumvent its need for parameter knowledge. However, results from \cite{tuynman_finding_2024} and \cite{zurek_span-based_2025} imply that a similar approach cannot be used to obtain the optimal span-based complexity, both showing that it is not generally possible to obtain a multiplicative approximation of $\spannorm{h^\star}$ with $\text{poly}(SA\spannorm{h^\star})$ samples. In the uniformly mixing setting, \cite{jin_feasible_2024} use a Q-learning-style algorithm with increasing discount factors to remove the need for knowledge of $\tmix$. \cite{neu_dealing_2024} develop an algorithm based on stochastic saddle-point optimization that does not require parameter knowledge in the general weakly communicating setting, but their bounds depend on $\spannorm{h^{\pihat}}$, the bias of the algorithm output policy, which is not generally related to $\spannorm{h^\star}$.

\setlength{\textfloatsep}{20pt plus 1.0pt minus 2.0pt}
\begin{table}[t]
{
\renewcommand{\arraystretch}{1.6} 
\centering
\begin{tabular}{cccc}
\toprule
Reference & Sample Complexity &  \parbox[c]{2.1cm}{Sample Size\\Requirement} &  \parbox[c]{2.1cm}{Requires\\Perturbation?} \\ \midrule
\cite{azar_minimax_2013} & $SA \frac{1}{(1-\gamma)^3\varepsilon^2} $ & $\varepsilon \leq \sqrt{\frac{1}{(1-\gamma)S}}$ & No \\ 
\cite{agarwal_model-based_2020} & $SA \frac{1}{(1-\gamma)^3\varepsilon^2} $ & $\varepsilon \leq \sqrt{\frac{1}{1-\gamma}}$ & No \\ 
\cite{li_breaking_2020} & $SA \frac{1}{(1-\gamma)^3\varepsilon^2} $ & None & Yes \\ 
\cite{wang_optimal_2023} & $SA \frac{\tmix}{(1-\gamma)^2\varepsilon^2} $ & $\varepsilon\leq \sqrt{\frac{\tmix}{1-\gamma}}$  & Yes \\ 
%\cite{wang_optimal_2023} & $\Otilde\left(SA \frac{\tmix}{(1-\gamma)^2\varepsilon^2} \right)$ & $\varepsilon\leq \sqrt{(1-\gamma)\tmix}$  & Yes \\ \midrule
\cite{zurek_span-based_2025} & $SA \frac{\spannorm{h^\star}+1}{(1-\gamma)^2\varepsilon^2} $ & $\varepsilon\leq \spannorm{h^\star}$  & Yes \\ \midrule
Our Theorem \ref{thm:DMDP_pert_thm} & $SA \frac{\spannorm{V^\star} + \spannorm{V^{\pihat}}+1}{(1-\gamma)^2\varepsilon^2} $ & None  & Yes \\ 
Our Theorem \ref{thm:DMDP_main_thm} & $SA \frac{\spannorm{V^\star} + \spannorm{\Vhat^\star}+1}{(1-\gamma)^2\varepsilon^2} $ & None  & No \\ 
Our Theorem \ref{thm:DMDP_main_thm}+Lemma \ref{lem:small_sample_span_bound} &$SA \frac{\spannorm{h^\star}+1}{(1-\gamma)^2\varepsilon^2} $ & $\varepsilon\leq \spannorm{h^\star}$  & No  \\ \bottomrule
%Reference & Sample Complexity &  \parbox[c]{2.1cm}{Sample Size\\Requirement} &  \parbox[c]{2.1cm}{Requires\\Perturbation?} \\ \midrule
%\cite{azar_minimax_2013} & $\Otilde\cigl(SA \frac{1}{(1-\gamma)^3\varepsilon^2} \cigr)$ & $\varepsilon \leq \sqrt{\frac{1}{(1-\gamma)S}}$ & No \\ 
%\cite{agarwal_model-based_2020} & $\Otilde\cigl(SA \frac{1}{(1-\gamma)^3\varepsilon^2} \cigr)$ & $\varepsilon \leq \sqrt{\frac{1}{1-\gamma}}$ & No \\ 
%\cite{li_breaking_2020} & $\Otilde\cigl(SA \frac{1}{(1-\gamma)^3\varepsilon^2} \cigr)$ & None & Yes \\ 
%\cite{wang_optimal_2023} & $\Otilde\cigl(SA \frac{\tmix}{(1-\gamma)^2\varepsilon^2} \cigr)$ & $\varepsilon\leq \sqrt{\frac{\tmix}{1-\gamma}}$  & Yes \\ 
%%%%\cite{wang_optimal_2023} & $\Otilde\left(SA \frac{\tmix}{(1-\gamma)^2\varepsilon^2} \right)$ & $\varepsilon\leq \sqrt{(1-\gamma)\tmix}$  & Yes \\ \midrule
%\cite{zurek_span-based_2025} & $\Otilde\cigl(SA \frac{\spannorm{h^\star}+1}{(1-\gamma)^2\varepsilon^2} \cigr)$ & $\varepsilon\leq \spannorm{h^\star}$  & Yes \\ \midrule
%Our Theorem \ref{thm:DMDP_pert_thm} & $\Otilde\cigl(SA \frac{\spannorm{V^\star} + \spannorm{V^{\pihat}}+1}{(1-\gamma)^2\varepsilon^2} \cigr)$ & None  & Yes \\ 
%Our Theorem \ref{thm:DMDP_main_thm} & $\Otilde\cigl(SA \frac{\spannorm{V^\star} + \spannorm{\Vhat^\star}+1}{(1-\gamma)^2\varepsilon^2} \cigr)$ & None  & No \\ 
%Our Theorem \ref{thm:DMDP_main_thm}+Lemma \ref{lem:small_sample_span_bound} &$\Otilde\cigl(SA \frac{\spannorm{h^\star}+1}{(1-\gamma)^2\varepsilon^2} \cigr)$ & $\varepsilon\leq \spannorm{h^\star}$  & No  \\ \bottomrule
\end{tabular}
\caption{\textbf{Sample complexity bounds for the \plugin approach in $\gamma$-discounted MDPs} with $S$ states and $A$ actions, for finding an $\varepsilon$-optimal policy under a generative model (up to $\log$ factors). The sample size requirement is the valid range of $\varepsilon$ for the respective complexity guarantee, and a result is said to require perturbation if it utilizes a randomly perturbed reward vector. All results containing $\spannorm{h^\star}$ assume $P$ is weakly communicating, and all results containing $\tmix$ assume uniform mixing. If $P$ is weakly communicating then $\spannorm{V^\star} \leq 2 \spannorm{h^\star}$, and if $P$ is uniformly mixing then $\spannorm{V^\star}, \spannorm{V^{\pihat}} \leq 3 \tmix$. We also always have the naive bounds $\spannorm{V^\star}, \spannorm{V^{\pihat}} \leq \frac{1}{1-\gamma}$.}
\label{table:DMDPs}
}
\end{table}

We present related work on learning optimal policies in discounted MDPs (DMDPs) in Table~\ref{table:DMDPs}. This problem also has been extensively studied, and we include only results on the \plugin approach which obtain minimax-optimal sample complexities. The optimal complexity $\Otilde(SA \frac{1}{(1-\gamma)^3 \varepsilon^2})$ for learning a $\varepsilon$-discounted-optimal policy was first obtained by \cite{azar_minimax_2013} for a restrictive range of $\varepsilon$. This range was enlarged by \cite{agarwal_model-based_2020}, who introduce the absorbing MDP construction for decoupling statistical dependence which also finds use in our analysis. The matching lower bound is established by \cite{azar_sample_2012, sidford_near-optimal_2018}.
\cite{li_breaking_2020} are the first to achieve the optimal $\Otilde(SA \frac{1}{(1-\gamma)^3 \varepsilon^2})$ complexity for the full range $\varepsilon \in (0, \frac{1}{1-\gamma}]$, and their results actually yield a stronger instance-dependent bound. This stronger guarantee is used in both \cite{wang_optimal_2023} and \cite{zurek_span-based_2025} to obtain the complexity bounds of $\Otilde\cig(SA \frac{\tmix}{(1-\gamma)^2\varepsilon^2} \cig)$ and $\Otilde\cig(SA \frac{\spannorm{h^\star}}{(1-\gamma)^2\varepsilon^2} \cig)$, respectively, in the restricted situations that $P$ is uniformly mixing or is weakly communicating. However, the arguments within \cite{li_breaking_2020} implicitly require that $n \geq \Omega(\frac{1}{1-\gamma})$. This is without loss of generality if the goal is to show $\Otilde(SA \frac{1}{(1-\gamma)^3 \varepsilon^2})$ complexity, since this is equivalent to an error bound of $\Otilde\cigl(\sqrt{\frac{1}{(1-\gamma)^3n}} \cigr)$, which is only nontrivial error if it is below $\frac{1}{1-\gamma}$ which requires $n \geq \Omega(\frac{1}{1-\gamma})$. 
However, once the target is strengthened to an improved complexity like $\Otilde\cig(SA \frac{\tmix}{(1-\gamma)^2\varepsilon^2} \cig)$, the condition $n \geq \Omega(\frac{1}{1-\gamma})$ is equivalent to a sample size barrier of $\varepsilon \leq O(\sqrt{\frac{\tmix}{1-\gamma}})$. 
%(The work of \cite{zurek_span-based_2025} has an even stronger requirement that $\varepsilon \leq \spannorm{h^\star}$, due to additional restrictions. We believe that the arguments of \cite{li_breaking_2020} and \cite{zurek_span-based_2025} could be used to show a complexity bound like $\Otilde\left(SA \frac{\spannorm{V^\star} + \spannorm{V^{\pihat}}}{(1-\gamma)^2\varepsilon^2} \right)$ with a sample size restriction like $\varepsilon \leq O(\sqrt{\frac{\spannorm{V^\star} + \spannorm{V^{\pihat}}}{1-\gamma}})$, but then \cite{zurek_span-based_2025} provide additional arguments which show that $\spannorm{V^{\pihat}} \leq O(\spannorm{h^\star})$ when $\varepsilon \leq \spannorm{h^\star}$.)

\section{Problem Setup}

A Markov decision process (MDP) is a tuple $(\S, \A, P, r)$, where $\S$ is the finite set of states, $\A$ is the finite set of actions, $P : \S \times \A \to \Delta(\S)$ is the transition kernel with $\Delta(\S)$ denoting the probability simplex over $\S$, and $r : \S \times \A \to [0,1]$ is the reward function. We denote the cardinality of the state and action spaces as $S = |\S|$ and $A = |\A|$, respectively. Unless otherwise noted, all policies considered are Markovian (stationary) policies of the form $\pi : \S \to \Delta(\A)$. For any initial state $s_0 \in \S$ and policy $\pi$, we let $\E^\pi_{s_0}$ denote the expectation with respect to the probability distribution over trajectories $(S_0, A_0, S_1, A_1, \dots)$ where $S_0 = s_0$, $A_t \sim \pi(S_t)$, and $S_{t+1} \sim P(\cdot \mid S_t, A_t)$. 
We let $P_\pi$ denote the transition probability matrix of the Markov chain induced by $\pi$, that is, $\left(P_\pi\right)_{s,s'} := \sum_{a \in \A} \pi(a | s) P(s' \mid s, a)$. Likewise define $(r_\pi)_{s} := \sum_{a \in \A} \pi(a | s) r(s, a)$. We also consider $P$ as an $(\S \times \A)$-by-$ \S$ matrix where $P_{sa, s'} = P(s' \mid s, a)$.

We assume access to a generative model \citep{kearns_finite-sample_1998}, also known as a simulator, which provides independent samples from $P(\cdot \mid s, a)$ for any given $s \in \S, a \in \A$. $P$ itself is unknown. We assume the $r$ is deterministic and known, which is standard in generative settings (e.g., \citealt{agarwal_model-based_2020, li_breaking_2020}) since otherwise estimating the mean rewards is relatively easy.

\textbf{Discounted reward criterion~~} A discounted MDP is a tuple $(\S, \A, P, r, \gamma)$, where $\gamma \in (0,1)$ is the discount factor. For a stationary policy $\pi$, the (discounted) value function $V^\pi_\gamma : \S \to [0, \infty)$ is defined, for each $s \in \S$, as $V^\pi_\gamma(s) := \E^\pi_s \left[\sum_{t=0}^\infty \gamma^t R_t \right]$,
%\begin{align}
%    V^\pi_\gamma(s) := \E^\pi_s \left[\sum_{t=0}^\infty \gamma^t R_t \right] \label{eq:value_fn_defn}
%\end{align}
where $R_t = r(S_t, A_t)$ is the reward received at time $t$. There always exists an optimal policy $\pistar_\gamma$ that is deterministic and satisfies $V_\gamma^{\pistar_\gamma}(s) = V_\gamma^\star(s) := \sup_{\pi} V_\gamma^\pi(s)$ for all $s \in \S$ \citep{puterman_markov_1994}.

%\smallskip
\textbf{Average-reward criterion~~}
In an MDP $(\S, \A, P, r)$, the average reward per stage or the \emph{gain} of a policy $\pi$ starting from state $s$ is defined as $\rho^\pi(s)  := \lim_{T \to \infty} \frac{1}{T} \E_s^\pi \big[\sum_{t=0}^{T-1} R_t \big].$
The \emph{bias function} of any stationary policy $\pi$ is
$h^\pi(s) := \Clim_{T \to \infty} \E_s^\pi \big[\sum_{t=0}^{T-1} \left(R_t - \rho^\pi(S_t)\right) \big]$,
where $\Clim$ denotes the Cesaro limit. When the Markov chain induced by $P_\pi$ is aperiodic, $\Clim$ can be replaced with the usual limit. For any policy $\pi$, $\rho^\pi$ and $h^\pi$ satisfy $\rho^\pi = P_\pi \rho^\pi$ and $\rho^\pi + h^\pi = r_\pi + P_\pi h^\pi$. We let $P_\pi^\infty = \Clim_{T \to \infty} P^T$ denote the limiting matrix, and note that $P_\pi^\infty P_\pi = P_\pi P_\pi^\infty = P_\pi^\infty$ and $\rho^\pi = P_\pi^\infty r_\pi$.
A policy $\pistar$ is \emph{Blackwell-optimal} if there exists some discount factor $\Bar{\gamma} \in [0,1)$ such that for all $\gamma \geq \Bar{\gamma}$ we have $V^{\pistar}_\gamma \geq V^{\pi}_\gamma$ for all policies $\pi$. When $S$ and $A$ are finite, there always exists some Blackwell-optimal policy which is Markovian and deterministic, which we denote $\pistar$ \citep{puterman_markov_1994}. We define the optimal gain $\rho^\star\in \R^\S$ by $\rho^\star(s) = \sup_{\pi} \rho^\pi(s)$ and note that we have $\rho^\star = \rho^{\pistar}$. We also define $h^\star = h^{\pistar}$ (and we note that this definition does not depend on which Blackwell-optimal $\pistar$ is used if there are multiple). A policy $\pi$ is \textit{gain-optimal} if $\rho^\pi = \rho^\star$ and it is \textit{bias-optimal} if it is gain-optimal and $h^\pi = h^\star$. For $x \in \R^{\S} $, we define the span semi-norm  $\spannorm{x} := \max_{s \in \S} x(s) - \min_{s \in \S} x(s).$
An MDP is communicating if for any initial and target states, some Markovian policy can reach the target state from the initial state (with probability 1). The \emph{diameter} is defined as
$D := \max_{s_1 \neq s_2} \inf_{\pi\in \Pi} \E^\pi_{s_1} \left[\eta_{s_2}\right]$, where $\eta_{s} = \inf \{ t \geq 1 : S_t = s \}$ denotes the hitting time of a state $s\in\S$. $D < \infty$ if and only if the MDP is communicating.
A weakly communicating MDP is such that the states can be partitioned into two disjoint subsets $\S = \S_1 \cup \S_2$ such that all states in $\S_1$ are transient under all stationary policies and $\S_2$ is communicating. In weakly communicating MDPs $\rho^\star$ is a constant vector (all entries are equal). 
For each policy $\pi$, if the Markov chain induced by $P_\pi$ has a unique stationary distribution $\nu_\pi$, we define the mixing time of $\pi$ as
$\tau_\pi := \inf \cig\{t \geq 1 : \max_{s \in \S} \cigl\|e_s^\top \left(P_\pi \right)^t - \nu^\top_\pi \cigr\|_1 \leq \frac{1}{2} \cig\}.$
If all policies in the set of Markovian \textit{deterministic} policies, denoted $\Pi$, satisfy this assumption, we define the \emph{uniform mixing time} $\tmix := \sup_{\pi\in\Pi} \tau_\pi$. An MDP is unichain if all Markovian deterministic policies induce a Markov chain $P_\pi$ with a single recurrent class (and possibly some transient states).
We note that this definition of mixing time requires the Markov chain $P_\pi$ to be unichain but not irreducible.

When using transition kernels besides $P$, for example denoted $\Phat$, we will accordingly write $\Vhat_\gamma^\pi, \hhat^\pi, \rhohat^\pi$ for the associated value, bias, and gain functions respectively. We also occasionally drop the subscript from discounted value functions and write $V^\pi$ when $\gamma$ is clear from context.

We use standard Big-Oh notation $O, \Omega, \Theta$, and we also use the notation $\widetilde{O}, \widetilde{\Omega}$ to hide logarithmic factors in $n, S, A, \frac{1}{\delta}$, as well $\frac{1}{\xi}$ where $\xi$ is a perturbation size parameter appearing in some results. We use $C_1,C_2,\ldots$ to denote absolute constants.

\section{Main Results}
\label{sec:main_results}

\begin{algorithm}[h]
\caption{\Plugin approach for AMDP} \label{alg:generic_amdp_plugin_alg}
\begin{algorithmic}[1]
\Require Sample size per state-action pair $n$; optional anchoring state $s_0$, optional anchor probability $\eta$, optional perturbation level $\xi$
\For{each state-action pair $(s,a) \in \S \times \A$}
\State Collect $n$ samples $S^1_{s,a}, \dots, S^n_{s,a}$ from $P(\cdot \mid s,a)$
\State Form the empirical transition kernel $\Phat(s' \mid s, a) = \frac{1}{n}\sum_{i=1}^n \ind\{S^i_{s,a} = s'\}$, for all $s' \in \S$
\EndFor
\State Form \textit{anchored} empirical transition matrix $\Phatanc = (1-\eta) \Phat + \eta \one e_{s_0}^\top$ \algorithmiccomment{Set $\eta = 0$ for no anchoring} \label{alg:anchoring_step}
\State Form \textit{perturbed} reward $\widetilde{r} = r + \Delta$ where $\Delta(s,a) \stackrel{\text{i.i.d.}}{\sim} \text{Uniform}[0 , \xi]$ \algorithmiccomment{Set $\xi = 0$ for no perturbation}\label{alg:perturbation_step}
\State $\pihat = \SolveAMDP(\Phatanc, \widetilde{r})$\label{alg:solver_step}
\State \Return $\pihat$
\end{algorithmic}
\end{algorithm}

We present a meta-algorithm, Algorithm \ref{alg:generic_amdp_plugin_alg}, which encapsulates several variants of the \plugin approach for solving AMDPs. There are three key choices in Algorithm \ref{alg:generic_amdp_plugin_alg}, within lines \ref{alg:anchoring_step}, \ref{alg:perturbation_step}, and \ref{alg:solver_step}. The first is that instead of solving for an optimal policy in the empirical MDP $\Phat$, we may instead choose to use the \textit{anchored} MDP $\Phatanc = (1-\eta) \Phat + \eta \one e_{s_0}^\top$, which adds a small probability $\eta$ of transitioning to an arbitrary \textit{anchor} state $s_0$ from all states and actions \citep{fruit_efficient_2018}. We discuss the anchoring technique in more detail shortly. This step is optional, and can be skipped by setting $\eta = 0$. Secondly, we may use a slightly perturbed reward vector $\widetilde{r} = r + \Delta$ where $\Delta \in \R^{SA}$ has each entry sampled independently from the $\text{Uniform}[0 , \xi]$ distribution. This step can also be skipped by setting $\xi = 0$. Finally, any AMDP solver $\SolveAMDP$ can be used in line \ref{alg:solver_step}, but our theorems each require certain conditions on the degree of suboptimality of the output policy $\pihat$ guaranteed by the solver.

\subsection{Standard \Plugin Approach}
We first analyze arguably the most natural algorithm for learning optimal policies in AMDPs, the \plugin approach: we form an empirical transition matrix $\Phat$ using transition counts from the generative model and then compute a bias-optimal policy $\pihat$ for the AMDP $(\Phat, r)$. This corresponds to Algorithm~\ref{alg:generic_amdp_plugin_alg} with no perturbation ($\xi = 0$) and no anchoring ($\eta = 0$).
To the best of our knowledge, the following results are the first for this simple algorithm.

\begin{thm}
    \label{thm:AMDP_plugin_thm}
    Suppose $P$ is weakly communicating. Consider Algorithm~\ref{alg:generic_amdp_plugin_alg} with $\eta=0$ and $\xi=0$.
    Suppose that the policy $\pihat$ returned by $\SolveAMDP$ is guaranteed to be a bias-optimal policy of the AMDP $(\Phat, r)$. Let $\hhat^{\star}$ be the optimal bias of $(\Phat, r)$, and let $\empbod$ be the (random) smallest discount factor such that for all $\gamma \geq \empbod$, there exists $c \in \R$ (which may depend on $\gamma$) such that
    \begin{align}
        \infnorm{\Vhat_\gamma^\star - \hhat^\star - c \one} \leq \frac{1}{n}. \label{eq:bias_optimal_discount_cond}
    \end{align}
    Then with probability $1 - \delta$, if $\Phat$ is weakly communicating, then
    \begin{align*}
        \rho^{\pihat} - \rho^\star & \leq  \sqrt{\frac{C_7 \log^3 \left( \frac{SAn}{\delta(1-\empbod)} \right)}{n} \left(\spannorm{h^\star} + \spannorm{\hhat^\star}+1 \right)}\one.
    \end{align*}
\end{thm}

Compared to the minimax optimal rate of $\Otilde \Big(\sqrt{\frac{\spannorm{h^\star}}{n} }\Big)$ (equivalent to $\Otilde \cig( SA\frac{\spannorm{h^\star}}{\varepsilon^2} \cig)$ sample complexity), Theorem \ref{thm:AMDP_plugin_thm} has the additional term $\spannorm{\hhat^\star}$. We show in Theorem \ref{thm:plug_in_lower_bound} that this additional term is unavoidable for the \plugin approach, in the sense that there exist instances where the \plugin approach satisfies a high-probability bound of $\Otilde \Big(\sqrt{\frac{\spannorm{h^\star} + \spannorm{\hhat^\star}}{n} }\Big)$ but not $\Otilde \Big(\sqrt{\frac{\spannorm{h^\star}}{n} }\Big)$. One key feature of the optimal algorithm \citep{zurek_span-based_2025}, based on DMDP reduction, is that it requires prior knowledge of $\spannorm{h^\star}$ to set the discount factor, whereas the \plugin method has no need for such information.
Lemma \ref{lem:prop_of_empbod} establishes basic properties of the quantity $\empbod$ appearing in Theorem \ref{thm:AMDP_plugin_thm}. In particular it is well-defined when $\Phat$ is weakly communicating.

Within the proof of Theorem \ref{thm:AMDP_plugin_thm} we analyze the accuracy of using $(\Phat, r)$ to estimate the gain of a fixed policy, leading to the following policy evaluation result of independent interest.

\begin{thm}
\label{thm:AMDP_policy_eval}
    Fix a policy $\pi$ such that $\rho^\pi$ is constant. Let $\rhohat^\pi$ be the gain of $\pi$ in the empirical AMDP $(\Phat, r)$. Then with probability at least $1 - \delta$,
    \begin{align*}
        \infnorm{\rhohat^\pi - \rho^\pi} & \leq  \sqrt{\frac{C_4 \log^3 \left( \frac{SAn}{\delta} \right)}{n} \left( \spannorm{h^\pi}+1 \right)} . 
    \end{align*}
\end{thm}

\subsection{Anchoring-Based \Plugin Approach}

Although the standard \plugin approach is arguably the most natural algorithm for learning in AMDPs with a generative model, it and our Theorem \ref{thm:AMDP_plugin_thm} have a few limitations. First, the performance bound in Theorem \ref{thm:AMDP_plugin_thm} only holds on the event that $\Phat$ is weakly communicating, which can be understood as a consequence of the fact that the standard \plugin approach does not provide a way for us to incorporate prior information that $P$ is weakly communicating. Additionally, the $\log \left(\frac{1}{1-\empbod}\right)$ term is not bounded as $\log \text{poly}(SAn\delta)$ in the worst case, preventing us from applying Theorem \ref{thm:AMDP_plugin_thm} to obtain optimal $D$ or $\tmix$-based rates. Finally, since arbitrary weakly communicating MDPs do not possess sufficient stability properties for a definition of policy near-optimality which suffices for our purposes, our Theorem \ref{thm:AMDP_plugin_thm} requires finding an exactly bias-optimal policy in $(\Phat, r)$. 

Fortunately, we can overcome all of these limitations with a simple technique which has been used many times (for various purposes) in the literature on average-reward reinforcement learning, which we term \textit{anchoring}. 
For a small probability $\eta \in [0,1]$ and any arbitrarily chosen state $s_0$, we can form the \textit{anchored} transition matrix $\Phatanc = (1-\eta) \Phat + \eta \one e_{s_0}^\top$ (where $\one \in \R^{SA}$ is all-$1$ and $e_{s_0} \in \R^S$ is all-$0$ except for a $1$ in entry $s_0$). In words, $\Phatanc$ follows $\Phat$ a $(1-\eta)$ fraction of the time, but all state-action pairs have a small chance $\eta$ to return to $s_0$. This technique has been used in average-reward and related settings (e.g., \citealt{fruit_efficient_2018, yin_offline_2022}) for essentially \textit{computational} reasons, since it ensures that the associated Bellman operator is a $1-\eta$ (span-)contraction \cite[Theorem 6.6.6]{puterman_markov_1994}, whereas without anchoring there is no guarantee of contractivity and thus standard average-reward value iteration has no finite-time convergence guarantee. Such works often set $\eta = \frac{1}{n}$, in which case these computational benefits are essentially without loss of statistical efficiency since an order $O(1/n)$ perturbation contributes a lower-order term relative to the statistical error. An arguably more standard perspective would be to consider anchoring + value iteration as a particular solver for the empirical AMDP $(\Phat, r)$, but we can incorporate anchoring within the \plugin framework and thus allow arbitrary AMDP solvers by having them solve the anchored AMDP $(\Phatanc, r)$. $\Phatanc$ is always weakly communicating (in fact unichain), thus providing a simple way to enforce our model estimate to be weakly communicating. 

Lemma \ref{lem:anchoring_optimality_properties} summarizes all these (and other) properties of the anchoring technique, in particular showing that anchoring with $\eta = \frac{1}{n}$ is essentially equivalent to DMDP reduction with an effective horizon of $\frac{1}{1-\gamma} = n$. This is a much larger effective horizon than those used in prior work on DMDP reduction for solving AMDP \citep{jin_towards_2021, wang_optimal_2023, wang_near_2022, zurek_span-based_2025} and unlike prior work, does not require knowledge of complexity parameters such as $\tmix$ or $\spannorm{h^\star}$. Prior analysis of DMDP (e.g. \cite{agarwal_model-based_2020, li_breaking_2020}) does not allow or gives vacuous guarantees for $\frac{1}{1-\gamma} = n$, but our novel analysis for the AMDP \plugin method (which heuristically is a DMDP reduction with arbitrarily large effective horizon) can be repurposed to handle this situation. We thus believe anchoring is better understood as a stabilized method for directly solving the AMDP ($\Phat$, $r$) rather than as a discounted reduction. Still, the DMDP reduction method (with horizon $n$) obtains nearly identical guarantees to those in this section, which we provide in Appendix \ref{sec:DMDP_reduction_approach}.

Now we present our first result on the anchored AMDP \plugin approach. We define $\rhohatanc^{\pi}$ and $ \hhatanc^{\pi}$ as the gain and bias of a policy $\pi$ in the anchored AMDP $(\Phatanc, r)$, and likewise define $\rhohatanc^\star$ and $\hhatanc^\star$ as the optimal gain and bias in the anchored AMDP.
\begin{thm}
    \label{thm:AMDP_anchored_nopert}
    Suppose $P$ is weakly communicating. Let $s_0$ be an arbitrary state, let $\eta = \frac{1}{n}$, and set $\xi = 0$ in Algorithm \ref{alg:generic_amdp_plugin_alg}. Also suppose that $\SolveAMDP$ is guaranteed to return a policy $\pihat$ satisfying
    \begin{align}
       \rhohatanc^{\pihat} \geq \rhohatanc^\star - \frac{1}{3n^2} \quad\text{and}\quad \infnorm{\hhatanc^{\pihat} - \hhatanc^\star} \leq \frac{1}{3n^2} . \label{eq:solveamdp_opt_cond_1}
    \end{align}
   Then with probability at least $1-\delta$,
    \begin{align*}
        \rho^{\pihat} - \rho^\star & \leq  \sqrt{\frac{C_5 \log^3 \left( \frac{SAn}{\delta} \right)}{n} \left( \spannorm{h^\star} + \spannorm{\hhatanc^{\star}} + 1 \right) }\one.
    \end{align*}
\end{thm}
Regarding the difference between the terms $\spannorm{\hhatanc^\star}$ and $\spannorm{\hhat^\star}$ appearing in Theorems \ref{thm:AMDP_anchored_nopert} and \ref{thm:AMDP_plugin_thm} respectively, by Lemma \ref{lem:anchoring_optimality_properties}, whenever $\Phat$ is weakly communicating (which is required for the bound within Theorem \ref{thm:AMDP_plugin_thm} to hold), we have that $\spannorm{\hhatanc^\star} \leq O(\spannorm{\hhat^\star})$.

We now apply Theorem \ref{thm:AMDP_anchored_nopert} to the diameter-based complexity setting, where we assume $P$ is communicating with diameter $D$ and derive a complexity bound depending on $D$. Theorem \ref{thm:AMDP_anchored_nopert} will yield an optimal sample complexity $\Otilde \left( SA \frac{D}{\varepsilon^2} \right)$ (matching the lower bound in \citealt{wang_near_2022}), and this optimal complexity follows directly from upper-bounding the guarantee of Theorem \ref{thm:AMDP_anchored_nopert} in terms of $D$ without any algorithmic modifications. In particular, no prior knowledge of $D$ is required.
The optimal bias span is always bounded by the diameter \citep{bartlett_regal_2012, lattimore_bandit_2020}, so we have $\spannorm{h^\star} \leq D$, and similarly it is possible to show $\spannorm{\hhatanc^\star} \leq O(\widehat{D})$ where $\widehat{D}$ is the diameter of the empirical MDP $\Phat$. The key fact is that whenever $n \geq \widetilde{\Omega}(D)$, we additionally have $\widehat{D} \leq O(D)$, that is, $\Phat$ will be communicating and have diameter order $D$.

\begin{lem}
\label{lem:empirical_diameter_bound}
Suppose that the MDP $P$ is communicating and has diameter $D$. Then there exists a constant $C_7$ such that if $n \geq C_8 D \log^3 \left(\frac{SADn}{\delta} \right)$, then with probability at least $1-\delta$,
$\widehat{D} \leq 14 D.$
    %\[
    %\widehat{D} \leq 14 D.
    %\]
In particular, in this same event, $\Phat$ is communicating.
\end{lem}
Lemma \ref{lem:empirical_diameter_bound} follows from our later results on DMDPs, using the fact that the maximum travel time between states in $\Phat$ can be bounded by analyzing certain discounted value functions associated with auxiliary problems each measuring travel time to a certain state.
\begin{cor}
\label{thm:diameter_complexity}
    Suppose $P$ is communicating and has diameter $D$. Let $s_0$ be an arbitrary state, let $\eta = \frac{1}{n}$, and set $\xi = 0$ in Algorithm \ref{alg:generic_amdp_plugin_alg}. Also suppose that $\SolveAMDP$ is guaranteed to return a policy $\pihat$ which satisfies condition~\eqref{eq:solveamdp_opt_cond_1}.
    Then with probability at least $1-\delta$,
    \begin{align*}
        \rho^{\pihat} - \rho^\star & \leq  \sqrt{\frac{C_9 \log^3 \left( \frac{SADn}{\delta} \right)}{n} D }\one.
    \end{align*}
\end{cor}
The only other method which does not require prior knowledge of $D$ and which yields a diameter-based complexity guarantee is that of \cite{tuynman_finding_2024}.
In contrast to the explicit diameter estimation employed by \cite{tuynman_finding_2024} which leads to a worse error bound (since for large values of $\varepsilon$, the diameter estimation subroutine dominates the complexity), Theorem \ref{thm:diameter_complexity} yields the optimal diameter-based complexity for the full range of $\varepsilon$ and does so with a simpler algorithm.

Another important and heavily studied sub-setting is the uniformly mixing setting, wherein all deterministic Markovian policies are assumed to have bounded mixing time $\tmix$ in $P$, and the goal is to obtain a complexity bound in terms of $\tmix$. We conjecture that Theorem \ref{thm:AMDP_anchored_nopert} should also imply an optimal complexity of $\Otilde(SA\frac{\tmix}{\varepsilon^2})$ for this setting by an analogous argument. We always have $\spannorm{h^\star} \leq O(\tmix)$ (Lemma \ref{lem:mixing_param_relationships}; \citealt{wang_near_2022}), however we are unaware of how to bound the uniform mixing time of $\Phat$ by $\tmix$ for small values of $n$ (which would imply that $\spannorm{\hhatanc^\star} \leq O(\tmix)$). For this reason we present a slightly different algorithm and guarantee which replaces the $\spannorm{\hhatanc^\star}$ term with $\spannorm{\hanc^{\pihat}}$, the bias span of the returned policy $\pihat$ in the \textit{true} anchored MDP $\Panc = (1-\eta) P + (1-\eta) \one e_{s_0}^\top$, which can be straightforwardly bounded as $O(\tmix)$.
\begin{thm}
\label{thm:AMDP_anchored_perturbed}
    Suppose $P$ is weakly communicating. Let $s_0$ be an arbitrary state, let $\eta = \frac{1}{n}$, and set $\xi \in (0, \frac{1}{n}]$ in Algorithm \ref{alg:generic_amdp_plugin_alg}. Also suppose that the policy $\pihat$ returned by $\SolveAMDP$ is guaranteed to be the exact Blackwell-optimal policy of the AMDP $(\Phatanc, \widetilde{r})$.
    Then with probability at least $1 - \delta$,
    \begin{align*}
        \rho^\star - \rho^{\pihat}  \leq \sqrt{\frac{C_3  \log^3 \left( \frac{SAn}{ \delta \xi}\right) }{n}\left( \spannorm{h^\star} + \spannorm{\hanc^{\pihat}} + 1\right) } \one.
    \end{align*}
\end{thm}
While finding a Blackwell-optimal policy may generally be computationally expensive, similarly to the discounted setting \cite{li_breaking_2020}, with high probability the perturbation ensures a small separation between the bias of the Blackwell optimal policy of $(\Phatanc, \widetilde{r})$ and all other (Markovian, deterministic) gain-optimal policies, thus ensuring that Blackwell optimality for $(\Phatanc, \widetilde{r})$ reduces to bias optimality and also that $\Otilde\left(n \right)$ steps of value iteration suffice to find an exactly Blackwell-optimal policy. See Lemma \ref{lem:anchpert_exact_VI_convergence} for the formal statement.

Now, as promised, we can show that the \plugin approach with anchoring and perturbation obtains the optimal $\tmix$-based sample complexity.
\begin{cor}
\label{thm:mixing_complexity}
    Suppose $P$ has a finite uniform mixing time $\tmix$.
    Let $s_0$ be an arbitrary state, let $\eta = \frac{1}{n}$, and set $\xi \in (0, \frac{1}{n}]$ in Algorithm \ref{alg:generic_amdp_plugin_alg}. Also suppose that the policy $\pihat$ returned by $\SolveAMDP$ is the exact Blackwell-optimal policy of the AMDP $(\Phatanc, \widetilde{r})$.
    Then with probability at least $1 - \delta$,
    \begin{align*}
        \rho^\star - \rho^{\pihat}  \leq \sqrt{\frac{7C_3  \log^3 \left( \frac{SAn}{ \delta \xi}\right) }{n}\tmix }\one.
    \end{align*}
\end{cor}
Compared to extensive prior work on this setting, this is the first algorithm which achieves the optimal complexity without requiring prior knowledge of $\tmix$, and additionally we believe the algorithm is much simpler than previous approaches.

Rather than having two different (albeit highly similar, differing only in whether or not the reward vector $r$ is perturbed) algorithms which are optimal for different settings, one might prefer to have one algorithm which achieves the best of all the aforementioned guarantees. In fact, we can view the exact solution of the \textit{perturbed} empirical AMDP $(\Phat, \widetilde{r})$ as an approximate solution of the \textit{unperturbed} empirical AMDP $(\Phat, r)$, with the degree of suboptimality depending on the perturbation magnitude $\xi$, and thus for sufficiently small $\xi$, we can also apply the performance guarantees for the unperturbed anchored \plugin approach from Theorem \ref{thm:AMDP_anchored_nopert} to the perturbed anchored \plugin approach.

\begin{thm}
    \label{thm:AMDP_best_of_both}
    Suppose $P$ is weakly communicating. Let $s_0$ be an arbitrary state, let $\eta = \frac{1}{n}$, and set $\xi = \frac{1}{2n^2}$ in Algorithm \ref{alg:generic_amdp_plugin_alg}. Also suppose that the policy $\pihat$ returned by $\SolveAMDP$ is guaranteed to be the exact Blackwell-optimal policy of the AMDP $(\Phatanc, \widetilde{r})$.
    Then with probability at least $1 - \delta$,
    \begin{align*}
        \rho^\star - \rho^{\pihat}  \leq \sqrt{\frac{C_6  \log^3 \left( \frac{SAn}{ \delta }\right) }{n}\left( \spannorm{h^\star} + \min\{\spannorm{\hhatanc^{\pihat}},  \spannorm{\hanc^{\pihat}} \} + 1\right) }.
    \end{align*}
\end{thm}
Following identical steps as in the Corollaries \ref{thm:diameter_complexity} and \ref{thm:mixing_complexity}, we can thus show that the anchored \plugin approach with $\xi = \frac{1}{2n^2}$ automatically satisfies the optimal error bounds for both the diameter and uniform-mixing-based settings, without any required prior knowledge. Since some prior DMDP-reduction-based methods for these settings require the discount factor to be set in terms of $D$ or $\tmix$, prior algorithms which are optimal for the diameter-based setting may not be optimal for the uniformly-mixing setting and vice versa, unlike our result.

\subsection{\Plugin Approach for Discounted MDPs}
The new analysis techniques developed for the AMDP \plugin approach also lead to improvements for the DMDP \plugin method. Similarly to the previous section, we present a meta-algorithm, Algorithm \ref{alg:generic_dmdp_plugin_alg}, and allow different choices of perturbation $\xi$ and solver $\SolveDMDP$ within the theorems.

\begin{algorithm}[t]
\caption{\Plugin approach for DMDP} \label{alg:generic_dmdp_plugin_alg}
\begin{algorithmic}[1]
\Require Sample size per state-action pair $n$, discount factor $\gamma$; optional perturbation level $\xi$
\For{each state-action pair $(s,a) \in \S \times \A$}
\State Collect $n$ samples $S^1_{s,a}, \dots, S^n_{s,a}$ from $P(\cdot \mid s,a)$
\State Form the empirical transition kernel $\Phat(s' \mid s, a) = \frac{1}{n}\sum_{i=1}^n \ind\{S^i_{s,a} = s'\}$, for all $s' \in \S$
\EndFor
\State Form \textit{perturbed} reward $\widetilde{r} = r + \Delta$ where $\Delta(s,a) \stackrel{\text{i.i.d.}}{\sim} \text{Uniform}[0 , \xi]$ \algorithmiccomment{Set $\xi = 0$ for no perturbation}\label{alg:perturbation_step_DMDP}
\State $\pihat = \SolveDMDP(\Phat, \widetilde{r}, \gamma)$
\State \Return $\pihat$
\end{algorithmic}
\end{algorithm}

\begin{thm}
    \label{thm:DMDP_main_thm}
Let $\xi = 0$ in Algorithm \ref{alg:generic_dmdp_plugin_alg}. Suppose that $\SolveDMDP$ returns a policy $\pihat$ satisfying
\begin{align}
    \Vhat^{\pihat} \geq \Vhat^\star - \frac{1}{n}\one \label{eq:solvedmdp_opt_cond}.
\end{align}
Then with probability at least $1 - \delta$,
    \begin{align*}
    \infnorm{V^{\pihat} - V^\star} &\leq  \frac{1 }{1-\gamma}\sqrt{\frac{C_1  \log^3 \left( \frac{SAn}{(1-\gamma) \delta}\right) }{n}\left( \spannorm{V^{\star}} + \spannorm{\Vhat^{\pihat}} + 1\right) }.
\end{align*}
\end{thm}
Using the bounds $\spannorm{V^\star}, \spannorm{\Vhat^{\pihat}} \leq \frac{1}{1-\gamma}$, Theorem \ref{thm:DMDP_main_thm} is the first to imply that the discounted \plugin approach attains the minimax optimal sample complexity of $\Otilde\left(\frac{SA}{(1-\gamma)^3 \varepsilon^2}\right)$ \textit{without perturbation}. 

Analogously to the situation for the anchored AMDP \plugin approach, by adding perturbation we can replace the $\spannorm{\Vhat^{\pihat}}$ term with $\spannorm{V^{\pihat}}$.

\begin{thm}
\label{thm:DMDP_pert_thm}
Set $\xi \in (0,   \frac{1}{n}]$ in Algorithm \ref{alg:generic_dmdp_plugin_alg}. Suppose that the policy $\pihat$ returned by $\SolveDMDP$ is guaranteed to be exactly optimal for the DMDP $(\Phat, \widetilde{r}, \gamma)$. Then with probability at least $1 - \delta$,
    \begin{align*}
    \infnorm{V^{\pihat} - V^\star} &\leq  \frac{1 }{1-\gamma}\sqrt{\frac{C_2  \log^3 \left( \frac{SAn}{(1-\gamma) \delta \xi}\right) }{n}\left( \spannorm{V^{\star}} + \spannorm{V^{\pihat}} + 1\right) }.
\end{align*}
\end{thm}
As shown in \cite{li_breaking_2020}, the perturbation ensures that the exact optimal policy of the DMDP $(\Phat, \widetilde{r}, \gamma)$ can be computed in finite time, for example with $\Otilde\left( \frac{1}{1-\gamma}\right)$ value iteration steps.

Also analogously to the situation for the AMDP \plugin approach, for sufficiently small perturbation $\xi$ we can combine both guarantees for one algorithm.
%it is possible to view the exact optimal policy of the perturbed empirical DMDP $(\Phat, \widetilde{r}, \gamma)$ as an approximately optimal solution to the unperturbed empricial DMDP $(\Phat, r, \gamma)$, and thus for sufficiently small $\xi$ (order $\frac{(1-\gamma)}{n}$ suffices) we can apply Theorem \ref{thm:DMDP_main_thm} to the setting described in Theorem \ref{thm:DMDP_pert_thm}, obtaining the best of both guarantees for one algorithm. 
\begin{thm}
\label{thm:DMDP_best_of_both}
Set $\xi = \frac{1-\gamma}{2n}$ in Algorithm \ref{alg:generic_dmdp_plugin_alg}. Suppose that the policy $\pihat$ returned by $\SolveDMDP$ is the exact optimal policy of the DMDP $(\Phat, \widetilde{r}, \gamma)$. Then with probability at least $1 - \delta$,
    \begin{align*}
    \infnorm{V^{\pihat} - V^\star} &\leq  \frac{1 }{1-\gamma}\sqrt{\frac{C_{10}  \log^3 \left( \frac{SAn}{(1-\gamma) \delta }\right) }{n}\left( \spannorm{V^{\star}} + \min \{\spannorm{\Vhat^{\pihat}}, \spannorm{V^{\pihat}} \} + 1\right) }.
\end{align*}
\end{thm}

Now we discuss the relationship between the terms $\spannorm{V^\star}$, $\spannorm{\Vhat^{\pihat}}$, and $\spannorm{V^{\pihat}}$ appearing in our theorems and the usual complexity parameters $\tmix, D$, and $\spannorm{h^\star}$. If $P$ is weakly communicating, we have $\spannorm{V^\star} \leq 2\spannorm{h^\star}$ \citep[Lemma 2]{wei_model-free_2020}, and as previously mentioned we have $\spannorm{h^\star} \leq 3 \tmix$ (Lemma \ref{lem:mixing_param_relationships}, \cite{wang_near_2022}) and $\spannorm{h^\star} \leq D$ \citep{bartlett_regal_2012, lattimore_bandit_2020}. Therefore, under the event described in Lemma \ref{lem:empirical_diameter_bound}, since $\Phat$ is communicating with diameter $\widehat{D} \leq O(D)$, we can apply these same statements to $\spannorm{\Vhat^{\star}}$ to obtain that $\spannorm{\Vhat^{\star}} \leq 2\spannorm{\hhat^\star}$ (since $\Phat$ is weakly communicating) and that $\spannorm{\hhat^\star} \leq \widehat{D}$, and thus that $\spannorm{\Vhat^{\star}} \leq O(D)$, and finally the optimality condition~\eqref{eq:solvedmdp_opt_cond} implies that $\spannorm{\Vhat^{\pihat}} \leq \spannorm{\Vhat^{\star}} + \frac{1}{n} \leq O(D)$. Therefore, Theorem \ref{thm:DMDP_main_thm} (combined with Lemma \ref{lem:empirical_diameter_bound}) implies a $\Otilde\left(SA\frac{D}{(1-\gamma)^2\varepsilon^2} \right)$ sample complexity bound, for the full nontrivial range of $\varepsilon \in (0, \frac{1}{1-\gamma}]$.
Also, we have the bound $\spannorm{V^{\pihat}} \leq 3 \tmix$ (see Lemma \ref{lem:mixing_param_relationships}), which when combined with Theorem \ref{thm:DMDP_pert_thm} implies a $\Otilde\left(SA\frac{\tmix}{(1-\gamma)^2\varepsilon^2} \right)$ sample complexity bound, also for the entire nontrivial range of $\varepsilon \in (0, \frac{1}{1-\gamma}]$. This improves on \cite{wang_optimal_2023} which only obtains this complexity for $\varepsilon \leq \sqrt{\frac{\tmix}{1-\gamma}}$. (See \cite{wang_optimal_2023} for the matching lower bound.)

When $n$ is sufficiently large relative to other problem parameters we can show that $\spannorm{\Vhat^\star}$ is bounded in terms of $\spannorm{h^\star}$ akin to $\spannorm{V^\star}$.
\begin{lem}
    \label{lem:small_sample_span_bound}
    If $P$ is weakly communicating and $n \geq \frac{C_1  \log^3 \left( \frac{SAn}{(1-\gamma) \delta}\right)}{(1-\gamma)^2 (\spannorm{h^{\star}} +1)}$, then with probability at least $1 - \delta$, $\spannorm{\Vhat^\star} \leq 4(\spannorm{h^\star} + 1)$.
\end{lem}
When $P$ is weakly communicating, the algorithm of \cite{zurek_span-based_2025} achieves the span-based bound $\Otilde\cig(SA \frac{\spannorm{h^\star}+1}{(1-\gamma)^2 \varepsilon^2} \cig)$ under the restriction that $\varepsilon \leq \spannorm{h^\star}$, or equivalently that $n \geq \widetilde{\Omega}\cig( \frac{1}{(1-\gamma)^2 (\spannorm{h^\star}+1)}\cig)$. Under this condition the requirement of Lemma \ref{lem:small_sample_span_bound} will be met, and by combining it with our Theorem \ref{thm:DMDP_main_thm}, we recover the result of \cite{zurek_span-based_2025}. An analogous version of Lemma \ref{lem:small_sample_span_bound} could also be shown to bound $\spannorm{V^{\pihat}}$. However, similarly to the situation for the average-reward \plugin method, generally the terms $\spannorm{V^{\pihat}}$ and $\spannorm{\Vhat^{\pihat}}$ cannot be removed from the analysis of the DMDP \plugin approach, as is shown in Theorem \ref{thm:plug_in_lower_bound}.

\paragraph{AMDP-to-DMDP Reduction Approach} %\subsubsection{AMDP-to-DMDP Reduction Approach} this looks terrible for some reason

While our focus is not on analyzing the well-studied AMDP-to-DMDP reduction approach for solving AMDPs, we briefly mention some corollaries of our DMDP results for the complexity of this method. First, for target AMDP error $\varepsilon$, if $\spannorm{h^\star}$ is known, then we can use an effective horizon of $\frac{1}{1-\gamma} = C \frac{\spannorm{h^\star}}{\varepsilon}$ (as do \cite{wang_near_2022, zurek_span-based_2025}) and the condition in Lemma \ref{lem:small_sample_span_bound} will be satisfied as long as we have $n \geq \widetilde{\Omega}(\frac{\spannorm{h^\star} + 1}{\varepsilon^2})$. Combining the resulting error bound with Theorem \ref{thm:DMDP_main_thm} and with standard DMDP reduction results \citep{wang_near_2022}, this recovers the result of \cite{zurek_span-based_2025} which obtains the optimal $\Otilde\cig(SA\frac{\spannorm{h^\star} + 1}{\varepsilon^2} \cig)$ sample complexity, but we remove the need for reward perturbation.

More interestingly, we can satisfy the conditions of Lemma \ref{lem:small_sample_span_bound} with a smaller effective horizon of approximately $\sqrt{n}$. This is too small to yield the optimal complexity, since a DMDP reduction with discount $\gamma$ incurs error of order $(1-\gamma)\spannorm{h^\star} \approx \frac{\spannorm{h^\star}}{\sqrt{n}}$ (even with infinite samples) \citep{wang_near_2022}. However, since Lemma \ref{lem:small_sample_span_bound} holds, we can obtain the first complexity bound depending only on $S, A, \varepsilon,$ and $\spannorm{h^\star}$ \textit{without} requiring prior knowledge of $\spannorm{h^\star}$.
\begin{thm}
\label{thm:span_based_without_knowledge}
    Suppose $P$ is weakly communicating. Let $\xi = 0$ and $\gamma = 1 - \sqrt{\frac{C_1 \log^3(\frac{SAn^2}{\delta})}{n}}$ in Algorithm \ref{thm:DMDP_main_thm} and suppose that $\SolveDMDP$ guarantees~\eqref{eq:solvedmdp_opt_cond}. Then with probability $1 - \delta$,
    \begin{align*}
        \rho^\star - \rho^{\pihat} & \leq \sqrt{\frac{C_{11}\log^3(\frac{SAn}{\delta})}{n} \left( \spannorm{h^\star}^2 + 1\right)} \one.
    \end{align*}
\end{thm}

\subsection{Hard Instance for \Plugin Approach}

Now we provide a concrete MDP where both the average-reward and discounted \plugin methods fail to achieve an $\spannorm{h^\star}$-based complexity, implying that the terms $\spannorm{\hhat^\star}$, $\spannorm{\Vhat^{\pihat}}$, and $\spannorm{V^{\pihat}}$ cannot be generally removed from the Theorems \ref{thm:AMDP_plugin_thm}, \ref{thm:DMDP_main_thm}, and \ref{thm:DMDP_pert_thm}, respectively.
\begin{thm}
\label{thm:plug_in_lower_bound}
    For any fixed $n \geq 10$, there exists an MDP $P$ (depending on $n$) with $S = A = 2$ such that $P$ has $\spannorm{h^\star} = 1$, diameter $D = n$, and a uniform mixing time $\tmix = \Theta(n)$.
    Also, with probability at least $\frac{1}{25}$, 
        \begin{enumerate}[itemsep=0pt, topsep=0pt]
            \item $\Phat$ is communicating.
            \item $\spannorm{\hhat^\star} = \Theta(n)$.
            \item $1/(1-\empbod) \leq O(n^3)$.
            \item Letting $\pihstar$ be the Blackwell-optimal policy of the AMDP $(\Phat, r)$, $\infnorm{\rho^\star - \rho^{\pihstar}} \geq  \frac{1}{5}$.
            \item Letting $\pihstar_\gamma$ be the optimal policy for the DMDP $(\Phat, \gamma, r)$ with effective horizon $\frac{1}{1-\gamma} = n^2$, $\infnorm{V_\gamma^\star - V_\gamma^{\pihstar}} \geq \frac{n^2}{5} = \frac{1}{1-\gamma}\frac{1}{5}$.
        \end{enumerate}
    Consequently, for any constant $C>0$, there exists $n, P, r$ such that the statement
        \[
        \P \Bigg( \infnorm{\rho^{\pihstar} - \rho^\star}\leq C \sqrt{\frac{\spannorm{h^\star} \log\left(\spannorm{h^\star} n \right)}{n}} \Bigg) > 1 - \frac{1}{25}
        \]
    is false.
    Additionally, for any constant $C>0$, there exists $n, P, r, \gamma$ such that the statement
        \[
        \P \Bigg( \infnorm{V_\gamma^{\pihstar_\gamma} - V_\gamma^{\star}} \leq C \frac{1}{1-\gamma}\sqrt{\frac{\spannorm{h^\star} \log\left(\spannorm{h^\star} n \right)}{n}} \Bigg) > 1 - \frac{1}{25}
        \]
    is false.
\end{thm}
We show the construction for $P$ in Appendix \ref{sec:proof_of_plug_in_lower_bound}, along with the proof of Theorem \ref{thm:plug_in_lower_bound}. At a high level, $P$ causes a constant probability of sampling a $\Phat$ that has optimal bias span $\spannorm{\hhat^\star} \gg \spannorm{h^\star}$. If we had knowledge of the true $\spannorm{h^\star}$ we could use it to find a near-optimal policy with controlled complexity, which is accomplished by DMDP reduction using a $\spannorm{h^\star}$-based effective horizon \citep{zurek_span-based_2025}. In contrast, the AMDP \plugin method (and similarly the DMDP \plugin method with a sufficiently large horizon) has no way of controlling the span of the empirical optimal policy, leading to a potentially greater span than $\spannorm{h^\star}$ and correspondingly a larger error.

\section{Proof Techniques}
\label{sec:proof_techniques}
At the heart of all our main results is a novel decomposition of the difference between limiting distributions associated with $\Phat$ and $P$. This technique bears some resemblance to the ``higher-order'' simulation lemma expansion introduced by \cite{li_breaking_2020} to show the DMDP \plugin method achieves $\Otilde \cigl(SA \frac{1}{(1-\gamma)^3 \varepsilon^2} \cigr)$ sample complexity for the full range $\varepsilon \in (0, \frac{1}{1-\gamma}]$. Both techniques decompose the error with the simulation lemma, use law-of-total-variance-style arguments to bound some leading terms, and obtain lower-order error terms which can be inductively bounded again with the simulation lemma.
However, there are many subtle differences. Most importantly for the average-reward setting, the arguments of \cite{li_breaking_2020} require $n \geq \Omega(\frac{1}{1-\gamma})$, which is without loss of generality for the minimax rate of $\Otilde \cigl( \frac{1}{\sqrt{(1-\gamma)^3n}} \cigr)$ (it is necessary for nontrivial accuracy), but breaks for the arbitrarily large effective horizons needed in average-reward problems.
Also, the arguments of \cite{li_breaking_2020} are designed to use concentration inequalities involving variance parameters of certain auxiliary MDPs, which requires a more delicate leave-one-out analysis (hence their use of reward perturbation), whereas our argument requires concentration bounds on terms which are simpler functions of the original/empirical MDPs, enabling the flexibility to utilize the absorbing MDP arguments of \cite{agarwal_model-based_2020} or those of \cite{li_breaking_2020}.

We briefly illustrate our techniques as applied to the proof of the policy evaluation bound Theorem \ref{thm:AMDP_policy_eval}. Hence we fix a policy $\pi$ with constant gain $\rho^\pi$ and attempt to bound $\infnorm{\rhohat^\pi - \rho^\pi }$. In this sketch we use the $\lesssim$ notation to ignore constants and $\log$ factors.
By an average-reward version of the simulation lemma (see Lemma \ref{lem:average_reward_simulation_lemma}), since $\rho^\pi$ is a constant vector, 
\begin{align}
    \rhohat^\pi - \rho^\pi = \Phat_\pi^\infty r_\pi - P_\pi^\infty r_\pi =  \Phat_\pi^\infty ( \Phat_\pi  - P_\pi) h^\pi \label{eq:pf_sketch_sim_lemma}
\end{align}
Using Bernstein's inequality ($S$ times), we can obtain an elementwise inequality
\begin{align}
    \left|( \Phat_\pi  - P_\pi) h^\pi \right| \lesssim \sqrt{\frac{\Var_{P_\pi}\left[ h^\pi\right]}{n}} + \frac{\infnorm{h^\pi}}{n} \one \label{eq:pf_sketch_bernstein}
\end{align}
where $\Var_{P_\pi}\left[ h^\pi\right] = P_\pi (h^\pi)^{\circ 2} - (P_\pi h^\pi)^{\circ 2}$ is a (vector) variance of the next-state bias function. ($x^{\circ k}$ denotes the elementwise $k$th-power operation.) Since all entries of $\Phat^\infty_\pi$ are nonnegative, we can combine this with~\eqref{eq:pf_sketch_sim_lemma} and obtain
\begin{align}
    \left|\rhohat^\pi - \rho^\pi \right| & \leq \Phat^\infty_\pi \left|( \Phat_\pi  - P_\pi) h^\pi \right| \lesssim \Phat^\infty_\pi \sqrt{\frac{\Var_{P_\pi}\left[ h^\pi\right]}{n}} + \frac{\infnorm{h^\pi}}{n} \Phat^\infty_\pi \one. \label{eq:pf_sketch_Phat_monotonicity}
\end{align}
The second term of~\eqref{eq:pf_sketch_Phat_monotonicity} is $\frac{\infnorm{h^\pi}}{n} \Phat^\infty_\pi \one = \frac{\infnorm{h^\pi}}{n} \one$, which is smaller than the desired bound  $\sqrt{\frac{\infnorm{h^\pi}}{n}}$ (in the nontrivial accuracy regime where both of these terms are $\leq 1$) so we focus on bounding the first term in the RHS of~\eqref{eq:pf_sketch_Phat_monotonicity}. Using Jensen's inequality to move $\Phat^\infty_\pi$ inside the square root and the Poisson/Bellman equation $\rho^\pi + h^\pi = r_\pi + P_\pi h^\pi$, we have (elementwise)
\begin{align}
    \Phat^\infty_\pi \sqrt{\Var_{P_\pi}\left[ h^\pi\right]} &\leq \sqrt{\Phat^\infty_\pi \Var_{P_\pi}\left[ h^\pi\right]} 
    = \sqrt{\Phat^\infty_\pi \left( P_\pi (h^\pi)^{\circ 2} - (P_\pi h^\pi)^{\circ 2} \right)} \nonumber\\
    &=  \sqrt{\Phat^\infty_\pi \left( P_\pi (h^\pi)^{\circ 2} - \left((\rho^\pi - r_\pi) + h^\pi\right)^{\circ 2} \right)} \nonumber\\
    &=   \sqrt{\Phat^\infty_\pi \left( P_\pi (h^\pi)^{\circ 2} - (h^\pi)^{\circ 2} + 2(\rho^\pi - r_\pi) \circ h^{\pi} - (\rho^\pi - r_\pi)^{\circ 2} \right)} \nonumber\\
    &\leq   \sqrt{\Phat^\infty_\pi (P_\pi - I) (h^\pi)^{\circ 2}  + 2 \Phat^\infty_\pi \infnorm{\rho^\pi - r_\pi}\infnorm{h^\pi} \one } \nonumber\\
    & \lesssim \sqrt{\left|\Phat^\infty_\pi (P_\pi - I) (h^\pi)^{\circ 2} \right|}  + \sqrt{\infnorm{h^\pi}}\one \label{eq:pf_sketch_var_param_bound}.
\end{align}
Combining all these steps, we have shown
\begin{align}
    \left|\rhohat^\pi - \rho^\pi \right| & \lesssim \frac{\infnorm{h^\pi}}{n}\one + \sqrt{\frac{\infnorm{h^\pi}}{n}}\one + \frac{1}{\sqrt{n}} \sqrt{\left|\Phat^\infty_\pi (P_\pi - I) (h^\pi)^{\circ 2} \right|}. \label{eq:pf_sketch_var_bound_one_step}
\end{align}
Using that $\Phat_\pi^\infty \Phat_\pi = \Phat^\infty_\pi$, we can recognize that $\Phat^\infty_\pi (P_\pi - I) (h^\pi)^{\circ 2} = \Phat^\infty_\pi (P_\pi - \Phat_\pi) (h^\pi)^{\circ 2}$, a term of a very similar form to the RHS of the average-reward simulation lemma step~\eqref{eq:pf_sketch_sim_lemma}. This suggests that we can apply analogous steps to bound this term and thus replace the final term in the RHS of~\eqref{eq:pf_sketch_var_bound_one_step} with lower-order quantities.
Using Bernstein's inequality again,
\begin{align*}
    \left|( \Phat_\pi  - P_\pi) (h^\pi)^{\circ 2} \right| \lesssim \sqrt{\frac{\Var_{P_\pi}\left[ (h^\pi)^{\circ 2} \right]}{n}} + \frac{\infnorm{(h^\pi)^{\circ 2}}}{n} \one \leq \sqrt{\frac{ P_\pi (h^\pi)^{\circ 4} - (P_\pi h^\pi)^{\circ 4}}{n}} + \frac{\infnorm{h^\pi}^{2}}{n}\one  %\label{eq:pf_sketch_bernstein_2}
\end{align*}
where we used $\Var_{P_\pi}\left[ (h^\pi)^{\circ 2} \right] = P_\pi (h^\pi)^{\circ 4} - \left(P_\pi (h^\pi)^{\circ 2} \right)^{\circ 2} \leq P_\pi (h^\pi)^{\circ 4} - (P_\pi h^\pi)^{\circ 4}$ by Jensen's inequality. Thus similarly to steps~\eqref{eq:pf_sketch_Phat_monotonicity} and~\eqref{eq:pf_sketch_var_param_bound} we can bound
\begin{align}
    \cig|\Phat^\infty_\pi (\Phat_\pi - P_\pi) (h^\pi)^{\circ 2} \cig| &\leq \Phat^\infty_\pi \cig|( \Phat_\pi  - P_\pi) (h^\pi)^{\circ 2} \cig| \nonumber \\
    & \lesssim \frac{1}{\sqrt{n}}\sqrt{\Phat^\infty_\pi \left(P_\pi (h^\pi)^{\circ 4} - (P_\pi h^\pi)^{\circ 4} \right) } + \frac{\infnorm{h^\pi}^2}{n}\one \nonumber \\
    & = \frac{1}{\sqrt{n}}\sqrt{\Phat^\infty_\pi \left(P_\pi (h^\pi)^{\circ 4} - \left( (\rho^\pi - r_\pi) + h^\pi\right)^{\circ 4} \right) } + \frac{\infnorm{h^\pi}^2}{n}\one  \nonumber \\
    & \leq \frac{1}{\sqrt{n}}\sqrt{\Phat^\infty_\pi (P_\pi - I) (h^\pi)^{\circ 4} + (2^4 - 1) (\infnorm{h^\pi}+1)^{3}\one } + \frac{\infnorm{h^\pi}^2}{n}\one \label{eq:pf_sketch_recursive_var_bound_key} \\
    & \lesssim \frac{1}{\sqrt{n}}\sqrt{\left|\Phat^\infty_\pi (P_\pi - I) (h^\pi)^{\circ 4} \right|} + \sqrt{\frac{2^4 (\infnorm{h^\pi}+1)^{3}}{n}} \one + \frac{\infnorm{h^\pi}^2}{n}\one \nonumber
\end{align}
where we obtain inequality~\eqref{eq:pf_sketch_recursive_var_bound_key} by noticing $\left( (\rho^\pi - r_\pi) + h^\pi\right)^{\circ 4}$ expands to $2^4$ terms, one of which is $(h^\pi)^{\circ 4}$ and the rest of which have $\infnorm{\cdot}$ bounded by $\infnorm{h^\pi}^{4-k} \infnorm{\rho^\pi - r_\pi}^{k}$ for some $k \in \{1,2,3,4\}$, and also $\infnorm{\rho^\pi - r_\pi} \leq 1$. We have shown
\begin{align*}
    \frac{1}{\sqrt{n}} \sqrt{\cig|\Phat^\infty_\pi (P_\pi - I) (h^\pi)^{\circ 2} \cig|} & =  \frac{1}{\sqrt{n}} \sqrt{\cig|\Phat^\infty_\pi (\Phat_\pi - P_\pi) (h^\pi)^{\circ 2} \cig|} \\
    & \lesssim \frac{1}{\sqrt{n}} \sqrt{\frac{1}{\sqrt{n}}\sqrt{\cig|\Phat^\infty_\pi (P_\pi - I) (h^\pi)^{\circ 4} \cig|} + \sqrt{\frac{2^4 (\infnorm{h^\pi}+1)^{3}}{n}} \one + \frac{\infnorm{h^\pi}^2}{n}\one} \\
    & \leq (2^4)^{1/4} \Big(\frac{\infnorm{h^\pi}+ 1}{n} \Big)^{3/4} \one + \frac{\infnorm{h^\pi}}{n} \one + \frac{1}{n^{3/4}} \cig|\Phat^\infty_\pi (P_\pi - I) (h^\pi)^{\circ 4} \cig|^{1/4}
\end{align*}
and plugging back into~\eqref{eq:pf_sketch_var_bound_one_step} and simplifying, we have
\begin{align*}
    \left|\rhohat^\pi - \rho^\pi \right| & \lesssim 2\frac{\infnorm{h^\pi}}{n} \one + \Big(\frac{\infnorm{h^\pi}+1}{n}\Big)^{1/2} \one + \Big(\frac{\infnorm{h^\pi}+1}{n}\Big)^{3/4} \one + \frac{1}{n^{3/4}} \cig|\Phat^\infty_\pi (P_\pi - I) (h^\pi)^{\circ 4} \cig|^{1/4}.
\end{align*}
As this argument suggests, we can continue bounding terms of the form $\big| \Phat^\infty_\pi (P_\pi - I) (h^\pi)^{\circ 2^k} \big|^{2^{-k}}$, picking up additional terms which are lower-order relative to $\cig(\frac{\infnorm{h^\pi}+1}{n}\cig)^{1/2}$ and increasing the powers of $2$. After roughly $\log_2 \log_2 \infnorm{h^\pi}$ steps all terms will be small enough to end the argument, yielding the desired bound $\infnorm{\rhohat^\pi - \rho^\pi} \lesssim \sqrt{\frac{\infnorm{h^\pi}+1}{n}}$. See Lemmas \ref{lem:recursive_variance_param_bound} and \ref{lem:AMDP_full_var_param_bound}.

Now we briefly outline the additional steps required for our additional results. First, we note that a basically analogous argument, but with $(I - \gamma \Phat_\pi)^{-1}$ replacing $\Phat_\pi^\infty$, $V_\gamma^\pi$ replacing $h^\pi$, and other straightforward adaptations, can be used in the DMDP setting, leading to our DMDP results. One important difference is that, while $\spannorm{h^\pi}$ and $\infnorm{h^\pi}$ are equivalent up to a factor of $2$, we generally have $\infnorm{V_\gamma^\pi} \gg \spannorm{V_\gamma^\pi}$. However, all steps of the argument still go through if we replace $V_\gamma^\pi$ by $\overline{V} = V_\gamma^\pi - (\min_s V_\gamma^\pi(s))\one$, and thus the resulting bound will be in terms of $\infnorm{\overline{V}} = \spannorm{V_\gamma^\pi}$. See Lemmas \ref{lem:DMDP_recursive_var_param_bound} and \ref{lem:DMDP_full_var_param_bound} for details.

One final point is that unlike the sketched Theorem \ref{thm:AMDP_policy_eval}, our other results show the optimality of an empirical (near-)optimal policy $\pihat$. In the average-reward case (e.g. for Theorem \ref{thm:AMDP_plugin_thm}) this requires bounding the two terms $\infnorm{\rhohat^{\pistar} - \rho^{\pistar}}$ and $\infnorm{\rhohat^{\pihat} - \rho^{\pihat}}$. The same technique as sketched above is still used, but the Bernstein inequality steps (e.g.~\eqref{eq:pf_sketch_bernstein}) require more care in order to decouple statistical dependency between $\Phat$ and $h^{\pihat}$. \cite{agarwal_model-based_2020} and \cite{li_breaking_2020} have developed different leave-one-out techniques for this purpose in DMDPs, either of which can be used to establish the ``Bernstein-like'' inequalities required in our argument. Anchoring plays a key role in facilitating the use of their DMDP-based bounds for AMDPs, since by Lemma \ref{lem:anchoring_optimality_properties}, the bias functions in anchored AMDPs are equivalent (up to a constant shift) to certain DMDP value functions. See Lemmas \ref{lem:LOO_DMDP_bernstein_bound}, \ref{lem:LOO_DMDP_bernstein_bound_pert}, \ref{lem:LOO_anch_AMDP_bernstein_bound}, and \ref{lem:LOO_AMDP_bernstein_bound_plug_in} where we establish the Bernstein-like inequalities needed for our different results.

\section{Conclusion}
In this paper we performed the first analysis of the \plugin approach for average-reward MDPs, showing that this simple method obtains optimal rates for the diameter- and mixing-based settings without requiring prior knowledge. Our techniques also lead to improved results for DMDPs. While Theorem \ref{thm:plug_in_lower_bound} suggests our span-based results cannot be improved for the \plugin method, it remains an interesting open question as to whether an improved algorithm can achieve the optimal $\Otilde\cig(\frac{\spannorm{h^\star}}{\varepsilon^2} \cig)$ sample complexity without knowledge of $\spannorm{h^\star}$. In conclusion, we believe this work fills a gap in our understanding of average-reward RL algorithms, and we hope that our techniques can be more broadly useful for the analysis of natural average-reward algorithms.

\section*{Acknowledgments}
{Y.\ Chen and M.\ Zurek were supported in part by National Science Foundation grants CCF-2233152 and DMS-2023239.}

\bibliographystyle{plainnat}
\bibliography{refs}

\clearpage

\appendix

\section{Additional Notation and Guide to Appendices}

\subsection{Additional Notation}
Here we provide definitions for some miscellaneous additional notation used within the appendices.
Letting $x \in \R^S$, we define a next-state transition variance vector with respect to $P$, $\Var_{P}\left[x \right] \in \R^{SA}$, as $\left(\Var_{P} \left[x \right]\right)_{sa} := \sum_{s' \in \S} \left(P\right)_{sa, s'} \big[x(s') - \sum_{s''}(P_{sa,s''}x(s'')\big]^2.$
For a policy $\pi$, we also define a policy-specific version $\Var_{P_{\pi}}\left[x \right] \in \R^{S}$ as
$\left(\Var_{P_{\pi}} \left[x \right]\right)_s := \sum_{s' \in \S} \left(P_{\pi}\right)_{s, s'} \big[x(s') - \sum_{s''}\left(P_{\pi}\right)_{s, s''}x(s'')\big]^2.$
For any policy $\pi$, we define the policy matrix $M^\pi \in \R^{S \times SA}$ by $M^\pi_{s, s a} = \pi(a \mid s)$, and $M^\pi_{s, s' a} = 0$ if $s \neq s'$. We also define the maximization operator $M : \R^{SA} \to \R^S$ by $M(x)_s = \max_{a} x_{sa}$. We also note that for any $x \in \R^{SA}$ and any policy $\pi$, $M(x) \geq M^\pi(x)$.
We say a policy $\pi$ is greedy with respect to a vector $x \in \R^{SA}$ if $M(x) = M^\pi(x)$.
For any transition matrix $P$ and policy $\pi$, we define the deviation matrix $H_{P_\pi}$ as the Drazin inverse of $I - P_\pi$ (see \cite[Appendix A]{puterman_markov_1994} for its basic properties).
Given a DMDP $(P, r, \gamma)$, we define the $Q$-function of policy $\pi$ as $Q_\gamma^\pi = r + \gamma P V_\gamma^\pi$, and we define the optimal $Q$-function as $Q^\star_\gamma = r + \gamma P V_\gamma^\star$. Note that for any policy $\pi$, we have that $V_\gamma^\pi = (I - \gamma P_\pi)^{-1}r_\pi$. We let $\infinfnorm{B}$ denote the $\infnorm{\cdot}$ to $\infnorm{\cdot}$ operator norm of a matrix $B$, and we note that this is equal to the maximum of the $\ell^1$-norms of the rows of $B$. In particular, if $B$ is a stochastic matrix (all rows are probability distributions) then $\infinfnorm{B} =1$.

\subsection{Guide to Appendices}

Now we provide an outline of the appendices.
In Appendix \ref{sec:DMDP_appendix_main} we prove our main results for the DMDP \plugin method, Theorems \ref{thm:DMDP_main_thm} and \ref{thm:DMDP_pert_thm}. This section is further split into Subsection \ref{sec:DMDP_high_order_var_bounds}, where we show a deterministic error decomposition (which can be understood as a DMDP version of the arguments sketched in Section \ref{sec:proof_techniques}), Subsection \ref{sec:DMDP_bernstein_bounds}, where we check the concentration inequalities that are required for this error decomposition, and Subsections \ref{sec:DMDP_main_pf} and \ref{sec:DMDP_pert_pf}, where we complete the proofs of Theorems \ref{thm:DMDP_main_thm} and \ref{thm:DMDP_pert_thm}, respectively. 

In Appendix \ref{sec:AMDP_appendix_main} we prove our main results for the AMDP \plugin method, Theorems \ref{thm:AMDP_plugin_thm}, \ref{thm:AMDP_policy_eval}, \ref{thm:AMDP_anchored_nopert}, \ref{thm:AMDP_anchored_perturbed}, and \ref{thm:AMDP_best_of_both}. This section is likewise split into further subsections. Subsection \ref{sec:key_AMDP_lemmas} contains useful lemmas, including many properties of the anchoring technique. Subsection \ref{sec:AMDP_anch_pert_pf} contains a proof of Theorem \ref{thm:AMDP_anchored_perturbed}, which can be shown as a consequence of the DMDP result Theorem \ref{thm:DMDP_pert_thm}. Analogous to the proofs of the DMDP results, Subsection \ref{sec:AMDP_high_order_var_bounds} contains a deterministic error decomposition (the formal version of the sketched proof techniques) and Subsection \ref{sec:AMDP_bernstein_bounds} checks the required concentration inequalities. We then complete the proofs of Theorems \ref{thm:AMDP_policy_eval}, \ref{thm:AMDP_anchored_nopert}, \ref{thm:AMDP_best_of_both}, and \ref{thm:AMDP_plugin_thm} in Subsections \ref{sec:AMDP_policy_eval_pf}, \ref{sec:AMDP_anch_nopert_pf}, \ref{sec:AMDP_best_of_both_pf}, and \ref{sec:AMDP_plugin_pf}, respectively.

In Appendix \ref{sec:opt_diam_tmix_cors} we provide proofs which lead to our corollaries for the diameter- and mixing-based settings, proving Lemma \ref{lem:empirical_diameter_bound} in Subsection \ref{sec:empirical_diam_bound_pf}, proving the diameter-based Corollary \ref{thm:diameter_complexity} in Subsection \ref{sec:diam_complexity_thm_pf}, and proving the mixing-based Corollary \ref{thm:mixing_complexity} in Subsection \ref{sec:mixing_complexity_thm_pf}. 

In Appendix \ref{sec:other_DMDP_results} we prove additional DMDP-related results, Theorem \ref{thm:DMDP_best_of_both}, Lemma \ref{lem:small_sample_span_bound}, and  Theorem \ref{thm:span_based_without_knowledge}. 

In Appendix \ref{sec:DMDP_reduction_approach} we provide guarantees for the DMDP reduction approach with effective horizon $n$, using the close connection to the anchored AMDP plugin approach.

Finally, in Appendix \ref{sec:proof_of_plug_in_lower_bound} we provide the proof of Theorem \ref{thm:plug_in_lower_bound} on the impossibility of obtaining a purely $\spannorm{h^\star}$-based complexity with the \plugin approach.

\section{Proofs of DMDP Theorems}
\label{sec:DMDP_appendix_main}
In this section we prove Theorems \ref{thm:DMDP_main_thm} and \ref{thm:DMDP_pert_thm}.

\subsection{Higher-order variance bounds}
\label{sec:DMDP_high_order_var_bounds}
\begin{lem}
\label{lem:DMDP_recursive_var_param_bound}
    Fix $k \geq 0$ and let $\Vc = V^{\pi} - \left(\min_s V^{\pi}(s)\right)\one$. If the inequality
    \begin{align}
        \left| \left(\Phat_{\pi} - P_{\pi}\right) \left( \Vc \right)^{\circ 2^k} \right|
    &\leq \sqrt{\frac{\alpha   \Var_{P_{\pi}} \left[ \left(\Vc \right)^{\circ 2^k}  \right]    }{n }} + \frac{ \alpha \cdot 2^k }{n}  \left( \infnorm{\Vc} + 1\right)^{2^k}   \one \label{eq:bernstein_like_inequality}
    \end{align}
    holds for some $\alpha \in \R$, then
    \begin{align*}
        &\left(\frac{2\alpha}{(1-\gamma)n}\right)^{1-2^{-k}}\infnorm{ \gamma^{2^k}(I - \gamma^{2^k} \Phat_{\pi})^{-1} (\Phat_{\pi} - P_\pi) \left( \Vc \right)^{\circ 2^k} }^{2^{-k}} \\
       & \leq \frac{4 \alpha}{(1-\gamma)n} \left( \infnorm{\Vc} + 1\right) + \frac{ 2 }{1-\gamma} \left(\frac{2\alpha }{n} \left( \infnorm{\Vc} + 1\right) \right)^{1-2^{-{(k+1)}}} \\
        & \quad \quad + \left(\frac{2\alpha}{(1-\gamma)n}\right)^{1-2^{-(k+1)}} \infnorm{\gamma^{2^{k+1}}(I - \gamma^{2^{k+1}} \Phat_{\pi})^{-1}\left(\Phat_\pi - P_\pi \right) \left(\Vc \right)^{\circ 2^{k+1}}}^{2^{-(k+1)}}.
    \end{align*}
\end{lem}
We note that this lemma is purely algebraic, and the same statement holds if we swap all appearances of $P$ and $V^\pi$ with $\Phat$ and $\Vhat^\pi$ (respectively), both within the assumption and conclusion of the lemma.

\begin{proof}
    Since $(I - \gamma^{2^k} \Phat_{\pi})^{-1}$ is elementwise non-negative and using the assumption of the lemma,
    \begin{align}
         \left| (I - \gamma^{2^k} \Phat_{\pi})^{-1} (\Phat_{\pi} - P_\pi) \left( \Vc \right)^{\circ 2^k} \right| & \leq  (I - \gamma^{2^k} \Phat_{\pi})^{-1} \left|(\Phat_{\pi} - P_\pi) \left( \Vc \right)^{\circ 2^k} \right| \nonumber \\
         & \leq \sqrt{\frac{\alpha}{n}} (I - \gamma^{2^k} \Phat_{\pi})^{-1} \sqrt{\Var_{P_{\pi}} \left[ \left(\Vc \right)^{\circ 2^k}  \right] }+ \frac{ \alpha \cdot 2^k }{(1-\gamma) n}  \left( \infnorm{\Vc} + 1\right)^{2^k} \one \label{eq:DMDP_recursive_step_1}
    \end{align}
    using the fact that $(I - \gamma^{2^k} \Phat_{\pi})^{-1} \one = \frac{1}{1-\gamma^{2^k}} \one \leq \frac{1}{1-\gamma} \one$. 
Therefore using Jensen's inequality (since each row of $(1-\gamma^{2^k}) (I - \gamma^{2^{k}} \Phat_{\pi})^{-1}$ is a probability distribution), and then using the elementary inequality $\infnorm{(I - \beta P)^{-1}x} \leq 2\infnorm{(I - \beta^2 P)^{-1}x}$ for $\beta \in (0,1)$ \citep{agarwal_model-based_2020}, we have
\begin{align}
    \infnorm{(I - \gamma^{2^k} \Phat_{\pi})^{-1} \sqrt{\Var_{P_{\pi}} \left[ \left(\Vc \right)^{\circ 2^k}  \right] } }& \leq \frac{1}{\sqrt{1-\gamma^{2^k}}} \sqrt{ \infnorm{(I - \gamma^{2^k} \Phat_{\pi})^{-1}\Var_{P_{\pi}} \left[ \left(\Vc \right)^{\circ 2^k}  \right]} } \nonumber\\
    & \leq \frac{\sqrt{2}}{\sqrt{1-\gamma^{2^k}}} \sqrt{\infnorm{(I - \gamma^{2^{k+1}} \Phat_{\pi})^{-1}\Var_{P_{\pi}} \left[ \left(\Vc \right)^{\circ 2^k}  \right]} }\nonumber \\
    & \leq \frac{\sqrt{2}}{\sqrt{1-\gamma}}\sqrt{\infnorm{(I - \gamma^{2^{k+1}} \Phat_{\pi})^{-1}\Var_{P_{\pi}} \left[ \left(\Vc \right)^{\circ 2^k}  \right]} }\label{eq:DMDP_recursive_step_2}.
\end{align}
Note that $(I - \gamma^{2^{k+1}} \Phat_{\pi})^{-1}\Var_{P_{\pi}} \left[ \left(\Vc \right)^{\circ 2^k}  \right] \geq 0$ so it suffices to upper bound $(I - \gamma^{2^{k+1}} \Phat_{\pi})^{-1}\Var_{P_{\pi}} \left[ \left(\Vc \right)^{\circ 2^k}  \right]$ elementwise. 

Abbreviating $\nu = \left(\min_s V^{\pi}(s)\right)(s)$, by the Bellman equation for $V^\pi$ we have
\begin{align*}
    P_\pi \Vc &= P_\pi (V^\pi - \nu \one) \\
    &= P_\pi V^\pi - \nu \one \\
    &= \frac{1}{\gamma} (V^\pi - r_\pi) - \nu \one \\
    &= \frac{1}{\gamma} (V^\pi - r_\pi) + \frac{(1-\gamma) - 1}{\gamma }\nu \one \\
    &= \frac{1}{\gamma} (V^\pi -\nu \one - r_\pi ) + \frac{1-\gamma}{\gamma}\nu \one \\
    &= \frac{1}{\gamma} \left( \Vc - r_\pi + (1-\gamma) \nu \one\right).
\end{align*}
Using this, we can calculate that
\begin{align*}
    \Var_{P_{\pi}} \left[ \left(\Vc \right)^{\circ 2^k}  \right] &= P_{\pi} \left(\Vc \right)^{\circ 2^{k+1}} - \left(P_{\pi} \left(\Vc \right)^{\circ 2^k} \right)^{\circ 2}\\
    & \leq P_{\pi} \left(\Vc \right)^{\circ 2^{k+1}} - \left(\left( P_{\pi} \Vc \right)^{\circ 2^k} \right)^{\circ 2} \\
    &= P_{\pi} \left(\Vc \right)^{\circ 2^{k+1}} - \left( P_{\pi} \Vc \right)^{\circ 2^{k+1}} \\
    &= P_{\pi} \left(\Vc \right)^{\circ 2^{k+1}} - \frac{1}{\gamma^{2^{k+1}}}\left( \Vc - r_\pi + (1-\gamma) \nu \one \right)^{\circ 2^{k+1}} \\
    &= \frac{1}{\gamma^{2^{k+1}}} \left(\gamma^{2^{k+1}} P_{\pi} - I\right) \left(\Vc \right)^{\circ 2^{k+1}} - \frac{1}{\gamma^{2^{k+1}}} \left( \left( \Vc - r_\pi + (1-\gamma) \nu \one \right)^{\circ 2^{k+1}} - \left(\Vc \right)^{\circ 2^{k+1}} \right) \\
    & \leq \frac{1}{\gamma^{2^{k+1}}} \left(\gamma^{2^{k+1}} P_{\pi} - I\right) \left(\Vc \right)^{\circ 2^{k+1}} + \frac{2^{2^{k+1}}}{\gamma^{2^{k+1}}} \max\{\infnorm{\Vc},1\}^{2^{k+1}-1} \one \\
    & \leq \frac{1}{\gamma^{2^{k+1}}} \left(\gamma^{2^{k+1}} P_{\pi} - I\right) \left(\Vc \right)^{\circ 2^{k+1}} + \frac{2^{2^{k+1}}}{\gamma^{2^{k+1}}} \left(\infnorm{\Vc} + 1 \right)^{2^{k+1}-1} \one
\end{align*}
where we used the fact that $P_{\pi} \left(\Vc \right)^{\circ 2^k} \geq \left( P_{\pi} \Vc \right)^{\circ 2^k}$ by Jensen's inequality (applied to each row), and then the fact that
\[
 \left( \Vc - (r_\pi - (1-\gamma) \nu \one) \right)^{\circ 2^{k+1}} - \left(\Vc \right)^{\circ 2^{k+1}} 
\]
contains $< 2^{2^{k+1}}$ terms, all of which have magnitude bounded by $\max\{\infnorm{\Vc}, 1\}^{2^{k+1}-1}$ since $\infnorm{r_\pi - (1-\gamma) \nu} \leq 1$ (because $r_\pi \in [0,1]$ elementwise and also $(1-\gamma) \nu \in [0,1]$ since $0 \leq \nu \leq \frac{1}{1-\gamma}$).
Plugging this into the RHS of~\eqref{eq:DMDP_recursive_step_2}, we obtain
\begin{align*}
    &(I - \gamma^{2^{k+1}} \Phat_{\pi})^{-1}\Var_{P_{\pi}} \left[ \left(\Vc \right)^{\circ 2^k}  \right] \\
    & \leq  \frac{1}{\gamma^{2^{k+1}}} (I - \gamma^{2^{k+1}} \Phat_{\pi})^{-1} \left(\gamma^{2^{k+1}} P_{\pi} - I\right) \left(\Vc \right)^{\circ 2^{k+1}}  + \frac{2^{2^{k+1}}}{\gamma^{2^{k+1}}} \left(\infnorm{\Vc} + 1 \right)^{2^{k+1}-1} (I - \gamma^{2^{k+1}} \Phat_{\pi})^{-1} \one \\
    &= \frac{1}{\gamma^{2^{k+1}}} (I - \gamma^{2^{k+1}} P_{\pi})^{-1} \left(\gamma^{2^{k+1}} P_{\pi} - I\right) \left(\Vc \right)^{\circ 2^{k+1}} \\
    & \quad \quad + \frac{1}{\gamma^{2^{k+1}}} \left((I - \gamma^{2^{k+1}} \Phat_{\pi})^{-1} - (I - \gamma^{2^{k+1}} P_{\pi})^{-1} \right)\left(\gamma^{2^{k+1}} P_{\pi} - I\right) \left(\Vc \right)^{\circ 2^{k+1}} \\
    & \quad \quad + \frac{2^{2^{k+1}}}{(1-\gamma^{2^{k+1}})\gamma^{2^{k+1}}} \left(\infnorm{\Vc} + 1 \right)^{2^{k+1}-1}  \one \\
    & = -\frac{1}{\gamma^{2^{k+1}}} \left(\Vc \right)^{\circ 2^{k+1}} + \frac{2^{2^{k+1}}}{(1-\gamma^{2^{k+1}})\gamma^{2^{k+1}}} \left(\infnorm{\Vc} + 1 \right)^{2^{k+1}-1}  \one \\
    & \quad \quad + \frac{1}{\gamma^{2^{k+1}}} \gamma^{2^{k+1}}(I - \gamma^{2^{k+1}} \Phat_{\pi})^{-1}\left(\Phat_\pi - P_\pi \right) (I - \gamma^{2^{k+1}} P_{\pi})^{-1} \left(\gamma^{2^{k+1}} P_{\pi} - I\right) \left(\Vc \right)^{\circ 2^{k+1}} \\
    & = -\frac{1}{\gamma^{2^{k+1}}} \left(\Vc \right)^{\circ 2^{k+1}} + \frac{2^{2^{k+1}}}{(1-\gamma^{2^{k+1}})\gamma^{2^{k+1}}} \left(\infnorm{\Vc} + 1 \right)^{2^{k+1}-1}  \one \\
    & \quad \quad - \frac{1}{\gamma^{2^{k+1}}} \gamma^{2^{k+1}}(I - \gamma^{2^{k+1}} \Phat_{\pi})^{-1}\left(\Phat_\pi - P_\pi \right)  \left(\Vc \right)^{\circ 2^{k+1}} \\
    & \leq \frac{1}{\gamma^{2^{k+1}}} \left(\frac{2^{2^{k+1}}}{1-\gamma} \left(\infnorm{\Vc} + 1 \right)^{2^{k+1}-1} + \infnorm{\gamma^{2^{k+1}}(I - \gamma^{2^{k+1}} \Phat_{\pi})^{-1}\left(\Phat_\pi - P_\pi \right) \left(\Vc \right)^{\circ 2^{k+1}}} \right) \one.
\end{align*}
Thus inequality~\eqref{eq:DMDP_recursive_step_2}, combined with the fact that $\sqrt{a+b} \leq \sqrt{a} + \sqrt{b}$, yields that
\begin{align}
    &\infnorm{(I - \gamma^{2^k} \Phat_{\pi})^{-1} \sqrt{\Var_{P_{\pi}} \left[ \left(\Vc \right)^{\circ 2^k}  \right] } } \\
    & \leq \frac{\sqrt{2}}{\sqrt{1-\gamma}} \sqrt{\frac{1}{\gamma^{2^{k+1}}}} \sqrt{\frac{2^{2^{k+1}}}{1-\gamma} \left(\infnorm{\Vc} + 1 \right)^{2^{k+1}-1} }  \nonumber \\
    & \quad \quad + \frac{\sqrt{2}}{\sqrt{1-\gamma}} \sqrt{\frac{1}{\gamma^{2^{k+1}}}} \sqrt{\infnorm{\gamma^{2^{k+1}}(I - \gamma^{2^{k+1}} \Phat_{\pi})^{-1}\left(\Phat_\pi - P_\pi \right) \left(\Vc \right)^{\circ 2^{k+1}}}} \\
    & = \frac{\sqrt{2}}{\sqrt{1-\gamma}} \frac{1}{\gamma^{2^k}} \sqrt{\frac{2^{2^{k+1}}}{1-\gamma} \left(\infnorm{\Vc} + 1 \right)^{2^{k+1}-1} }  \nonumber \\
    & \quad \quad + \frac{\sqrt{2}}{\sqrt{1-\gamma}}\frac{1}{\gamma^{2^k}} \sqrt{\infnorm{\gamma^{2^{k+1}}(I - \gamma^{2^{k+1}} \Phat_{\pi})^{-1}\left(\Phat_\pi - P_\pi \right) \left(\Vc \right)^{\circ 2^{k+1}}}}.
\end{align}
Combining this with inequality~\eqref{eq:DMDP_recursive_step_1} gives
\begin{align*}
         &\infnorm{ (I - \gamma^{2^k} \Phat_{\pi})^{-1} (\Phat_{\pi} - P_\pi) \left( \Vc \right)^{\circ 2^k} }\\
         & \leq \sqrt{\frac{\alpha}{n}} \infnorm{(I - \gamma^{2^k} \Phat_{\pi})^{-1} \sqrt{\Var_{P_{\pi}} \left[ \left(\Vc \right)^{\circ 2^k}  \right] }} + \frac{ \alpha \cdot 2^k }{(1-\gamma) n}  \left( \infnorm{\Vc} + 1\right)^{2^k} \\
         & \leq \frac{ \alpha \cdot 2^k }{(1-\gamma) n}  \left( \infnorm{\Vc} + 1\right)^{2^k} + \sqrt{\frac{\alpha}{n}}\frac{\sqrt{2}}{\sqrt{1-\gamma}} \frac{1}{\gamma^{2^k}} \sqrt{\frac{2^{2^{k+1}}}{1-\gamma} \left(\infnorm{\Vc} + 1 \right)^{2^{k+1}-1} }  \nonumber \\
    & \quad \quad +  \sqrt{\frac{\alpha}{n}} \frac{\sqrt{2}}{\sqrt{1-\gamma}} \frac{1}{\gamma^{2^k}} \sqrt{\infnorm{\gamma^{2^{k+1}}(I - \gamma^{2^{k+1}} \Phat_{\pi})^{-1}\left(\Phat_\pi - P_\pi \right) \left(\Vc \right)^{\circ 2^{k+1}}}}
    \end{align*}
so using the fact that $(a+b)^{2^{-k}} \leq a^{2^{-k}} + b^{2^{-k}}$ and that $\gamma^{2^k} < 1$,
\begin{align*}
         &\left(\frac{2\alpha}{(1-\gamma)n}\right)^{1-2^{-k}}\infnorm{ \gamma^{2^k}(I - \gamma^{2^k} \Phat_{\pi})^{-1} (\Phat_{\pi} - P_\pi) \left( \Vc \right)^{\circ 2^k} }^{2^{-k}}  \\
         &  \quad \quad \leq \left(\frac{2\alpha}{(1-\gamma)n}\right)^{1-2^{-k}} \left(\frac{ \alpha \cdot 2^k }{(1-\gamma) n}  \left( \infnorm{\Vc} + 1\right)^{2^k} \right)^{2^{-k}}  \\
         & \quad \quad + \left(\frac{2\alpha }{(1-\gamma)n}\right)^{1-2^{-k}} \left(\sqrt{\frac{\alpha}{n}}\frac{\sqrt{2}}{\sqrt{1-\gamma}}  \sqrt{\frac{2^{2^{k+1}}}{1-\gamma} \left(\infnorm{\Vc} + 1 \right)^{2^{k+1}-1} } \right)^{2^{-k}}  \\
        & \quad \quad + \left(\frac{2\alpha}{(1-\gamma)n}\right)^{1-2^{-k}} \left( \sqrt{\frac{\alpha}{n}} \frac{\sqrt{2}}{\sqrt{1-\gamma}}  \sqrt{\infnorm{\gamma^{2^{k+1}}(I - \gamma^{2^{k+1}} \Phat_{\pi})^{-1}\left(\Phat_\pi - P_\pi \right) \left(\Vc \right)^{\circ 2^{k+1}}}} \right)^{2^{-k}} \\
        & = \frac{2^{1-2^{-k} +k2^{-k}} \alpha}{(1-\gamma)n} \left( \infnorm{\Vc} + 1\right) + \frac{ 2 }{1-\gamma} \left(\frac{2\alpha }{n} \left( \infnorm{\Vc} + 1\right) \right)^{1-2^{-{(k+1)}}} \\
        & \quad \quad + \left(\frac{2\alpha}{(1-\gamma)n}\right)^{1-2^{-(k+1)}} \infnorm{\gamma^{2^{k+1}}(I - \gamma^{2^{k+1}} \Phat_{\pi})^{-1}\left(\Phat_\pi - P_\pi \right) \left(\Vc \right)^{\circ 2^{k+1}}}^{2^{-(k+1)}} \\
        & \leq \frac{4 \alpha}{(1-\gamma)n} \left( \infnorm{\Vc} + 1\right) + \frac{ 2 }{1-\gamma} \left(\frac{2\alpha }{n} \left( \infnorm{\Vc} + 1\right) \right)^{1-2^{-{(k+1)}}} \\
        & \quad \quad + \left(\frac{2\alpha}{(1-\gamma)n}\right)^{1-2^{-(k+1)}} \infnorm{\gamma^{2^{k+1}}(I - \gamma^{2^{k+1}} \Phat_{\pi})^{-1}\left(\Phat_\pi - P_\pi \right) \left(\Vc \right)^{\circ 2^{k+1}}}^{2^{-(k+1)}}
    \end{align*}
    as desired.
\end{proof}

\begin{lem}
\label{lem:DMDP_full_var_param_bound}
    Let $\Vc = V^{\pi} - \left(\min_s V^{\pi}(s)\right)\one$ and $\ell = \lceil\log_2 \log_2 \left( \infnorm{\Vc} + 4\right) \rceil$. Suppose that for some $\alpha \in \R$, the inequalities
    \begin{align*}
        \left| \left(\Phat_{\pi} - P_{\pi}\right) \left( \Vc \right)^{\circ 2^k} \right|
    &\leq \sqrt{\frac{\alpha   \Var_{P_{\pi}} \left[ \left(\Vc \right)^{\circ 2^k}  \right]    }{n }} + \frac{ \alpha \cdot 2^k }{n}  \left( \infnorm{\Vc} + 1\right)^{2^k}   \one 
    \end{align*}
    hold for all $k = 0, \dots, \ell$.
    Then
    \begin{align*}
    \infnorm{ \Vhat^\pi - V^\pi} &\leq \frac{4 (\ell + 1) \alpha}{(1-\gamma)n} \left( \infnorm{\Vc} + 1\right) + \frac{2(\ell + 1)}{1-\gamma} \sqrt{\frac{2 \alpha \left( \infnorm{\Vc} + 1\right)}{n}} .
\end{align*}
\end{lem}

\begin{proof}
    First we give a weaker but non-recursive bound which can be used on the final term. 
Note that
\begin{align*}
        \left| \left(\Phat_{\pi} - P_{\pi}\right) \left( \Vc \right)^{\circ 2^\ell} \right|
    &\leq \sqrt{\frac{\alpha   \Var_{P_{\pi}} \left[ \left(\Vc \right)^{\circ 2^\ell}  \right]    }{n }} + \frac{ \alpha \cdot 2^\ell }{n}  \left( \infnorm{\Vc} + 1\right)^{2^\ell}   \one \\
    & \leq \sqrt{\frac{\alpha}{n}} \infnorm{\Vc}^{2^\ell} \one  + \frac{ \alpha \cdot 2^\ell }{n}  \left( \infnorm{\Vc} + 1\right)^{2^\ell} \one 
    \end{align*}
    so
    \begin{align*}
        \infnorm{ \gamma^{2^\ell}(I - \gamma^{2^\ell} \Phat_{\pi})^{-1} (\Phat_{\pi} - P_\pi) \left( \Vc \right)^{\circ 2^\ell} } & \leq \gamma^{2^\ell} \infnorm{(I - \gamma^{2^\ell} \Phat_{\pi})^{-1} \left| (\Phat_{\pi} - P_\pi) \left( \Vc \right)^{\circ 2^\ell} \right|} \\
        & \leq \frac{1}{1-\gamma} \sqrt{\frac{\alpha}{n}} \infnorm{\Vc}^{2^\ell} + \frac{1}{1-\gamma} \frac{ \alpha \cdot 2^\ell }{n}  \left( \infnorm{\Vc} + 1\right)^{2^\ell} \\
        & \leq \frac{1}{1-\gamma} \sqrt{\frac{\alpha}{n}} \left( \infnorm{\Vc} + 1\right)^{2^\ell} + \frac{1}{1-\gamma} \frac{ \alpha \cdot 2^\ell }{n}  \left( \infnorm{\Vc} + 1\right)^{2^\ell}
    \end{align*}
    since $(I - \gamma^{2^\ell} \Phat_{\pi})^{-1} \one = \frac{1}{1-\gamma^{2^\ell}} \one \leq \frac{1}{1-\gamma} \one$ and $\gamma < 1$.
    Therefore
     \begin{align}
        &\left(\frac{2\alpha}{(1-\gamma)n}\right)^{1-2^{-\ell}}\infnorm{ \gamma^{2^\ell}(I - \gamma^{2^\ell} \Phat_{\pi})^{-1} (\Phat_{\pi} - P_\pi) \left( \Vc \right)^{\circ 2^\ell} }^{2^{-\ell}} \nonumber\\
        &\quad \leq \left(\frac{2\alpha}{(1-\gamma)n}\right)^{1-2^{-\ell}} \left( \frac{1}{1-\gamma} \sqrt{\frac{\alpha}{n}} \left( \infnorm{\Vc} + 1\right)^{2^\ell} \right)^{2^{-\ell}} + \left(\frac{2\alpha}{(1-\gamma)n}\right)^{1-2^{-\ell}} \left(  \frac{1}{1-\gamma} \frac{ \alpha \cdot 2^\ell }{n}  \left( \infnorm{\Vc} + 1\right)^{2^\ell} \right)^{2^{-\ell}} \nonumber \\
        & \quad \leq \frac{2}{1-\gamma} \left(\frac{\alpha}{n}\right)^{1-2^{-(\ell+1)}} \left( \infnorm{\Vc} + 1\right) + \frac{1}{1-\gamma}\frac{4 \alpha}{n} \left( \infnorm{\Vc} + 1\right). \label{eq:non_recursive_bound_final_term}
    \end{align}

Note that
\begin{align*}
    \Vhat^\pi - V^\pi &= \gamma (I - \gamma \Phat_\pi)^{-1} \left(\Phat_\pi - P_\pi\right) V^\pi \\
    &= \gamma  (I - \gamma \Phat_\pi)^{-1} \left(\Phat_\pi - P_\pi\right) \left(V^\pi -  \left( \min_s V^\pi(s) \right) \one \right)
\end{align*}
since $\left(\Phat_\pi - P_\pi\right) \one = 0$.
Using Lemma \ref{lem:DMDP_recursive_var_param_bound} $\ell$ times and using the above bound~\eqref{eq:non_recursive_bound_final_term} for the final term, we obtain
\begin{align}
    \infnorm{ \Vhat^\pi - V^\pi} &= \infnorm{\gamma  (I - \gamma \Phat_\pi)^{-1} \left(\Phat_\pi - P_\pi\right) \left(V^\pi -  \left( \min_s V^\pi(s) \right) \one \right)} \nonumber\\
    & \leq \sum_{k = 0}^{\ell - 1} \left( \frac{4 \alpha}{(1-\gamma)n} \left( \infnorm{\Vc} + 1\right) + \frac{ 2 }{1-\gamma} \left(\frac{2\alpha }{n} \left( \infnorm{\Vc} + 1\right) \right)^{1-2^{-{(k+1)}}} \right) \nonumber \\
    & \quad \quad + \left(\frac{2\alpha}{(1-\gamma)n}\right)^{1-2^{-\ell}} \infnorm{\gamma^{2^{\ell}}(I - \gamma^{2^{\ell}} \Phat_{\pi})^{-1}\left(\Phat_\pi - P_\pi \right) \left(\Vc \right)^{\circ 2^{\ell}}}^{2^{-\ell}} \nonumber \\
    & \leq \sum_{k = 0}^{\ell - 1} \left( \frac{4 \alpha}{(1-\gamma)n} \left( \infnorm{\Vc} + 1\right) + \frac{ 2 }{1-\gamma} \left(\frac{2\alpha }{n} \left( \infnorm{\Vc} + 1\right) \right)^{1-2^{-{(k+1)}}} \right) \nonumber \\
    & \quad \quad + \frac{2}{1-\gamma} \left(\frac{\alpha}{n}\right)^{1-2^{-(\ell+1)}} \left( \infnorm{\Vc} + 1\right) + \frac{1}{1-\gamma}\frac{4 \alpha}{n} \left( \infnorm{\Vc} + 1\right)\nonumber  \\
    & \leq \frac{4 (\ell + 1) \alpha}{(1-\gamma)n} \left( \infnorm{\Vc} + 1\right) + \frac{2\ell}{1-\gamma} \sqrt{\frac{2 \alpha \left( \infnorm{\Vc} + 1\right)}{n}} + \frac{2}{1-\gamma} \left(\frac{\alpha}{n}\right)^{1-2^{-(\ell+1)}} \left( \infnorm{\Vc} + 1\right) \label{eq:DMDP_error_bound_penultimate}
\end{align}
where we assume that $\frac{2\alpha }{n} \left( \infnorm{\Vc} + 1\right)  < 1$ in the final inequality step, so that the $k=0$ term is the largest term in $\sum_{k = 0}^{\ell - 1}  \left(\frac{2\alpha }{n} \left( \infnorm{\Vc} + 1\right) \right)^{1-2^{-{(k+1)}}} $. Now we check that $\ell = \lceil\log_2 \log_2 \left( \infnorm{\Vc} + 4\right) \rceil$ is sufficiently large so that the rightmost term in the RHS of~\eqref{eq:DMDP_error_bound_penultimate} is smaller than $\frac{2}{1-\gamma} \sqrt{\frac{2\alpha \left( \infnorm{\Vc} + 1\right)}{n}}$. We have
\begin{align*}
    \left(\frac{\alpha}{n}\right)^{1-2^{-(\ell+1)}} \left( \infnorm{\Vc} + 1\right) &= \left(\frac{\alpha}{n} \left( \infnorm{\Vc} + 1\right)\right)^{1-2^{-(\ell+1)}} \left( \infnorm{\Vc} + 1\right)^{2^{-(\ell+1)}} \\
    &\leq  \left(\frac{\alpha}{n} \left( \infnorm{\Vc} + 1\right)\right)^{1/2} \left( \infnorm{\Vc} + 1\right)^{2^{-(\ell+1)}}
\end{align*}
since $\ell \geq 0$. Furthermore we have the equivalences
\begin{align*}
    \left( \infnorm{\Vc} + 1\right)^{2^{-(\ell+1)}} \leq \sqrt{2} & \iff 2^{-(\ell+1)} \log_2 \left( \infnorm{\Vc} + 1\right) \leq 1/2 \\
    & \iff \log_2 \left( \infnorm{\Vc} + 1\right) \leq 2^{\ell} \\
    & \iff \ell \geq  \log_2 \log_2 \left( \infnorm{\Vc} + 1\right) 
\end{align*}
and this RHS of the final inequality is smaller than our choice of $\ell$. Thus our choice of $\ell $ indeed guarantees that
\begin{align*}
    \frac{2}{1-\gamma} \left(\frac{\alpha}{n}\right)^{1-2^{-(\ell+1)}} \left( \infnorm{\Vc} + 1\right) &\leq \frac{2}{1-\gamma} \sqrt{\frac{\alpha \left( \infnorm{\Vc} + 1\right)}{n}}  \left( \infnorm{\Vc} + 1\right)^{2^{-(\ell+1)}} \\
    &\leq \frac{2}{1-\gamma} \sqrt{\frac{2\alpha \left( \infnorm{\Vc} + 1\right)}{n}}.
\end{align*}
Plugging this into~\eqref{eq:DMDP_error_bound_penultimate}, we conclude that
\begin{align*}
    \infnorm{ \Vhat^\pi - V^\pi} &\leq \frac{4 (\ell + 1) \alpha}{(1-\gamma)n} \left( \infnorm{\Vc} + 1\right) + \frac{2(\ell + 1)}{1-\gamma} \sqrt{\frac{2 \alpha \left( \infnorm{\Vc} + 1\right)}{n}} 
\end{align*}
as desired.
\end{proof}

\subsection{Bernstein-like inequalities}
\label{sec:DMDP_bernstein_bounds}
Now we show that each of the different versions of the Bernstein-like inequality~\eqref{eq:bernstein_like_inequality} which are needed for our different results hold with high probability.

For an optimal $\gamma$-discounted policy $\pistar_\gamma$, which is fixed and independent of $\Phat$, showing the the Bernstein-like inequality~\eqref{eq:bernstein_like_inequality} follows almost immediately from Bernstein's inequality.

\begin{lem}
\label{lem:simple_DMDP_bernstein_bound}
    Let $\Vc = V^{\star} - \left(\min_s V^{\star}(s)\right)\one$. With probability at least $1 - \delta$, we have that for all $k = 0, \dots, \left\lceil\log_2 \log_2 \left( \spannorm{V^\star} + 4\right) \right\rceil$,
    \begin{align*}
        \left| \left(\Phat_{\pistar_\gamma} - P_{\pistar_\gamma}\right) \left( \Vc \right)^{\circ 2^k} \right|
    &\leq \sqrt{\frac{\alpha   \Var_{P_{\pistar_\gamma}} \left[ \left(\Vc \right)^{\circ 2^k}  \right]    }{n }} + \frac{\alpha  }{n}  \infnorm{\Vc}^{2^k}   \one
    \end{align*}
    where $\alpha = 2 \log \left( \frac{6 S \log_2 \log_2 \left(\spannorm{V^\star}+4 \right)}{\delta} \right)$.
\end{lem}
\begin{proof}
    First we fix $k \in \{0, \dots, \left\lceil\log_2 \log_2 \left( \spannorm{V^\star} + 4\right) \right\rceil\}$. Fix a state $s \in \S$ and note that we defined $\pistar_\gamma$ to be a deterministic policy, so we may treat $\pistar_\gamma(s) $ as an element of $\A$. By applying Bernstein's inequality (e.g. \citealt[Theorem 3]{maurer_empirical_2009}), we have that with probability at least $1 - 2 \delta'$,
    \begin{align*}
        \left| \left(\Phat_{s,\pistar_\gamma(s)} - P_{s,\pistar_\gamma(s)}\right) \left( \Vc \right)^{\circ 2^k} \right|
    &\leq \sqrt{\frac{ 2\log \left(\frac{1}{\delta'} \right)   \Var_{P_{s,\pistar_\gamma(s)}} \left[ \left(\Vc \right)^{\circ 2^k}  \right]    }{n }} + \frac{  \log \left(\frac{1}{\delta'} \right)  }{3n}  \infnorm{\left(\Vc \right)^{\circ 2^k}}  \\
    &= \sqrt{\frac{ 2\log \left(\frac{1}{\delta'} \right)   \Var_{P_{s,\pistar_\gamma(s)}} \left[ \left(\Vc \right)^{\circ 2^k}  \right]    }{n }} + \frac{  \log \left(\frac{1}{\delta'} \right)  }{3n}  \infnorm{\Vc}^{2^k} \\
    & \leq \sqrt{\frac{ 2\log \left(\frac{1}{\delta'} \right)   \Var_{P_{s,\pistar_\gamma(s)}} \left[ \left(\Vc \right)^{\circ 2^k}  \right]    }{n }} + \frac{ 2 \log \left(\frac{1}{\delta'} \right)  }{n}  \infnorm{\Vc}^{2^k}.
    \end{align*}
    Now taking a union bound over $s \in \S$, we have that the above inequality holds for all $s$ simultaneously with probability at least $1 - 2S \delta'$, in which case we have (elementwise)
    \begin{align*}
        \left| \left(\Phat_{\pistar_\gamma} - P_{\pistar_\gamma}\right) \left( \Vc \right)^{\circ 2^k} \right| & \leq \sqrt{\frac{ 2\log \left(\frac{1}{\delta'} \right)   \Var_{P_{\pistar_\gamma}} \left[ \left(\Vc \right)^{\circ 2^k}  \right]    }{n }} + \frac{ 2 \log \left(\frac{1}{\delta'} \right)  }{n}  \infnorm{\Vc}^{2^k} \one.
    \end{align*}
    Finally, taking another union bound over all possible values of $k$, of which there are
    \begin{align*}
        1+\left\lceil\log_2 \log_2 \left( \spannorm{V^\star} + 4\right) \right\rceil &\leq 2 + \log_2 \log_2 \left( \spannorm{V^\star} + 4\right)  \leq 3 \log_2 \log_2 \left(\spannorm{V^\star} +4\right),
    \end{align*}
    (since $\log_2 \log_2 \left(\spannorm{V^\star} +4\right) \geq \log_2 \log_2 4 \geq 1$) and setting $\delta' = \frac{ \delta}{6 S \log_2 \log_2 \left(\spannorm{V^\star} +4\right)}$, we obtain the desired conclusion.
\end{proof}

For the empirical (near)-optimal policy $\pihat$, which is statistically dependent on $\Phat$, this requires more effort, in particular the use of the absorbing MDP construction pioneered by \cite{agarwal_model-based_2020} to decouple the statistical dependency. We first present their construction.

\begin{thm}[\cite{agarwal_model-based_2020}]
\label{thm:agarwal_LOO_construction}
    There exists a collection of random variables $\Vhatstar_{s,u}$ for $s \in \S$ and $u \in [0,1]$ such that
    \begin{enumerate}
        \item For all $s \in \S, u \in [0,1]$, $\Vhatstar_{s,u}$ is independent from all of the random variables $S^{i}_{s,a}$ for all $i = 1, \dots, n$ and all $a \in \A$ (recall these are the $n$ observed transitions from state-action pair $(s,a)$ which are used to form $\Phat$).
        \item For all $s \in \S, u,u' \in [0,1]$, $\infnorm{\Vhatstar_{s,u} - \Vhatstar_{s,u'}} \leq \frac{|u-u'|}{1-\gamma}$.
        \item Letting $u^\star(s) = (1-\gamma)\Vhatstar(s)$, $\Vhatstar_{s,u^\star(s)} = \Vhatstar$.
    \end{enumerate}
    As a consequence, there exists a finite set $U$ of $\left\lceil \frac{1}{2(1-\gamma)\varepsilon} \right\rceil$ equally-spaced points in $[0,1]$ such that for all $s \in \S$, there exists $u \in U$ such that $\infnorm{\Vhatstar_{s, u} - \Vhatstar} \leq \varepsilon$.
\end{thm}

We also note a useful elementary inequality.
\begin{lem}
\label{lem:difference_of_powers}
For any natural number $n$,
    \[\left| a^n - b^n\right| \leq n|a-b| \left(\max\{|a|, |b|\} \right)^{n-1}.\]
\end{lem}
\begin{proof}
This follows immediately from the algebraic identity
    \[
    a^n - b^n = (a-b)(a^{n-1} + a^{n-2}b + a^{n-3}b^2 + \dots + b^{n-1}).
    \]
\end{proof}

Now we can use the leave-one-out construction of \cite{agarwal_model-based_2020} to check a version of the Bernstein-like inequalities.
\begin{lem}
\label{lem:LOO_DMDP_bernstein_bound}
    If $n \geq 4$, then with probability at least $1-\delta$, for all $\pihat$ which satisfy $\infnorm{\Vhat^{\pihat} - \Vhatstar} \leq \frac{1}{n}$, letting $\Vc = \Vhat^{\pihat} - \left(\min_s \Vhat^{\pihat}(s)\right)\one$, for all $k = 0, \dots, \left\lceil\log_2 \log_2 \left( \spannorm{\Vhat^{\pihat}} + 4\right) \right\rceil$, we have
    \begin{align*}
        \left| \left(\Phat_{\pihat} - P_{\pihat}\right) \left( \Vc \right)^{\circ 2^k} \right|
    &\leq \sqrt{\frac{\alpha   \Var_{\Phat_{\pihat}} \left[ \left(\Vc \right)^{\circ 2^k}  \right]    }{n }} + \frac{ \alpha \cdot 2^k }{n}  \left( \infnorm{\Vc} + 1\right)^{2^k}   \one
    \end{align*}
where $\alpha = 16 \log \left(12 \frac{SAn}{(1-\gamma)^2 \delta} \right)$.
\end{lem}
\begin{proof}

We use the leave-one-out construction from \cite{agarwal_model-based_2020}: by Theorem \ref{thm:agarwal_LOO_construction}, there exists a finite set $U$ with $|U|=\left\lceil \frac{n}{2(1-\gamma)} \right\rceil$ such that for all $s \in \S$, there exists $u \in U$ such that $\infnorm{\Vhatstar_{s, u} - \Vhatstar} \leq \frac{1}{n}$. Also note that since $n \geq 2$ by assumption, $1 \leq \frac{n}{2(1-\gamma)}$, so we can bound $|U| \leq \frac{n}{2(1-\gamma)} + 1 \leq \frac{n}{1-\gamma}$.
We define, for all $s \in \S$ and $u \in U$, $\Vc_{s,u} = \Vhatstar_{s,u} - \left(\min_{s'} \Vhatstar_{s,u}(s')\right)\one$.

We note that the quantity $\left\lceil\log_2 \log_2 \left( \spannorm{\Vhat^{\pihat}} + 4\right) \right\rceil$ appearing in the lemma statement is random, but since $\spannorm{\Vhat^{\pihat}} \leq \frac{1}{1-\gamma}$, we can simply bound
\begin{align}
    \left\lceil\log_2 \log_2 \left( \spannorm{\Vhat^{\pihat}} + 4\right) \right\rceil &\leq \left\lceil\log_2 \log_2 \left( \frac{1}{1-\gamma} + 4\right) \right\rceil \nonumber\\
     &\leq 1 + \log_2 \log_2 \left( \frac{1}{1-\gamma} + 4\right) \nonumber\\
     &\leq 1 + 2 \frac{1}{1-\gamma} \nonumber\\
     & \leq 3 \frac{1}{1-\gamma} \label{eq:LOO_ineq_bound_on_ks}
\end{align}
(using that $\log_2 \log_2 (x+4) \leq 2x$ for $x \geq 1$), so we can check the inequality for all values of $k$ up to the upper bound $3 \frac{1}{1-\gamma}$, which is at most $1 + 3 \frac{1}{1-\gamma} \leq 4 \frac{1}{1-\gamma}$ values of $k$.
Therefore for the rest of the proof we will focus on showing the desired conclusion only for some fixed $k$, from which we can immediately obtain the desired conclusion by taking a union bound and adjusting the failure probability.

For each state $s$, action $a$, and $u \in U$, we will use \cite[Theorem 10]{maurer_empirical_2009} to show that, with probability at least $1 - \delta'$,
\begin{align}
        \sqrt{\Var_{P_{sa}} \left[ \left(\Vc_{s,u} \right)^{\circ 2^k} \right] }\leq \sqrt{2 \Var_{\Phat_{sa}} \left[ \left(\Vc_{s,u} \right)^{\circ 2^k} \right] } + 2 \infnorm{\left(\Vc_{s,u} \right)^{\circ 2^k}}\sqrt{\frac{\log (1/\delta')}{n}}. \label{eq:maurer_empirical_conclusion}
\end{align}
To match the notation of \cite[Theorem 10]{maurer_empirical_2009}, let $x_i := \left(\Vc_{s,u} \right)^{\circ 2^k}(S^i_{s,a})$, that is the value of $\left(\Vc_{s,u} \right)^{\circ 2^k}$ at the state $S^i_{s,a}$, which is the $i$th sample from the state-action distribution $P(\cdot \mid s,a)$. \cite[Theorem 10]{maurer_empirical_2009} assumes that $x_i \in [0,1]$, but we can apply their theorem to the quantities $x_i' := \frac{x_i}{\infnorm{\left(\Vc_{s,u} \right)^{\circ 2^k}}}$ to obtain that if $n \geq 2$, with probability at least $1-\delta'$,
\begin{align*}
        \sqrt{\Var_{P_{sa}} \left[ \left(\Vc_{s,u} \right)^{\circ 2^k} \right] }\leq \sqrt{\frac{n}{n-1} \Var_{\Phat_{sa}} \left[ \left(\Vc_{s,u} \right)^{\circ 2^k} \right] } +  \spannorm{\left(\Vc_{s,u} \right)^{\circ 2^k}}\sqrt{\frac{2\log (1/\delta')}{n-1}}.
\end{align*}
This is true because we have
\begin{align*}
    \frac{n}{n-1} \Var_{\Phat_{sa}} \left[ \left(\Vc_{s,u} \right)^{\circ 2^k} \right] &= \frac{n}{n-1} \left(\Phat_{sa} \left( \left(\Vc_{s,u} \right)^{\circ 2^k}\right)^{\circ 2}  - \left( \Phat_{sa} \left(\Vc_{s,u} \right)^{\circ 2^k} \right)^{\circ 2} \right) \\
    &= \frac{n}{n-1} \left( \frac{1}{n} \sum_{i=1}^n x_i^2 - \left(\frac{1}{n} \sum_{i=1}^n x_i \right)^2 \right) \\
    &=\frac{1}{n-1} \sum_{i=1}^n x_i^2 + \frac{1}{n(n-1)}\left(\sum_{i=1}^n x_i \right) \left(\sum_{j=1}^n x_j \right) \\
    &= \frac{1}{n(n-1)} \left(\sum_{i=1}^n\sum_{j=1}^n \frac{x_i^2}{2} + \sum_{i=1}^n\sum_{j=1}^n \frac{x_j^2}{2} \right) + \frac{1}{n(n-1)}\left(\sum_{i=1}^n x_i \right) \left(\sum_{j=1}^n x_j \right) \\
    &= \frac{1}{n(n-1)} \sum_{i=1}^n \sum_{j=1}^n \frac{x_i^2 - 2x_i x_j + x_j^2}{2} \\
    &= \frac{1}{n(n-1)} \sum_{i=1}^n \sum_{j=1}^n \frac{(x_i - x_j)^2}{2}
\end{align*}
which is the quantity $V_n(X)$ appearing in \cite[Theorem 10]{maurer_empirical_2009}. Also since $n \geq 2$ we have $\frac{n}{n-1} \leq 2$ and $\frac{1}{n-1} \leq 2 \frac{1}{n}$, yielding~\eqref{eq:maurer_empirical_conclusion}. Using Bernstein's inequality (e.g. \cite[Theorem 3]{maurer_empirical_2009}), with an additional failure probability of at most $2\delta'$, we have 
\begin{align}
        \left|\left(\Phat_{sa} - P_{sa} \right) \left(\Vc_{s,u} \right)^{\circ 2^k}\right| & \leq \sqrt{\frac{2\log\left( \frac{1}{\delta'}\right) \Var_{P_{sa}} \left[ \left(\Vc_{s,u} \right)^{\circ 2^k} \right]    }{n }} + \frac{\log\left( \frac{1}{\delta'}\right) \infnorm{\left(\Vc_{s,u} \right)^{\circ 2^k}}}{3n}.
\end{align}
Combining this with~\eqref{eq:maurer_empirical_conclusion}, we have
\begin{align}
        \left|\left(\Phat_{sa} - P_{sa} \right) \left(\Vc_{s,u} \right)^{\circ 2^k}\right| & \leq \sqrt{\frac{4\log\left( \frac{1}{\delta'}\right) \Var_{\Phat_{sa}} \left[ \left(\Vc_{s,u} \right)^{\circ 2^k} \right]    }{n }} + \frac{4\log\left( \frac{1}{\delta'}\right) \infnorm{\left(\Vc_{s,u} \right)^{\circ 2^k}}}{n}\label{eq:maurer_empirical_conclusion_2}
\end{align}
(using that $\frac{1}{3} + 2 \sqrt{2} \leq 4$ for the second term.)

Also $\infnorm{\left(\Vc_{s,u} \right)^{\circ 2^k}} = \infnorm{\Vc_{s,u}} ^{ 2^k}$.
Combining this with~\eqref{eq:maurer_empirical_conclusion_2} yields
\begin{align}
        \left|\left(\Phat_{sa} - P_{sa} \right) \left(\Vc_{s,u} \right)^{\circ 2^k}\right| & \leq \sqrt{\frac{4\log\left( \frac{1}{\delta'}\right) \Var_{\Phat_{sa}} \left[ \left(\Vc_{s,u} \right)^{\circ 2^k} \right]    }{n }} + \frac{4 \log\left( \frac{1}{\delta'}\right) }{n}\infnorm{\Vc_{s,u}}^{2^k}\label{eq:maurer_empirical_conclusion_3}.
\end{align}
From this point, we will take $\delta' = \frac{\delta}{3SA |U|}$ and operate under the event that the above inequality holds for all $s \in S, a \in A, u \in U$, which by the union bound and our choice of $\delta'$ has probability at least $1 - \delta$.
$U$ is chosen so that there exists some $u \in U$ such that $\infnorm{\Vhatstar_{s,u} - \Vhatstar} \leq \frac{1}{n}$. Additionally, by assumption $\infnorm{\Vhat^{\pihat} - \Vhatstar} \leq \frac{1}{n}$, so by triangle inequality $\infnorm{\Vhat^{\pihat} - \Vhatstar_{s,u}} \leq \frac{2}{n}$. 
Therefore
\begin{align*}
    \infnorm{\Vc - \Vc_{s,u}} = \infnorm{\Vhat^{\pihat} - \Vhatstar_{s,u} + \min_{s'}\Vhatstar_{s,u}(s') - \min_{s'}\Vhat^{\pihat}(s')} \leq 2 \infnorm{\Vhat^{\pihat} - \Vhatstar_{s,u}} \leq \frac{4}{n}.
\end{align*}
We now use this to obtain a version of~\eqref{eq:maurer_empirical_conclusion_3} but with $\Vc$ in place of $\Vc_{s,u}$.

First we note that by Lemma \ref{lem:difference_of_powers} (applied elementwise) we have
\begin{align}
     \infnorm{\left( \Vc \right)^{\circ 2^k} - \left( \Vc_{s,u} \right)^{\circ 2^k}} & 
     \leq 2^k \infnorm{\Vc - \Vc_{s,u}} \left( \infnorm{\Vc} + \frac{4}{n}\right)^{2^k - 1} \leq 2^{k} \frac{4}{n} \left( \infnorm{\Vc} + \frac{4}{n}\right)^{2^k - 1}\label{eq:difference_of_value_fn_powers}
\end{align}
since $\max \{\infnorm{\Vc}, \infnorm{\Vc_{s,u}}\} \leq \infnorm{\Vc} + \frac{4}{n}$ because $ \infnorm{\Vc - \Vc_{s,u}} \leq \frac{4}{n}$.
For any $s, a$, we thus have
\begin{align}
    &\left| \left(\Phat_{sa} - P_{sa}\right) \left( \Vc \right)^{\circ 2^k} \right| \nonumber \\
    & \leq \left| \left(\Phat_{sa} - P_{sa}\right) \left( \Vc_{s,u} \right)^{\circ 2^k} \right| + \left| \left(\Phat_{sa} - P_{sa}\right) \left(\left( \Vc \right)^{\circ 2^k} - \left( \Vc_{s,u} \right)^{\circ 2^k} \right) \right| \nonumber \\
    & \leq \left| \left(\Phat_{sa} - P_{sa}\right) \left( \Vc_{s,u} \right)^{\circ 2^k} \right| +  \onenorm{\Phat_{sa} - P_{sa}}\infnorm{\left( \Vc \right)^{\circ 2^k} - \left( \Vc_{s,u} \right)^{\circ 2^k}} \nonumber \\
    & \leq \left| \left(\Phat_{sa} - P_{sa}\right) \left( \Vc_{s,u} \right)^{\circ 2^k} \right| +  2^k\frac{8}{n} \left( \infnorm{\Vc} + \frac{4}{n}\right)^{2^k - 1} \nonumber \\
    & \leq \sqrt{\frac{4\log\left( \frac{1}{\delta'}\right) \Var_{\Phat_{sa}} \left[ \left(\Vc_{s,u} \right)^{\circ 2^k} \right]    }{n }} + \frac{4 \log \left( \frac{1}{\delta'} \right)}{n} \left( \infnorm{\Vc} + \frac{4}{n}\right)^{2^k } + 2^{k}\frac{8}{n} \left( \infnorm{\Vc} + \frac{4}{n}\right)^{2^k - 1} \label{eq:vihat_pihat_ineq_1}
\end{align}
using the triangle inequality, Holder's inequality, $ \onenorm{\Phat_{sa} - P_{sa}} \leq 2$ and~\eqref{eq:difference_of_value_fn_powers}, and~\eqref{eq:maurer_empirical_conclusion_3} combined with the fact that $\infnorm{\Vc_{s,u}} \leq \infnorm{\Vc} + \frac{4}{n}$.

We also have that
\begin{align*}
    \sqrt{\Var_{\Phat_{sa}} \left[ \left(\Vc_{s,u} \right)^{\circ 2^k} \right] }
    &= \sqrt{\Var_{\Phat_{sa}} \left[ \left(\Vc \right)^{\circ 2^k} + \left(\Vc_{s,u} \right)^{\circ 2^k} - \left(\Vc \right)^{\circ 2^k} \right] } \\
    & \leq \sqrt{\Var_{\Phat_{sa}} \left[ \left(\Vc \right)^{\circ 2^k}  \right] } + \sqrt{\Var_{\Phat_{sa}} \left[  \left(\Vc_{s,u} \right)^{\circ 2^k} - \left(\Vc \right)^{\circ 2^k} \right] } \\
    & \leq \sqrt{\Var_{\Phat_{sa}} \left[ \left(\Vc \right)^{\circ 2^k}  \right] } + \infnorm{\left( \Vc \right)^{\circ 2^k} - \left( \Vc_{s,u} \right)^{\circ 2^k}} \\
    & \leq \sqrt{\Var_{\Phat_{sa}} \left[ \left(\Vc \right)^{\circ 2^k}  \right] } + 2^{k} \frac{4}{n} \left( \infnorm{\Vc} + \frac{4}{n}\right)^{2^k - 1}
\end{align*}
using the triangle inequality for the norm $\sqrt{\E X^2}$, $\Var_{\Phat_{sa}} \left[ X\right] \leq \infnorm{X}^2$, and using~\eqref{eq:difference_of_value_fn_powers}. Plugging this into~\eqref{eq:vihat_pihat_ineq_1} and simplifying, we have
\begin{align*}
    \left| \left(\Phat_{sa} - P_{sa}\right) \left(\Vc \right)^{\circ 2^k} \right| 
    &\leq \sqrt{\frac{4\log\left( \frac{1}{\delta'}\right) \Var_{\Phat_{sa}} \left[ \left(\Vc \right)^{\circ 2^k}  \right]    }{n }} +   \frac{2 \sqrt{\log \left(\frac{1}{\delta'} \right)}}{\sqrt{n}} 2^{k} \frac{4}{n} \left( \infnorm{\Vc} + \frac{4}{n}\right)^{2^k - 1} \\
    & \quad \quad + \frac{4 \log \left( \frac{1}{\delta'} \right)}{n} \left( \infnorm{\Vc} + \frac{4}{n}\right)^{2^k } + 2^{k}\frac{8}{n} \left( \infnorm{\Vc} + \frac{4}{n}\right)^{2^k - 1} \\
    & \leq \sqrt{\frac{4\log\left( \frac{1}{\delta'}\right) \Var_{\Phat_{sa}} \left[ \left(\Vc \right)^{\circ 2^k}  \right]    }{n }} +   \frac{16 \cdot 2^k \log \left(\frac{1}{\delta'} \right)}{n}  \left( \infnorm{\Vc} + 1\right)^{2^k} \\
    & \leq \sqrt{\frac{\alpha \Var_{\Phat_{sa}} \left[ \left(\Vc \right)^{\circ 2^k}  \right]    }{n }} +   \frac{ \alpha \cdot 2^k }{n}  \left( \infnorm{\Vc} + 1\right)^{2^k}
\end{align*}
for $\alpha = 16 \log\left( \frac{1}{\delta'}\right)  $ in the final inequality. (For the simplification steps, we have $\log \left( \frac{1}{\delta'}\right) \geq 1$ so $\sqrt{\log \left( \frac{1}{\delta'}\right)} \leq \log \left( \frac{1}{\delta'}\right)$, and, since $n \geq 4$ we have $ \infnorm{\Vc} + \frac{4}{n} \leq \infnorm{\Vc} + 1$, and also $\infnorm{\Vc} + 1 \geq 1$. Also since $n \geq 4$, $\frac{2}{\sqrt{n}} \leq 1$.) Since this holds for all $s,a$, by definition of the vector $ \Var_{\Phat} \left[ \left(\Vc \right)^{\circ 2^k}  \right]$ we thus have (elementwise)
\begin{align}
    \left| \left(\Phat - P\right) \left( \Vc \right)^{\circ 2^k} \right| & \leq \sqrt{\frac{\alpha  \Var_{\Phat} \left[ \left(\Vc \right)^{\circ 2^k}  \right]    }{n }} + \frac{ \alpha \cdot 2^k }{n}  \left( \infnorm{\Vc} + 1\right)^{2^k} \one \label{eq:DMDP_LOO_bernstein_cond_wo_pihat}
\end{align}
and so
\begin{align}
    \left| \left(\Phat_{\pihat} - P_{\pihat}\right) \left( \Vc \right)^{\circ 2^k} \right| &= \left|M^{\pihat} \left(\Phat - P\right) \left( \Vc \right)^{\circ 2^k} \right|  \nonumber \\
    & \leq M^{\pihat} \left| \left(\Phat - P\right) \left( \Vc \right)^{\circ 2^k} \right| \nonumber \\
    & \leq M^{\pihat} \sqrt{\frac{\alpha  \Var_{\Phat} \left[ \left(\Vc \right)^{\circ 2^k}  \right]    }{n }} + \frac{ \alpha \cdot 2^k }{n}  \left( \infnorm{\Vc} + 1\right)^{2^k}  M^{\pihat} \one \nonumber \\
    & \leq  \sqrt{\frac{\alpha  M^{\pihat} \Var_{\Phat} \left[ \left(\Vc \right)^{\circ 2^k}  \right]    }{n }} + \frac{ \alpha \cdot 2^k }{n}  \left( \infnorm{\Vc} + 1\right)^{2^k}   \one \nonumber \\
    &= \sqrt{\frac{\alpha   \Var_{\Phat_{\pihat}} \left[ \left(\Vc \right)^{\circ 2^k}  \right]    }{n }} + \frac{ \alpha \cdot 2^k }{n}  \left( \infnorm{\Vc} + 1\right)^{2^k}   \one \label{eq:DMDP_LOO_bernstein_cond_add_pihat}
\end{align}
where the first inequality is because all entries of $M^{\pihat}$ are $\geq 0$, and the third inequality is using Jensen's inequality (since the rows of $M^{\pihat}$ are probability distributions) and using the fact that $M^{\pihat} \one = \one$ (note $\one$ has different dimensions on each side of this equation). (We note that this step, which replaces all appearances of $P, \Phat$ in~\eqref{eq:DMDP_LOO_bernstein_cond_wo_pihat} with $P_{\pihat}, \Phat_{\pihat}$, could be done for any arbitrary policy $\pi$.)

Finally, taking a union bound over all $\leq 4 \frac{1}{1-\gamma}$ values of $k$ and adjusting the failure probabilities so that the overall failure probability is $\leq \delta$, we can upper bound the resulting value of $\alpha$ by
\begin{align}
    16 \log \left(\frac{3 SA |U|}{\delta} 4 \frac{1}{1-\gamma} \right) & \leq 16 \log \left(\frac{12 SA n}{(1-\gamma)^2\delta} \right). \label{eq:DMDP_LOO_bernstein_alpha_bound}
\end{align}
\end{proof}

The following result summarizes the leave-one-out construction from \cite{li_breaking_2020} (specifically it is a direct combination of \cite[Lemma 6]{li_breaking_2020} and \cite[Lemma 4]{li_breaking_2020}).
\begin{thm}[\cite{li_breaking_2020}]
\label{thm:li_LOO_construction}
    Let $\widetilde{r} = r + \Delta$ where $\Delta(s,a) \sim \textnormal{Uniform}[0, \xi]$ independently for each $s \in \S, a \in \A$. Let $\Vhatpert^\pi$ denote the value of policy $\pi$ the discounted MDP with transition matrix $\Phat$ and reward $\widetilde{r}$, and likewise let $\Vhatpert^\star$ denote the optimal value function in this MDP. We also let $\Qhatpert^\pi$ and $\Qhatpert^\star$ be the Q-functions of policy $\pi$ and of an optimal policy in this MDP. Then
    \begin{enumerate}
        \item There exists a family of MDPs $\Phat^{(s,a,u)}$ for $s \in \S, a \in \A, u \in U$ (for some finite set $U$) such that
        \begin{enumerate}
            \item $|U| \leq \frac{24 SA^2}{(1-\gamma)^3 \xi \delta}$.
            \item For each $s \in \S, a \in \A, u \in U$, the MDP $\Phat^{(s,a,u)}$ is independent of all of the random variables $(S^i_{s,a})_{i=1}^n$ (the $n$ observed transitions from state-action pair $s,a$).
        \end{enumerate}
        \item With probability at least $1 - \delta$,
        \begin{enumerate}
            \item The optimal policy $\pihatpert$ in the DMDP $(\Phat, \widetilde{r}, \gamma)$ is unique and is a deterministic policy.
            \item For all $s \in \S, a \in \A$ such that $a \neq \pihatpert(s)$,
            \[
                \Qhatpert^\star(s, \pihatpert(s)) - \Qhatpert^\star(s,a) \geq \frac{\xi \delta (1-\gamma)}{3 S A^2}.
            \]
            \item For each $s \in \S, a \in \A$, there exists $u^\star \in U$ such that the unique optimal policy in the DMDP $(\Phat^{(s,a,u^\star)}, \widetilde{r}, \gamma)$, which we label $\pihstar_{s,a,u}$, is equal to $\pihatpert$.
        \end{enumerate}
    \end{enumerate}
\end{thm}
This construction allows us to check the Bernstein-like inequality~\eqref{eq:bernstein_like_inequality} but with $V^{\pihatpert}$, the true value function (in the true DMDP $(P, r, \gamma)$) of a policy $\pihatpert$ which is optimal in the perturbed empirical DMDP $(\Phat, \widetilde{r}, \gamma)$.
\begin{lem}
\label{lem:LOO_DMDP_bernstein_bound_pert}
    Let $\widetilde{r} = r + \Delta$ where $\Delta(s,a) \sim \textnormal{Uniform}[0, \xi]$ independently for each $s \in \S, a \in \A$. Let $\pihatpert$ be a policy which is optimal for the DMDP $(\Phat, \widetilde{r}, \gamma)$. With probability at least $1 - \delta$, letting $\Vc = V^{\pihatpert} - \left(\min_s V^{\pihatpert}(s)\right)\one$, for all $k = 0, \dots, \left\lceil\log_2 \log_2 \left( \spannorm{V^{\pihatpert}} + 4\right) \right\rceil$, we have
    \begin{align*}
        \left| \left(\Phat_{\pihatpert} - P_{\pihatpert}\right) \left( \Vc \right)^{\circ 2^k} \right|
    &\leq \sqrt{\frac{\alpha   \Var_{P_{\pihatpert}} \left[ \left(\Vc \right)^{\circ 2^k}  \right]    }{n }} + \frac{ \alpha \cdot 2^k }{n}  \left( \infnorm{\Vc} + 1\right)^{2^k}   \one
    \end{align*}
where $\alpha = 2 \log \left(\frac{768 S^2A^3}{(1-\gamma)^4 \xi \delta^2} \right)$.
\end{lem}

\begin{proof} 
    From Theorem \ref{thm:li_LOO_construction}, the policy $\pihstar_{s,a,u}$ is independent of the observed transitions $(S^i_{s,a})_{i=1}^n$ from state-action pair $(s,a)$, so in particular the random variable $V^{\pihstar_{s,a,u}}$ is independent from $(S^i_{s,a})_{i=1}^n$ (and so is $\left(V^{\pihstar_{s,a,u}} - \left( \min_{s'} V^{\pihstar_{s,a,u}}(s')\right) \one \right)^{\circ 2^k}$ for any $k$).
    Furthermore, Theorem \ref{thm:li_LOO_construction} guarantees that with probability at least $1 - \delta/2$, for all $s \in \S$ and $a \in \A$ there exists some $u^\star \in U$ such that $\pihstar_{s,a,u^\star} = \pihatpert$, which implies that $V^{\pihstar_{s,a,u^\star}}  = V^{\pihatpert}$.
    Also we have $|U| \leq \frac{48 SA^2}{(1-\gamma)^3 \xi \delta}$. Therefore, letting $\Vc_{s,a,u} = V^{\pihstar_{s,a,u}} - \left( \min_{s'} V^{\pihstar_{s,a,u}}(s')\right) \one$, if we check that
    \begin{align}
        \left| \left(\Phat_{sa} - P_{sa}\right) \left( \Vc_{s,a,u} \right)^{\circ 2^k} \right|
    &\leq \sqrt{\frac{\alpha   \Var_{P_{sa}} \left[ \left(\Vc_{s,a,u} \right)^{\circ 2^k}  \right]    }{n }} + \frac{ \alpha \cdot 2^k }{n}  \left( \infnorm{\Vc_{s,a,u}} + 1\right)^{2^k}   \label{eq:LOO_DMDP_bernstein_ineq_pert}
    \end{align}
    for all combinations of $s \in \S, a \in \A$, $u \in U$, and $k = 0, \dots, \left\lceil\log_2 \log_2 \left( \spannorm{V^{\pihatpert}} + 4\right) \right\rceil$ with probability at least $1 - \delta/2$, then we can combine with Theorem \ref{thm:li_LOO_construction} to obtain that
    \begin{align}
        \left| \left(\Phat_{sa} - P_{sa}\right) \left( \Vc \right)^{\circ 2^k} \right|
        &\leq \sqrt{\frac{\alpha   \Var_{P_{sa}} \left[ \left(\Vc \right)^{\circ 2^k}  \right]    }{n }} + \frac{ \alpha \cdot 2^k }{n}  \left( \infnorm{\Vc} + 1\right)^{2^k}   \label{eq:LOO_DMDP_bernstein_ineq_pert_2}
    \end{align}
    (for all $a,s,k$), and consequently that
    \begin{align*}
        \left| \left(\Phat_{\pihatpert} - P_{\pihatpert}\right) \left( \Vc \right)^{\circ 2^k} \right|
    &\leq \sqrt{\frac{\alpha   \Var_{P_{\pihatpert}} \left[ \left(\Vc \right)^{\circ 2^k}  \right]    }{n }} + \frac{ \alpha \cdot 2^k }{n}  \left( \infnorm{\Vc} + 1\right)^{2^k}   \one
    \end{align*}
    (for all $k$) as desired (since the scalar inequality~\eqref{eq:LOO_DMDP_bernstein_ineq_pert_2} applies in particular to all $(s,a)$ of the form $(s, \pihatpert(a))$).
    
    Identically to the proof of Lemma \ref{lem:LOO_DMDP_bernstein_bound}, we can bound the number of values of $k$ to be checked as $\leq 4 \frac{1}{1-\gamma}$. Fixing $s \in \S, a \in \A, u \in U$, and a value of $k$, by Bernstein's inequality (e.g. \cite[Theorem 3]{maurer_empirical_2009}), since $(S^i_{sa})_{i=1}^n$ (which determine $\Phat_{sa}$) are independent of $\left(\Vc_{s,a,u}\right)^{\circ 2^k}$, we have that with probability at least $1 - 2 \delta'$,
    \begin{align*}
        \left| \left(\Phat_{sa} - P_{sa}\right) \left( \Vc_{s,a,u} \right)^{\circ 2^k} \right|
    &\leq \sqrt{\frac{ 2\log \left(\frac{1}{\delta'} \right)   \Var_{P_{sa}} \left[ \left(\Vc_{s,a,u} \right)^{\circ 2^k}  \right]    }{n }} + \frac{  \log \left(\frac{1}{\delta'} \right)  }{3n}  \infnorm{\left(\Vc_{s,a,u} \right)^{\circ 2^k}}  \\
    &= \sqrt{\frac{ 2\log \left(\frac{1}{\delta'} \right)   \Var_{P_{sa}} \left[ \left(\Vc_{s,a,u} \right)^{\circ 2^k}  \right]    }{n }} + \frac{  \log \left(\frac{1}{\delta'} \right)  }{3n}  \infnorm{\Vc_{s,a,u}}^{2^k} \\
    & \leq \sqrt{\frac{ 2\log \left(\frac{1}{\delta'} \right)   \Var_{P_{sa}} \left[ \left(\Vc_{s,a,u} \right)^{\circ 2^k}  \right]    }{n }} + \frac{ 2 \log \left(\frac{1}{\delta'} \right)  }{n}  \infnorm{\Vc_{s,a,u}}^{2^k} \\
    &\leq \sqrt{\frac{2 \log \left(\frac{1}{\delta'} \right)  \Var_{P_{sa}} \left[ \left(\Vc_{s,a,u} \right)^{\circ 2^k}  \right]    }{n }} + \frac{ 2 \log \left(\frac{1}{\delta'} \right) \cdot 2^k }{n}  \left( \infnorm{\Vc_{s,a,u}} + 1\right)^{2^k}.
    \end{align*} 
    Taking a union bound over all $\leq S \cdot A \cdot \frac{48 SA^2}{(1-\gamma)^3 \xi \delta} \cdot \frac{4}{1-\gamma} = \frac{192 S^2A^3}{(1-\gamma)^4 \xi \delta}$ combinations of $s, a, u, k$ and choosing $\delta' = \frac{\delta}{4} \frac{(1-\gamma)^4 \xi \delta}{192 S^2A^3}$, we thus obtain~\eqref{eq:LOO_DMDP_bernstein_ineq_pert} with $\alpha = 2 \log \left(\frac{768 S^2A^3}{(1-\gamma)^4 \xi \delta^2} \right)$ as desired. 
\end{proof}

\subsection{Completing the proof of Theorem \ref{thm:DMDP_main_thm}}
\label{sec:DMDP_main_pf}

\begin{proof}[Proof of Theorem \ref{thm:DMDP_main_thm}]
    Let $\alpha = 16 \log \left(\frac{12 SA n}{(1-\gamma)^2\delta}  \right)$, which is larger than $2 \log \left( \frac{6 S \log_2 \log_2 \left(\spannorm{V^\star}+4 \right)}{\delta} \right)$ since $\log_2 \log_2 (\spannorm{V^\star}+4) \leq 2\frac{1}{1-\gamma}$ by the same arguments as those within the proof of Lemma \ref{lem:LOO_DMDP_bernstein_bound}. Combining Lemma \ref{lem:simple_DMDP_bernstein_bound} with Lemma \ref{lem:DMDP_full_var_param_bound}, we obtain that with probability at least $1 - \delta$,
    \begin{align}
        \infnorm{ \Vhat^{\pistar_\gamma} - V^{\pistar_\gamma}} &\leq \frac{4 \cdot 3 \log_2 \log_2 \left( \frac{1}{1-\gamma} + 4\right) \alpha}{(1-\gamma)n} \left( \spannorm{V^{\pistar_\gamma}} + 1\right) \nonumber\\
        & \quad \quad  + \frac{2 \cdot 3 \log_2 \log_2 \left( \frac{1}{1-\gamma} + 4\right)}{1-\gamma} \sqrt{\frac{2 \alpha \left( \spannorm{V^{\pistar_\gamma}} + 1\right)}{n}} \nonumber\\
        &\leq \frac{12 \log_2 \log_2 \left( \frac{1}{1-\gamma} + 4\right) \alpha}{(1-\gamma)n} \left( \spannorm{V^{\pistar_\gamma}} + 1\right) \nonumber\\
        & \quad \quad  + \frac{12 \log_2 \log_2 \left( \frac{1}{1-\gamma} + 4\right) }{1-\gamma} \sqrt{\frac{ \alpha \left( \spannorm{V^{\pistar_\gamma}} + 1\right)}{n}} \label{eq:DMDP_main_thm_eval_bound_1}\\
        & \leq \frac{24 \log_2 \log_2 \left( \frac{1}{1-\gamma} + 4\right)  }{1-\gamma}\sqrt{\frac{ \alpha \left( \spannorm{V^{\pistar_\gamma}} + 1\right)}{n}}\label{eq:DMDP_main_thm_eval_bound}
    \end{align}
    where we used that $1+\left\lceil\log_2 \log_2 \left( \spannorm{V^\star} + 4\right) \right\rceil \leq 2 + \log_2 \log_2 \left( \frac{1}{1-\gamma} + 4\right) \leq 3 \log_2 \log_2 \left( \frac{1}{1-\gamma} + 4\right)$. We also use that the second term on the RHS of~\eqref{eq:DMDP_main_thm_eval_bound_1} is always the larger of the two when $\alpha\frac{\spannorm{V^{\pistar_\gamma}} + 1 }{n} \leq 1$, and if this quantity is not $\leq 1$ then by the trivial bound that $\infnorm{ \Vhat^{\pistar_\gamma} - V^{\pistar_\gamma}} \leq \frac{1}{1-\gamma}$, inequality~\eqref{eq:DMDP_main_thm_eval_bound} still holds.
    Similarly, combining Lemma \ref{lem:LOO_DMDP_bernstein_bound} with Lemma \ref{lem:DMDP_full_var_param_bound} and assuming $n \geq 4$ (so that the bound in Lemma \ref{lem:LOO_DMDP_bernstein_bound} satisfies the assumptions of Lemma \ref{lem:DMDP_full_var_param_bound}, noting also that the assumptions are satisfied since we assume $\pihat$ satisfies $\Vhat^{\pihat} \geq \Vhat^\star - \frac{1}{n}\one$, which implies $\infnorm{\Vhat^{\pihat} - \Vhatstar} \leq \frac{1}{n}$ since $\Vhatstar \geq \Vhat^{\pihat}$), we can perform analogous calculations to obtain that with probability at least $1 - \delta$,
    \begin{align}
        \infnorm{ \Vhat^{\pihat} - V^{\pihat}} &\leq  \frac{24 \log_2 \log_2 \left( \frac{1}{1-\gamma} + 4\right)  }{1-\gamma}\sqrt{\frac{ \alpha \left( \spannorm{\Vhat^{\pihat}} + 1\right)}{n}}.\label{eq:DMDP_main_thm_opt_bound}
    \end{align}
    Since $\Vhat^{\pihat} \geq \Vhat^\star - \frac{1}{n}\one$, we have that elementwise
    \begin{align}
        V^\star \geq V^{\pihat} &\geq \Vhat^{\pihat} - \infnorm{\Vhat^{\pihat} - V^{\pihat}} \one \nonumber\\
        & \geq \Vhat^\star - \frac{1}{n}\one - \infnorm{\Vhat^{\pihat} - V^{\pihat}} \one \nonumber\\
        & \geq \Vhat^{\pistar_\gamma} - \frac{1}{n}\one - \infnorm{\Vhat^{\pihat} - V^{\pihat}} \one \nonumber\\
        & \geq V^\star - \infnorm{\Vhat^{\pistar_\gamma} - V^{\pistar_\gamma}} \one - \frac{1}{n}\one - \infnorm{\Vhat^{\pihat} - V^{\pihat}} \one \label{eq:DMDP_chain_of_val_fn_bounds}
    \end{align}
    (since $\Vhat^\star \geq \Vhat^{\pistar_\gamma}$ and $V^\star = V^{\pistar_\gamma}$). Combining this inequality with the bounds~\eqref{eq:DMDP_main_thm_eval_bound} and~\eqref{eq:DMDP_main_thm_opt_bound} (and taking the union bound, giving an overall failure probability bounded by $2\delta$) yields
    \begin{align*}
        \infnorm{V^{\pihat} - V^\star} & \leq \frac{1}{n}+ \frac{24 \log_2 \log_2 \left( \frac{1}{1-\gamma} + 4\right)  }{1-\gamma}\left(\sqrt{\frac{ \alpha \left( \spannorm{\Vhat^{\pihat}} + 1\right)}{n}} + \sqrt{\frac{ \alpha \left( \spannorm{V^{\pistar_\gamma}} + 1\right)}{n}} \right) \\
        & \leq \frac{1}{n}+ \frac{24 \log_2 \log_2 \left( \frac{1}{1-\gamma} + 4\right)  }{1-\gamma}\sqrt{\frac{2 \alpha \left( \spannorm{\Vhat^{\pihat}} +\spannorm{V^{\pistar_\gamma}}+ 2\right)}{n}} \\
        & \leq  \frac{25 \log_2 \log_2 \left( \frac{1}{1-\gamma} + 4\right)  }{1-\gamma}\sqrt{\frac{2 \alpha \left( \spannorm{\Vhat^{\pihat}} +\spannorm{V^{\pistar_\gamma}}+ 2\right)}{n}} \\
        & \leq  \frac{25 \log_2 \log_2 \left( \frac{1}{1-\gamma} + 4\right)  }{1-\gamma}\sqrt{\frac{4 \alpha \left( \spannorm{\Vhat^{\pihat}} +\spannorm{V^{\pistar_\gamma}}+ 1\right)}{n}} \\
        &= \frac{1  }{1-\gamma}\sqrt{4\left(25 \log_2 \log_2 \left( \frac{1}{1-\gamma} + 4\right)\right)^2 16 \log \left(\frac{12 SA n}{(1-\gamma)^2\delta}  \right)\frac{ \left( \spannorm{\Vhat^{\pihat}} +\spannorm{V^{\pistar_\gamma}}+ 1\right)}{n}} \\
        &\leq \frac{1  }{1-\gamma}\sqrt{4\left(25 \log \left( \frac{4}{1-\gamma}\right)\right)^2 16 \log \left(\frac{12 SA n}{(1-\gamma)^2\delta}  \right)\frac{ \left( \spannorm{\Vhat^{\pihat}} +\spannorm{V^{\pistar_\gamma}}+ 1\right)}{n}} \\
        & \leq \frac{1  }{1-\gamma}\sqrt{ C_1\log^3 \left(\frac{ SA n}{(1-\gamma)\delta}  \right)\frac{ \left( \spannorm{\Vhat^{\pihat}} +\spannorm{V^{\pistar_\gamma}}+ 1\right)}{n}} 
    \end{align*}
    where we used the fact that $\sqrt{a} + \sqrt{b} \leq \sqrt{2(a+b)}$, the definition of $\alpha$, that $\log_2 \log_2 (x+4) \leq \log 4x$ for $x \geq 1$, and finally chose $C_1$ sufficiently large (in particular large enough to ensure the above bound is vacuous when the $n \geq 4$ assumption is not satisfied; $C_1 \geq 4$ suffices).
\end{proof}

\subsection{Completing the proof of Theorem \ref{thm:DMDP_pert_thm}}
\label{sec:DMDP_pert_pf}

First we briefly outline the proof. The proof of Theorem \ref{thm:DMDP_main_thm} already bounds the quantity $\infnorm{ \Vhat^{\pistar_\gamma} - V^{\pistar_\gamma}}$, which can be reused here. The more difficult step is bounding the ``evaluation error'' of the empirical optimal policy of the perturbed MDP, $\pihatpert$. (While in the statement of Theorem \ref{thm:DMDP_pert_thm} we referred to this policy as $\pihat$, here we rename it to $\pihatpert$ to emphasize that it is optimal for the perturbed MDP.) There are several possible choices for which pair of value functions to bound (for example, $\infnorm{\Vhat^{\pihatpert}  - V^{\pihatpert}}$ and $\infnorm{\Vhatpert^{\pihatpert} - V^{\pihatpert}} $ are two reasonable choices). Since Lemma \ref{lem:DMDP_recursive_var_param_bound} assumes that the reward is bounded by $1$ while we only have $\widetilde{r} \leq 1 + \xi$, it is most convenient to bound the term $\infnorm{\Vhat^{\pihatpert}  - V^{\pihatpert}}$ since it does not involve the perturbed reward function $\widetilde{r}$. (Lemma \ref{lem:DMDP_recursive_var_param_bound} could be trivially modified to handle differently scaled rewards, but by bounding $\infnorm{\Vhat^{\pihatpert}  - V^{\pihatpert}}$ we avoid this extra bookkeeping.) This consideration motivated our choice of the particular Bernstein-like-condition to check within Lemma \ref{lem:LOO_DMDP_bernstein_bound_pert}. Thus, combining Lemma \ref{lem:LOO_DMDP_bernstein_bound_pert} with Lemma \ref{lem:DMDP_recursive_var_param_bound}, we obtain a bound on $\infnorm{\Vhat^{\pihatpert}  - V^{\pihatpert}}$. We can use the fact that the perturbation is small to bound $\infnorm{\Vhatpert^{\pihatpert} - \Vhat^{\pihatpert}}$ and $\infnorm{\Vhatpert^{\pistar_\gamma} - \Vhat^{\pistar_\gamma}}$, which by triangle inequality gives us bounds on the quantities $\infnorm{\Vhatpert^{\pihatpert} - V^{\pihatpert}}$ and $\infnorm{ \Vhatpert^{\pistar_\gamma} - V^{\pistar_\gamma}}$. We conclude by using the fact that $\Vhatpert^{\pihatpert} \geq \Vhatpert^{\pistar_\gamma}$, since $\pihatpert$ is optimal for the perturbed empirical MDP.

\begin{proof}[Proof of Theorem \ref{thm:DMDP_pert_thm}]
    We follow the above sketch. First we bound $ \infnorm{ \Vhat^{\pistar_\gamma} - V^{\pistar_\gamma}} $. Combining Lemma \ref{lem:simple_DMDP_bernstein_bound} with Lemma \ref{lem:DMDP_full_var_param_bound}, similarly to the proof of Theorem \ref{thm:DMDP_main_thm} we obtain that with probability at least $1-\delta$, with $\alpha_1 = 2 \log \left( \frac{6 S \log_2 \log_2 \left(\spannorm{V^{\pistar_\gamma}}+4 \right)}{\delta} \right)$,
    \begin{align}
        \infnorm{ \Vhat^{\pistar_\gamma} - V^{\pistar_\gamma}} &\leq \frac{4 \cdot 3 \log_2 \log_2 \left( \frac{1}{1-\gamma} + 4\right) \alpha_1}{(1-\gamma)n} \left( \spannorm{V^{\pistar_\gamma}} + 1\right) \nonumber\\
        & \quad \quad  + \frac{2 \cdot 3 \log_2 \log_2 \left( \frac{1}{1-\gamma} + 4\right)}{1-\gamma} \sqrt{\frac{2 \alpha_1 \left( \spannorm{V^{\pistar_\gamma}} + 1\right)}{n}} \nonumber\\
        &\leq \frac{12 \log_2 \log_2 \left( \frac{1}{1-\gamma} + 4\right) \alpha_1}{(1-\gamma)n} \left( \spannorm{V^{\pistar_\gamma}} + 1\right) \nonumber\\
        & \quad \quad  + \frac{12 \log_2 \log_2 \left( \frac{1}{1-\gamma} + 4\right) }{1-\gamma} \sqrt{\frac{ \alpha_1 \left( \spannorm{V^{\pistar_\gamma}} + 1\right)}{n}} \nonumber \\
        & \leq \frac{24 \log_2 \log_2 \left( \frac{1}{1-\gamma} + 4\right)  }{1-\gamma}\sqrt{\frac{ \alpha_1 \left( \spannorm{V^{\pistar_\gamma}} + 1\right)}{n}}\label{eq:DMDP_pert_thm_eval_bound}
    \end{align}
    where again we used that $1+\left\lceil\log_2 \log_2 \left( \spannorm{V^\star} + 4\right) \right\rceil \leq 2 + \log_2 \log_2 \left( \frac{1}{1-\gamma} + 4\right) \leq 3 \log_2 \log_2 \left( \frac{1}{1-\gamma} + 4\right)$ and the fact that the second term is always larger whenever the bound on $\infnorm{ \Vhat^{\pistar_\gamma} - V^{\pistar_\gamma}}$ is non-vacuous.
    To bound $\infnorm{\Vhat^{\pihatpert}  - V^{\pihatpert}}$, we can combine Lemma \ref{lem:LOO_DMDP_bernstein_bound_pert} with Lemma \ref{lem:DMDP_full_var_param_bound} and perform analogous calculations to obtain that with additional failure probability at most $\delta$,
    \begin{align}
        \infnorm{\Vhat^{\pihatpert}  - V^{\pihatpert}} & \leq \frac{24 \log_2 \log_2 \left( \frac{1}{1-\gamma} + 4\right)  }{1-\gamma}\sqrt{\frac{ \alpha_2 \left( \spannorm{V^{\pihatpert}} + 1\right)}{n}}\label{eq:DMDP_pert_thm_opt_bound}
    \end{align}
    where $\alpha_2 = 2 \log \left(\frac{768 S^2A^3}{(1-\gamma)^4 \xi \delta^2} \right)$.

    Next we bound the terms $\infnorm{\Vhatpert^{\pihatpert} - \Vhat^{\pihatpert}}$ and $\infnorm{\Vhatpert^{\pistar_\gamma} - \Vhat^{\pistar_\gamma}}$. We have
    \begin{align}
        \infnorm{\Vhatpert^{\pihatpert} - \Vhat^{\pihatpert}} &= \infnorm{(I - \gamma \Phat_{\pihatpert})^{-1} \widetilde{r}_{\pihatpert} - (I - \gamma \Phat_{\pihatpert})^{-1} r_{\pihatpert}} \nonumber\\
        & \leq \infinfnorm{(I - \gamma \Phat_{\pihatpert})^{-1}} \infnorm{\widetilde{r}_{\pihatpert} - r_{\pihatpert}} \nonumber\\
        & \leq \frac{1}{1-\gamma} \infnorm{\widetilde{r} - r} \nonumber \\
        & \leq \frac{\xi}{1-\gamma} \label{eq:pert_error_bound_1}
    \end{align}
    and likewise
    \begin{align}
        \infnorm{\Vhatpert^{\pistar_\gamma} - \Vhat^{\pistar_\gamma}} \leq \frac{\xi}{1-\gamma}. \label{eq:pert_error_bound_2}
    \end{align}

    Finally, we can combine all of these bounds with the fact that $\Vhatpert^{\pihatpert} \geq \Vhatpert^{\pistar_\gamma}$ (since $\pihatpert$ is optimal for the perturbed empirical MDP) to obtain that with probability at least $1 - 2 \delta$,
    \begin{align*}
        V^{\pihatpert} & \geq \Vhat^{\pihatpert} - \infnorm{\Vhat^{\pihatpert}  - V^{\pihatpert}} \one \\
        & \geq \Vhat^{\pihatpert} - \frac{24 \log_2 \log_2 \left( \frac{1}{1-\gamma} + 4\right)  }{1-\gamma}\sqrt{\frac{ \alpha_2 \left( \spannorm{V^{\pihatpert}} + 1\right)}{n}} \one && \text{\eqref{eq:DMDP_pert_thm_opt_bound}} \\
        & \geq \Vhatpert^{\pihatpert} - \infnorm{\Vhatpert^{\pihatpert} - \Vhat^{\pihatpert}} \one - \frac{24 \log_2 \log_2 \left( \frac{1}{1-\gamma} + 4\right)  }{1-\gamma}\sqrt{\frac{ \alpha_2 \left( \spannorm{V^{\pihatpert}} + 1\right)}{n}} \one \\
        & \geq \Vhatpert^{\pihatpert} - \frac{\xi}{1-\gamma}\one - \frac{24 \log_2 \log_2 \left( \frac{1}{1-\gamma} + 4\right)  }{1-\gamma}\sqrt{\frac{ \alpha_2 \left( \spannorm{V^{\pihatpert}} + 1\right)}{n}} \one && \text{\eqref{eq:pert_error_bound_1}} \\
        & \geq \Vhatpert^{\pistar_\gamma} - \frac{\xi}{1-\gamma}\one - \frac{24 \log_2 \log_2 \left( \frac{1}{1-\gamma} + 4\right)  }{1-\gamma}\sqrt{\frac{ \alpha_2 \left( \spannorm{V^{\pihatpert}} + 1\right)}{n}} \one \\
        & \geq \Vhatpert^{\pistar_\gamma}- \infnorm{\Vhatpert^{\pistar_\gamma} - \Vhat^{\pistar_\gamma}} \one - \frac{\xi}{1-\gamma}\one - \frac{24 \log_2 \log_2 \left( \frac{1}{1-\gamma} + 4\right)  }{1-\gamma}\sqrt{\frac{ \alpha_2 \left( \spannorm{V^{\pihatpert}} + 1\right)}{n}} \one \\
        & \geq \Vhatpert^{\pistar_\gamma}-  \frac{2\xi}{1-\gamma}\one - \frac{24 \log_2 \log_2 \left( \frac{1}{1-\gamma} + 4\right)  }{1-\gamma}\sqrt{\frac{ \alpha_2 \left( \spannorm{V^{\pihatpert}} + 1\right)}{n}} \one && \text{\eqref{eq:pert_error_bound_2}} \\
        & \geq V^{\pistar_\gamma}-  \infnorm{ \Vhat^{\pistar_\gamma} - V^{\pistar_\gamma}}\one - \frac{2\xi}{1-\gamma}\one - \frac{24 \log_2 \log_2 \left( \frac{1}{1-\gamma} + 4\right)  }{1-\gamma}\sqrt{\frac{ \alpha_2 \left( \spannorm{V^{\pihatpert}} + 1\right)}{n}} \one  \\
        & \geq V^{\pistar_\gamma} - \frac{2\xi}{1-\gamma}\one - \frac{24 \log_2 \log_2 \left( \frac{1}{1-\gamma} + 4\right)  }{1-\gamma}\Biggg(\sqrt{\frac{ \alpha_2 \left( \spannorm{V^{\pihatpert}} + 1\right)}{n}} \\
        & \quad \quad \quad \quad\quad \quad \quad \quad\quad \quad \quad \quad\quad \quad \quad \quad\quad \quad \quad \quad + \sqrt{\frac{ \alpha_1 \left( \spannorm{V^{\pistar_\gamma}} + 1\right)}{n}} \Biggg)\one && \text{\eqref{eq:DMDP_pert_thm_eval_bound}}.
    \end{align*}
    Now using that $\xi \leq \frac{1}{n}$, the fact that $\sqrt{a} + \sqrt{b} \leq 2\sqrt{a + b}$, and that
    \begin{align*}
        \alpha_1 = 2 \log \left( \frac{6 S \log_2 \log_2 \left(\spannorm{V^{\pistar_\gamma}}+4 \right)}{\delta} \right) \leq 2 \log \left( \frac{12 S }{(1-\gamma)\delta} \right) \leq \alpha_2
    \end{align*}
    (since $\log_2 \log_2 (\spannorm{V^{\pistar_\gamma}}+4) \leq 2\frac{1}{1-\gamma}$ as shown within the proof of Lemma \ref{lem:LOO_DMDP_bernstein_bound}), we can simplify
    \begin{align*}
        &\frac{2\xi}{1-\gamma} + \frac{24 \log_2 \log_2 \left( \frac{1}{1-\gamma} + 4\right)  }{1-\gamma}\left(\sqrt{\frac{ \alpha_2 \left( \spannorm{V^{\pihatpert}} + 1\right)}{n}} + \sqrt{\frac{ \alpha_1 \left( \spannorm{V^{\pistar_\gamma}} + 1\right)}{n}} \right) \\
        & \leq \frac{2}{n(1-\gamma)} + \frac{48 \log_2 \log_2 \left( \frac{1}{1-\gamma} + 4\right)  }{1-\gamma}\sqrt{\frac{ \alpha_2 \left( \spannorm{V^{\pistar_\gamma}} + \spannorm{V^{\pihatpert}} + 2 \right)}{n}} \\
        & \leq  \frac{96 \log_2 \log_2 \left( \frac{1}{1-\gamma} + 4\right)  }{1-\gamma}\sqrt{\frac{ \alpha_2 \left( \spannorm{V^{\pistar_\gamma}} + \spannorm{V^{\pihatpert}} + 2 \right)}{n}} \\
        & \leq   \frac{96 \sqrt{2} \log_2 \log_2 \left( \frac{1}{1-\gamma} + 4\right)  }{1-\gamma}\sqrt{\frac{ \alpha_2 \left( \spannorm{V^{\pistar_\gamma}} + \spannorm{V^{\pihatpert}} + 1 \right)}{n}} \\
        & \leq   \frac{ C_2  }{1-\gamma}\sqrt{\frac{ \log^3 \left( \frac{S  A n}{(1-\gamma)\delta \xi}\right) \left( \spannorm{V^{\pistar_\gamma}} + \spannorm{V^{\pihatpert}} + 1 \right)}{n}} 
    \end{align*}
    where in the final inequality we use that $\log_2 \log_2 \left( \frac{1}{1-\gamma} + 4\right) \leq \alpha_2$ and chosen sufficiently large constant $C_1$ (including absorbing an additional constant due to adjusting the failure probability to be at most $\delta$ rather than $2 \delta$).
\end{proof}

\section{Proofs of AMDP Theorems}
\label{sec:AMDP_appendix_main}

\subsection{Useful Lemmas}

\label{sec:key_AMDP_lemmas}

The following is an average-reward version of the simulation lemma. Such techniques are well-known \citep{cao_single_1999, meyer_jr_condition_1980}.
\begin{lem}
\label{lem:average_reward_simulation_lemma}
    Fix a policy $\pi$, and let $P, \Phat$ be any two MDP transition matrices. Let $\Delta = \Phat_\pi - P_\pi$. Then
    \[
    \Phat_\pi^\infty - P_\pi^\infty = \Phat_\pi^\infty \Delta H_{P_\pi} - (\Phat_\pi^\infty - P_\pi^\infty)P^\infty_\pi.
    \]
    Consequently,
    \begin{enumerate}
        \item If the Markov chain $P_\pi$ satisfies $P^\infty_\pi = \one \mu^\top$ for some probability distribution $\mu^\top$, then
        \[
        \Phat_\pi^\infty - P_\pi^\infty = \Phat_\pi^\infty \Delta H_{P_\pi}.
        \]

        \item If the quantity $\rho^\pi = P^\infty_\pi r_\pi$ is constant (has the form $\alpha \one$ for some $\alpha \in \R$), then
        \[
        \rhohat^\pi - \rho^\pi = \Phat_\pi^\infty r_\pi - P_\pi^\infty r_\pi = \Phat_\pi^\infty \Delta H_{P_\pi} r_\pi = \Phat_\pi^\infty \Delta h^\pi.
        \]
    \end{enumerate}
\end{lem}
\begin{proof}
By the properties of limiting matrices,
    \begin{align*}
        \Phat_\pi^\infty - P_\pi^\infty &= \Phat_\pi^\infty \Phat_\pi - P_\pi^\infty P_\pi \\
        &= \Phat_\pi^\infty (P_\pi + \Delta) - P_\pi^\infty P_\pi \\
        &= (\Phat_\pi^\infty - P_\pi^\infty ) P_\pi + \Phat_\pi^\infty \Delta .
    \end{align*}
    Therefore $(\Phat_\pi^\infty - P_\pi^\infty)(I - P_\pi) = \Phat_\pi^\infty \Delta$. Now post-multiplying both sides by the deviation matrix $H_{P_\pi}$, which satisfies $(I - P_\pi) H_{P_\pi} = I - P_\pi^\infty$, we obtain
    \begin{align*}
        \Phat_\pi^\infty \Delta H_{P_\pi} &= (\Phat_\pi^\infty - P_\pi^\infty)(I - P_\pi) H_{P_\pi} = (\Phat_\pi^\infty - P_\pi^\infty)(I - P^\infty_\pi) = \Phat_\pi^\infty - P_\pi^\infty + (\Phat_\pi^\infty - P_\pi^\infty)P^\infty_\pi.
    \end{align*}

    For the first consequence, note that if $P^\infty_\pi = \one \mu^\top$, then since $\Phat_\pi^\infty$ and $P_\pi^\infty$ are stochastic matrices and have $\one$ as a right eigenvector with eigenvalue $1$, we have
    \[(\Phat_\pi^\infty - P_\pi^\infty)P^\infty_\pi = (\Phat_\pi^\infty - P_\pi^\infty) \one \mu^\top = \one \mu^\top - \one \mu^\top = 0.\]

    Similarly, for the second consequence, we have
    \[(\Phat_\pi^\infty - P_\pi^\infty)P^\infty_\pi r_\pi = (\Phat_\pi^\infty - P_\pi^\infty) \alpha \one  = \alpha \one  - \alpha \one = 0.\]
\end{proof}

\begin{lem}
\label{lem:anchoring_optimality_properties}
    Let $P$ be any transition matrix, let $\eta \in (0,1)$, and let $s_0 \in \S$ be an arbitrary state. Form the anchored transition matrix $\widetilde{P} = (1-\eta) P + \eta \one e_{s_0}^\top \in \R^{SA \times S}$. Also we use $\widetilde{h}^\pi$, $\widetilde{P}_\pi^\infty$ to denote the bias function of policy $\pi$ in $\widetilde{P}$ and the limiting matrix of $\widetilde{P}_\pi$.
    Then
    \begin{enumerate}
        \item For all policies $\pi$, the state $s_0$ is recurrent in the Markov chain $\widetilde{P}_\pi$. Consequently $\widetilde{P}$ is unichain.

        \item Fix a policy $\pi$. Then
        \begin{enumerate}
            \item $\widetilde{P}_\pi^\infty  = \eta \one e_{s_0}^\top (I - (1-\eta)P_\pi)^{-1}$.
            \item $\widetilde{\rho}^\pi = \eta \one V_{1-\eta}^\pi(s_0)$, where $V_{1-\eta}^\pi$ is the discounted value function for policy $\pi$ with discount factor $1-\eta$ (or equivalently effective horizon $\frac{1}{\eta}$).
            \item $\widetilde{h}^\pi = V_{1-\eta}^\pi + c \one$ for some scalar $c$.
            \item $\infnorm{\rhot^\pi - \rho^\pi} \leq 2\eta \infnorm{\htilde^{\pi}}$. %$\infnorm{\rhot^\pi - \rho^\pi} \leq \eta \spannorm{\htilde^{\pi}}$.
            \item If $\rho^\pi$ is a state-independent constant, then $\spannorm{\htilde^\pi} \leq 2\spannorm{h^\pi}$.
        \end{enumerate}

        \item Letting $\rhot^\star$ and $\htilde^\star$ be the optimal gain and optimal bias of $\widetilde{P}$, respectively, we have
        \begin{enumerate}
            \item $\widetilde{\rho}^\star = \eta \one V_{1-\eta}^\star(s_0)$, where $V_{1-\eta}^\star$ is the optimal discounted value function with discount factor $1-\eta$.
            \item $\widetilde{h}^\star = V_{1-\eta}^\star + c \one$ for some scalar $c$.
            \item If $\rho^\star$ is a state-independent constant, then $\spannorm{\htilde^\star} \leq 2 \spannorm{h^\star}$.
            \item The average-reward Bellman optimality operator for $\widetilde{P}$, $\Topt(h) := M(r + \widetilde{P} h)$, is a $(1-\eta)$-span contraction:
            $\spannorm{\Topt(h) - \Topt(h')} \leq (1-\eta) \spannorm{h - h'}$.
        \end{enumerate}
    \end{enumerate}
\end{lem}

\begin{proof}[Proof of Lemma \ref{lem:anchoring_optimality_properties}]
    We start with 1. Fix a policy $\pi$ and consider the Markov chain $\widetilde{P}_\pi$. Since this is a finite Markov chain, there must exist some recurrent state $s$, and since in $\widetilde{P}_\pi$ there is probability at least $\eta> 0$ of transitioning to $s_0$ from $s$, the state $s_0$ is also recurrent \cite[Chapter 5.3]{durrett_probability_2019}. Furthermore this shows that for any state $s$ which is recurrent, it is in the same recurrent class as $s_0$, and therefore there is only one recurrent class in $\widetilde{P}_\pi$. Since this holds for arbitrary $\pi$ (and in particular for all determininstic $\pi$), the MDP $\widetilde{P}$ is unichain.

    Now we show all the properties in statement 2. Again fix a policy $\pi$. Since $\widetilde{P}_\pi^\infty$ must satisfy $\widetilde{P}_\pi^\infty \widetilde{P}_\pi = \widetilde{P}_\pi^\infty$, expanding the definition of $\widetilde{P}_\pi^\infty$ we have
    \begin{align*}
        \widetilde{P}_\pi^\infty = \widetilde{P}_\pi \widetilde{P}_\pi^\infty = (1-\eta) \widetilde{P}_\pi^\infty P_\pi + \eta \widetilde{P}_\pi^\infty \one e_{s_0}^\top = (1-\eta) \widetilde{P}_\pi^\infty P_\pi + \eta  \one e_{s_0}^\top
    \end{align*}
    using that $\widetilde{P}_\pi^\infty\one = \one$ in the last equality. By rearranging we have
    \[
        \widetilde{P}_\pi^\infty \left(I - (1-\eta) P_\pi\right) = \eta \one e_{s_0}^\top 
    \]
    and since $\infinfnorm{(1-\eta) P_\pi} = (1-\eta) < 1$, it is a standard fact that the matrix $\left(I - (1-\eta) P_\pi\right)$ is invertible, and so
    \[
    \widetilde{P}_\pi^\infty = \eta \one e_{s_0}^\top \left(I - (1-\eta) P_\pi\right)^{-1}
    \]
    as desired. We can then calculate that
    \[
        \rhot^\pi = \widetilde{P}_\pi^\infty r_\pi = \eta \one e_{s_0}^\top \left(I - (1-\eta) P_\pi\right)^{-1} r_\pi =  \eta \one e_{s_0}^\top V_{1-\eta}^\pi = \eta \one V_{1-\eta}^\pi(s_0).
    \]
    Next, to compute $\htilde^\pi$, we check that $\eta \one V_{1-\eta}^\pi(s_0)$ and $ V_{1-\eta}^\pi$ satisfy the evaluation equations \cite[Section 8.2.3]{puterman_markov_1994}. We have that
    \begin{align*}
        \eta \one V_{1-\eta}^\pi(s_0) + (I - \widetilde{P}_\pi) V_{1-\eta}^\pi &= \eta \one V_{1-\eta}^\pi(s_0) +  (I - (1-\eta)P_\pi - \eta \one e_{s_0}^\top) V_{1-\eta}^\pi \\
        &= (I - (1-\eta)P_\pi) V_{1-\eta}^\pi \\
        &= (I - (1-\eta)P_\pi) (I - (1-\eta)P_\pi)^{-1} r_\pi \\
        &= r_\pi
    \end{align*}
    so by \cite[Corollary 8.2.7]{puterman_markov_1994}, since $\widetilde{P}_\pi$ is unichain, we have that $\widetilde{h}^\pi = V_{1-\eta}^\pi + c \one$ for some scalar $c$. Next, since we have already checked that $\rhot^\pi$ is constant, we can apply Lemma \ref{lem:average_reward_simulation_lemma} to obtain
    \[
        \rhot^\pi - \rho^\pi = P_\pi^\infty (\widetilde{P}_\pi - P_\pi) \htilde^\pi
    \]
    and thus
    \[
    \infnorm{\rhot^\pi - \rho^\pi} \leq \infinfnorm{P_\pi^\infty} \infinfnorm{\widetilde{P}_\pi - P_\pi} \infnorm{\htilde^\pi} \leq 1 \cdot \infinfnorm{\eta \one e_{s_0}^\top -\eta P_\pi } \infnorm{\htilde^\pi} \leq 2\eta \infnorm{\htilde^\pi}.
    \]
    Finally, assuming that $\rho^\pi$ is a constant vector, we want to show that $\spannorm{\htilde^\pi} \leq 2\spannorm{h^\pi}$. Since we have shown
    $\spannorm{\htilde^\pi}  = \spannorm{V^\pi_{1-\eta}}$, it suffices to bound $\spannorm{V^\pi_{1-\eta}}$. We calculate
    \begin{align*}
         \spannorm{\htilde^\pi} &=  \spannorm{V^\pi_{1-\eta}} \\
         &= \spannorm{V^\pi_{1-\eta} - \frac{1}{\eta} \rho^\pi} && \text{because $\rho^\pi$ is constant} \\
         &= \spannorm{(I - (1-\eta) P_{\pi})^{-1} r_\pi - \frac{1}{\eta} \rho^\pi} \\
         &= \spannorm{(I - (1-\eta) P_{\pi})^{-1} \left(\rho^\pi + (I - P_\pi)h^\pi \right) - \frac{1}{\eta} \rho^\pi} && \text{$\rho^\pi + h^\pi = r_\pi +  P_\pi h^\pi$}\\
         &= \spannorm{(I - (1-\eta) P_{\pi})^{-1}(I - P_\pi)h^\pi} && \text{$(I - (1-\eta) P_{\pi})^{-1} \rho^\pi = \frac{1}{\eta}\rho^\pi$}.
    \end{align*}
    The fact that $(I - (1-\eta) P_{\pi})^{-1} \rho^\pi = \frac{1}{\eta}\rho^\pi$ for general policies $\pi$ follows from the fact that $P_{\pi} \rho^\pi = P_{\pi} P_\pi^{\infty } r_\pi = P_\pi^{\infty } r_\pi = \rho^\pi$, which implies that $P_\pi^t \rho^\pi = \rho^\pi$ which we can then combine with the Neumann series to obtain that
    $(I - (1-\eta) P_{\pi})^{-1} \rho^\pi = \sum_{t=0}^\infty (1-\eta)^t P_\pi^t \rho^\pi = \sum_{t=0}^\infty (1-\eta)^t \rho^\pi = \frac{1}{\eta}\rho^\pi$.
    
    For convenience writing $\gamma = 1-\eta$, using the Neumann series formula we have
    \begin{align*}
        (I - \gamma P_{\pi})^{-1}(I - P_\pi) &= \sum_{t=0}^\infty \gamma^t P_\pi^t - \sum_{t=0}^\infty \gamma^t P_\pi^{t+1} \\
        &= I + \sum_{t=1}^\infty \gamma^t P_\pi^t - \sum_{t=0}^\infty \gamma^t P_\pi^{t+1} \\
        &= I + \sum_{t=0}^\infty \gamma^{t+1} P_\pi^{t+1} - \sum_{t=0}^\infty \gamma^t P_\pi^{t+1} \\
        &= I - (1-\gamma)\sum_{t=0}^\infty \gamma^{t} P_\pi^{t+1}
    \end{align*}
    and we note that $(1-\gamma)\sum_{t=0}^\infty \gamma^{t} P_\pi^{t+1}$ is a stochastic matrix (since all terms are nonnegative, and, since each row of $P_\pi^t$ sums to $1$ for any $t$, the rows all sum to $(1-\gamma) \sum_{t=0}^\infty \gamma^t 1= 1$). Therefore continuing the previous calculation,
    \begin{align*}
        \spannorm{\htilde^\pi} &= \spannorm{(I - (1-\eta) P_{\pi})^{-1}(I - P_\pi)h^\pi} \\
        &= \spannorm{h^\pi - (1-\gamma)\sum_{t=0}^\infty \gamma^t P_\pi^{t+1} h^\pi} \\
        &\leq \spannorm{h^\pi} + \spannorm{ \left( (1-\gamma)\sum_{t=0}^\infty \gamma^t P_\pi^{t+1}\right) h^\pi} \\
        & \leq 2\spannorm{h^\pi}
    \end{align*}
    where the last inequality is because for any stochastic matrix $P'$, $\spannorm{P' h^\pi} \leq \spannorm{h^\pi}$.

    Now we verify statement 3. First we show that $\Topt$ is a $(1-\eta)$-span contraction. This follows from existing results, since the fact that all states have probability $\geq \eta$ of transitioning to $s_0$ means we could apply \cite[Theorem 8.5.2]{puterman_markov_1994}. However, we will provide a direct proof due to its simplicity. Letting $h, h' \in \R^S$ be arbitrary, we can calculate
    \begin{align*}
        \spannorm{\Topt(h) - \Topt(h')} &= \spannorm{M(r + \widetilde{P}h) - M(r+\widetilde{P}h')} \\
        &= \spannorm{M\Big(r + (1-\eta) Ph + \eta \one h(s_0)\Big) - M\Big(r+(1-\eta) Ph' + \eta \one h'(s_0)\Big)} \\
        &= \spannorm{M\Big(r + (1-\eta) Ph  \Big) +\eta \one h(s_0) - M\Big(r+(1-\eta) Ph' \Big) - \eta \one h'(s_0)} \\
        &= \spannorm{M\Big(r + (1-\eta) Ph  \Big)  - M\Big(r+(1-\eta) Ph' \Big)} \\
        &\leq \spannorm{r + (1-\eta) Ph    - r-(1-\eta) Ph' } \\
        & = (1-\eta) \spannorm{Ph - Ph'} \\
        & \leq (1-\eta) \spannorm{h - h'}
    \end{align*}
    where we used the fact that $M$ is $\spannorm{\cdot}$-nonexpansive, which we verify now. Letting $x, x' \in \R^{SA}$ be arbitrary and letting $\pi$ and $\pi'$ satisfy $M(x) = M^\pi x$ and $M(x') = M^{\pi'} x'$, we have
    \begin{align*}
        M(x) - M(x') &= M^\pi x - M^{\pi'} x' \leq M^\pi x - M^\pi x' =M^\pi(x - x') \leq \left(\max_{s \in \S} x(s) - x'(s) \right) \one
    \end{align*}
    and analogously
    \begin{align*}
        M(x) - M(x') &= M^\pi x - M^{\pi'} x' \geq M^{\pi'} x - M^{\pi'} x' =M^{\pi'}(x - x') \geq \left(\min_{s \in \S} x(s) - x'(s) \right) \one
    \end{align*}
    so $\spannorm{M(x)-M(x')} \leq \spannorm{x-x'}$ as desired. Now we check that the claimed forms of $\rhot^\star$ and $\htilde^\star$ satisfy the (unichain) average optimality equation. We have
    \begin{align*}
        M(r + \widetilde{P} V^\star_{1-\eta}) &= M\left(r + (1-\eta) P V^\star_{1-\eta} + \eta \one V^\star_{1-\eta}(s_0)\right) \\
        &= M\left(r + (1-\eta) P V^\star_{1-\eta} \right) + \eta \one V^\star_{1-\eta}(s_0) \\
        &= V^\star_{1-\eta} + \eta \one V^\star_{1-\eta}(s_0)
    \end{align*}
    (using the discounted Bellman equation in the final equality), so indeed $\rhot^\star = \eta \one V^\star_{1-\eta}(s_0)$ and $\htilde^\star = V^\star_{1-\eta} + c \one$ for some scalar $c$ \cite[Theorem 8.4.3]{puterman_markov_1994}. (In general satisfying the average optimality equation only determines the optimal gain, but in unichain models the optimal bias is also determined up to a constant by the optimality equation \cite[Section 8.4.2]{puterman_markov_1994} \cite{schweitzer_functional_1978}. In our setting it is also possible to show this directly as a consequence of the span-non-expansiveness of $\Topt$.)
    %(using the discounted Bellman equation in the final equality), so indeed $\rhot^\star = \eta \one V^\star_{1-\eta}(s_0)$ since the average optimality equation characterizes the optimal gain \cite[Section 8.4.2]{puterman_markov_1994}. In general satisfying the average optimality equation does not determine the optimal bias, even for unichain models \citep{schweitzer_functional_1978}, but in our setting we can obtain this using the span-contractivity of $\Topt$: letting $\pi$ be the Blackwell-optimal policy of $(\widetilde{P},r)$, we have $\rhot^\star + \htilde^\pi = \Topt (\htilde^\pi)$, and therefore by span-contractivity of $\Topt$,
    %\[
    %\spannorm{V^\star_{1-\eta} - \htilde^\pi} = \spannorm{\Topt(V^\star_{1-\eta}) - \Topt(\htilde^\pi)} \leq (1-\eta)\spannorm{V^\star_{1-\eta} - \htilde^\pi}
    %\]
    %and since $(1-\eta) < 1$ this is only possible if $\spannorm{V^\star_{1-\eta} - \htilde^\pi}= 0$, which implies $\htilde^\star = \htilde^\pi = V^\star_{1-\eta} + c \one$ for some scalar $c$.

    Finally we check that $\spannorm{\htilde^\star} \leq 2 \spannorm{h^\star}$ in the case that $\rho^\star$ is constant. Note that $\spannorm{\htilde^\star} = \spannorm{V^\star_{1-\eta}} = \spannorm{V_{1-\eta}^{\pistar_{1-\eta}}}$, whereas $\spannorm{h^\star} = \spannorm{h^{\pistar}}$, so we are comparing two different policies. Essentially the same bound has appeared in prior work, for instance \cite[Lemma 2]{wei_model-free_2020}, but for completeness we reprove it with a manner of calculation very similar to the previous case concerning the a fixed policy. First, notice that (letting $\gamma = 1-\eta$ for notational convenience)
    \begin{align*}
        V^\star_{1-\eta} &\geq V^{\pistar}_{1-\eta} \\
        &= (I - \gamma P_{\pistar})^{-1} r_{\pistar} \\
        &= (I - \gamma P_{\pistar})^{-1} \left(\rho^\star + (I - P_{\pistar})h^\star \right) && \text{$\rho^\star + h^{\star} = r_{\pistar} + P_{\pistar} h^\star$} \\
        &= (I - \gamma P_{\pistar})^{-1} \rho^\star + (I - \gamma P_{\pistar})^{-1} (I - P_{\pistar})h^\star \\
        &= \frac{1}{1-\gamma} \rho^\star + (I - \gamma P_{\pistar})^{-1} (I - P_{\pistar})h^\star
    \end{align*}
    where the last equality holds because $\rho^\star$ is a state-independent constant.
    Next, using the fact that
    \[
        \rho^\star + h^\star = M(r + P h^\star) \geq M^{\pistar_\gamma} (r + P h^\star) = r_{\pistar_\gamma} + P_{\pistar_\gamma} h^\star,
    \]
    we have
    \begin{align*}
        V^\star_{1-\eta} & = (I - \gamma P_{\pistar_\gamma})^{-1} r_{\pistar_\gamma} \\
        & \leq (I - \gamma P_{\pistar_\gamma})^{-1} \left( \rho^\star + (I - P_{\pistar_\gamma}) h^\star \right) && \text{above inequality, monotonicity of $(I - \gamma P_{\pistar_\gamma})^{-1}$} \\
        &= (I - \gamma P_{\pistar_\gamma})^{-1} \rho^\star + (I - \gamma P_{\pistar_\gamma})^{-1} (I - P_{\pistar_\gamma}) h^\star \\
        &= \frac{1}{1-\gamma}\rho^\star + (I - \gamma P_{\pistar_\gamma})^{-1} (I - P_{\pistar_\gamma}) h^\star && \text{$\rho^\star$ is a constant vector}.
    \end{align*}
    Combining these two calculations we have
    \[
       \frac{1}{1-\gamma} \rho^\star + (I - \gamma P_{\pistar})^{-1} (I - P_{\pistar})h^\star \leq  V^\star_{1-\eta} \leq \frac{1}{1-\gamma}\rho^\star + (I - \gamma P_{\pistar_\gamma})^{-1} (I - P_{\pistar_\gamma}) h^\star.
    \]
    We can also reuse our previous calculation that for any $\pi$, $(I - \gamma P_{\pi})^{-1} (I - P_{\pi}) = I - Q$ for some stochastic matrix $Q$, to obtain
    \[
       \frac{1}{1-\gamma} \rho^\star + h^\star - Q_1 h^\star \leq  V^\star_{1-\eta} \leq \frac{1}{1-\gamma}\rho^\star + h^\star - Q_2 h^\star
    \]
    (for stochastic matrices $Q_1, Q_2$). We have the elementwise bounds $h^\star - Q_2 h^\star \leq h^\star - \left(\min_s h^\star(s)\right) \one \leq \spannorm{h^\star} \one$ and likewise $h^\star - Q_1 h^\star \geq - \spannorm{h^\star} \one$, which combined with the above display inequalities imply that $\spannorm{\htilde^\star} = \spannorm{V^\star_{1-\eta}} \leq 2 \spannorm{h^\star}$.    
\end{proof}

\begin{lem}
    \label{lem:prop_of_empbod}
    Let $\empbod = \inf \left \{\gamma \in [0, 1) : \exists c \in \R \text{ such that }\infnorm{\Vhat^\star_{\gamma} - \hhat^\star - c\one} \leq \frac{1}{n} \right \}$. Then if $\Phat$ is weakly communicating, then the above set is nonempty and the above infimum is attained, that is there exists $c \in \R$ such that $\infnorm{\Vhat^\star_{\empbod} - \hhat^\star - c\one} \leq \frac{1}{n}$ (and so $\empbod$ may be defined as the \textit{smallest} discount factor satisfying this property).
\end{lem}
\begin{proof}
    If $\Phat$ is weakly communicating, then we have that $\rhohat^{\star}$ is a constant vector. Letting $\pihstar$ be a Blackwell-optimal policy for $(\Phat, r)$ and letting $\widehat{\gamma}_{\textrm{BW}} < 1$ be the Blackwell discount factor, we have that $\rhohat^\star = \rhohat^{\pihstar}$, that $\hhat^{\star} = \hhat^{\pihstar}$, and that $\Vhat^\star_{\gamma} = \Vhat^{\pihstar}_{\gamma}$ for all $\gamma \in [\widehat{\gamma}_{\textrm{BW}}, 1)$ \cite{puterman_markov_1994}. By the well-known Laurent series expansion (e.g. \cite[Corollary 8.2.4]{puterman_markov_1994}), we also have $\Vhat^{\pihstar}_\gamma = \frac{1}{1-\gamma}\rhohat^{\pihstar} + \hhat^{\pihstar} + g(\gamma)$ where $g(\gamma) \to 0$ as $\gamma \uparrow 1$. Combining these facts, we have that for all $\gamma \geq \widehat{\gamma}_{\textrm{BW}}$,
        \[
            \Vhat^\star_{\gamma} = \Vhat^{\pihstar}_{\gamma} = \frac{1}{1-\gamma}\rhohat^{\pihstar} + \hhat^{\pihstar} + g(\gamma) = \frac{1}{1-\gamma}\rhohat^{\star} + \hhat^{\star} + g(\gamma)
        \]
        and also that $\rhohat^{\star}$ is a constant vector. Therefore, there exists sufficiently large $\gamma$ (such that the $g(\gamma)$ term is bounded by $\frac{1}{n}$ in $\infnorm{\cdot}$ norm) such that $\infnorm{\Vhat^\star_{\gamma} - \hhat^\star - \frac{1}{1-\gamma}\rhohat^\star\one} \leq \frac{1}{n}$, and thus the set in the definition of $\empbod$ is nonempty.

        Now we argue that the infimum is attained. We have already argued that the set
        \[
            \left \{\gamma \in [0, 1) : \exists c \in \R \text{ such that }\infnorm{\Vhat^\star_{\gamma} - \hhat^\star - c\one} \leq \frac{1}{n} \right \}
        \]
        is nonempty and thus contains some $\overline{\gamma} \in [0, 1)$. Therefore we can write
        \begin{align}
            \empbod = \inf \left \{\gamma \in [0, \overline{\gamma}] : \exists c \in \R \text{ such that }\infnorm{\Vhat^\star_{\gamma} - \hhat^\star - c\one} \leq \frac{1}{n} \right \}. \label{eq:empbod_def_2}
        \end{align}
        Additionally, $\hhat^\star$ must have some entry which is $\leq 0$ and some entry which is $\geq 0$ (since $\Phat^\infty_{\pihstar} \hhat^\star = 0$ and each row of $\Phat^\infty_{\pihstar}$ is a probability distribution; note these may be the same entry). Furthermore, for all $\gamma \in [0, 1)$, $\Vhat_\gamma^\star \in [0, \frac{1}{1-\gamma}]$. Thus if for some $\gamma$ there exists $c \in \R$ such that$\infnorm{\Vhat^\star_{\gamma} - \hhat^\star - c\one} \leq \frac{1}{n}$, then letting $\hhat^\star(s)\leq 0$, we must have
        \[
            \Vhat_\gamma^\star(s) - \hhat^\star(s) - c \leq \infnorm{\Vhat^\star_{\gamma} - \hhat^\star - c\one} \leq \frac{1}{n}
        \]
        which implies that
        \[
            c \geq \Vhat^\star_\gamma(s) - \hhat^\star(s) - \frac{1}{n} \geq 0 - 0 - \frac{1}{n} = -\frac{1}{n}.
        \]
        Likewise looking at $s$ such that $\hhat^\star(s) \geq 0$, we must have
        \[
            \Vhat_\gamma^\star(s) - \hhat^\star(s) - c \geq -\infnorm{\Vhat^\star_{\gamma} - \hhat^\star - c\one} \geq -\frac{1}{n}
        \]
        which implies that
        \[
            c \leq \Vhat^\star_\gamma(s) - \hhat^\star(s) + \frac{1}{n} \leq \frac{1}{1-\gamma} - 0 + \frac{1}{n} = \frac{1}{1-\gamma} + \frac{1}{n}.
        \]
        Therefore, the set
        \begin{align}
            \left\{ \gamma \in [0, \overline{\gamma}], c \in \R : \infnorm{\Vhat^\star_{\gamma} - \hhat^\star - c\one} \leq \frac{1}{n} \right\} \label{eq:empbod_set}
        \end{align}
        is bounded since it is contained within $[0, \overline{\gamma}] \times \left[-\frac{1}{n}, \frac{1}{1-\overline{\gamma}} + \frac{1}{n}\right]$. Therefore it remains to show that the set~\eqref{eq:empbod_set} is closed, since this would imply that the set~\eqref{eq:empbod_set} is compact, and then since the continuous image of a compact set is compact and the projection of the set~\eqref{eq:empbod_set} onto its first coordinate is exactly the set in the expression~\eqref{eq:empbod_def_2} for $\empbod$, meaning that the infimum of this set is contained within the set, which is what we are trying to prove.

        To show that the set~\eqref{eq:empbod_set} is closed, we first show that $\gamma \mapsto \Vhat^\star_\gamma$ (with domain restricted to $[0, \overline{\gamma}]$) is a continuous function. This is a known result but we prove it for completeness. First, if $\pi$ is fixed, then letting $\gamma, \gamma' \in [0, \overline{\gamma}]$, we have
        \begin{align*}
            \Vhat_\gamma^\pi -  \Vhat_{\gamma'}^\pi &= (I - \gamma \Phat_\pi)^{-1} r_\pi - (I - \gamma' \Phat_\pi)^{-1} r_\pi \\
            &= (I - \gamma \Phat_\pi)^{-1} (I - \gamma' \Phat_\pi) (I - \gamma' \Phat_\pi)^{-1} r_\pi - (I - \gamma \Phat_\pi)^{-1} (I - \gamma \Phat_\pi) (I - \gamma' \Phat_\pi)^{-1} r_\pi \\
            &= (I - \gamma \Phat_\pi)^{-1} \left[(I - \gamma' \Phat_\pi) - (I - \gamma \Phat_\pi) \right] (I - \gamma' \Phat_\pi)^{-1} r_\pi \\
            &= (I - \gamma \Phat_\pi)^{-1} (\gamma - \gamma') \Phat_\pi \Vhat_{\gamma'}^\pi
        \end{align*}
        and thus
        \begin{align*}
            \infnorm{\Vhat_\gamma^\pi -  \Vhat_{\gamma'}^\pi} &\leq |\gamma - \gamma'| \infinfnorm{(I - \gamma \Phat_\pi)^{-1}} \infinfnorm{\Phat_\pi} \infnorm{\Vhat_{\gamma'}^\pi} \\
            &\leq |\gamma - \gamma'| \frac{1}{1-\gamma} 1 \frac{1}{1-\gamma'} \\
            & \leq \frac{|\gamma - \gamma'|}{(1-\overline{\gamma})^2}
        \end{align*}
        so the function $\gamma \mapsto \Vhat_\gamma^\pi$ is Lipschitz and thus continuous. Now $\Vhat^\star_\gamma$ is equal to the maximum over (the finite number of) all Markovian deterministic policies $\pi$, and thus $\gamma \mapsto \Vhat^\star_\gamma$ is also continuous. This means that the function $f : [0, \overline{\gamma}] \times \R \to \R$ defined by $f(\gamma, c) = \infnorm{\Vhat^\star_\gamma - \hhat^\star - c \one}$ is a continuous function, since we have shown that $\gamma \mapsto \Vhat^\star_\gamma$ is continuous, and $\infnorm{\cdot}$ and addition are continuous functions. 
        Therefore the preimage of $[1, \frac{1}{n}]$ under $f$ is a closed set since $[0, \frac{1}{n}]$ is closed. (Technically it we only immediately know that it is closed in the topology of the domain of $f $, $[0, \overline{\gamma}] \times \R$, but its closed sets are exactly the closed sets of $\R^2$ intersected with $[0, \overline{\gamma}] \times \R$ \cite{pugh_real_2015}.) Thus we have shown that the set~\eqref{eq:empbod_set} is closed, and thus as argued we can conclude that the infimum in the definition of $\empbod$ is attained.
\end{proof}
We also remark that the asymptotic (partial) Laurent series expansion, used to bound $\empbod$ within this proof, could be replaced with a non-asymptotic version to give a more explicit bound.

\begin{lem}
\label{lem:greedy_error_amplification_amdp}
    If there exists $h\in \R^S $ such that $\pi$ is greedy with respect to $r + \Phatanc h$, then
    \[
        \Vhat_{1-\frac{1}{n}}^\star - \Vhat_{1-\frac{1}{n}}^\pi \leq (n-1)\spannorm{\hhatanc^\star - h}.
    \]
\end{lem}
\begin{proof}
    First we note a classic result for $\gamma$-discounted MDPs, that if there exists $V$ such that $\pi$ is greedy with respect to $r + \gamma \Phatanc V$, then
    \begin{align}
        \Vhat_\gamma^\star - \Vhat^{\pi} \leq \frac{2 \gamma \infnorm{V - \Vhat_\gamma^\star}}{1-\gamma} \label{eq:value_error_amplification_bound}
    \end{align}
    \citep{singh_upper_1994}. The desired result will follow from~\eqref{eq:value_error_amplification_bound} and the connection between AMDP and DMDP provided by Lemma \ref{lem:anchoring_optimality_properties}. Specifically, we will try to find a vector $V$ such that $\pi$ is greedy with respect to $r + \gamma \Phat V$ and such that $\infnorm{V - \Vhat^\star_\gamma}$ is small and bounded in terms of $\spannorm{\hhatanc^\star - h}$.

    Let $\gamma = 1 - \frac{1}{n}$. First note that if $\pi$ is greedy with respect to $r + \Phatanc h$, then it is also greedy with respect to $r + \Phatanc h + \alpha \one$ for any $\alpha$ (since this shifts all entries by the same amount $\alpha$). Now we try to choose $\alpha$ to meet the aforementioned conditions. First, note that
    \[
    r + \Phatanc h + \alpha \one = r + \gamma \Phat h + (1-\gamma) h(s_0) \one + \alpha \one = r + \gamma \Phat \left(h + (1-\gamma ) h(s_0) \one + \alpha \one \right)
    \]
    (since $\Phat \one = \one$) so we can define $V = h + (1-\gamma ) h(s_0) \one + \alpha \one$ and then try to minimize the quantity $\infnorm{V - \Vhat_\gamma^\star}$ by appropriately choosing $\alpha$ (as $V$ is a function of $\alpha$). We also know that $\hhatanc^\star = \Vhat^\star_\gamma + \alpha^\star \one$ for some $\alpha^\star$ from Lemma \ref{lem:anchoring_optimality_properties}. 
    Therefore if we choose 
    \[
        \alpha = \frac{\max_s (\hhatanc^\star(s) - h(s)) + \min_s (\hhatanc^\star(s) - h(s))}{2}\one - \alpha^\star \one - (1-\gamma) h(s_0) \one, 
    \] 
    then
    \begin{align*}
        \infnorm{\Vhat^\star_\gamma - V} &= \infnorm{\hhatanc^\star - \alpha^\star \one - h - (1-\gamma)h(s_0)\one - \alpha \one}  \\
        &= \infnorm{\hhatanc^\star - h -  \frac{\max_s (\hhatanc^\star(s) - h(s)) + \min_s (\hhatanc^\star(s) - h(s))}{2}\one }   \\
        &= \frac{\spannorm{\hhatanc^\star - h}}{2}.
    \end{align*}
    Now we can conclude by applying~\eqref{eq:value_error_amplification_bound} and noting that $\frac{\gamma}{1-\gamma} = \frac{1-\frac{1}{n}}{1 - 1 + \frac{1}{n}} = n-1$.
\end{proof}

\begin{lem}
\label{lem:anchored_AMDP_DMDP_opt_cond_relns}
    For any $h \in \R^S$, let $\Thatanc(h) := M(r + \Phatanc h)$ be the average-reward Bellman optimality operator for the anchored MDP $\Phatanc = (1-\eta) \Phat + \eta \one e_{s_0}^\top$.
    Suppose that one of the following conditions are satisfied for some policy $\pi$.
    \begin{enumerate}
        \item $\pi$ is greedy with respect to $r + \Phatanc h$ for some $h$ such that $\spannorm{\hhatanc^\star - h} \leq \frac{1}{n^2}$.

        \item $\Thatanc(\hhatanc^\pi) \leq \hhatanc^\pi + \rhohatanc^\pi + \frac{1}{n^2} \one$.

        \item $\rhohatanc^{\pi} \geq \rhohatanc^\star - \frac{1}{3n^2}$ and $\infnorm{\hhatanc^{\pi} - \hhatanc^\star} \leq \frac{1}{3n^2}.$

    \end{enumerate}
    Then we have that
    \begin{align}
        \infnorm{\Vhat^{\star}_{1-\frac{1}{n}} - \Vhat^{\pi}_{1-\frac{1}{n}}} & \leq \frac{1}{n}. \label{eq:anc_AMDP_central_opt_cond}
    \end{align}

    Furthermore, if inequality~\eqref{eq:anc_AMDP_central_opt_cond} holds, we have $\rhohatanc^{\pi} \geq \rhohatanc^\star - \frac{1}{n^2}$ and $\spannorm{\hhatanc^{\pi} - \hhatanc^\star} \leq \frac{2}{n}$.
\end{lem}
\begin{proof}
    The fact that the first condition implies~\eqref{eq:anc_AMDP_central_opt_cond} follows immediately from Lemma \ref{lem:greedy_error_amplification_amdp}.

    For the second condition, similar to the proof of Lemma \ref{lem:greedy_error_amplification_amdp}, we first note an optimality condition for DMDPs, which we will later verify using the second condition. Letting $\That_\gamma(V) := M(r + \gamma \Phat V)$ be the Bellman optimality operator for the $\gamma$-discounted MDP $\Phat$, for any policy $\pi$, we have
    \begin{align}
        \infnorm{\Vhat_\gamma^\star - \Vhat_\gamma^\pi} & \leq \frac{\infnorm{\That(\Vhat_\gamma^\pi) - \Vhat_\gamma^\pi}}{1-\gamma}. \label{eq:DMDP_error_bound_Bellman_residual}
    \end{align}
    Inequality~\eqref{eq:DMDP_error_bound_Bellman_residual} is well-known but we give a proof for completeness: using $\gamma$-contractivity of $\That_\gamma$,
    \begin{align*}
        \infnorm{\Vhat_\gamma^\star - \Vhat_\gamma^\pi} &= \infnorm{\That_\gamma (\Vhat_\gamma^\star) - \Vhat_\gamma^\pi} \\
            & \leq \infnorm{\That_\gamma(\Vhat_\gamma^\star) - \That_\gamma(\Vhat_\gamma^\pi)} + \infnorm{\That_\gamma(\Vhat_\gamma^\pi) - \Vhat_\gamma^\pi} \\
            & \leq \gamma \infnorm{\Vhat_\gamma^\star - \Vhat_\gamma^\pi} + \infnorm{\That_\gamma(\Vhat_\gamma^\pi) - \Vhat_\gamma^\pi}
    \end{align*}
        which implies $\infnorm{\Vhat_\gamma^\star - \Vhat_\gamma^\pi} \leq \frac{\infnorm{\That_\gamma(\Vhat_\gamma^\pi) - \Vhat_\gamma^\pi}}{1-\gamma}$ after rearranging. Now we relate condition 2 from the lemma to the quantity $\infnorm{\That(\Vhat_\gamma^\pi) - \Vhat_\gamma^\pi}$ from~\eqref{eq:DMDP_error_bound_Bellman_residual}. Letting $\gamma = 1-\frac{1}{n}$, by Lemma \ref{lem:anchoring_optimality_properties} we have that $\Vhat_\gamma^\pi = \hhatanc^\pi + c \one$ for some scalar $c$. We also have that 
        \begin{align*}
            \Thatanc(x) = M(r + \Phatanc x) = M(r + (1-\eta)\Phat x + \eta x(s_0) \one ) = M(r + (1-\eta)\Phat x  ) + \eta x(s_0) \one = \That_\gamma(x)  + \eta x(s_0) \one
        \end{align*}
        for any $x \in \R^S$.
        Then we can calculate
        \begin{align}
            \That_\gamma (\Vhat_\gamma^\pi) - \Vhat_\gamma^\pi &= \That_\gamma (\Vhat_\gamma^\pi) - \hhatanc^\pi - c \one \nonumber\\
            &= \Thatanc (\Vhat_\gamma^\pi) - \eta \Vhat_\gamma^\pi(s_0)\one - \hhatanc^\pi - c \one \nonumber \\
            &= \Thatanc ( \hhatanc^\pi + c \one) - \eta \Vhat_{1-\eta}^\pi(s_0)\one - \hhatanc^\pi - c \one \nonumber \\
            &= \Thatanc ( \hhatanc^\pi) - \eta \Vhat_{1-\eta}^\pi(s_0)\one - \hhatanc^\pi \nonumber \\
            &= \Thatanc ( \hhatanc^\pi) - \rhohatanc^\pi - \hhatanc^\pi \label{eq:AMDP_DMDP_Bellman_residual}
        \end{align}
        where in the last step we used the fact that $\eta \Vhat_{1-\eta}^\pi(s_0)\one = \rhohatanc^\pi$ by Lemma \ref{lem:anchoring_optimality_properties}. Also we have that $\That_\gamma (\Vhat_\gamma^\pi) - \Vhat_\gamma^\pi \geq 0$ (this is a standard fact, but to see this note that $\Vhat_\gamma^\pi$ satisfies the Bellman equation $\Vhat_\gamma^\pi = M^\pi(r + \gamma \Phat \Vhat_\gamma^\pi)$, and thus $\That_\gamma (\Vhat_\gamma^\pi) - \Vhat_\gamma^\pi = M(r + \gamma \Phat \Vhat_\gamma^\pi) - M^\pi(r + \gamma \Phat \Vhat_\gamma^\pi) \geq M^\pi(r + \gamma \Phat \Vhat_\gamma^\pi) - M^\pi(r + \gamma \Phat \Vhat_\gamma^\pi) = 0$.) Thus combining this with~\eqref{eq:AMDP_DMDP_Bellman_residual}, we have the equivalence
        \begin{align*}
            \infnorm{ \That_\gamma (\Vhat_\gamma^\pi) - \Vhat_\gamma^\pi} \leq \varepsilon \quad \iff \quad \Thatanc ( \hhatanc^\pi) - \rhohatanc^\pi - \hhatanc^\pi \leq \varepsilon \one.
        \end{align*}
        Therefore, by this equivalence, if $\Thatanc ( \hhatanc^\pi) - \rhohatanc^\pi - \hhatanc^\pi \leq \frac{1}{n^2} \one$, then $\infnorm{ \That_\gamma (\Vhat_\gamma^\pi) - \Vhat_\gamma^\pi} \leq \frac{1}{n^2}$, and plugging into~\eqref{eq:DMDP_error_bound_Bellman_residual}, we obtain that
        \[
            \Vhat_\gamma^\star - \Vhat_\gamma^\pi  \leq \frac{\infnorm{\That(\Vhat_\gamma^\pi) - \Vhat_\gamma^\pi}}{1-\gamma} \leq \frac{1/n^2}{1/n} = \frac{1}{n}
        \]
        as desired. Thus we have justified the second condition.

    Next we will show that the third condition implies the second condition, and thus it also implies that $\infnorm{\Vhat_\gamma^\star - \Vhat_\gamma^\pi} \leq \frac{1}{n}$.
    We can write
    \begin{align*}
        \Thatanc(\hhatanc^\pi) - \hhatanc^\pi - \rhohatanc^\pi &= \Thatanc(\hhatanc^\pi) - \Thatanc(\hhatanc^\star) + \Thatanc(\hhatanc^\star) - \hhatanc^\star + \hhatanc^\star - \hhatanc^\pi - \rhohatanc^\star + \rhohatanc^\star - \rhohatanc^\pi \\
        & \leq \infnorm{\Thatanc(\hhatanc^\pi) - \Thatanc(\hhatanc^\star)} \one + \Thatanc(\hhatanc^\star) - \hhatanc^\star - \rhohatanc^\star + \infnorm{\hhatanc^\star - \hhatanc^\pi}\one + \infnorm{\rhohatanc^\star - \rhohatanc^\pi}\one \\
        & \leq \gamma \infnorm{\hhatanc^\star - \hhatanc^\pi}\one + 0 + \infnorm{\hhatanc^\star - \hhatanc^\pi}\one + \infnorm{\rhohatanc^\star - \rhohatanc^\pi}\one \\
        & \leq \frac{\gamma}{3n^2}\one + \frac{1}{3n^2}\one + \frac{1}{3 n^2}\one \leq \frac{1}{n^2}\one
    \end{align*}
    where in the penultimate inequality we use the assumptions of condition 3.

    Finally, we assume that condition~\eqref{eq:anc_AMDP_central_opt_cond} holds and try to show that $\rhohatanc^{\pi} \geq \rhohatanc^\star - \frac{1}{n^2}$ and $\spannorm{\hhatanc^{\pi} - \hhatanc^\star} \leq \frac{2}{n}$. 
    First, by Lemma \ref{lem:anchoring_optimality_properties} we have $\rhohatanc^{\pi} = \frac{\Vhat_{1-\frac{1}{n}}^{\pi}(s_0)}{n} \one$ and $\rhohatanc^{\star} = \frac{\Vhat_{1-\frac{1}{n}}^{\star}(s_0)}{n} \one$. Combining with condition~\eqref{eq:anc_AMDP_central_opt_cond} we can obtain that
    \begin{align*}
        \rhohatanc^{\pi} &= \frac{\Vhat_{1-\frac{1}{n}}^{\pi}(s_0)}{n} \one \geq \frac{\Vhat_{1-\frac{1}{n}}^{\star}(s_0) -\left|\Vhat_{1-\frac{1}{n}}^{\star}(s_0) - \Vhat_{1-\frac{1}{n}}^{\pi}(s_0)\right| }{n} \one \\
        &\geq  \frac{\Vhat_{1-\frac{1}{n}}^{\star}(s_0) - \infnorm{\Vhat^{\star}_{1-\frac{1}{n}} - \Vhat^{\pi}_{1-\frac{1}{n}}}  }{n} \one \geq \rhohatanc^\star - \frac{1}{n^2}\one.
    \end{align*}
    Second, note that by Lemma \ref{lem:anchoring_optimality_properties} we have $\spannorm{\hhatanc^{\pi} - \hhatanc^\star} = \spannorm{\Vhat_{1 - \frac{1}{n}}^{\pi} - \Vhat_{1 - \frac{1}{n}}^{\star}}$, so we have that
    \begin{align*}
        \spannorm{\hhatanc^{\pi} - \hhatanc^\star} = \spannorm{\Vhat_{1 - \frac{1}{n}}^{\pi} - \Vhat_{1 - \frac{1}{n}}^{\star}} \leq 2 \infnorm{\Vhat_{1 - \frac{1}{n}}^{\pi} - \Vhat_{1 - \frac{1}{n}}^{\star}} \leq \frac{2}{n}.
    \end{align*}
\end{proof}

\begin{lem}
    \label{lem:anchpert_exact_VI_convergence}
    Under the same event that the conclusions of Theorem \ref{thm:AMDP_anchored_perturbed} hold, there is a unique Markovian bias-optimal policy $\pihatpert$ for the AMDP $(\Phatanc, \widetilde{r})$, and thus this policy is also the unique Markovian Blackwell optimal policy. The policy $\pihatpert$ is deterministic.
    Furthermore, $\pihatpert$ can be computed with $K = n \lceil \log\left( \frac{12 n^3 S A^2}{\xi \delta}\right) \rceil$ iterations of average-reward value iteration: Letting $\Thatancpert(h) := M(\widetilde{r} + \Phatanc h)$ be the average-reward Bellman operator for the AMDP $(\Phatanc, \widetilde{r})$, $\pihatpert$ is the greedy policy with respect to $\Thatancpert^{(K)}(0)$.
\end{lem}
\begin{proof}
First we note that under the event that the conclusions of Theorem \ref{thm:AMDP_anchored_perturbed} hold, we have the separation property described in Theorem \ref{thm:li_LOO_construction} with $\gamma = 1 - \frac{1}{n}$, namely that the optimal Markovian policy $\pihatpert$ in the DMDP $(\Phat, \widetilde{r}, 1 - \frac{1}{n})$ is unique and deterministic, and for all $s \in \S, a \in \A$ such that $a \neq \pihatpert(s)$,
\begin{align}
    \Qhatpert_{1-1/n}^\star(s, \pihatpert(s)) - \Qhatpert_{1-1/n}^\star(s,a) \geq \frac{\xi \delta }{3 nS A^2}. \label{eq:separation_cond_anch_pert}
\end{align}
(This property is used within the proof of Theorem \ref{thm:AMDP_anchored_perturbed}, namely within Lemma \ref{lem:LOO_DMDP_bernstein_bound_pert}).

First we check that $\pihatpert$ is the unique bias-optimal policy. Note that since Blackwell-optimal implies bias-optimal, this would imply that $\pihatpert$ is also the unique Blackwell-optimal policy. Since any bias-optimal policy $\pi$ satisfies $\hhatancpert^\pi = \hhatancpert^\star$, and $\Thatancpert(\hhatancpert^\star) = \rho^\star + \hhatancpert^\star$, and also $\rho^\star + \hhatancpert^\pi = M^\pi (r + \Phatanc \hhatancpert^\pi)$, we must have that $\pi$ is greedy with respect to $\widetilde{r} + \Phatanc \hhatancpert^\star$, that is $M(\widetilde{r} + \Phatanc \hhatancpert^\star) = M^\pi(\widetilde{r} + \Phatanc \hhatancpert^\star)$. By the definition of $\Qhatpert_{1-1/n}^\star$, the definition of $\Phatanc$, and Lemma \ref{lem:anchoring_optimality_properties} (which ensures $\Vhatpert^\star = \hhatancpert^\star + c \one$), we have
\begin{align*}
    \Qhatpert_{1-1/n}^\star(s, a) &= \widetilde{r} + (1-\frac{1}{n})\Phat \Vhatpert_{1-1/n}^\star = \widetilde{r} + \Phatanc \Vhatpert_{1-1/n}^\star - \eta \one \Vhatpert_{1-1/n}^\star(s_0) \\
    &= \widetilde{r} + \Phatanc \hhatancpert^\star  + \Phatanc c \one - \eta \one \Vhatpert^\star(s_0) 
     = \widetilde{r} + \Phatanc \hhatancpert^\star + c' \one
\end{align*}
for some scalars $c, c'$. Therefore $\pi$ is also greedy with respect to $\Qhatpert^\star_{1-1/n}$, and by the separation property~\eqref{eq:separation_cond_anch_pert}, this implies that we must have $\pi = \pihatpert$. Therefore the unique bias- and Blackwell-optimal policy is $\pihatpert$.

To show finite convergence we can combine the separation condition~\eqref{eq:separation_cond_anch_pert} with Lemma \ref{lem:greedy_error_amplification_amdp}, which also holds with the perturbed $\widetilde{r}$ reward function and thus guarantees that if there exists $h\in \R^S $ such that $\pi$ is greedy with respect to $\widetilde{r} + \Phatanc h$, then
    \begin{align}
        \Vhatpert_{1-\frac{1}{n}}^\star - \Vhatpert_{1-\frac{1}{n}}^\pi \leq (n-1)\spannorm{\hhatancpert^\star - h}. \label{eq:greedy_policy_DMDP_suboptimality}
    \end{align}
where $\hhatancpert^\star$ is the optimal bias function of the perturbed anchored empirical AMDP $(\Phatanc, \widetilde{r})$. From Lemma \ref{lem:anchoring_optimality_properties} we have that $\Thatancpert$ is a $1 - \eta = 1 - \frac{1}{n}$ span-contraction, and it has fixed point $\hhatancpert^\star$. Thus
\begin{align}
    \spannorm{\Thatancpert^{(K)}(0) - \hhatancpert^\star} &= \spannorm{\Thatancpert^{(K)}(0) - \Thatancpert^{(K)}(\hhatancpert^\star)} \nonumber\\
    & \leq \left(1 - \frac{1}{n} \right)^{K} \spannorm{\hhatancpert^\star - 0} \nonumber \\
    & \leq \left( \left(1 - \frac{1}{n} \right)^{n} \right)^{\lceil \log\left( \frac{12 n^3 S A^2}{\xi \delta}\right) \rceil} 2n \nonumber \\
    & \leq \left( \frac{1}{e}\right)^{ \log\left( \frac{12 n^3 S A^2}{\xi \delta}\right) } 2n \nonumber \\
    & \leq \frac{\xi \delta}{6 n^2 S A^2} \label{eq:hhatancpert_err_bd}
\end{align}
where we used that
\[
\spannorm{\hhatancpert^\star - 0} = \spannorm{\hhatancpert^\star} = \spannorm{\Vhatpert_{1-1/n}} \leq \infnorm{\Vhatpert_{1-1/n}} \leq \frac{1}{1-(1-\frac{1}{n})} \infnorm{\widetilde{r}} \leq n (1 + \frac{1}{n}) \leq 2n
\]
and also that $(1-\frac{1}{n})^n \leq \frac{1}{e}$. Now combining the bound~\eqref{eq:hhatancpert_err_bd} with~\eqref{eq:greedy_policy_DMDP_suboptimality}, if $\pi$ is chosen to be deterministic and greedy with respect to $\widetilde{r} + \Phatanc \Thatancpert^{(K)}(0)$, we have that $\Vhatpert_{1-\frac{1}{n}}^\star - \Vhatpert^{\pi}_{1-\frac{1}{n}} \leq \frac{\xi \delta}{6 n S A^2}$.
By~\eqref{eq:separation_cond_anch_pert}, if $\pi$ is a deterministic policy, then this implies $\pi = \pihatpert$. However, if $\pi$ is not deterministic, then if must be possible to write $\pi$ as the convex combination of some distinct policies which are deterministic and such that each of these policies is also greedy with respect to $\widetilde{r } + \Phatanc \Thatancpert^{(K)}(0)$, but then each of these deterministic policies must also be equal to $ \pihatpert$, giving a contradiction, so $\pi$ must have been deterministic.
\end{proof}

\subsection{Proof of Theorem \ref{thm:AMDP_anchored_perturbed}}
\label{sec:AMDP_anch_pert_pf}
Here we complete the proof of Theorem \ref{thm:AMDP_anchored_perturbed}. Since Lemma \ref{lem:anchoring_optimality_properties} relates the gains and bias functions of anchored AMDPs to DMDPs, we can prove this theorem by simply combining Lemma \ref{lem:anchoring_optimality_properties} with Theorem \ref{thm:DMDP_pert_thm}.

\begin{proof}[Proof of Theorem \ref{thm:AMDP_anchored_perturbed}]
Following the conditions listed in Theorem \ref{thm:AMDP_anchored_perturbed}, let $\pihat$ be an exact Blackwell-optimal policy of the AMDP $(\Phatanc, \widetilde{r})$.

    First, by using triangle inequality twice, we have the elementwise inequality
    \begin{align}
        \rho^{\pistar} - \rho^{\pihat}  & \leq \infnorm{\rho^{\pihat} - \rhoanc^{\pihat}}\one + \left(\rhoanc^{\pistar} - \rhoanc^{\pihat} \right) + \infnorm{\rhoanc^{\pistar} - \rho^{\pistar}} \one. \label{eq:anchored_perturbed_pf_1}
    \end{align}
    Note that the first and third terms of~\eqref{eq:anchored_perturbed_pf_1} are controlled by Lemma \ref{lem:anchoring_optimality_properties} as
    \begin{align}
        \infnorm{\rho^{\pistar} - \rhoanc^{\pistar}}  \leq \frac{2\spannorm{\hanc^{\pistar}}}{n} \leq \frac{4\spannorm{h^{\pistar}}}{n} = \frac{4\spannorm{h^\star}}{n}\label{eq:anchored_perturbed_pf_2}
    \end{align}
    (because $\rho^{\pistar} = \rho^\star$ is constant, for the second inequality)
    and
    \begin{align}
        \infnorm{\rho^{\pihat} - \rhoanc^{\pihat}}  \leq \frac{2\spannorm{\hanc^{\pihat}}}{n}.  \label{eq:anchored_perturbed_pf_3}
    \end{align}
    For the middle term on the RHS of~\eqref{eq:anchored_perturbed_pf_1}, using the identity for the gain in the anchored AMDP from Lemma \ref{lem:anchoring_optimality_properties}, we have
    \begin{align}
        \rhoanc^{\pistar} - \rhoanc^{\pihat} &=  \frac{V^{\pistar}_{1-\frac{1}{n}}(s_0)}{n} - \frac{V^{\pihat}_{1-\frac{1}{n}}(s_0)}{n} \nonumber \\
        & \leq \frac{V^{\star}_{1-\frac{1}{n}}(s_0)}{n} - \frac{V^{\pihat}_{1-\frac{1}{n}}(s_0)}{n} \nonumber\\
        & \leq \frac{1}{n} \infnorm{V^{\star}_{1-\frac{1}{n}} - V^{\pihat}_{1-\frac{1}{n}}}. \label{eq:anchored_perturbed_pf_4}
    \end{align}

    On the high-probability event in the conlusion of Theorem \ref{thm:DMDP_pert_thm}, by Lemma \ref{lem:anchpert_exact_VI_convergence}, the policy $\pihat$ (defined as a Blackwell-optimal policy for $(\Phatanc, \widetilde{r})$) is identical to the optimal policy for the DMDP $(\Phat, \widetilde{r}, 1-1/n)$. Therefore Theorem \ref{thm:DMDP_pert_thm} (which is stated for the optimal policy of the DMDP $(\Phat, \widetilde{r}, 1-1/n)$) also applies to $\pihat$ as defined in this proof. Thus by Theorem \ref{thm:DMDP_pert_thm}, with probability at least $1 - \delta$,
    \begin{align*}
        \infnorm{V^{\star}_{1-\frac{1}{n}} - V^{\pihat}_{1-\frac{1}{n}}} & \leq \frac{1 }{1-(1-\frac{1}{n})}\sqrt{\frac{C_2  \log^3 \left( \frac{SAn}{(1-(1-\frac{1}{n})) \delta \xi}\right) }{n}\left( \spannorm{V_{1-\frac{1}{n}}^{\star}} + \spannorm{V_{1-\frac{1}{n}}^{\pihat}} + 1\right) } \\
        &= n \sqrt{\frac{C_2  \log^3 \left( \frac{SAn^2}{ \delta \xi}\right) }{n}\left( \spannorm{\hanc^\star} + \spannorm{\hanc^{\pihat}} + 1\right) } \\
        &\leq n \sqrt{\frac{C_2  \log^3 \left( \frac{SAn^2}{ \delta \xi}\right) }{n}\left( 2\spannorm{h^\star} + \spannorm{\hanc^{\pihat}} + 1\right) }
    \end{align*}
    (using Lemma \ref{lem:anchoring_optimality_properties} in the second two steps, specifically the facts that $\spannorm{V_{1-\frac{1}{n}}^{\star}} = \spannorm{\hanc^\star}$,  $ \spannorm{V_{1-\frac{1}{n}}^{\pihat}} = \spannorm{\hanc^{\pihat}}$, and then that $\spannorm{\hanc^\star} \leq 2 \spannorm{h^\star}$ since $\rho^\star$ is constant).
    Combining this with~\eqref{eq:anchored_perturbed_pf_4},
    \begin{align}
        \rhoanc^{\pistar} - \rhoanc^{\pihat} &\leq \sqrt{\frac{C_2  \log^3 \left( \frac{SAn^2}{ \delta \xi}\right) }{n}\left( 2\spannorm{h^\star} + \spannorm{\hanc^{\pihat}} + 1\right) }\label{eq:anchored_perturbed_pf_5}.
    \end{align}
    Combining~\eqref{eq:anchored_perturbed_pf_2},~\eqref{eq:anchored_perturbed_pf_3}, and~\eqref{eq:anchored_perturbed_pf_5} with~\eqref{eq:anchored_perturbed_pf_1}, and then simplifying, we obtain
    \begin{align*}
        \rho^{\pistar} - \rho^{\pihat} & \leq \frac{4\spannorm{h^\star}}{n} + \sqrt{\frac{C_2  \log^3 \left( \frac{SAn^2}{ \delta \xi }\right) }{n}\left( 2\spannorm{h^\star} + \spannorm{\hanc^{\pihat}} + 1\right) } + \frac{2\spannorm{\hanc^{\pihat}}}{n} \\
        &\leq \sqrt{\frac{4\spannorm{h^\star}}{n}} + \sqrt{\frac{C_2  \log^3 \left( \frac{SAn^2}{ \delta \xi}\right) }{n}\left( 2\spannorm{h^\star} + \spannorm{\hanc^{\pihat}} + 1\right) } + \sqrt{\frac{2\spannorm{\hanc^{\pihat}}}{n}} \\
        & \leq 3\sqrt{ \frac{4\spannorm{h^\star}}{n} + \frac{C_2  \log^3 \left( \frac{SAn^2}{ \delta \xi}\right) }{n}\left( 2\spannorm{h^\star} + \spannorm{\hanc^{\pihat}} + 1\right) + \frac{2\spannorm{\hanc^{\pihat}}}{n} } \\
        & \leq \sqrt{\frac{C_3  \log^3 \left( \frac{SAn}{ \delta \xi}\right) }{n}\left( 2\spannorm{h^\star} + \spannorm{\hanc^{\pihat}} + 1\right) }
    \end{align*}
    where the second inequality holds since the terms $\frac{4\spannorm{h^\star}}{n}, \frac{2\spannorm{\hanc^{\pihat}}}{n}$ must be $\leq 1$ or else the bound holds trivially, the third inequality uses $\sqrt{a} + \sqrt{b} + \sqrt{c} \leq 3\sqrt{a + b + c}$, and the final inequality sets $C_3$ sufficiently large.
\end{proof}

\subsection{Higher-order variance bounds}
\label{sec:AMDP_high_order_var_bounds}
\begin{lem}
\label{lem:moment_inequality}
    \[\left(P_\pi (h^\pi)^{\circ 2^k}\right)^{\circ 2} \geq (P_\pi h)^{\circ 2^{k+1}}.\]
\end{lem}
\begin{proof}
    For arbitrary $\ell \geq 0$, we have (elementwise)
    \begin{align*}
        0 \leq \Var_{P_\pi} \left[(h^{\pi})^{\circ 2^\ell}\right] = P_\pi (h^\pi)^{\circ 2^{\ell+1}} - \left(P_\pi(h^\pi)^{\circ 2^\ell}\right)^{\circ 2}
    \end{align*}
    so $\left(P_\pi(h^\pi)^{\circ 2^\ell}\right)^{\circ 2} \leq P_\pi (h^\pi)^{\circ 2^{\ell+1}}$. Now we apply this fact $k$ times to obtain the desired conclusion:
    \begin{align*}
        \left(P_\pi (h^\pi)^{\circ 2^k}\right)^{\circ 2} &\geq \left( \left( P_\pi (h^\pi)^{\circ 2^{k-1}}\right)^{\circ 2} \right)^{\circ 2} =  \left( P_\pi (h^\pi)^{\circ 2^{k-1}}\right)^{\circ 2^2} \\
        &~~\vdots \\
        & \geq \left( P_\pi (h^\pi)^{\circ 2^{k-j}}\right)^{\circ 2^{j+1}} \\
        &~~\vdots \\
        & \geq (P_\pi h)^{\circ 2^{k+1}}.
    \end{align*}
\end{proof}

\begin{lem}
\label{lem:recursive_variance_param_bound}
Fix an integer $k \geq 0$ and let $\hc = h^\pi - \left(\min_s h^\pi(s)\right)\one $. If
    \begin{align}
        \left|(\Phat_\pi - P_\pi) (\hc)^{\circ 2^k} \right| \leq \sqrt{\frac{\alpha \Var_{P_\pi}\left[(\hc)^{\circ 2^k} \right] }{n}} + \frac{\alpha \cdot 2^k  }{n} \left(\infnorm{\hc} + 1 \right)^{2^k} \one \label{eq:bernstein_condition_shifted}
    \end{align}
    holds, then
    \begin{align*}
        &\left(\frac{\alpha}{n}\right)^{1-2^{-k}}\left|\Phat_\pi^\infty (\Phat_\pi - P_\pi) (\hc)^{\circ 2^k} \right|^{2^{-k}} \\
        &\leq \left(\frac{\alpha}{n}\right)^{1-2^{-(k+1)}}  \left|  \Phat_\pi^\infty \left(\Phat_\pi - P_\pi \right) (\hc)^{\circ 2^{k+1} } \right|^{2^{-(k+1)}} + 2 \left(\frac{\alpha \left(\infnorm{\hc} + 1\right)}{n} \right)^{1 - 2^{-(k+1)}}\one+ \frac{2\alpha   }{n}\left(\infnorm{\hc} + 1 \right) \one.
    \end{align*}
\end{lem}
\begin{proof}
    \begin{align}
        \left|\Phat_\pi^\infty (\Phat_\pi - P_\pi) (\hc)^{\circ 2^k} \right| & \leq \Phat_\pi^\infty \left|(\Phat_\pi - P_\pi) (\hc)^{\circ 2^k} \right| \\
        & \leq \Phat_\pi^\infty  \sqrt{\frac{\alpha \Var_{P_\pi}\left[(\hc)^{\circ 2^k} \right] }{n}} + \frac{\alpha \cdot 2^k  }{n}\left(\infnorm{\hc} + 1 \right)^{2^k} \one \label{eq:var_bound_step_1}
    \end{align}
    using the fact that all entries of $\Phat_\pi^\infty$ are non-negative, and then the condition~\eqref{eq:bernstein_condition_shifted}. Focusing on the first term in~\eqref{eq:var_bound_step_1},
    \begin{align}
         \Phat_\pi^\infty  \sqrt{\Var_{P_\pi}\left[(\hc)^{\circ 2^k} \right]} &\leq  \sqrt{\Phat_\pi^\infty  \Var_{P_\pi}\left[(\hc)^{\circ 2^k} \right]} \\
         &= \sqrt{\Phat_\pi^\infty \left( P_\pi (\hc)^{\circ 2^{k+1}} - \left(P_\pi (\hc)^{\circ 2^k}\right)^{\circ 2} \right)} \\
         &\leq  \sqrt{\Phat_\pi^\infty \left( P_\pi (\hc)^{\circ 2^{k+1}} - \left(P_\pi \hc\right)^{\circ 2^{k+1}}  \right) } \\
         &=  \sqrt{\Phat_\pi^\infty \left( P_\pi (\hc)^{\circ 2^{k+1}} - \left( \hc + P_\pi^\infty r_\pi - r_\pi \right)^{\circ 2^{k+1}}  \right) } \label{eq:var_bound_step_2}
    \end{align}
    where we used Jensen's inequality since each row of $\Phat_\pi^\infty$ is a probability distribution, the definition of $\Var_{P_\pi}\left[(\hc)^{\circ 2^k} \right]$, the non-negativity of $\Phat_\pi^\infty$ along with the inequality from Lemma \ref{lem:moment_inequality}, and then finally the Bellman equation $h^\pi + P_\pi^\infty r_\pi = r_\pi + P_\pi h^\pi$, which after subtracting $\left(\min_s h^\pi(s) \right) \one$ from both sides yields
    \begin{align*}
        \hc + P_\pi^\infty r_\pi &= h^\pi - \left(\min_s h^\pi(s) \right) \one + P_\pi^\infty r_\pi \\
        &= r_\pi + P_\pi h^\pi - \left(\min_s h^\pi(s) \right) \one \\
        &= r_\pi + P_\pi \left( h^\pi - \left(\min_s h^\pi(s) \right) \one  \right) \\
        &= r_\pi + P_\pi \hc
    \end{align*}
    since $P_\pi \one = \one$.
    Now note that if we expand $\left( \hc + P_\pi^\infty r_\pi - r_\pi \right)^{\circ 2^{k+1}}$ into the sum of $2^{k+1}$ individual terms, the leading term will be $(\hc)^{\circ 2^{k+1}}$, while all other terms will be (entrywise) bounded in magnitude by $\max \{\infnorm{\hc} , 1 \}^{2^{k+1}-1} $ since $\infnorm{P_\pi^\infty r_\pi - r_\pi} \leq 1$.
    Thus resuming from~\eqref{eq:var_bound_step_2}, we have
    \begin{align*}
        &\sqrt{\Phat_\pi^\infty \left( P_\pi (\hc)^{\circ 2^{k+1}} - \left( \hc + P_\pi^\infty r_\pi - r_\pi \right)^{\circ 2^{k+1}}  \right) } \\
         &\leq  \sqrt{\Phat_\pi^\infty \left( P_\pi (\hc)^{\circ 2^{k+1}} - \left( \hc\right)^{\circ 2^{k+1}} + 2^{2^{k+1}} \max\{\infnorm{\hc}, 1 \}^{2^{k+1}-1} \one  \right) } \\
         &\leq  \sqrt{\Phat_\pi^\infty \left( P_\pi (\hc)^{\circ 2^{k+1}} - \left( \hc\right)^{\circ 2^{k+1}} + 2^{2^{k+1}} \left(\infnorm{\hc} + 1 \right)^{2^{k+1}-1} \one  \right) } \\
         & = \sqrt{\Phat^\infty_\pi \left(P_\pi - I \right) (\hc)^{\circ 2^{k+1}} + 2^{2^{k+1}} \left(\infnorm{\hc} + 1 \right)^{2^{k+1}-1} \one} \\
         & = \sqrt{\Phat^\infty_\pi \left(P_\pi - \Phat_\pi \right) (\hc)^{\circ 2^{k+1}} + 2^{2^{k+1}} \left(\infnorm{\hc} + 1 \right)^{2^{k+1}-1} \one} \\
         & \leq \left| \Phat_\pi^\infty \left(\Phat_\pi - P_\pi \right) (\hc)^{\circ 2^{k+1} } \right|^{1/2} + 2^{2^k} \left(\infnorm{\hc} + 1 \right)^{\frac{1}{2}(2^{k+1}-1)} \one
    \end{align*}
where in the final equality step we use that $\Phat_\pi^\infty \Phat_\pi = \Phat_\pi^\infty$, and then in the final inequality step we use that $\sqrt{a + b} \leq \sqrt{|a|} + \sqrt{b}$.
Combining these steps we have that
\begin{align*}
    & \left|\Phat_\pi^\infty (\Phat_\pi - P_\pi) (h^{\pi})^{\circ 2^k} \right| \\
    & \quad \quad \leq \Phat_\pi^\infty  \sqrt{\frac{\alpha \Var_{P_\pi}\left[(\hc)^{\circ 2^k} \right] }{n}} + \frac{\alpha \cdot 2^k  }{n}\left(\infnorm{\hc} + 1 \right)^{2^k} \one \\
    & \quad \quad \leq \sqrt{\frac{\alpha}{n}} \left( \left| \Phat_\pi^\infty \left(\Phat_\pi - P_\pi \right) (\hc)^{\circ 2^{k+1} } \right|^{1/2} + 2^{2^k} \left(\infnorm{\hc} + 1 \right)^{\frac{1}{2}(2^{k+1}-1)} \one \right) +\frac{\alpha \cdot 2^k  }{n}\left(\infnorm{\hc} + 1 \right)^{2^k} \one \\
    & \quad \quad =  \left| \frac{\alpha}{n} \Phat_\pi^\infty \left(\Phat_\pi - P_\pi \right) (\hc)^{\circ 2^{k+1} } \right|^{1/2} + 2^{2^k} \left(\frac{\alpha \left(\infnorm{\hc} + 1\right)^{2^{k+1}-1}}{n} \right)^{1/2}\one  + \frac{\alpha \cdot 2^k  }{n}\left(\infnorm{\hc} + 1 \right)^{2^k} \one.
\end{align*}
Therefore
\begin{align*}
    &\left(\frac{\alpha}{n}\right)^{1-2^{-k}} \left|\Phat_\pi^\infty (\Phat_\pi - P_\pi) (\hc)^{\circ 2^k} \right|^{2^{-k}} \\
    &\leq \left(\frac{\alpha}{n}\right)^{1-2^{-k}} \Biggg( \left| \frac{\alpha}{n} \Phat_\pi^\infty \left(\Phat_\pi - P_\pi \right) (\hc)^{\circ 2^{k+1} } \right|^{1/2} + 2^{2^k} \left(\frac{\alpha \left(\infnorm{\hc} + 1\right)^{2^{k+1}-1}}{n} \right)^{1/2}\one\\
    & \quad \quad + \frac{\alpha \cdot 2^k  }{n}\left(\infnorm{\hc} + 1 \right)^{2^k} \one   \Biggg)^{2^{-k}} \\
    &\leq \left(\frac{\alpha}{n}\right)^{1-2^{-k}} \Biggg( \left| \frac{\alpha}{n} \Phat_\pi^\infty \left(\Phat_\pi - P_\pi \right) (\hc)^{\circ 2^{k+1} } \right|^{2^{-k}/2} + 2^{{2^k}\cdot 2^{-k}} \left(\frac{\alpha \left(\infnorm{\hc} + 1\right)^{2^{k+1}-1}}{n} \right)^{2^{-k}/2}\one \\
    &\quad \quad + \left(\frac{\alpha \cdot 2^k  }{n}\right)^{2^{-k}}\left(\infnorm{\hc} + 1 \right)^{2^k \cdot 2^{-k}} \one   \Biggg) \\
    &\leq \left(\frac{\alpha}{n}\right)^{1-2^{-k}} \Biggg( \left| \frac{\alpha}{n} \Phat_\pi^\infty \left(\Phat_\pi - P_\pi \right) (\hc)^{\circ 2^{k+1} } \right|^{2^{-k}/2} + 2^{{2^k}\cdot 2^{-k}} \left(\frac{\alpha \left(\infnorm{\hc} + 1\right)^{2^{k+1}-1}}{n} \right)^{2^{-k}/2}\one \\
    & \quad \quad + 2\left(\frac{\alpha   }{n}\right)^{2^{-k}}\left(\infnorm{\hc} + 1 \right)^{2^k \cdot 2^{-k}} \one   \Biggg) \\
    &= \left(\frac{\alpha}{n}\right)^{1-2^{-k}} \Biggg( \left| \frac{\alpha}{n} \Phat_\pi^\infty \left(\Phat_\pi - P_\pi \right) (\hc)^{\circ 2^{k+1} } \right|^{2^{-(k+1)}} + 2 \left(\frac{\alpha \left(\infnorm{\hc} + 1\right)^{2^{k+1}-1}}{n} \right)^{2^{-(k+1)}}\one\\
    &\quad \quad + 2\left(\frac{\alpha   }{n}\right)^{2^{-k}}\left(\infnorm{\hc} + 1 \right) \one   \Biggg) \\
    &= \left(\frac{\alpha}{n}\right)^{1-2^{-k}} \Bigg( \left| \frac{\alpha}{n} \Phat_\pi^\infty \left(\Phat_\pi - P_\pi \right) (\hc)^{\circ 2^{k+1} } \right|^{2^{-(k+1)}} + 2 \left(\frac{\alpha}{n} \right)^{2^{-(k+1)}} \left(\infnorm{\hc} + 1\right)^{1 - 2^{-(k+1)}}\one\\
    & \quad \quad + 2\left(\frac{\alpha   }{n}\right)^{2^{-k}}\left(\infnorm{\hc} + 1 \right) \one   \Bigg) \\
    &= \left(\frac{\alpha}{n}\right)^{1-2^{-(k+1)}}  \left|  \Phat_\pi^\infty \left(\Phat_\pi - P_\pi \right) (\hc)^{\circ 2^{k+1} } \right|^{2^{-(k+1)}} + 2 \left(\frac{\alpha \left(\infnorm{\hc} + 1\right)}{n} \right)^{1 - 2^{-(k+1)}}\one+ \frac{2\alpha   }{n}\left(\infnorm{\hc} + 1 \right) \one   
\end{align*}
    as desired, where for the inequality steps we used the previous calculations, then that $(a + b + c)^{2^{-k}} \leq a^{2^{-k}} + b^{2^{-k}} + c^{2^{-k}}$, and then that $2^k \leq 2^{2^k}$ so $(2^{k})^{2^{-k}} \leq 2^{2^k \cdot 2^{-k}} = 2$.
\end{proof}

\begin{lem}
\label{lem:AMDP_full_var_param_bound}
    Let $\hc = h^{\pi} - \left(\min_s h^{\pi}(s)\right)\one$ and $\ell = \lceil\log_2 \log_2 \left( \infnorm{\hc} + 4\right) \rceil$. Suppose that for some $\alpha \in \R$, the inequalities
    \begin{align*}
       \left|(\Phat_\pi - P_\pi) (\hc)^{\circ 2^k} \right| \leq \sqrt{\frac{\alpha \Var_{P_\pi}\left[(\hc)^{\circ 2^k} \right] }{n}} + \frac{\alpha \cdot 2^k  }{n} \left(\infnorm{\hc} + 1 \right)^{2^k} \one
    \end{align*}
    hold for all $k = 0, \dots, \ell$. Also suppose that $\rho^\pi$ is constant.
    Then
    \begin{align*}
    \infnorm{\rhohat^\pi - \rho^\pi} \leq 2 (\ell + 1) \left(\frac{\alpha \left(\infnorm{\hc} + 1\right)}{n} \right)^{ \frac{1}{2}}+ (\ell+1) \frac{2\alpha   }{n}\left(\infnorm{\hc} + 1 \right).
\end{align*}
\end{lem}
We also note that this Lemma is purely algebraic, and thus we can accordingly replace the transition matrices $P, \Phat$ (and all their derived quantities, that is, $h^\pi$, $\rho^\pi$, $\rhohat^\pi$) with any other pair of transition matrices.

\begin{proof}
    Similarly to the proof of Lemma \ref{lem:DMDP_full_var_param_bound}, first we give a weaker but non-recursive bound which can be used on the final term. 
Note that
\begin{align}
       \left|(\Phat_\pi - P_\pi) (\hc)^{\circ 2^\ell} \right| 
       &\leq \sqrt{\frac{\alpha \Var_{P_\pi}\left[(\hc)^{\circ 2^\ell} \right] }{n}} \one + \frac{\alpha \cdot 2^\ell  }{n} \left(\infnorm{\hc} + 1 \right)^{2^\ell} \one \nonumber \\
       & \leq  \sqrt{\frac{\alpha}{n}} \infnorm{(\hc)^{\circ  2^\ell}} \one + \frac{\alpha \cdot 2^\ell  }{n} \left(\infnorm{\hc} + 1 \right)^{2^\ell} \one \nonumber\\
       & =  \sqrt{\frac{\alpha}{n}} \infnorm{\hc}^{2^\ell} \one + \frac{\alpha \cdot 2^\ell  }{n} \left(\infnorm{\hc} + 1 \right)^{2^\ell} \one \label{eq:last_term_bound_AMDP}
    \end{align}
and so (elementwise)
\begin{align*}
    \left|\Phat_\pi^\infty (\Phat_\pi - P_\pi) (\hc)^{\circ 2^\ell} \right| & \leq \Phat_\pi^\infty \left| (\Phat_\pi - P_\pi) (\hc)^{\circ 2^\ell} \right| \\
    & \leq \Phat_\pi^\infty \left( \sqrt{\frac{\alpha}{n}} \infnorm{\hc}^{2^\ell} \one + \frac{\alpha \cdot 2^\ell  }{n} \left(\infnorm{\hc} + 1 \right)^{2^\ell} \one \right) \\
    &= \sqrt{\frac{\alpha}{n}} \infnorm{\hc}^{2^\ell} \one + \frac{\alpha \cdot 2^\ell  }{n} \left(\infnorm{\hc} + 1 \right)^{2^\ell} \one
\end{align*}
using that all entries of $\Phat_\pi^\infty$ are non-negative in the first inequality, then the bound~\eqref{eq:last_term_bound_AMDP}, then the fact $\Phat_\pi^\infty \one = \one$.
Therefore
\begin{align}
    \left(\frac{\alpha}{n}\right)^{1-2^{-\ell}}\left|\Phat_\pi^\infty (\Phat_\pi - P_\pi) (\hc)^{\circ 2^\ell} \right|^{2^{-\ell}} & \leq \left(\frac{\alpha}{n}\right)^{1-2^{-\ell}} \left(\sqrt{\frac{\alpha}{n}} \infnorm{\hc}^{2^\ell}  + \frac{\alpha \cdot 2^\ell  }{n} \left(\infnorm{\hc} + 1 \right)^{2^\ell}  \right)^{2^{-\ell}} \one \nonumber \\
    & \leq \left(\frac{\alpha}{n}\right)^{1-2^{-\ell}} \left(\left(\sqrt{\frac{\alpha}{n}} \infnorm{\hc}^{2^\ell} \right)^{2^{-\ell}}  + \left(\frac{\alpha \cdot 2^\ell  }{n} \left(\infnorm{\hc} + 1 \right)^{2^\ell}  \right)^{2^{-\ell}} \right)\one  \nonumber \\
    & \leq \left(\frac{\alpha}{n}\right)^{1-2^{-\ell}} \left(\left(\sqrt{\frac{\alpha}{n}} \infnorm{\hc}^{2^\ell} \right)^{2^{-\ell}}  + \left(\frac{\alpha \cdot 2^{2^\ell}  }{n} \left(\infnorm{\hc} + 1 \right)^{2^\ell}  \right)^{2^{-\ell}} \right)\one  \nonumber \\
    & = \left(\frac{\alpha}{n}\right)^{1-2^{-(\ell+1)}}  \infnorm{\hc} \one + 2\frac{\alpha}{n} \left(\infnorm{\hc} + 1 \right) \one \label{eq:non_recursive_bound_final_term_AMDP}.
\end{align}

Since we have assumed that $\rho^\pi$ is constant, using Lemma \ref{lem:average_reward_simulation_lemma} and then the fact that $\left(\Phat_\pi - P_\pi\right) \one = 0$, we have
\begin{align*}
    \rhohat^\pi - \rho^\pi &= \Phat^\infty_\pi (\Phat_\pi - P_\pi) h^\pi \\
    &= \Phat^\infty_\pi (\Phat_\pi - P_\pi) \left( h^{\pi} - \left(\min_s h^{\pi}(s)\right)\one \right) \\
    &= \Phat^\infty_\pi (\Phat_\pi - P_\pi) \hc.
\end{align*}

Now, using this equation, then Lemma \ref{lem:recursive_variance_param_bound} $\ell$ times (for $k = 0, \dots, \ell - 1$), then assuming $n \geq \alpha (\infnorm{\hc} + 1)$, then using the bound~\eqref{eq:non_recursive_bound_final_term_AMDP} for the final term, we obtain
\begin{align}
    \infnorm{\rhohat^\pi - \rho^\pi} & = \infnorm{ \Phat^\infty_\pi (\Phat_\pi - P_\pi) \hc} \nonumber \\
    &\leq \sum_{k=0}^{\ell - 1} \left( 2 \left(\frac{\alpha \left(\infnorm{\hc} + 1\right)}{n} \right)^{1 - 2^{-(k+1)}}+ \frac{2\alpha   }{n}\left(\infnorm{\hc} + 1 \right) \right) \nonumber \\
    & \quad \quad \quad \quad \quad \quad \quad \quad + \left(\frac{\alpha}{n}\right)^{1-2^{-\ell}}  \left|  \Phat_\pi^\infty \left(\Phat_\pi - P_\pi \right) (\hc)^{\circ 2^{\ell} } \right|^{2^{-\ell}} \nonumber \\
    & \leq \sum_{k=0}^{\ell - 1} \left( 2 \left(\frac{\alpha \left(\infnorm{\hc} + 1\right)}{n} \right)^{1 - \frac{1}{2}}+ \frac{2\alpha   }{n}\left(\infnorm{\hc} + 1 \right) \right)  \nonumber \\
     & \quad \quad \quad \quad \quad \quad \quad \quad + \left(\frac{\alpha}{n}\right)^{1-2^{-\ell}}  \left|  \Phat_\pi^\infty \left(\Phat_\pi - P_\pi \right) (\hc)^{\circ 2^{\ell} } \right|^{2^{-\ell}} \nonumber \\
    & =   2 \ell \left(\frac{\alpha \left(\infnorm{\hc} + 1\right)}{n} \right)^{ \frac{1}{2}}+ \ell \frac{2\alpha   }{n}\left(\infnorm{\hc} + 1 \right)  + \left(\frac{\alpha}{n}\right)^{1-2^{-\ell}}  \left|  \Phat_\pi^\infty \left(\Phat_\pi - P_\pi \right) (\hc)^{\circ 2^{\ell} } \right|^{2^{-\ell}} \nonumber \\
    & \leq    2 \ell \left(\frac{\alpha \left(\infnorm{\hc} + 1\right)}{n} \right)^{ \frac{1}{2}}+ \ell \frac{2\alpha   }{n}\left(\infnorm{\hc} + 1 \right)  +    \left(\frac{\alpha}{n}\right)^{1-2^{-(\ell+1)}}  \infnorm{\hc}  + 2\frac{\alpha  }{n}  \left(\infnorm{\hc} + 1 \right) \nonumber \\
    & =    2 \ell \left(\frac{\alpha \left(\infnorm{\hc} + 1\right)}{n} \right)^{ \frac{1}{2}}+ (\ell+1) \frac{2\alpha   }{n}\left(\infnorm{\hc} + 1 \right)  +    \left(\frac{\alpha}{n}\right)^{1-2^{-(\ell+1)}}  \infnorm{\hc} .\label{eq:AMDP_error_bound_penultimate}
\end{align}
Note that the assumption $n \geq \alpha (\infnorm{\hc} + 1)$ was used to guarantee that the largest term in the initial summation was the $k = 0$ term.

Finally, we need to ensure that $\ell = \lceil\log_2 \log_2 \left( \infnorm{\hc} + 4\right) \rceil$ is sufficiently large so that the rightmost term in~\eqref{eq:AMDP_error_bound_penultimate} is bounded by $2\left(\frac{\alpha \infnorm{\hc}}{n} \right)^{ \frac{1}{2}}$. This rightmost term can be bounded as
\begin{align*}
    \left(\frac{\alpha}{n}\right)^{1-2^{-(\ell+1)}}  \infnorm{\hc} &= \left(\frac{\alpha \infnorm{\hc} }{n}\right)^{1-2^{-(\ell+1)}}  \infnorm{\hc}^{2^{-(\ell+1)}} 
    \leq \left(\frac{\alpha \infnorm{\hc} }{n}\right)^{1/2}  \infnorm{\hc}^{2^{-(\ell+1)}}
\end{align*}
(again using the assumption that $n \geq \alpha (\infnorm{\hc} + 1)$), and then we have the equivalences
\begin{align*}
    &\left(\frac{\alpha \infnorm{\hc} }{n}\right)^{1/2}  \infnorm{\hc}^{2^{-(\ell+1)}}
    \leq 2\left(\frac{\alpha \infnorm{\hc}}{n} \right)^{ \frac{1}{2}} \\
    \iff & \infnorm{\hc}^{2^{-(\ell+1)}}  \leq 2 \\
    \iff & 2^{-(\ell+1)} \log_2 \left( \infnorm{\hc} \right)  \leq 1 \\
    \iff & \log_2 \left( \infnorm{\hc} \right) \leq 2^{\ell + 1} \\
    \iff & \ell + 1 \geq \log_2 \log_2 \left( \infnorm{\hc} \right).
\end{align*}
The final inequality is true for our definition of $\ell$, so we have that the rightmost term in~\eqref{eq:AMDP_error_bound_penultimate} is bounded by $2\left(\frac{\alpha \infnorm{\hc}}{n} \right)^{ \frac{1}{2}}$ as desired. Thus combining this fact with~\eqref{eq:AMDP_error_bound_penultimate}, we have
\begin{align*}
    \infnorm{\rhohat^\pi - \rho^\pi} & \leq  2 \ell \left(\frac{\alpha \left(\infnorm{\hc} + 1\right)}{n} \right)^{ \frac{1}{2}}+ (\ell+1) \frac{2\alpha   }{n}\left(\infnorm{\hc} + 1 \right)  +   2\left(\frac{\alpha \infnorm{\hc}}{n} \right)^{ \frac{1}{2}} \\
    & \leq 2 (\ell + 1) \left(\frac{\alpha \left(\infnorm{\hc} + 1\right)}{n} \right)^{ \frac{1}{2}}+ (\ell+1) \frac{2\alpha   }{n}\left(\infnorm{\hc} + 1 \right)
\end{align*}
as desired.
\end{proof}

\subsection{Bernstein-like inequalities}
\label{sec:AMDP_bernstein_bounds}

First we check the Bernstein-like inequality required for the proof of Theorem \ref{thm:AMDP_policy_eval}.

\begin{lem}
    \label{lem:simple_AMDP_bernstein_bound}
    Fix a policy $\pi$, and let $\hc = h^\pi - \left( \min_s h^\pi(s)\right) \one$. With probability at least $1 - \delta$, for all $k = 0, \dots, \left\lceil \log_2 \log_2 \left(\infnorm{\hc} + 4 \right) \right\rceil$ we have
    \begin{align*}
        \left|(\Phat_\pi - P_\pi) (\hc)^{\circ 2^k} \right| \leq \sqrt{\frac{\alpha \Var_{P_\pi}\left[(\hc)^{\circ 2^k} \right] }{n}} + \frac{\alpha  }{n} \infnorm{\hc}^{2^k} \one
    \end{align*}
    where $\alpha = 2 \log \left( \frac{3 SA   \log_2 \log_2 \left(\spannorm{h^\pi} + 4 \right) }{\delta }\right)$.
\end{lem}
\begin{proof}
    First, note that we only need to check this inequality for a fixed $k$, and then the desired result follows by taking a union bound and adjusting the failure probability. Fix $k$. Also fix $s \in \S$ and $a \in \A$. Using Bernstein's inequality (e.g. \cite[Theorem 3]{maurer_empirical_2009}), we have that with probability at least $1 - 2 \delta'$,
    \begin{align*}
        \left|(\Phat_{sa} - P_{sa}) (\hc)^{\circ 2^k} \right| &\leq \sqrt{\frac{2 \log \left( \frac{1}{\delta'}\right) \Var_{P_{sa}}\left[(\hc)^{\circ 2^k} \right] }{n}} + \frac{\log \left( \frac{1}{\delta'}\right)   }{ 3n} \infnorm{(\hc)^{\circ 2^k}} \\
        &\leq \sqrt{\frac{2 \log \left( \frac{1}{\delta'}\right) \Var_{P_{sa}}\left[(\hc)^{\circ 2^k} \right] }{n}} + \frac{\log \left( \frac{1}{\delta'}\right)   }{ 3n} \infnorm{\hc}^{2^k} \\
        &\leq \sqrt{\frac{2 \log \left( \frac{1}{\delta'}\right) \Var_{P_{sa}}\left[(\hc)^{\circ 2^k} \right] }{n}} + \frac{2 \log \left( \frac{1}{\delta'}\right)   }{ n} \infnorm{\hc}^{2^k}.
    \end{align*}
    Now taking a union bound over all possible $s \in \S$ and $a \in \A$, we have the elementwise inequality
    \begin{align}
        \left|(\Phat - P) (\hc)^{\circ 2^k} \right| 
        &\leq \sqrt{\frac{2 \log \left( \frac{1}{\delta'}\right) \Var_{P}\left[(\hc)^{\circ 2^k} \right] }{n}} + \frac{2 \log \left( \frac{1}{\delta'}\right)   }{ n} \infnorm{\hc}^{2^k} \one \label{eq:AMDP_eval_bernstein_elementwise}
    \end{align}
    with probability at least $1 - 2SA\delta'$. We can use this to obtain that
    \begin{align*}
        \left|(\Phat_\pi - P_\pi) (\hc)^{\circ 2^k} \right| &= \left|M^\pi(\Phat - P) (\hc)^{\circ 2^k} \right| \\
        & \leq M^\pi \left|(\Phat - P) (\hc)^{\circ 2^k} \right| \\
        & \leq  M^\pi \left( \sqrt{\frac{2 \log \left( \frac{1}{\delta'}\right) \Var_{P}\left[(\hc)^{\circ 2^k} \right] }{n}} + \frac{2 \log \left( \frac{1}{\delta'}\right)   }{ n} \infnorm{\hc}^{2^k} \one \right) \\
        & \leq   \sqrt{\frac{2 \log \left( \frac{1}{\delta'}\right) M^\pi \Var_{P}\left[(\hc)^{\circ 2^k} \right] }{n}} + \frac{2 \log \left( \frac{1}{\delta'}\right)   }{ n} \infnorm{\hc}^{2^k} \one \\
        &= \sqrt{\frac{2 \log \left( \frac{1}{\delta'}\right)  \Var_{P_\pi}\left[(\hc)^{\circ 2^k} \right] }{n}} + \frac{2 \log \left( \frac{1}{\delta'}\right)   }{ n} \infnorm{\hc}^{2^k} \one
    \end{align*}
    where we used Jensen's inequality for the first inequality step (since each row of $M^\pi$ is a probability distribution), then~\eqref{eq:AMDP_eval_bernstein_elementwise}, then Jensen's inequality again since $\sqrt{\cdot}$ is concave and the fact that $M^\pi \one = \one$. Now taking a union bound over all values of $k$, of which there are at most
    \begin{align*}
        1 + \left\lceil \log_2 \log_2 \left(\infnorm{\hc} + 4 \right) \right\rceil \leq 2 +  \log_2 \log_2 \left(\infnorm{\hc} + 4 \right) \leq 3   \log_2 \log_2 \left(\infnorm{\hc} + 4 \right),
    \end{align*}
    we can set $\delta' = \frac{\delta}{3 SA    \log_2 \log_2 \left(\infnorm{\hc} + 4 \right)}$ and $\alpha = 2 \log \frac{1}{\delta'}$ to complete the proof. (Note $\infnorm{\hc} = \spannorm{h^\pi}$.)
\end{proof}

Now we set out to check the Bernstein-like inequalities required for the proof of Theorem \ref{thm:AMDP_anchored_nopert}. While this could be done by essentially copying the arguments of Lemma \ref{lem:LOO_DMDP_bernstein_bound} but replacing $\Phat$ with $\Phatanc$, we can instead reuse Lemma \ref{lem:LOO_DMDP_bernstein_bound} more directly. The following Lemma \ref{lem:anchored_var_param_bound} will help us do so.

\begin{lem}
\label{lem:anchored_var_param_bound}
Let $\pi$ be an arbitrary policy, $P$ be an arbitrary MDP transition matrix, let $\widetilde{P} = (1-\eta)P + \eta 1 e_{s_0}^\top$ be an anchored version of $P$, and let $x \in \R^S$. Then
    \begin{align*}
        \Var_{P_\pi} \left[ x \right] &\leq \frac{1}{1-\eta}\Var_{\widetilde{P}_\pi} \left[ x \right]. 
    \end{align*}
\end{lem}
\begin{proof}
Since the desired inequality is an elementwise inequality, it suffices to show for an arbitrary entry $s$. Thus let $p = (P_\pi)_s$ and $\widetilde{p} = (\widetilde{P}_\pi)_s$ be row vectors denoting the $s$th row of $P_\pi$ and $\widetilde{p}$, respectively. Note that $\widetilde{p} = (1-\eta)p + \eta e_{s_0}^\top$. Then we can calculate
    \begin{align*}
         \left(\Var_{\widetilde{P}_\pi} \left[ x \right] \right)_s &= \widetilde{p} (x^{\circ 2}) - (\widetilde{p} x)^2 \\
         &= (1-\eta)p  (x^{\circ 2}) + \eta (x(s_0))^2 - ((1-\eta)px + \eta x(s_0))^2 \\
         &= (1-\eta)p  (x^{\circ 2}) + \eta (x(s_0))^2 - (1-\eta)^2 (px)^2 - \eta^2 (x(s_0))^2 - 2 \eta (1-\eta) (px) x(s_0)\\
         &= (1-\eta) \left(p  (x^{\circ 2}) - (px)^2 \right) + \eta (x(s_0))^2 + \eta(1-\eta) (px)^2 - \eta^2 (x(s_0))^2 - 2 \eta (1-\eta) (px) x(s_0) \\
         &= (1-\eta) \left(p  (x^{\circ 2}) - (px)^2 \right) + \eta (1-\eta) \left( (x(s_0))^2 + (px)^2 - 2 (px) x(s_0) \right) \\
         & \geq (1-\eta) \left(p  (x^{\circ 2}) - (px)^2 \right) \\
         & = (1-\eta) \left(\Var_{P_\pi} \left[ x \right] \right)_s
    \end{align*}
    where the inequality step is by the AM-GM inequality.
\end{proof}

Using the above lemma, as well as the connection between discounted value functions and the bias functions in anchored MDPs, we are able to repurpose Lemma \ref{lem:LOO_anch_AMDP_bernstein_bound} to verify the Bernstein-like inequality conditions for the bias function of a near-optimal policy in an AMDP using anchoring.

\begin{lem}
    \label{lem:LOO_anch_AMDP_bernstein_bound}
    If $n \geq 4$, then with probability at least $1-\delta$, for all $\pihat$ which satisfy $\infnorm{\Vhat_{1-\frac{1}{n}}^{\pihat} - \Vhatstar_{1-\frac{1}{n}}} \leq \frac{1}{n}$, letting $\hc = \hhatanc^{\pihat} - \left(\min_s \hhatanc^{\pihat}(s)\right)\one$, for all $k = 0, \dots, \left\lceil\log_2 \log_2 \left( \spannorm{\hhatanc^{\pihat}} + 4\right) \right\rceil$, we have
    \begin{align*}
        \left| \left(\Phatanc_{\pihat} - P_{\pihat}\right) \left( \hc \right)^{\circ 2^k} \right|
    &\leq \sqrt{\frac{\alpha   \Var_{\Phatanc_{\pihat}} \left[ \left(\hc \right)^{\circ 2^k}  \right]    }{n }} + \frac{ \alpha \cdot 2^k }{n}  \left( \infnorm{\hc} + 1\right)^{2^k}   \one
    \end{align*}
where $\alpha = 18 \log \left(12 \frac{SAn^3}{\delta} \right)$.
\end{lem}
\begin{proof}
    First, we can directly use Lemma \ref{lem:LOO_DMDP_bernstein_bound} to obtain that with probability at least $1 - \delta$, for all $\pihat$ which satisfy $\infnorm{\Vhat_{1-\frac{1}{n}}^{\pihat} - \Vhatstar_{1-\frac{1}{n}}} \leq \frac{1}{n}$, we have
    \begin{align}
        \left| \left(\Phat_{\pihat} - P_{\pihat}\right) \left( \Vc \right)^{\circ 2^k} \right|
    &\leq \sqrt{\frac{\alpha'   \Var_{\Phat_{\pihat}} \left[ \left(\Vc \right)^{\circ 2^k}  \right]    }{n }} + \frac{ \alpha' \cdot 2^k }{n}  \left( \infnorm{\Vc} + 1\right)^{2^k}   \one \label{eq:LOO_AMDP_bernstein_bound_1}
    \end{align}
    for all $k = 0, \dots, \left\lceil\log_2 \log_2 \left( \spannorm{\Vhat_{1-\frac{1}{n}}^{\pihat}} + 4\right) \right\rceil$, where $\Vc = \Vhat_{1 - \frac{1}{n}}^{\pihat} - \left(\min_s \Vhat_{1 - \frac{1}{n}}^{\pihat}(s)\right)\one$ and
    \[
    \alpha' = 16 \log \left(12 \frac{SAn}{(1-(1-\frac{1}{n}))^2 \delta} \right) = 16 \log \left(12 \frac{SAn^3}{ \delta} \right).
    \]
    Now using Lemma \ref{lem:anchoring_optimality_properties}, we have that $\hhatanc^{\pihat} = \Vhat_{1- \frac{1}{n}}^{\pihat} - c \one$ for some $c$, which immediately implies that $\Vc = \hc$ and also that $\spannorm{\Vhat_{1-\frac{1}{n}}^{\pihat}} = \spannorm{\hhatanc^{\pihat}}$. Applying these facts to~\eqref{eq:LOO_AMDP_bernstein_bound_1} we obtain that (under the same event)
    \begin{align}
        \left| \left(\Phat_{\pihat} - P_{\pihat}\right) \left( \hc \right)^{\circ 2^k} \right|
    &\leq \sqrt{\frac{\alpha'   \Var_{\Phat_{\pihat}} \left[ \left(\hc \right)^{\circ 2^k}  \right]    }{n }} + \frac{ \alpha' \cdot 2^k }{n}  \left( \infnorm{\hc} + 1\right)^{2^k}   \one \label{eq:LOO_AMDP_bernstein_bound_2}
    \end{align}
    for all $k = 0, \dots, \left\lceil\log_2 \log_2 \left( \spannorm{\hhatanc^{\pihat}} + 4\right) \right\rceil$. Therefore it remains to replace $\Phat$ with $\Phatanc$ within~\eqref{eq:LOO_AMDP_bernstein_bound_2}. Fix $k$. First notice that 
    \[
        \Phatanc_\pi - P_\pi = (1-\frac{1}{n})\Phat_\pi + \frac{1}{n}\one e_{s_0}^\top - P_\pi = (1-\frac{1}{n}) \left(\Phat_\pi - P_\pi \right) + \frac{1}{n} \left( \one e_{s_0}^\top - P_\pi\right)
    \]
    so
    \begin{align}
         \left| \left(\Phatanc_{\pihat} - P_{\pihat}\right) \left( \hc \right)^{\circ 2^k} \right| 
         &\leq \left( 1 - \frac{1}{n}\right)\left| \left(\Phat_{\pihat} - P_{\pihat}\right) \left( \hc \right)^{\circ 2^k} \right| + \frac{1}{n} \infinfnorm{\one e_{s_0}^\top - P_\pi} \infnorm{\left( \hc \right)^{\circ 2^k}} \one  \nonumber \\
         & \leq \left( 1 - \frac{1}{n}\right)\left| \left(\Phat_{\pihat} - P_{\pihat}\right) \left( \hc \right)^{\circ 2^k} \right| + \frac{2}{n}  \infnorm{\hc}^{2^k} \one. \label{eq:LOO_AMDP_bernstein_bound_3}
    \end{align}
    Also, we can use Lemma \ref{lem:anchored_var_param_bound} to obtain that
    \begin{align}
        \Var_{\Phat_{\pihat}} \left[ \left(\Vc \right)^{\circ 2^k}  \right] & \leq \frac{1}{1- \frac{1}{n}} \Var_{\Phatanc_{\pihat}} \left[ \left(\Vc \right)^{\circ 2^k}  \right].\label{eq:LOO_AMDP_bernstein_bound_4}
    \end{align}
    Now combining inequalities~\eqref{eq:LOO_AMDP_bernstein_bound_2},~\eqref{eq:LOO_AMDP_bernstein_bound_3}, and~\eqref{eq:LOO_AMDP_bernstein_bound_4}, we obtain
    \begin{align*}
        \left| \left(\Phatanc_{\pihat} - P_{\pihat}\right) \left( \hc \right)^{\circ 2^k} \right| 
        & \leq \left( 1 - \frac{1}{n}\right)\left| \left(\Phat_{\pihat} - P_{\pihat}\right) \left( \hc \right)^{\circ 2^k} \right| + \frac{2}{n}  \infnorm{\hc}^{2^k} \one \\
        & \leq \left( 1 - \frac{1}{n}\right)\sqrt{\frac{\alpha'   \Var_{\Phat_{\pihat}} \left[ \left(\hc \right)^{\circ 2^k}  \right]    }{n }} + \frac{ \alpha' \cdot 2^k }{n}  \left( \infnorm{\hc} + 1\right)^{2^k}   \one + \frac{2}{n}  \infnorm{\hc}^{2^k} \one \\
        & \leq \left( 1 - \frac{1}{n}\right) \sqrt{\frac{\alpha'   \frac{1}{1- \frac{1}{n}} \Var_{\Phatanc_{\pihat}} \left[ \left(\Vc \right)^{\circ 2^k}  \right]  }{n }} + \frac{ \alpha' \cdot 2^k }{n}  \left( \infnorm{\hc} + 1\right)^{2^k}   \one + \frac{2}{n}  \infnorm{\hc}^{2^k} \one \\
        & \leq \sqrt{1 - \frac{1}{n}} \sqrt{\frac{\alpha'    \Var_{\Phatanc_{\pihat}} \left[ \left(\Vc \right)^{\circ 2^k}  \right]  }{n }} + \frac{ \alpha' \cdot 2^k }{n}  \left( \infnorm{\hc} + 1\right)^{2^k}   \one + \frac{2}{n}  \left( \infnorm{\hc} + 1\right)^{2^k} \one \\
        & \leq  \sqrt{\frac{\alpha'    \Var_{\Phatanc_{\pihat}} \left[ \left(\Vc \right)^{\circ 2^k}  \right]  }{n }} + \frac{ (\alpha' + 2) \cdot 2^k }{n}  \left( \infnorm{\hc} + 1\right)^{2^k}   \one.
    \end{align*}
    Therefore we can set $\alpha = 18 \log \left(12 \frac{SAn}{(1-\gamma)^2} \right) \geq \alpha' + 2$ and obtain the desired conclusion.
\end{proof}

Now we check the Bernstein-like inequalities used within the proof of Theorem \ref{thm:AMDP_plugin_thm}.
\begin{lem}
    \label{lem:LOO_AMDP_bernstein_bound_plug_in}
    Suppose $\pihstar$ is a bias-optimal policy in the AMDP $(\Phat, r)$.
    If $n \geq 4$, then with probability at least $1-\delta$, letting $\hc = \hhat^\star - \left(\min_s \hhat^{\star}(s)\right)\one$, for all $k = 0, \dots, \left\lceil\log_2 \log_2 \left( \spannorm{\hhat^{\star}} + 4\right) \right\rceil$, we have
    \begin{align*}
        \left| \left(\Phat_{\pihstar} - P_{\pihstar}\right) \left( \hc \right)^{\circ 2^k} \right|
    &\leq \sqrt{\frac{16 \log \left(\frac{18\cdot 8 SA n }{ \delta (1-\empbod)^3} \right)   \Var_{\Phat_{\pihstar}} \left[ \left(\hc \right)^{\circ 2^k}  \right]    }{n }} + \frac{ 16 \log \left(\frac{18\cdot 8 SA n }{ \delta (1-\empbod)^3} \right) \cdot 2^k }{n}  \left( \infnorm{\hc} + 1\right)^{2^k}   \one.
    \end{align*}
\end{lem}
\begin{proof}
    To handle the fact that $\empbod$ is random, we will prove a version of the inequality for each $\gamma$ such that $\frac{1}{1-\gamma} = 2^m$ for some integer $m \geq 0$, and then we will adjust the failure probability for each $m$ so that the overall failure probability is bounded by $\delta$. First, we fix $\gamma$, and we seek to show that with probability at least $1 - \delta'$, for all vectors $x \in \R^S$ such that $\infnorm{x - \Vhat_\gamma^\star} \leq \frac{1}{n}$, letting $\ox = x - (\min_s x(s))\one$, we have that
    \begin{align}
        \left| \left(\Phat_{\pihstar} - P_{\pihstar}\right) \left( \ox \right)^{\circ 2^k} \right|
    &\leq \sqrt{\frac{16 \log \left(\frac{18 SA n}{(1-\gamma)^2\delta'} \right)   \Var_{\Phat_{\pihstar}} \left[ \left(\ox \right)^{\circ 2^k}  \right]    }{n }} + \frac{ 16 \log \left(\frac{18 SA n}{(1-\gamma)^2\delta'} \right) \cdot 2^k }{n}  \left( \infnorm{\ox} + 1\right)^{2^k}   \one \label{eq:LOO_AMDP_bernstein_bound_plug_in_1}
    \end{align}
    for all $k = 0, \dots, \left\lceil\log_2 \log_2 \left( \spannorm{x} + 4\right) \right\rceil$. 
    We will argue that~\eqref{eq:LOO_AMDP_bernstein_bound_plug_in_1} follows from an identical argument to Lemma \ref{lem:LOO_DMDP_bernstein_bound}. Specifically, we will argue that we can replace $\Vc =\Vhat^{\pihat} - \left( \min_s \Vhat^{\pihat}(s) \right) \one$ within the proof of Lemma \ref{lem:LOO_DMDP_bernstein_bound} by $\ox$.
    We observe that the proof of Lemma \ref{lem:LOO_DMDP_bernstein_bound} only uses the following properties of the vector $\Vhat^{\pihat}$: that $\spannorm{\Vhat^{\pihat}} \leq \frac{1}{1-\gamma}$, and that $\infnorm{\Vhat^{\pihat} - \Vhat^{\star}} \leq \frac{1}{n}$. Furthermore, the bound $\spannorm{\Vhat^{\pihat}} \leq \frac{1}{1-\gamma}$ is only used to coarsely upper-bound the number of values of $k$ for which the desired inequality must be checked. We can instead use the fact that $\infnorm{x - \Vhat_\gamma^\star} \leq \frac{1}{n}$ to obtain that
    \[
    \spannorm{x} \leq \spannorm{\Vhat_\gamma^\star} + \spannorm{\Vhat_\gamma^\star - x} \leq \spannorm{\Vhat_\gamma^\star} + 2\infnorm{\Vhat_\gamma^\star - x} \leq \frac{1}{1-\gamma} + \frac{2}{n} \leq \frac{1}{1-\gamma} + 1 \leq \frac{2}{1-\gamma}.
    \]
    Thus repeating an argument similar to~\eqref{eq:LOO_ineq_bound_on_ks}, using that $\log_2 \log_2 (x+4) \leq 2x$ for $x \geq 1$, we can bound
    \begin{align*}
    \left\lceil\log_2 \log_2 \left( \spannorm{x} + 4\right) \right\rceil &\leq \left\lceil\log_2 \log_2 \left( \frac{2}{1-\gamma} + 4\right) \right\rceil \\
     &\leq 1 + \log_2 \log_2 \left( \frac{2}{1-\gamma} + 4\right) \\
     &\leq 1 + 4 \frac{1}{1-\gamma} \\
     & \leq 5 \frac{1}{1-\gamma} 
    \end{align*}
    and thus (counting $k = 0$) there are $\leq 1 + 5 \frac{1}{1-\gamma} \leq \frac{6}{1-\gamma}$ values of $k$ to check the inequality for if we check it for all values up to the upper bound $5 \frac{1}{1-\gamma}$. Comparing with the bound~\eqref{eq:DMDP_LOO_bernstein_alpha_bound}, this will cause us to obtain a factor of
    \begin{align*}
        16 \log \left(\frac{3 SA |U|}{\delta'} 6 \frac{1}{1-\gamma} \right) & = 16 \log \left(\frac{18 SA n}{(1-\gamma)^2\delta'} \right)
    \end{align*}
    (rather than the $16 \log \left(\frac{12 SA n}{(1-\gamma)^2\delta'} \right)$ which appears in Lemma \ref{lem:LOO_DMDP_bernstein_bound}).
    The rest of the proof of Lemma \ref{lem:LOO_DMDP_bernstein_bound} only uses the fact that $\infnorm{\Vhat^{\pihat} - \Vhat^{\star}} \leq \frac{1}{n}$, and thus goes through unchanged if we replace $\Vhat^{\pihat}$ with $x$ (and thus also $\Vc$ with $\ox$), which, following the proof up to the bound~\eqref{eq:DMDP_LOO_bernstein_cond_wo_pihat}, yields
    \begin{align*}
        \left| \left(\Phat - P\right) \left( \ox \right)^{\circ 2^k} \right|
    &\leq \sqrt{\frac{16 \log \left(\frac{18 SA n}{(1-\gamma)^2\delta'} \right)   \Var_{\Phat} \left[ \left(\ox \right)^{\circ 2^k}  \right]    }{n }} + \frac{ 16 \log \left(\frac{18 SA n}{(1-\gamma)^2\delta'} \right) \cdot 2^k }{n}  \left( \infnorm{\ox} + 1\right)^{2^k}   \one
    \end{align*}
    for all $k = 0, \dots, \left\lceil\log_2 \log_2 \left( \spannorm{x} + 4\right) \right\rceil$. From here, we can use identical steps as to the end of the proof of the inequality~\eqref{eq:DMDP_LOO_bernstein_cond_add_pihat} within Lemma \ref{lem:LOO_DMDP_bernstein_bound} (but with $\pihstar$ rather than $\pihat$) to conclude~\eqref{eq:LOO_AMDP_bernstein_bound_plug_in_1} as desired. 

    Now applying~\eqref{eq:LOO_AMDP_bernstein_bound_plug_in_1} with $\frac{1}{1-\gamma} = 2^m$ and $\delta' = \frac{\delta}{2^m}$ for each $m = 1, 2, \dots$, and taking a union bound over all $m$, we obtain that with probability at least $1 - \sum_{m=1}^\infty \frac{\delta}{2^m} = 1 - \delta$, we have that for all integers $m \geq 1$, for all $x$ such that $\infnorm{x - \Vhat^\star_{1-2^{-m}}}$, that
    \begin{align}
        \left| \left(\Phat_{\pihstar} - P_{\pihstar}\right) \left( \ox \right)^{\circ 2^k} \right|
    &\leq \sqrt{\frac{16 \log \left(\frac{18 SA n}{(2^{-m})^2 \frac{\delta}{2^m}} \right)   \Var_{\Phat_{\pihstar}} \left[ \left(\ox \right)^{\circ 2^k}  \right]    }{n }} + \frac{ 16 \log \left(\frac{18 SA n}{(2^{-m})^2 \frac{\delta}{2^m}} \right)  \cdot 2^k }{n}  \left( \infnorm{\ox} + 1\right)^{2^k}   \one \nonumber \\
    &= \sqrt{\frac{16 \log \left(\frac{18 SA n 2^{3m}}{ \delta} \right)   \Var_{\Phat_{\pihstar}} \left[ \left(\ox \right)^{\circ 2^k}  \right]    }{n }} + \frac{ 16 \log \left(\frac{18 SA n 2^{3m}}{ \delta} \right)  \cdot 2^k }{n}  \left( \infnorm{\ox} + 1\right)^{2^k}   \one .\label{eq:LOO_AMDP_bernstein_bound_plug_in_2}
    \end{align}
    On this event, recalling we have defined $\empbod$ as the smallest discount factor such that for all $\gamma \geq \empbod$, there exists $c \in \R$ such that
    \begin{align}
        \infnorm{\Vhat_\gamma^\star - \hhat^\star - c \one} \leq \frac{1}{n}, \label{eq:bias_optimal_discount_cond_repeated}
    \end{align}
    we now define $\empbodd$ as the smallest $\gamma$ such that $\gamma \geq \empbod$ and also there exists an integer $m \geq 1$ such that $\frac{1}{1-\empbodd} = 2^m$. Thus we have $\frac{1}{1-\empbod} \leq \frac{1}{1-\empbodd} \leq \frac{2}{1-\empbod}$. Also since $\empbodd \geq \empbod$, by~\eqref{eq:bias_optimal_discount_cond_repeated} we have that there exists some (random) scalar $c$ such that
    \begin{align}
        \infnorm{\Vhat_{\empbodd}^\star - \hhat^\star - c \one} \leq \frac{1}{n}. \label{eq:bias_optimal_discount_cond_empbodd}
    \end{align}
    By~\eqref{eq:bias_optimal_discount_cond_empbodd}, we may apply~\eqref{eq:LOO_AMDP_bernstein_bound_plug_in_2} to $x = \hhat^\star + c \one$. Also note that for this choice of $x$, $\ox = x - ( \min_s x(s))\one = \hhat^\star + c \one - ( \min_s \hhat^\star(s) + c)\one =\hhat^\star -  ( \min_s \hhat^\star(s) )\one = \hc$. Also $\spannorm{x} = \spannorm{\hhat^\star}$. Thus, plugging these observations into~\eqref{eq:LOO_AMDP_bernstein_bound_plug_in_2}, we obtain (still on the aforementioned event) that
    \begin{align*}
        \left| \left(\Phat_{\pihstar} - P_{\pihstar}\right) \left( \hc \right)^{\circ 2^k} \right|
    &= \sqrt{\frac{16 \log \left(\frac{18 SA n }{ \delta (1-\empbodd)^3}  \right)   \Var_{\Phat_{\pihstar}} \left[ \left(\hc \right)^{\circ 2^k}  \right]    }{n }} + \frac{ 16 \log \left(\frac{18 SA n }{ \delta (1-\empbodd)^3} \right)  \cdot 2^k }{n}  \left( \infnorm{\hc} + 1\right)^{2^k}   \one \\
    & \leq  \sqrt{\frac{16 \log \left(\frac{18\cdot 8 SA n }{ \delta (1-\empbod)^3}  \right)   \Var_{\Phat_{\pihstar}} \left[ \left(\hc \right)^{\circ 2^k}  \right]    }{n }} + \frac{ 16 \log \left(\frac{18\cdot 8 SA n }{ \delta (1-\empbod)^3} \right)  \cdot 2^k }{n}  \left( \infnorm{\hc} + 1\right)^{2^k}   \one.
    \end{align*}
\end{proof}

\subsection{Proof of Theorem \ref{thm:AMDP_policy_eval}}
\label{sec:AMDP_policy_eval_pf}

\begin{proof}[Proof of Theorem \ref{thm:AMDP_policy_eval}]
    Combining Lemma \ref{lem:simple_AMDP_bernstein_bound} with Lemma \ref{lem:AMDP_full_var_param_bound}, we obtain that with probability at least $1 - \delta$,
    \begin{align}
        \infnorm{\rhohat^\pi - \rho^\pi} \leq 2 (\ell + 1) \left(\frac{\alpha \left(\spannorm{h^\pi} + 1\right)}{n} \right)^{ \frac{1}{2}}+ (\ell+1) \frac{2\alpha   }{n}\left(\spannorm{h^\pi} + 1 \right) \label{eq:AMDP_eval_pf_1}
    \end{align}
    where $\alpha = 2 \log \left( \frac{3 SA   \log_2 \log_2 \left(\spannorm{h^\pi} + 4 \right) }{\delta }\right)$ and $\ell = \lceil\log_2 \log_2 \left( \spannorm{h^\pi} + 4\right) \rceil$.
    We can additionally assume without loss of generality that $n \geq \alpha \left(\spannorm{h^\pi} + 1 \right)$, since otherwise the desired theorem conclusion still holds trivially since we always have $\infnorm{\rhohat^\pi - \rho^\pi} \leq 1$. Then since $\frac{\alpha  \left(\spannorm{h^\pi} + 1 \right) }{n} \leq 1$, we have $\left(\frac{\alpha \left(\spannorm{h^\pi} + 1\right)}{n} \right)^{ \frac{1}{2}} \geq \frac{\alpha \left(\spannorm{h^\pi} + 1\right)}{n}$, and then we can use this to simplify the bound~\eqref{eq:AMDP_eval_pf_1} to obtain
    \begin{align*}
        \infnorm{\rhohat^\pi - \rho^\pi} &\leq 2 (\ell + 1) \left(\frac{\alpha \left(\spannorm{h^\pi} + 1\right)}{n} \right)^{ \frac{1}{2}}+ (\ell+1) \frac{2\alpha   }{n}\left(\spannorm{h^\pi} + 1 \right) \\
        & \leq 2 (\ell + 1) \left(\frac{\alpha \left(\spannorm{h^\pi} + 1\right)}{n} \right)^{ \frac{1}{2}}+ 2(\ell+1) \left(\frac{\alpha \left(\spannorm{h^\pi} + 1\right)}{n} \right)^{ \frac{1}{2}} \\
        &= 4 (\left\lceil\log_2 \log_2 \left( \spannorm{h^\pi} + 4\right) \right\rceil + 1) \sqrt{\frac{2 \log \left( \frac{3 SA   \log_2 \log_2 \left(\spannorm{h^\pi} + 4 \right) }{\delta }\right)}{n}\left(\spannorm{h^\pi} + 1\right)} \\
        & \leq 12 \left(\log_2 \log_2 \left( \spannorm{h^\pi} + 4\right)\right)   \sqrt{\frac{2 \log \left( \frac{3 SA   \log_2 \log_2 \left(\spannorm{h^\pi} + 4 \right) }{\delta }\right)}{n}\left(\spannorm{h^\pi} + 1\right)} \\
        & \leq \sqrt{\frac{C_4 \log^3 \left( \frac{SAn}{\delta} \right)}{n} \left( \spannorm{h^\pi}+1 \right)}
    \end{align*}
    where in the final inequality we use the upper-bound $\spannorm{h^\pi} \leq n$ (which follows from $\frac{\alpha  \left(\spannorm{h^\pi} + 1 \right) }{n} \leq 1$) and choose a sufficiently large constant $C_4$.
\end{proof}

\subsection{Proof of Theorem \ref{thm:AMDP_anchored_nopert}}
\label{sec:AMDP_anch_nopert_pf}
First, we show that the result follows from bounding certain ``policy evaluation error'' terms. 
\begin{lem}
\label{lem:AMDP_anchored_nopert_decomposition}
    Under the conditions of Theorem \ref{thm:AMDP_anchored_nopert},
    \begin{align}
        \rho^{\pihat} & \geq \rho^{\pistar} - \left(\infnorm{\rhohatanc^{\pistar_{1 - \frac{1}{n}}} - \rhoanc^{\pistar_{1-\frac{1}{n}}}} + \infnorm{ \rhohatanc^{\pihat} - \rho^{\pihat}} + \frac{ \spannorm{h^\star} + 1/n }{n}\right)\one .
    \end{align}
\end{lem}
\begin{proof}
Note that by Lemma \ref{lem:anchored_AMDP_DMDP_opt_cond_relns} and the conditions on the $\SolveAMDP$ procedure used in the statement of Theorem \ref{thm:AMDP_anchored_nopert}, we have that $\rhohatanc^{\pihat} \geq \rhohatanc^{\star} - \frac{1}{n^2}\one$. Also we recall that $\pistar_{1 - \frac{1}{n}}$ is defined as the optimal policy for the DMDP $(P, r, 1-\frac{1}{n})$, and by Lemma \ref{lem:anchoring_optimality_properties} this policy has optimal gain in the anchored AMDP with transition matrix $\Panc = (1-\frac{1}{n})P + \frac{1}{n}\one e_{s_0}^\top$.  For notational convenience we let $\gamma = 1-\frac{1}{n}$ so that we can abbreviate $\pistar_\gamma = \pistar_{1 - \frac{1}{n}}$. Then we can calculate that
    \begin{align*}
        \rho^{\pihat} & \geq \rhohatanc^{\pihat} - \infnorm{\rhohatanc^{\pihat} - \rho^{\pihat}} \one &&\text{triangle inequality} \\
        & \geq \rhohatanc^{\star} - \frac{1}{n^2}\one - \infnorm{\rhohatanc^{\pihat} - \rho^{\pihat}} \one &&\text{$\rhohatanc^{\pihat} \geq \rhohatanc^{\star} - \frac{1}{n^2}\one$} \\
        & \geq \rhohatanc^{\pistar_{\gamma}} - \frac{1}{n^2}\one - \infnorm{\rhohatanc^{\pihat} - \rho^{\pihat}} \one &&\text{$\rhohatanc^{\star} \geq \rhohatanc^{\pistar_{\gamma}} $} \\
        & \geq \rhoanc^{\pistar_{\gamma}} - \infnorm{\rhohatanc^{\pistar_{\gamma}} - \rhoanc^{\pistar_{\gamma}}} \one - \frac{1}{n^2}\one - \infnorm{\rhohatanc^{\pihat} - \rho^{\pihat}} \one &&\text{triangle inequality} \\
        & \geq \rhoanc^{\pistar}  - \infnorm{\rhohatanc^{\pistar_{\gamma}} - \rhoanc^{\pistar_{\gamma}}} \one - \frac{1}{n^2}\one - \infnorm{\rhohatanc^{\pihat} - \rho^{\pihat}} \one &&\text{$\rhoanc^{\pistar_{1-\frac{1}{n}}} = \rhoanc^\star \geq \rhoanc^{\pistar}$} \\
        & \geq \rho^{\pistar} - \infnorm{\rhoanc^{\pistar} - \rho^{\pistar}}\one - \infnorm{\rhohatanc^{\pistar_{\gamma}} - \rhoanc^{\pistar_{\gamma}}} \one - \frac{1}{n^2}\one - \infnorm{\rhohatanc^{\pihat} - \rho^{\pihat}} \one &&\text{triangle inequality} \\
        & \geq \rho^{\pistar} - \frac{\spannorm{h^\star}}{n}\one - \infnorm{\rhohatanc^{\pistar_{\gamma}} - \rhoanc^{\pistar_{\gamma}}} \one - \frac{1}{n^2}\one - \infnorm{\rhohatanc^{\pihat} - \rho^{\pihat}} \one. &&\text{Lemma \ref{lem:anchoring_optimality_properties}, $\eta = \frac{1}{n}$}
    \end{align*}
\end{proof}

We remark that with very similar arguments we could replace the term $\infnorm{\rhohatanc^{\pistar_{1 - \frac{1}{n}}} - \rhoanc^{\pistar_{1-\frac{1}{n}}}}$ with the term $\infnorm{\rhohatanc^{\pistar} - \rhoanc^{\pistar}}$ or the term $\infnorm{\rhohatanc^{\pistar} - \rho^{\pistar}}$ and it would still be possible to carry out the arguments, however as will be seen shortly, the term $\infnorm{\rhohatanc^{\pistar_{1 - \frac{1}{n}}} - \rhoanc^{\pistar_{1-\frac{1}{n}}}}$ enables us to reuse bounds from our DMDP results.

Now we complete the proof of the theorem.

\begin{proof}[Proof of Theorem \ref{thm:AMDP_anchored_nopert}]
    By Lemma \ref{lem:AMDP_anchored_nopert_decomposition}, it suffices to bound the terms $\infnorm{\rhohatanc^{\pistar_{1 - \frac{1}{n}}} - \rhoanc^{\pistar_{1-\frac{1}{n}}}}$ and $\infnorm{ \rhohatanc^{\pihat} - \rho^{\pihat}}$ with high probability.

    First we handle the easier term $\infnorm{\rhohatanc^{\pistar_{1 - \frac{1}{n}}} - \rhoanc^{\pistar_{1-\frac{1}{n}}}}$. By similar observations as those used in the proof of Theorem \ref{thm:AMDP_anchored_perturbed}, this term can be directly related to a difference of discounted value functions using Lemma \ref{lem:anchoring_optimality_properties}. Lemma \ref{lem:anchoring_optimality_properties} shows that $\rhohatanc^{\pistar_{1-\frac{1}{n}}} = \frac{\Vhat_{1-\frac{1}{n}}^{\pistar_{1-\frac{1}{n}}}(s_0)}{n}$ and that $\rhoanc^{\pistar_{1-\frac{1}{n}}} = \frac{V_{1-\frac{1}{n}}^{\pistar_{1-\frac{1}{n}}}(s_0)}{n}$, which implies that
    \begin{align}
        \infnorm{\rhohatanc^{\pistar_{1-\frac{1}{n}}} - \rhoanc^{\pistar_{1-\frac{1}{n}}}} &= \left|\frac{\Vhat_{1-\frac{1}{n}}^{\pistar_{1-\frac{1}{n}}}(s_0)}{n} - \frac{V_{1-\frac{1}{n}}^{\pistar_{1-\frac{1}{n}}}(s_0)}{n} \right| \nonumber\\
        & \leq \frac{1}{n} \infnorm{\Vhat_{1-\frac{1}{n}}^{\pistar_{1-\frac{1}{n}}} - V_{1-\frac{1}{n}}^{\pistar_{1-\frac{1}{n}}}}. \label{eq:AMDP_anch_nopert_1}
    \end{align}
    Now we can reuse part of the proof of Theorem \ref{thm:DMDP_pert_thm} which bounds $\infnorm{\Vhat_{1-\frac{1}{n}}^{\pistar_{1-\frac{1}{n}}} - V_{1-\frac{1}{n}}^{\pistar_{1-\frac{1}{n}}}}$. Specifically, setting $\gamma = 1- \frac{1}{n}$, then it is shown in inequality~\eqref{eq:DMDP_pert_thm_eval_bound} from the proof of Theorem \ref{thm:DMDP_pert_thm} that with probability at least $1 - \delta$, we have
    \begin{align}
         \infnorm{ \Vhat_{\gamma}^{\pistar_\gamma} - V_{\gamma}^{\pistar_\gamma}} & \leq \frac{24 \log_2 \log_2 \left( \frac{1}{1-\gamma} + 4\right)  }{1-\gamma}\sqrt{\frac{ \alpha_1 \left( \spannorm{V_\gamma^{\pistar_\gamma}} + 1\right)}{n}} \nonumber \\
         &\leq n 24 \left(\log_2 \log_2 \left( n + 4\right)\right) \sqrt{\frac{ 2\alpha_1 \left( \spannorm{h^\star} + 1\right)}{n}} \label{eq:DMDP_eval_bound_for_anch}
    \end{align}
    where $\alpha_1 = 2 \log \left( \frac{6 S \log_2 \log_2 \left(\spannorm{V_\gamma^{\pistar_\gamma}}+4 \right)}{\delta} \right) $, and in the second inequality we used that $\spannorm{V_\gamma^{\pistar_\gamma}} = \spannorm{\hanc^{\star}} \leq 2 \spannorm{h^\star}$, both steps of which follow from Lemma \ref{lem:anchoring_optimality_properties} (the inequality step because $\rho^\star$ is constant).
    Combining~\eqref{eq:DMDP_eval_bound_for_anch} with~\eqref{eq:AMDP_anch_nopert_1}, we obtain
    \begin{align}
        \infnorm{\rhohatanc^{\pistar_{1-\frac{1}{n}}} - \rhoanc^{\pistar_{1-\frac{1}{n}}}} & \leq 24\left( \log_2 \log_2 \left( n + 4\right) \right) \sqrt{\frac{ 2\alpha_1 \left( \spannorm{h^\star} + 1\right)}{n}} .\label{eq:AMDP_anch_nopert_2}
    \end{align}

    Now we bound the term $\infnorm{ \rhohatanc^{\pihat} - \rho^{\pihat}}$. First note that by Lemma \ref{lem:anchored_AMDP_DMDP_opt_cond_relns}, the requirement~\eqref{eq:solveamdp_opt_cond_1} implies that $\infnorm{\Vhat_{1-\frac{1}{n}}^{\pihat} - \Vhatstar_{1-\frac{1}{n}}} \leq \frac{1}{n}$. Thus, if we assume for now that $n \geq 4$, the conditions of Lemma \ref{lem:LOO_anch_AMDP_bernstein_bound} are satisfied, and thus by combining it with Lemma \ref{lem:AMDP_full_var_param_bound}, we have that with (additional) failure probability at most $\delta$,
    \begin{align}
        \infnorm{\rhohatanc^{\pihat} - \rho^{\pihat}} \leq 2 (\ell + 1) \left(\frac{\alpha_2 \left(\spannorm{\hhatanc^{\pihat}} + 1\right)}{n} \right)^{ \frac{1}{2}}+ (\ell+1) \frac{2\alpha_2   }{n}\left(\spannorm{\hhatanc^{\pihat}} + 1 \right) \label{eq:AMDP_anch_nopert_3}
    \end{align}
    where $\alpha_2 = 18 \log \left(12 \frac{SAn^3}{\delta} \right)$ and $\ell = \left\lceil\log_2 \log_2 \left( \spannorm{\hhatanc^{\pihat}} + 4\right) \right\rceil$. Also note that the use of Lemma \ref{lem:AMDP_full_var_param_bound} requires that $\rhohatanc^{\pihat}$ is constant, which follows from Lemma \ref{lem:anchoring_optimality_properties}. By following arguments which are analogous to the bounds in the proof of Theorem \ref{thm:AMDP_policy_eval}, we can simplify~\eqref{eq:AMDP_anch_nopert_3} and obtain
    \begin{align}
        \infnorm{\rhohatanc^{\pihat} - \rho^{\pihat}} \leq 12 \left(\log_2 \log_2 \left( \spannorm{\hhatanc^{\pihat}} + 4\right)\right)   \sqrt{\frac{\alpha_2}{n}\left(\spannorm{\hhatanc^{\pihat}} + 1\right)}. \label{eq:AMDP_anch_nopert_4}
    \end{align}
    We can also assume without loss of generality that $\spannorm{\hhatanc^{\pihat}} \leq n$, since otherwise the RHS of~\eqref{eq:AMDP_anch_nopert_4} is greater than $1$ and so the inequality~\eqref{eq:AMDP_anch_nopert_4} still holds since trivially always $\infnorm{\rhohatanc^{\pihat} - \rho^{\pihat}} \leq 1$. Thus we can bound
    \[
        \log_2 \log_2 \left( \spannorm{\hhatanc^{\pihat}} + 4\right) \leq \log_2 \log_2 \left( n + 4\right) \leq \log 4 n \leq \alpha_2
    \]
    since $\log_2 \log_2 (x+4) \leq \log 4x$ for $x \geq 1$. Also since $\spannorm{V_\gamma^{\pistar_\gamma}} \leq \frac{1}{1-\gamma} = n$ and also $\log_2 \log_2 (x+4) \leq 2x$ for $x \geq 1$,
    \[
        2\alpha_1 = 4 \log \left( \frac{6 S \log_2 \log_2 \left(\spannorm{V_\gamma^{\pistar_\gamma}}+4 \right)}{\delta} \right) \leq 4 \log \left( \frac{12 S n }{\delta} \right) \leq \alpha_2.
    \]
    Using these bounds and combining inequalities~\eqref{eq:AMDP_anch_nopert_2} and~\eqref{eq:AMDP_anch_nopert_4} with Lemma \ref{lem:AMDP_anchored_nopert_decomposition}, we obtain that with probability at least $1 - 2\delta$,
    \begin{align}
        \rho^{\pihat} & \geq \rho^{\pistar} - \left(24 \alpha_2 \sqrt{\frac{ \alpha_2 }{n} \left( \spannorm{h^\star} + 1 \right) } + 12 \alpha_2  \sqrt{\frac{\alpha_2}{n}\left(\spannorm{\hhatanc^{\pihat}} + 1\right)} + \frac{ \spannorm{h^\star} + 1/n }{n}\right)\one. \label{eq:AMDP_anch_nopert_5}
    \end{align}
    We can also bound
    \[
        \frac{ \spannorm{h^\star} + 1/n }{n} \leq \frac{ \spannorm{h^\star} + 1 }{n} \leq \sqrt{\frac{ \spannorm{h^\star} + 1 }{n}}
    \]
    (where similarly to before we are assuming without loss of generality that $\spannorm{h^\star} +1\leq n$ for the last inequality). Combining this fact with~\eqref{eq:AMDP_anch_nopert_5}, using~\eqref{eq:solveamdp_opt_cond_1} to bound $\spannorm{\hhatanc^{\pihat}} \leq \spannorm{\hhatanc^{\star}} + \spannorm{\hhatanc^{\pihat} - \hhatanc^{\star}} \leq \spannorm{\hhatanc^{\star}} + 2\infnorm{\hhatanc^{\pihat} - \hhatanc^{\star}} \leq \spannorm{\hhatanc^{\star}} + \frac{2}{3n^2} \leq  \spannorm{\hhatanc^{\star}} + 1$, and using the fact that $\sqrt{a} + \sqrt{b} \leq 2\sqrt{a + b}$, we conclude
    \begin{align*}
        \rho^{\pihat} & \geq \rho^{\pistar} - \left(24 \alpha_2 \sqrt{\frac{ \alpha_2 }{n} \left( \spannorm{h^\star} + 1 \right) } + 12 \alpha_2  \sqrt{\frac{\alpha_2}{n}\left(\spannorm{\hhatanc^{\pihat}} + 1\right)} + \sqrt{\frac{ \spannorm{h^\star} + 1 }{n}} \right)\one \\
        & \geq \rho^{\pistar} - \left(25 \alpha_2 \sqrt{\frac{ \alpha_2 }{n} \left( \spannorm{h^\star} + 1 \right) } + 12 \alpha_2  \sqrt{\frac{\alpha_2}{n}\left(\spannorm{\hhatanc^{\pihat}} + 1\right)} \right)\one \\
        & \geq \rho^{\pistar} - \left(25 \alpha_2 \sqrt{\frac{ \alpha_2 }{n} \left( \spannorm{h^\star} + 1 \right) } + 12 \alpha_2  \sqrt{\frac{\alpha_2}{n}\left(\spannorm{\hhatanc^{\star}} + 2\right)} \right)\one \\
        &\geq \rho^{\pistar} - 50 \alpha_2 \sqrt{\frac{ \alpha_2}{n} \left( \spannorm{h^\star} + \spannorm{\hhatanc^{\star}} + 3\right) }\one \\
        & \geq \rho^{\pistar} - C_5 \sqrt{\frac{ \log \left(\frac{SAn}{\delta} \right)}{n} \left( \spannorm{h^\star} + \spannorm{\hhatanc^{\star}} + 1\right) }\one, 
    \end{align*}
    also adjusting $\delta$ make the total failure probability $\delta$ rather than $2\delta$. Also note that we assumed $n \geq 4$ to derive this bound, but the inequality is also trivially true if $n \leq 3$ (since always $\rho^{\pihat} \geq \rho^{\pistar} - \one$), so this assumption can be removed without changing the result.
    
\end{proof}

\subsection{Proof of Theorem \ref{thm:AMDP_best_of_both}}
\label{sec:AMDP_best_of_both_pf}

First we provide a helper lemma to show that optimal policies are still near-optimal in perturbed DMDPs.
\begin{lem}
\label{lem:pert_near_optimality}
    Let $P$ be any transition matrix and fix a discount factor $\gamma$. Let $r_1, r_2 \in \R^{SA}$ be two reward vectors. Let $V_{\gamma, r_1}^\pi$ denote the value function of policy $\pi$ in the DMDP $(P, r_1, \gamma)$, and likewise let $V_{\gamma, r_2}^\pi$ denote the value function of policy $\pi$ in the DMDP $(P, r_2, \gamma)$. Also let $\pistar_1$ denote the optimal policy in $(P, r_1, \gamma)$ and let $\pistar_2$ denote the optimal policy in $(P, r_2, \gamma)$. Let $V_{\gamma, r_1}^\star = V_{\gamma, r_1}^{\pistar_1}$ and $V_{\gamma, r_2}^\star = V_{\gamma, r_2}^{\pistar_2}$ denote the respective optimal value functions. Then
    \begin{align*}
        V_{\gamma, r_2}^{\pistar_1} & \geq V_{\gamma, r_2}^{\star} - 2\frac{\infnorm{r_1 - r_2}}{1-\gamma} \one.
    \end{align*}
\end{lem}
\begin{proof}
Using the definitions for value functions as well as the facts that $\infinfnorm{(I - \gamma P_\pi)^{-1}} = \frac{1}{1-\gamma}$ and $\infinfnorm{M^\pi} = 1$ for any policy $\pi$, we can calculate
    \begin{align*}
        V_{\gamma, r_2}^{\pistar_1} &= (I - \gamma P_{\pistar_1})^{-1} M^{\pistar_1} r_2 \\
        &= (I - \gamma P_{\pistar_1})^{-1} M^{\pistar_1} r_1 + (I - \gamma P_{\pistar_1})^{-1} M^{\pistar_1} (r_2 - r_1) \\
        & \geq (I - \gamma P_{\pistar_1})^{-1} M^{\pistar_1} r_1 - \infinfnorm{(I - \gamma P_{\pistar_1})^{-1}} \infinfnorm{M^{\pistar_1}} \infnorm{ r_2 - r_1} \one \\
        &= V_{\gamma, r_1}^\star - \frac{\infnorm{r_2 - r_1}}{1-\gamma}\one \\
        & \geq V_{\gamma, r_1}^{\pistar_2} - \frac{\infnorm{r_2 - r_1}}{1-\gamma} \one \\
        & = (I - \gamma P_{\pistar_2})^{-1} M^{\pistar_2} r_1 - \frac{\infnorm{r_2 - r_1}}{1-\gamma} \one \\
        & = (I - \gamma P_{\pistar_2})^{-1} M^{\pistar_2} r_2 + (I - \gamma P_{\pistar_2})^{-1} M^{\pistar_2} (r_1 - r_2) - \frac{\infnorm{r_2 - r_1}}{1-\gamma} \one \\ 
        & \geq  (I - \gamma P_{\pistar_2})^{-1} M^{\pistar_2} r_2 + \infinfnorm{(I - \gamma P_{\pistar_2})^{-1} } \infinfnorm{M^{\pistar_2}} \infnorm{ r_1 - r_2} \one - \frac{\infnorm{r_2 - r_1}}{1-\gamma} \one \\
        &= V_{\gamma, r_2}^{\star} - 2\frac{\infnorm{r_1 - r_2}}{1-\gamma} \one.
    \end{align*}
\end{proof}

\begin{proof}[Proof of Theorem \ref{thm:AMDP_best_of_both}]
    The desired conclusion follows immediately once we verify that the guarantees within Theorems \ref{thm:AMDP_anchored_nopert} and \ref{thm:AMDP_anchored_perturbed} both hold. The guarantees of Theorem \ref{thm:AMDP_anchored_perturbed} (regarding the performance of the perturbed empirical optimal policy) obviously hold, so our main task is to verify that Theorem \ref{thm:AMDP_anchored_nopert} can be applied under the desired assumptions, which will be done by showing that the perturbation level $\xi = \frac{1}{2n^2}$ is sufficiently small so that, with high probability, $\pihat$ (the exact Blackwell-optimal policy of $(\Phatanc, \widetilde{r})$) is also near-optimal for the \textit{unperturbed} AMDP $(\Phatanc, r)$.
    
    Instead of checking condition~\eqref{eq:solveamdp_opt_cond_1} on the optimality of $\pihat$ for the unperturbed AMDP (which could also be done, with smaller $\xi$ and more effort), we instead notice that the proof of Theorem \ref{thm:AMDP_anchored_nopert} only uses condition~\eqref{eq:solveamdp_opt_cond_1} to apply Lemma \ref{lem:anchored_AMDP_DMDP_opt_cond_relns}, which in turn verifies that~\eqref{eq:anc_AMDP_central_opt_cond} holds. We can thus instead directly check condition~\eqref{eq:anc_AMDP_central_opt_cond}, which we recall is $\infnorm{\Vhat^{\star}_{1-\frac{1}{n}} - \Vhat^{\pihat}_{1-\frac{1}{n}}}  \leq \frac{1}{n}$ or equivalently $\Vhat^{\pihat}_{1-\frac{1}{n}} \geq \Vhat^{\star}_{1-\frac{1}{n}} - \frac{1}{n} \one$. Applying Lemma \ref{lem:pert_near_optimality} (with $P = \Phat$, $r_2 = r$, $r_1 = \widetilde{r}$, $\gamma = 1-\frac{1}{n}$, and thus $\pistar_1$ is equal to $\pihat$ since by Lemma \ref{lem:anchoring_optimality_properties} $\pihat$ is also optimal for the DMDP $(\Phat, \widetilde{r}, 1-\frac{1}{n})$), we immediately obtain that
    \[
        \Vhat^{\pihat}_{1-\frac{1}{n}} \geq \Vhat^{\star}_{1-\frac{1}{n}} - \frac{2 \infnorm{\widetilde{r} - r}}{1 - \left(1 - \frac{1}{n} \right)} \one \geq \Vhat^{\star}_{1-\frac{1}{n}} - 2 n \xi \one \geq \Vhat^{\star}_{1-\frac{1}{n}} - \frac{1}{n} \one
    \]
    since by construction $\infnorm{\widetilde{r} - r} \leq \xi$ and $\xi = \frac{1}{2n^2}$. Now the desired result follows from applying both Theorems \ref{thm:AMDP_anchored_nopert} and \ref{thm:AMDP_anchored_perturbed}, and simplifying the constant and log factors (in particular, applying the union bound to bound the failure probability by $2 \delta$ and then adjusting the failure probability and absorbing this factor of $2$, as well as using the fact that we have chosen $\xi = \frac{1}{2n^2}$).
\end{proof}

\subsection{Proof of Theorem \ref{thm:AMDP_plugin_thm}}
\label{sec:AMDP_plugin_pf}
\begin{proof}[Proof of Theorem \ref{thm:AMDP_plugin_thm}]
    For consistency with Lemma \ref{lem:LOO_AMDP_bernstein_bound_plug_in} we will use $\pihstar$ rather than $\pihat$ to denote the bias-optimal policy of $(\Phat, r)$ which is returned by $\SolveAMDP$. We have that
    \begin{align}
        \rho^{\pihstar} & \geq \rhohat^{\pihstar} + \infnorm{\rhohat^{\pihstar} - \rho^{\pihstar}} \one \nonumber\\
        & \geq \rhohat^{\pistar} + \infnorm{\rhohat^{\pihstar} - \rho^{\pihstar}} \one \nonumber \\
        & \geq \rho^{\pistar}+ \infnorm{\rhohat^{\pistar} - \rho^{\pistar} }\one + \infnorm{\rhohat^{\pihstar} - \rho^{\pihstar}} \one \label{eq:MDP_plugin_thm_pf_1}
    \end{align}
    where we used the fact that since $\pihstar$ is bias-optimal, it is also gain-optimal, and thus $\rhohat^{\pihstar} = \rhohat^{\star} \geq \rhohat^{\pistar}$. Thus it remains to bound the terms $\infnorm{\rhohat^{\pistar} - \rho^{\pistar} }$ and $\infnorm{\rhohat^{\pihstar} - \rho^{\pihstar}}$ with high probability. First, since $\rho^{\pistar} = \rho^\star$ is a constant vector, we can apply Theorem \ref{thm:AMDP_policy_eval} to bound
    \begin{align}
        \infnorm{\rhohat^{\pistar} - \rho^{\pistar} } \leq \sqrt{\frac{C_4 \log^3 \left( \frac{SAn}{\delta} \right)}{n} \left( \spannorm{h^{\pistar}}+1 \right)} \label{eq:MDP_plugin_thm_pf_2}
    \end{align}
    with probability at least $1 - \delta$. Next, to bound $\infnorm{\rhohat^{\pihstar} - \rho^{\pihstar}}$, we can combine Lemma \ref{lem:LOO_AMDP_bernstein_bound_plug_in} and Lemma \ref{lem:AMDP_full_var_param_bound} to obtain that with probability at least $1 - \delta$, if $\Phat$ is weakly communicating (which ensures $\rhohat^\star$ is a constant vector, as required by Lemma \ref{lem:AMDP_full_var_param_bound}) then
    \begin{align*}
        \infnorm{\rhohat^{\pihstar} - \rho^{\pihstar}} & \leq 2 (\ell + 1) \left(\frac{\alpha \left(\spannorm{\hhat^\star} + 1\right)}{n} \right)^{ \frac{1}{2}}+ (\ell+1) \frac{2\alpha   }{n}\left(\spannorm{\hhat^\star} + 1 \right) 
    \end{align*}
    where $\alpha = 16 \log \left(\frac{18\cdot 8 SA n }{ \delta (1-\empbod)^3}  \right)$ and $\ell = \lceil\log_2 \log_2 \left( \spannorm{\hhat^\star} + 4\right) \rceil$. As in previous proofs, we can simplify by assuming $\alpha\frac{ \spannorm{\hhat^\star}+1}{n} \leq 1$, in which case $ \left(\frac{\alpha \left(\spannorm{\hhat^\star} + 1\right)}{n} \right)^{ \frac{1}{2}} \geq \frac{\alpha   }{n}\left(\spannorm{\hhat^\star} + 1 \right)$ and also $\spannorm{\hhat^\star} \leq n$. This is because if actually $\alpha\frac{ \spannorm{\hhat^\star}+1}{n} > 1$, then since always $\infnorm{\rhohat^{\pihstar} - \rho^{\pihstar}} \leq 1$, the desired theorem conclusion follows trivially. Thus continuing with the case that $\alpha\frac{ \spannorm{\hhat^\star}+1}{n} \leq 1$, we can bound
    \[
        \ell \leq 1 + \log_2 \log_2 \left( \spannorm{\hhat^\star} + 4\right) \leq \log 4 \spannorm{\hhat^\star} \leq \log 4n 
    \]
    where we used that $\log_2 \log_2 (x+4) \leq \log 4x$ for $x \geq 1$ and also that $\spannorm{\hhat^\star} \leq n$. Using this in combination with $ \left(\frac{\alpha \left(\spannorm{\hhat^\star} + 1\right)}{n} \right)^{ \frac{1}{2}} \geq \frac{\alpha   }{n}\left(\spannorm{\hhat^\star} + 1 \right)$ to simplify, we have that
    \begin{align}
        \infnorm{\rhohat^{\pihstar} - \rho^{\pihstar}} & \leq 4 (\ell + 1) \sqrt{\frac{\alpha \left(\spannorm{\hhat^\star}+1 \right)}{n}}  \nonumber \\
        & \leq 4 (\log (4n) + 1) \sqrt{\frac{\alpha \left(\spannorm{\hhat^\star}+1 \right)}{n}} \nonumber \\
        & \leq 4\alpha \sqrt{\frac{\alpha \left(\spannorm{\hhat^\star}+1 \right)}{n}} \nonumber \\
        & \leq  \sqrt{\frac{16 \alpha^3 \left(\spannorm{\hhat^\star}+1 \right)}{n}} \label{eq:MDP_plugin_thm_pf_3}.
    \end{align}
    Plugging~\eqref{eq:MDP_plugin_thm_pf_2} and~\eqref{eq:MDP_plugin_thm_pf_3} into~\eqref{eq:MDP_plugin_thm_pf_1}, and also halving the failure probability parameter of each to get an overall failure probability of $\leq \delta$ by the union bound, we obtain
    \begin{align*}
        \rho^{\pihstar} & \geq \rho^{\pistar} - \sqrt{\frac{C_4 \log^3 \left( \frac{SAn}{\delta} \right)}{n} \left( \spannorm{h^{\pistar}}+1 \right)} -  \sqrt{\frac{16 \alpha^3 \left(\spannorm{\hhat^\star}+1 \right)}{n}}\one \\
        & \geq \rho^{\pistar} - \sqrt{\frac{C_4 \log^3 \left( \frac{SAn}{\delta} \right)}{n} \left( \spannorm{h^{\pistar}}+1 \right) + \frac{16 \alpha^3 \left(\spannorm{\hhat^\star}+1 \right)}{n}}\one \\
        & \geq \rho^{\pistar} -\sqrt{\frac{C_7 \log^3 \left( \frac{SAn}{\delta(1-\empbod)} \right)}{n} \left(\spannorm{h^\star} + \spannorm{\hhat^\star}+1 \right)}\one
    \end{align*}
    where we used that $\sqrt{a}+ \sqrt{b} \leq 2\sqrt{a + b}$ and then chose $C_7$ sufficiently large.
\end{proof}

\section{Corollaries for Bounded Diameter or Uniformly Mixing MDPs}
\label{sec:opt_diam_tmix_cors}

\subsection{Proof of Lemma \ref{lem:empirical_diameter_bound}}
\label{sec:empirical_diam_bound_pf}

\paragraph{Setup}
Let $s \in \S$, and we refer to it as the target state. Define the MDP $P^{\to s}$ and the reward vector $r^{\to s}$ by
\begin{align*}
        P^{\to s}_{s'a'} &= \begin{cases}
            e_{s}^\top & s' = s \\
            P_{s'a'} &  s' \neq s
        \end{cases} \\
        r^{\to s}(s',a') &= \begin{cases}
            1 & s' = s \\
            0 &  s' \neq s
        \end{cases}.
\end{align*}
    Note $e_{s}^\top$ is a vector which is all $0$ except for a $1$ in state $s$, meaning that the target state $s$ is absorbing in $P^{\to s}$. All other states have identical transitions as in $P$. Also the reward is $1$ in the target state $s$ and $0$ for all other states. Thus intuitively the optimal policy in $P^{\to s}$ should try to reach state $s$ as quickly as possible.

    In the rest of this subsection, we will exclusively use the discount factor $\gamma = 1 - \frac{1}{6D}$. We define $V_{\to s}^\pi$ to be the discounted value function in the MDP $(P^{\to s}, r^{\to s})$ with respect to a policy $\pi$, and we define $\Vhat_{\to s}^\pi$ to be the value function for the ``empirical'' MDP $(\Phat^{\to s}, r^{\to s})$ where
    \begin{align*}
        \Phat^{\to s}_{s'a'} &= \begin{cases}
            e_{s}^\top & s' = s \\
            \Phat_{s'a'} &  s' \neq s
        \end{cases}.
    \end{align*}
    We note that there are two equivalent sampling processes for generating $\Phat^{\to s}$. The first process, suggested by the above definition, is sampling $\Phat$ as usual, and then making state $s$ an absorbing state. The second process is to treat $P^{\to s}$ as if it were the true MDP transition matrix $P$, and then following the usual sampling process to generate $\Phat$ from $P$. These are equivalent because in the second process, with probability one all transitions from state $s$ will return to state $s$, and thus it will be absorbing with probability $1$. Using this correspondence, we will later be able to apply our Theorem \ref{thm:DMDP_main_thm} to $\Vhat_{\to s}^\pi$. 

    We make a few final definitions. Let $\widehat{D}$ be the diameter of $\Phat$. We also define the optimal value functions $V^\star_{\to s}$ and $\Vhat^\star_{\to s}$, and let $\pistar_{\to s}$ and $\pihstar_{\to s}$ be the policies which attain these optimal value functions, respectively.
    
    \paragraph{Correspondence between value functions and diameter}
    Next we establish some basic facts about the above-define value functions and their relationships to the diameters of $\Phat$ and $P$.

    Letting $\E$ be the expectation with respect to the original MDP $P$ and $T_{s} = \inf \{t \geq 0: S_t =s\}$ be the first hitting time of state $s$ (where we allow the ``hit'' to occur at time $0$), we have
    \begin{align*}
        V^\pi_{\to s}(s') &= \E_{s'}^\pi \sum_{t = 0}^\infty \gamma^t \ind\{T_{s} \leq t\} \\
        &= \E_{s'}^\pi \sum_{t = 0}^\infty \gamma^t\left(1 - \ind\{T_{s} > t\}\right) \\
        &\geq \frac{1}{1-\gamma} - \E_{s'}^\pi \sum_{t = 0}^\infty \ind\{T_{s} > t\} \\
        &= \frac{1}{1-\gamma} -  \sum_{t = 0}^\infty \P_{s'}^\pi\left(T_{s} > t\right) \\
        &= \frac{1}{1-\gamma} - \E_{s'}^\pi T_{s}
    \end{align*}
    where we used the fact that $\sum_{t=0}^\infty \gamma^t = \frac{1}{1-\gamma}$, the monotone convergence theorem to interchange the expectation and the infinite sum, and the fact that since $T_{s}$ takes values in the non-negative integers, $\sum_{t = 0}^\infty \P_{s'}^\pi\left(T_{s} > t\right) = \E_{s'}^\pi T_{s}$.

    Then
    \begin{align*}
        V^\star_{\to s}(s') &= \sup_{\pi} V^\pi_{\to s}(s') \geq \frac{1}{1-\gamma} - \inf_{\pi} \E_{s'}^\pi T_{s} = \frac{1}{1-\gamma} - D.
    \end{align*}
    Therefore $V^\star_{\to s} \geq \frac{1}{1-\gamma}\one - D\one$ elementwise. Since also  $V^\star_{\to s} \leq \frac{1}{1-\gamma}\one$, we have that $\spannorm{V^\star_{\to s}} \leq D$.

    %Letting $\Ehat$ denote the expectation with respect to $\Phat$, we similarly have
    %\begin{align}
    %    \Vhat^\star_{\to s}(s') &= \frac{1}{1-\gamma} - \inf_{\pi} \Ehat_{s'}^\pi T_{s} = \frac{1}{1-\gamma} - \widehat{D}. \label{eq:empirical_diameter_formula}
    %\end{align}

    Now we show another relationship, that if $\Vhat^\star_{\to s}$ is sufficiently large (for all $s$), then $\widehat{D}$ cannot be too large. We formalize this in the following lemma.
    \begin{lem}
    \label{lem:diameter_bound_alg}
        Suppose that $\gamma = 1-\frac{1}{6D}$ and for all $s \in \S$, $\Vhat_{\to s}^\star \geq \frac{2}{3}\frac{1}{1-\gamma}\one $. Then 
        \[
        \widehat{D} \leq 12 D \log(3).
        \]
    \end{lem}
    \begin{proof}
        As an intermediate step, we will show for all $s,s' \in \S$ that
        \begin{align}
            \widehat{\P}_{s'}^{\pihstar_{\to s}}\left(T_{s} \leq 6 D \log( 3)\right) \geq \frac{1}{2}. \label{eq:reaching_prob_lower_bound}
        \end{align}
    
        Fix a constant $k$ and states $s,s' \in \S$. Then (even if $k$ is not an integer) we have
        \begin{align}
            \Vhat^\star_{\to s}(s') & \leq \frac{1}{1-\gamma} \widehat{\P}_{s'}^{\pihstar_{\to s}}\left(T_{s} \leq k\right) + \frac{\gamma^k}{1-\gamma} \widehat{\P}_{s'}^{\pihstar_{\to s}}\left(T_{s} > k\right). \label{eq:value_bound_by_hitting_time_prob}
        \end{align}
        Now we want to choose $k$ so that only a small amount of value can be contributed from the $\frac{\gamma^k}{1-\gamma} \widehat{\P}_{s'}^{\pihstar_{\to s}}\left(T_{s} > k\right)$ term, specifically we will choose $k$ so that $\frac{\gamma^k}{1-\gamma} \leq \frac{1}{3}\frac{1}{1-\gamma}$. We calculate that
        \begin{align*}
            \gamma^k \leq \frac{1}{3} & \iff k \log(\gamma) \leq \log \frac{1}{3} \\
            & \iff k \geq \frac{\log \frac{1}{3}}{\log(\gamma)} = \frac{\log \frac{1}{3}}{\log \left( 1 - \frac{1}{6D}\right)}  = \frac{\log 3}{- \log \left( 1 - \frac{1}{6D}\right)}\\
            & \impliedby k \geq \frac{\log 3}{ \frac{1}{6D}} = 6 D \log( 3) 
        \end{align*}
        where in the final inequality we use the fact that $\log (1-x) \leq -x$ so $-\log (1-x) \geq x$. Then if we set $k = 6 D \log(3)$ in~\eqref{eq:value_bound_by_hitting_time_prob}, we have that 
        \begin{align*}
            \Vhat^\star_{\to s}(s') & \leq \frac{1}{1-\gamma} \widehat{\P}_{s'}^{\pihstar_{\to s}}\left(T_{s} \leq 6 D \log( 3)\right) + \frac{1}{3}\frac{1}{1-\gamma} \widehat{\P}_{s'}^{\pihstar_{\to s}}\left(T_{s} > 6 D \log( 3)\right) \\
            & =\frac{1}{1-\gamma} \widehat{\P}_{s'}^{\pihstar_{\to s}}\left(T_{s} \leq 6 D \log( 3)\right) + \frac{1}{3}\frac{1}{1-\gamma} \left(1 - \widehat{\P}_{s'}^{\pihstar_{\to s}}\left(T_{s} \leq 6 D \log( 3)\right) \right) \\
            &= \frac{1}{3}\frac{1}{1-\gamma} + \frac{2}{3}\frac{1}{1-\gamma} \widehat{\P}_{s'}^{\pihstar_{\to s}}\left(T_{s} \leq 6 D \log( 3)\right).
        \end{align*}
        Now since $\Vhat^\star_{\to s}(s') \geq \frac{2}{3} \frac{1}{1-\gamma}$ by assumption, we must have that $\widehat{\P}_{s'}^{\pihstar_{\to s}}\left(T_{s} \leq 6 D \log( 3)\right) \geq \frac{1}{2}$, establishing~\eqref{eq:reaching_prob_lower_bound}. 

        Next, for convenience let $b = \lfloor 6 D \log( 3)\rfloor $, and note that $\widehat{\P}_{s'}^{\pihstar_{\to s}}\left(T_{s} \leq 6 D \log( 3)\right) = \widehat{\P}_{s'}^{\pihstar_{\to s}}\left(T_{s} \leq b\right)$. Then for fixed $s,s' \in \S$, using the Markov property,
        \begin{align}
             \Ehat^{\pihstar_{\to s}}_{s'} T_s 
            & \leq  b \widehat{\P}_{s'}^{\pihstar_{\to s}}\left(T_{s} \leq b\right) + \widehat{\P}_{s'}^{\pihstar_{\to s}}\left(T_{s} > b\right)\left(b +\sup_{s'' \in \S} \Ehat^{\pihstar_{\to s}}_{s'} T_s \right) \nonumber \\
            & = b  + \widehat{\P}_{s'}^{\pihstar_{\to s}}\left(T_{s} > b\right) \sup_{s'' \in \S} \Ehat^{\pihstar_{\to s}}_{s''} T_s \nonumber \\
            & \leq b  + \frac{1}{2} \sup_{s'' \in \S} \Ehat^{\pihstar_{\to s}}_{s''} T_s \label{eq:recursive_diam_bound_alg}
        \end{align}
        using the fact that $\widehat{\P}_{s'}^{\pihstar_{\to s}}\left(T_{s} > b\right)  \leq \frac{1}{2}$ for any $s'$ from~\eqref{eq:reaching_prob_lower_bound}. Now taking the supremum over $s' \in \S$ in inequality~\eqref{eq:recursive_diam_bound_alg} and rearranging, we obtain that
        \begin{align*}
            \sup_{s' \in \S} \Ehat^{\pihstar_{\to s}}_{s'} T_s \leq 2b \leq 12 D \log(3).
        \end{align*}
        Since this holds for all $s \in \S$, we conclude that
        \begin{align*}
            \widehat{D} &= \sup_{s, s' \in \S} \inf_{\pi} \Ehat^{\pi}_{s'} T_{s} \\
            & \leq \sup_{s, s' \in \S} \Ehat^{\pihstar_{\to s}}_{s'} T_{s} \\
            & \leq 12 D \log(3)
        \end{align*}
        as desired.
    \end{proof}

\paragraph{Lower-bounding $\Vhat^\star_{\to s}$} 
Now it remains to complete the proof by setting $\gamma = 1 - \frac{1}{6D}$ and showing that for sufficiently large $n$, with high probability, we have $\Vhat^\star_{\to s} \geq \frac{2}{3}\frac{1}{1-\gamma}$ for all $s \in \S$, and thus checking the conditions  of Lemma \ref{lem:diameter_bound_alg}.

\begin{proof}[Proof of Lemma \ref{lem:empirical_diameter_bound}]
As observed earlier, we may apply our theorems on discounted MDPs to the $S$ MDPs $(\Phat^{\to s})_{s \in \S}$ (with discount factor $\gamma = 1 - \frac{1}{6D}$). Instead of using Theorem \ref{thm:DMDP_main_thm}, it is more direct to use the bound~\eqref{eq:DMDP_main_thm_eval_bound} which appears in the final stage of the proof of Theorem \ref{thm:DMDP_main_thm}. Taking a union bound over all $S$ MDPs, we obtain that with probability at least $1 -  S \delta$, for all $s \in \S$ we have
\begin{align}
        \infnorm{ \Vhat_{\to s}^{\pistar_{\to s}} - V_{\to s}^{\pistar_{\to s}} }
        & \leq \frac{24 \log_2 \log_2 \left( \frac{1}{1-\gamma} + 4\right)  }{1-\gamma}\sqrt{\frac{ \alpha \left( \spannorm{V_{\to s}^{\pistar_{\to s}}} + 1\right)}{n}}\label{eq:diameter_thm_eval_bound}
\end{align}
where $\alpha = 16 \log \left(\frac{12 SA n}{(1-\gamma)^2\delta}  \right) = 16 \log \left(\frac{12 \cdot 6^2 SA n D^2}{\delta}  \right)$. As observed above, $\spannorm{V_{\to s}^{\pistar_{\to s}}} \leq D$, so we can simplify~\eqref{eq:diameter_thm_eval_bound} to obtain (for all $s \in \S$)
\begin{align*}
    \infnorm{ \Vhat_{\to s}^{\pistar_{\to s}} - V_{\to s}^{\pistar_{\to s}} }
        & \leq 144 D \log_2 \log_2 \left( \frac{1}{1-\gamma} + 4\right)\sqrt{\frac{ \alpha \left( D + 1\right)}{n}} \\
        & \leq 144 D \log(24 D)\sqrt{\frac{ \alpha \left( D + 1\right)}{n}} \\
        & \leq  144 D \log(24 D) \sqrt{\frac{ 2\alpha D}{n}} \\
        & \leq D
\end{align*}
where we used that $\log_2 \log_2 (x+4) \leq \log 4x$ for $x > 1$, that $D\geq 1$, and in the final inequality we assume that $n \geq  2 \cdot 144^2 \alpha D \log^2(24D)$.

In this event, we thus have that
\begin{align*}
    \Vhat^\star_{\to s} \geq \Vhat^{\pistar_{\to s}}_{\to s} \geq V^{\pistar_{\to s}}_{\to s} - \infnorm{\Vhat^{\pistar_{\to s}}_{\to s} - V^{\pistar_{\to s}}_{\to s}} \one \geq \frac{1}{1-\gamma}\one - D\one - D\one = \frac{2}{3} \frac{1}{1-\gamma}\one
\end{align*}
since $\frac{1}{1-\gamma} = 6D$. We can thus combine this fact with Lemma \ref{lem:diameter_bound_alg} to conclude that $\widehat{D} \leq 12 D \log(3) \leq 14 D$. We conclude by choosing the constant $C_8$ such that $n \geq C_8 D \log^3 \left(\frac{SADn}{\delta} \right)$ implies that $n \geq  2 \cdot 144^2 \alpha D \log^2(24D)$.
\end{proof}

\subsection{Proof of optimal diameter-based complexity Corollary \ref{thm:diameter_complexity}}
\label{sec:diam_complexity_thm_pf}
\begin{proof}[Proof of Corollary \ref{thm:diameter_complexity}]
    This result follows from combining Theorem \ref{thm:AMDP_anchored_nopert} (on the performance of Algorithm \ref{alg:generic_amdp_plugin_alg} with anchoring and no perturbation) with Lemma \ref{lem:empirical_diameter_bound} which bounds the empirical diameter. First, by Theorem \ref{thm:AMDP_anchored_nopert}, with probability at least $1 - \delta$ we have that
    \begin{align}
        \rho^{\pihat} - \rho^\star & \leq  \sqrt{\frac{C_5 \log^3 \left( \frac{SAn}{\delta} \right)}{n} \left( \spannorm{h^\star} + \spannorm{\hhatanc^{\pihat}} + 1 \right) }\one. \label{eq:diam_cor_1}
    \end{align}
    It remains to bound the terms $\spannorm{h^\star}$ and $\spannorm{\hhatanc^{\pihat}}$ in terms of $D$. First, it is well-known that $\spannorm{h^\star} \leq D$ \cite{bartlett_regal_2012}. (As pointed out in \cite[Exercise 38.13]{lattimore_bandit_2020}, the proof of this bound provided in \cite{bartlett_regal_2012} is incomplete, but \cite{lattimore_bandit_2020} provide a complete proof.) Next, by applying Lemma \ref{lem:empirical_diameter_bound}, if $n \geq C_7 D \log^3 \left(\frac{SADn}{\delta} \right)$ then with additional failure probability at most $\delta$, the diameter of $\Phat$, $\widehat{D}$, is bounded by $14 D$. By condition~\eqref{eq:solveamdp_opt_cond_1}, we have $\infnorm{\hhatanc^{\pihat} - \hhatanc^{\star}} \leq \frac{1}{3n}$, which implies
    \[
        \spannorm{\hhatanc^{\pihat}} \leq \spannorm{\hhatanc^{\star}} + \spannorm{\hhatanc^{\pihat} - \hhatanc^{\star}} \leq \spannorm{\hhatanc^{\star}} + 2\infnorm{\hhatanc^{\pihat} - \hhatanc^{\star}} \leq \spannorm{\hhatanc^{\star}} + \frac{2}{3n} \leq \spannorm{\hhatanc^{\star}} + 1.
    \]
    Additionally, using Lemma \ref{lem:anchoring_optimality_properties} and then the bound from \cite{bartlett_regal_2012} again, and then the bound on $\widehat{D}$, we have that
    \[
        \spannorm{\hhatanc^{\star}} \leq 2\spannorm{\hhat^\star} \leq 2 \widehat{D} \leq 28 D.
    \]
    (Note that we can apply the bound $\spannorm{\hhatanc^{\star}} \leq 2\spannorm{\hhat^\star}$ from Lemma \ref{lem:anchoring_optimality_properties} because we are operating under the event that $\widehat{D} \leq 14 D$, which in particular implies that $\Phat$ is communicating so it an optimal gain $\rhohat^\star$ which is a constant vector.)
    Plugging these bounds into~\eqref{eq:diam_cor_1} and simplifying, we obtain that
    \begin{align*}
        \rho^{\pihat} - \rho^\star & \leq  \sqrt{\frac{C_5 \log^3 \left( \frac{SAn}{\delta} \right)}{n} \left( D + 28 D + 1 \right) }\one \\
        & \leq  \sqrt{\frac{C_5 \log^3 \left( \frac{SAn}{\delta} \right)}{n} 30 D }\one \\
        & \leq  \sqrt{\frac{\max \{30 C_5, C_8 \} \log^3 \left( \frac{SADn}{\delta} \right)}{n}  D }\one.
    \end{align*}
    The last inequality ensures that whenever the above bound is non-trivial (the RHS is $< 1$), then the condition $n \geq C_8 D \log^3 \left(\frac{SADn}{\delta} \right)$ will be satisfied. Therefore we can conclude by choosing $C_9$ so that $C_9  \log^3 \left( \frac{SADn}{\delta} \right) \geq \max \{30 C_5, C_8 \} \log^3 \left( \frac{SADn}{\delta/2} \right)$ (where we have added the of $2$ so that the total failure probability is bounded by $\delta$).
    
\end{proof}

\subsection{Proof of optimal mixing-based complexity Corollary \ref{thm:mixing_complexity}}
\label{sec:mixing_complexity_thm_pf}

First we collect some simple facts regarding the relationships between bias and discounted value functions and $\tmix$ for uniformly mixing MDPs. These results are completely standard but we provide their proofs for completeness.
\begin{lem}
    \label{lem:mixing_param_relationships}
    Let $P$ be any MDP which has a bounded uniform mixing parameter $\tmix$. Then for any Markovian deterministic policy $\pi$,
    \begin{enumerate}
        \item $\spannorm{h^\pi} \leq 3\tmix$.
        \item For any discount factor $\gamma \in [0,1)$, $\spannorm{V_\gamma^\pi} \leq 3 \tmix$.
    \end{enumerate}
\end{lem}
(We note that the results would also hold for any randomized policy as well if $\tmix$ were defined over randomized policies.)
\begin{proof}
    \begin{enumerate}
        \item A bound of this form is essentially claimed in \cite[Lemma 9, Proposition 10]{wang_near_2022}, although we believe that \cite[Proposition 10]{wang_near_2022} is not needed, since in a mixing MDP, in the Markov chain induced by a policy $\pi$, there must be a unique stationary distribution in order for the mixing time to be defined, and thus there should only be one closed recurrent class. We also believe that there may be a missing factor of $2$ in the proof of \cite[Lemma 9]{wang_near_2022} (in the second inequality step). Thus we choose to reprove this bound for completeness (and get a better constant), but we essentially follow their arguments.

        By \cite[Lemma 1]{jin_towards_2021}, by uniform mixing we have
        \begin{align}
            \infinfnorm{P_\pi^k - P_\pi^\infty } \leq 2^{- \lfloor k/\tmix \rfloor} \label{eq:jin_towards_bound}
        \end{align}
        for all $k \geq \tmix$. Also since $P_\pi$ must be aperiodic since the mixing time is finite, we have \citep{puterman_markov_1994} that
        \begin{align*}
            h^{\pi}(s) = \lim_{T \to \infty} \E^\pi_s \left[\sum_{t=0}^{T-1}R_t - T \rho^\pi(s) \right] =  \lim_{T \to \infty}\sum_{t=0}^{T-1} e_s^\top \left(P_\pi \right)^t r_\pi - T \rho^\pi(s).
        \end{align*}
        Therefore we have
        \begin{align*}
            \spannorm{h^\pi} &= \spannorm{ \lim_{T \to \infty} \sum_{t=0}^{T-1} \left(P_\pi \right)^t r_\pi - T \rho^\pi} \\
        &=  \lim_{T \to \infty} \spannorm{\sum_{t=0}^{T-1} \left(P_\pi \right)^t r_\pi - T \rho^\pi} \\
        &=  \lim_{T \to \infty} \spannorm{\sum_{t=0}^{T-1} \left(P_\pi^t  -  P_\pi^\infty  \right) r_\pi} \\
        &\leq  \lim_{T \to \infty} \sum_{t=0}^{T-1} \spannorm{ \left(P_\pi^t  -  P_\pi^\infty  \right) r_\pi} \\
        &=   \sum_{t=0}^{\infty} \spannorm{ \left(P_\pi^t  -  P_\pi^\infty  \right) r_\pi}.
        \end{align*}
        Now we can bound this using~\eqref{eq:jin_towards_bound}:
        \begin{align*}
            \sum_{t=0}^{\infty} \spannorm{ \left(P_\pi^t  -  P_\pi^\infty  \right) r_\pi} & \leq  \sum_{t=0}^{\tmix-1} \spannorm{ \left(P_\pi^t -  P_\pi^\infty  \right) r_\pi} +  \sum_{t=\tmix}^{\infty} \spannorm{ \left(P_\pi^t  -  P_\pi^\infty  \right) r_\pi} \\
            & = \sum_{t=0}^{\tmix-1} \spannorm{ P_\pi^t r_\pi} +  \sum_{t=\tmix}^{\infty} \spannorm{ \left(P_\pi^t -  P_\pi^\infty  \right) r_\pi} \\
            & \leq \tmix +  \sum_{t=\tmix}^{\infty} \spannorm{ \left(P_\pi^t -  P_\pi^\infty  \right) r_\pi} \\
            & \leq  \tmix + \sum_{t=\tmix}^{\infty} 2\infnorm{ \left(P_\pi^t -  P_\pi^\infty  \right) r_\pi} \\
            & \leq  \tmix + \sum_{t=\tmix}^{\infty} 2\infinfnorm{ P_\pi^t -  P_\pi^\infty  } \infnorm{r_\pi} \\
            & \leq  \tmix + 2\sum_{t=\tmix}^{\infty} 2^{- \lfloor t/\tmix \rfloor}  \\
            & = \tmix + 2\sum_{k=1}^{\infty} \tmix 2^{- k} \\
            &=  \tmix + 2\tmix = 3 \tmix
        \end{align*}
        where in the first equality step we used that $P_\pi^\infty r_\pi = \rho^\pi$ is a constant vector, then we used that $ \spannorm{ P_\pi^t r_\pi} \leq 1$ (since $1 \geq P_\pi^t r_\pi \geq 0$ elementwise).

        \item We repeat a very similar argument. By the Neumann series expansion,
        \begin{align*}
            V_\gamma^\pi &= (I - \gamma P_\pi)^{-1}r_\pi \\
            &= \sum_{t=0}^\infty \gamma^t P_\pi^t r_\pi.
        \end{align*}
        Since $\rho^\pi$ is a constant vector,
        \begin{align*}
            \spannorm{V_\gamma^\pi} &= \spannorm{V_\gamma^\pi - \frac{1}{1-\gamma} \rho^\pi} \\
            &= \spannorm{ \sum_{t=0}^\infty \gamma^t P_\pi^t r_\pi  - \frac{1}{1-\gamma} P_\pi^\infty r_\pi} \\
            &= \spannorm{ \sum_{t=0}^\infty \gamma^t\left( P_\pi^t   - P_\pi^\infty \right) r_\pi} \\
            &\leq \sum_{t=0}^\infty \spannorm{  \gamma^t\left( P_\pi^t   - P_\pi^\infty \right) r_\pi} \\
            &\leq \sum_{t=0}^\infty \spannorm{  \left( P_\pi^t   - P_\pi^\infty \right) r_\pi}
        \end{align*}
        where we used that $\sum_{t=0}^\infty \gamma^t =\frac{1}{1-\gamma}$ in the third equality. We can conclude by noting that in the previous part of this lemma we have already bounded this exact final term by $3 \tmix$.
    \end{enumerate}
\end{proof}

\begin{proof}[Proof of Corollary \ref{thm:mixing_complexity}]
    This follows immediately from Theorem \ref{thm:AMDP_anchored_perturbed}, since both $\spannorm{h^\star}$ and $\spannorm{\hanc^{\pihat}}$ can be bounded by $3 \tmix$ (using Lemma \ref{lem:mixing_param_relationships}, in particular noting that $\spannorm{\hanc^{\pihat}} = \spannorm{V_{1/n}}^{\pihat}$ by Lemma \ref{lem:anchoring_optimality_properties} so we can apply the second part of Lemma \ref{lem:mixing_param_relationships}.) Also $\tmix \geq 1$, so $\spannorm{h^\star} + \spannorm{\hanc^{\pihat}} + 1 \leq 7 \tmix$.
\end{proof}

\section{Other DMDP Results}
\label{sec:other_DMDP_results}
\begin{proof}[Proof of Theorem \ref{thm:DMDP_best_of_both}]
    Analogously to the proof of Theorem \ref{thm:AMDP_best_of_both}, we will show that for sufficiently small perturbation, the exact solution of the perturbed empirical DMDP $(\Phat, \widetilde{r}, \gamma)$ is a sufficiently small approximate solution of the unperturbed empirical DMDP $(\Phat, r, \gamma)$.

    Let $\pihat$ be the exact solution of the perturbed empirical DMDP $(\Phat, \widetilde{r}, \gamma)$. By applying Lemma \ref{lem:pert_near_optimality} with $r_1 = \widetilde{r}$, $r_2 = r$, and $P = \Phat$, we obtain that
    \[
    \Vhat^{\pihat} \geq \Vhat^\star - 2\frac{\infnorm{\widetilde{r} - r}}{1-\gamma} \geq \Vhat^\star - 2\frac{\xi}{1-\gamma} = \Vhat^\star - \frac{1}{n}.
    \]
    Therefore the policy $\pihat$ satisfies the conditions of both Theorem \ref{thm:DMDP_main_thm} and Theorem \ref{thm:DMDP_pert_thm}. Applying both theorems and taking the union bound, $\pihat$ satisfies both guarantees with probability at least $1 - 2\delta$. Adjusting the constants to make the overall failure probability $\delta$, and absorbing $\xi = \frac{1-\gamma}{2n}$ into the other terms within the $\log$ factor of the guarantee of Theorem \ref{thm:DMDP_pert_thm}, we can immediately conclude by choosing $C_{10}$ sufficiently large. 
\end{proof}

\begin{proof}[Proof of Lemma \ref{lem:small_sample_span_bound}]
    We will use Theorem \ref{thm:DMDP_main_thm} to prove this result. By inspecting the proof of the theorem, specifically the chain of inequalities~\eqref{eq:DMDP_chain_of_val_fn_bounds}, on the same event that the theorem holds (which is a probability at least $1-\delta$ event) we have that
    \begin{align}
        \infnorm{V^\star - \Vhat^\star} &\leq  \frac{1 }{1-\gamma}\sqrt{\frac{C_1  \log^3 \left( \frac{SAn}{(1-\gamma) \delta}\right) }{n}\left( \spannorm{V^{\star}} + \spannorm{\Vhat^{\pihat}} + 1\right) }. \label{eq:vstar_err_1}
    \end{align}
    Also the optimality condition on $\pihat$, equation~\eqref{eq:solvedmdp_opt_cond}, implies that $\Vhat^\star \geq \Vhat^{\pihat} \geq \Vhat^\star - \frac{1}{n}$, which in turn implies that $\spannorm{\Vhat^{\pihat}} \leq \spannorm{\Vhat^\star} + \frac{1}{n}$. Also by triangle inequality we have that
    \[
         \spannorm{\Vhat^\star} \leq \spannorm{V^\star} + \spannorm{\Vhat^\star - V^\star} \leq \spannorm{V^\star} + 2 \infnorm{V^\star - \Vhat^\star}.
    \]
    Plugging both of these bounds into~\eqref{eq:vstar_err_1}, we obtain that (again, on the event that Theorem \ref{thm:DMDP_main_thm} holds)
    \begin{align}
        \infnorm{V^\star - \Vhat^\star}  &\leq  \frac{1 }{1-\gamma}\sqrt{\frac{C_1  \log^3 \left( \frac{SAn}{(1-\gamma) \delta}\right) }{n}\left( \spannorm{V^{\star}} + \frac{1}{n} + \spannorm{V^\star} + 2 \infnorm{V^\star - \Vhat^\star} + 1\right) } \nonumber \\
        & \leq \frac{1 }{1-\gamma}\sqrt{\frac{C_1  \log^3 \left( \frac{SAn}{(1-\gamma) \delta}\right) }{n}\left( 2\spannorm{V^{\star}} + 2 \infnorm{V^\star - \Vhat^\star} + 2\right) } \nonumber \\
        & \leq \frac{1 }{1-\gamma}\sqrt{\frac{C_1  \log^3 \left( \frac{SAn}{(1-\gamma) \delta}\right) }{n}\left( 4\spannorm{h^{\star}} + 2 \infnorm{V^\star - \Vhat^\star} + 2\right) } \label{eq:vstar_err_2}
    \end{align}
    using that $\spannorm{V^\star} \leq 2 \spannorm{h^\star}$ \citep[Lemma 2]{wei_model-free_2020} in the last inequality, which holds since we assumed that $P$ is weakly communicating. Now squaring both sides and rearranging, we obtain
    \begin{align*}
        \infnorm{V^\star - \Vhat^\star}^2 - \frac{2C_1  \log^3 \left( \frac{SAn}{(1-\gamma) \delta}\right)}{n(1-\gamma)^2} \infnorm{V^\star - \Vhat^\star} - \frac{2C_1  \log^3 \left( \frac{SAn}{(1-\gamma) \delta}\right)}{n(1-\gamma)^2} \left(2\spannorm{h^{\star}} +1 \right) \leq 0.
    \end{align*}
    Using the larger root given by the quadratic formula for this polynomial in $\infnorm{V^\star - \Vhat^\star}$, we can bound
    \begin{align}
        &\infnorm{V^\star - \Vhat^\star} \nonumber\\
        &\leq \frac{C_1  \log^3 \left( \frac{SAn}{(1-\gamma) \delta}\right)}{n(1-\gamma)^2} + \frac{1}{2}\sqrt{\left( \frac{2C_1  \log^3 \left( \frac{SAn}{(1-\gamma) \delta}\right)}{n(1-\gamma)^2} \right)^2 + \frac{8C_1  \log^3 \left( \frac{SAn}{(1-\gamma) \delta}\right)}{n(1-\gamma)^2} \left(2\spannorm{h^{\star}} +1 \right) } \nonumber \\
        &\leq \frac{C_1  \log^3 \left( \frac{SAn}{(1-\gamma) \delta}\right)}{n(1-\gamma)^2} +  \frac{C_1  \log^3 \left( \frac{SAn}{(1-\gamma) \delta}\right)}{n(1-\gamma)^2}  + \sqrt{\frac{2C_1  \log^3 \left( \frac{SAn}{(1-\gamma) \delta}\right)}{n(1-\gamma)^2} \left(2\spannorm{h^{\star}} +1 \right) } \label{eq:vstar_err_3}
    \end{align}
    where in the second inequality we used that $\sqrt{a + b} \leq \sqrt{a} + \sqrt{b}$.
    Now if we assume that $n \geq \frac{C_1  \log^3 \left( \frac{SAn}{(1-\gamma) \delta}\right)}{(1-\gamma)^2 (\spannorm{h^{\star}} +1)}$, or equivalently that $\frac{C_1  \log^3 \left( \frac{SAn}{(1-\gamma) \delta}\right)}{n(1-\gamma)^2} \leq \spannorm{h^\star} + 1$, then plugging this into~\eqref{eq:vstar_err_3}, we obtain that
    \[
        \infnorm{V^\star - \Vhat^\star} \leq 2(\spannorm{h^\star} + 1) + \sqrt{2(\spannorm{h^\star} + 1)(2\spannorm{h^\star} + 1)} \leq 4 (\spannorm{h^\star} + 1)
    \]
    as desired.
\end{proof}

\begin{proof}[Proof of Theorem \ref{thm:span_based_without_knowledge}]
    From the choice of $\gamma$, we have that
    \begin{align*}
        \frac{1}{1-\gamma} = \sqrt{\frac{n}{C_1 \log^3(\frac{SAn^2}{\delta})}} \leq \sqrt{\frac{n}{C_1 \log^3(\frac{SAn}{(1-\gamma)\delta})}}
    \end{align*}
    (where for the inequality we use the coarse bound that $\frac{1}{1-\gamma} \leq n$). We can thus immediately check from this inequality that the condition of Lemma \ref{lem:small_sample_span_bound} is satisfied (note we also assume $P$ is weakly communicating), so we obtain that with probability at least $1 - \delta$, $\spannorm{\Vhat^\star} \leq 4(\spannorm{h^\star} + 1)$. Also by the condition~\eqref{eq:solvedmdp_opt_cond}, we have $\Vhat^\star \geq \Vhat^{\pihat} \geq \Vhat^{\star} - \frac{1}{n}$ so $\spannorm{\Vhat^{\pihat}} \leq \spannorm{\Vhat^\star} + \frac{1}{n}$.
    Plugging these bounds into the guarantee from Theorem \ref{thm:DMDP_main_thm}, as well as the bound $\spannorm{V^\star} \leq 2 \spannorm{h^\star}$ from \cite[Lemma 2]{wei_model-free_2020} (which holds since $P$ is weakly communicating), we obtain that
    \begin{align}
        \infnorm{V^{\pihat} - V^\star} &\leq  \frac{1 }{1-\gamma}\sqrt{\frac{C_1  \log^3 \left( \frac{SAn}{(1-\gamma) \delta}\right) }{n}\left( \spannorm{V^{\star}} + \spannorm{\Vhat^{\pihat}} + 1\right) } \nonumber \\
        & \leq  \frac{1 }{1-\gamma}\sqrt{\frac{C_1  \log^3 \left( \frac{SAn}{(1-\gamma) \delta}\right) }{n}\left(2\spannorm{h^\star} + \frac{1}{n} + 4 (\spannorm{h^\star} + 1) + 1\right) }  \nonumber \\
        & \leq  \frac{1 }{1-\gamma}\sqrt{\frac{6C_1  \log^3 \left( \frac{SAn}{(1-\gamma) \delta}\right) }{n}\left(\spannorm{h^\star} + 1\right) }.
    \end{align}
    Now, we pause to restate the the main AMDP-to-DMDP reduction result from \cite{wang_near_2022} (in a form closer more immediately useful for us):
    \begin{thm}{\cite[Theorem 1]{wang_near_2022}}
        If $P$ is weakly communicating, then for any policy $\pi$, we have
        \begin{align*}
            \rho^\star - \rho^{\pi} \leq (1-\gamma) \left(8\spannorm{h^\star} + 3 \infnorm{V_\gamma^{\pi} - V_\gamma^\star}\right)\one.
        \end{align*}
    \end{thm}
    Plugging our bound on $\infnorm{V^{\pihat} - V^\star}$ and our choice of $\gamma$ into this theorem, we obtain that
    \begin{align*}
        \rho^\star - \rho^{\pihat} & \leq 8\spannorm{h^\star} \sqrt{\frac{C_1 \log^3(\frac{SAn^2}{\delta})}{n}} + 3 \frac{1-\gamma}{1-\gamma} \sqrt{\frac{6C_1  \log^3 \left( \frac{SAn}{(1-\gamma) \delta}\right) }{n}\left(\spannorm{h^\star} + 1\right) } \one
    \end{align*}
    and by using the inequality $\sqrt{a} + \sqrt{b} \leq 2 \sqrt{a + b}$, that $\spannorm{h^\star} \leq \frac{1}{2}\spannorm{h^\star}^2 + \frac{1}{2}$ by AM-GM, and choosing $C_{11}$ appropriately, the RHS can be bounded by
    \begin{align*}
        \rho^\star - \rho^{\pihat} & \leq \sqrt{\frac{C_{11}\log^3(\frac{SAn}{\delta})}{n} \left( \spannorm{h^\star}^2 + 1\right)} \one
    \end{align*}
    as desired. (Note that the total failure probability is actually just $\leq \delta$ since by the proof of Lemma \ref{lem:small_sample_span_bound}, the event of Lemma \ref{lem:small_sample_span_bound} is contained within the event that the bound from Theorem \ref{thm:DMDP_main_thm} holds.)
\end{proof}

\section{DMDP Reduction Approach}
\label{sec:DMDP_reduction_approach}
Here we provide theorems with identical guarantees as to those of Theorems \ref{thm:AMDP_anchored_nopert} and \ref{thm:AMDP_anchored_perturbed}, but instead of requiring solutions to the anchored (resp., perturbed) emprical AMDP $(\Phatanc, r)$ (resp., $(\Phatanc, \widetilde{r})$), the optimality condition is expressed in terms of solutions to the empirical (resp., perturbed empirical) DMDPs $(\Phat, r, \gamma)$ (resp., $(\Phat, \widetilde{r}, \gamma)$) with $\gamma = 1 - \frac{1}{n}$. Thus, the same conclusions hold for the DMDP reduction approach using an effective horizon of $\frac{1}{1-\gamma} = n$, which does not require prior knowledge. Because of the close connection between anchoring and horizon-$n$-discounted reductions, the proofs are completely trivial.

\begin{thm}
    \label{thm:AMDP_anchored_nopert_DMDP_red}
    Suppose $P$ is weakly communicating. Set $\xi = 0$ and $\gamma = 1-\frac{1}{n}$ in Algorithm \ref{alg:generic_dmdp_plugin_alg}. Also suppose that $\SolveDMDP$ is guaranteed to return a policy $\pihat_\gamma$ satisfying
    \begin{align*}
       \infnorm{\Vhat^{\star}_{1-\frac{1}{n}} - \Vhat^{\pihat_\gamma}_{1-\frac{1}{n}}} & \leq \frac{1}{n}.
    \end{align*}
   Then with probability at least $1-\delta$,
    \begin{align*}
        \rho^{\pihat_\gamma} - \rho^\star & \leq  \sqrt{\frac{C_5 \log^3 \left( \frac{SAn}{\delta} \right)}{n} \left( \spannorm{h^\star} + \spannorm{\Vhat^\star_{1-\frac{1}{n}}} + 1 \right) }\one.
    \end{align*}
\end{thm}
\begin{proof}
    Note that the proof of Theorem \ref{thm:AMDP_anchored_nopert} uses Lemma \ref{lem:anchored_AMDP_DMDP_opt_cond_relns} to show that any $\pihat$ satisfying the optimality condition~\eqref{eq:solveamdp_opt_cond_1} which appears in the statement of Theorem \ref{thm:AMDP_anchored_nopert} also satisfies the condition~\eqref{eq:anc_AMDP_central_opt_cond}, and then the rest of the proof only uses the fact that $\pihat$ satisfies the condition~\eqref{eq:anc_AMDP_central_opt_cond}. Since the above requirement that $\infnorm{\Vhat^{\star}_{1-\frac{1}{n}} - \Vhat^{\pihat_\gamma}_{1-\frac{1}{n}}}  \leq \frac{1}{n}$ is exactly the condition~\eqref{eq:anc_AMDP_central_opt_cond}, the rest of the proof immediately goes through for $\pihat_\gamma$ in place of $\pihat$. Lastly, we can use Lemma \ref{lem:anchoring_optimality_properties} to obtain that $\spannorm{\hhatanc^{\star}} = \spannorm{\Vhat^\star_{1-\frac{1}{n}}}$ and thus replace the $\spannorm{\hhatanc^{\star}}$ term which appears in Theorem \ref{thm:AMDP_anchored_nopert}.
\end{proof}

\begin{thm}
\label{thm:AMDP_anchored_perturbed_DMDP_red}
    Suppose $P$ is weakly communicating. Set $\xi \in (0, \frac{1}{n}]$ and $\gamma = 1-\frac{1}{n}$ in Algorithm \ref{alg:generic_dmdp_plugin_alg}. Also suppose that the policy $\pihat_\gamma$ returned by $\SolveDMDP$ is guaranteed to be the exact discounted optimal policy of the DMDP $(\Phat, \widetilde{r}, \gamma)$.
    Then with probability at least $1 - \delta$,
    \begin{align*}
        \rho^\star - \rho^{\pihat_\gamma}  \leq \sqrt{\frac{C_3  \log^3 \left( \frac{SAn}{ \delta \xi}\right) }{n}\left( \spannorm{h^\star} + \spannorm{V_{1-\frac{1}{n}}^{\pihat_\gamma}} + 1\right) } \one.
    \end{align*}
\end{thm}
\begin{proof}
    As noted in the proof of Theorem \ref{thm:AMDP_anchored_perturbed}, under the event that the theorem's guarantee holds, the policy $\pihat$ which is the exact Blackwell-optimal policy of the AMDP $(\Phatanc, \widetilde{r})$ is identical to the exact discounted-optimal policy of the DMDP $(\Phat, \widetilde{r}, 1 - \frac{1}{n})$, so we immediately obtain the bound which appears in Theorem \ref{thm:AMDP_anchored_perturbed}. Then using Lemma \ref{lem:anchoring_optimality_properties}, we have $\spannorm{\hanc^{\pihat_\gamma}} = \spannorm{V_{1-\frac{1}{n}}^{\pihat_\gamma}}$, so we can replace the $\spannorm{\hanc^{\pihat_\gamma}}$ term which appears in Theorem \ref{thm:AMDP_anchored_perturbed}.
\end{proof}

\section{Proof of Theorem \ref{thm:plug_in_lower_bound}}
\label{sec:proof_of_plug_in_lower_bound}

\begin{proof}[Proof of Theorem \ref{thm:plug_in_lower_bound}]
First we provide the MDP, $P$ as well as an MDP $\Phat$ which has a constant probability of being sampled from $P$.
\begin{figure}[H]
\centering
\resizebox{0.49\textwidth}{!}{
\begin{tikzpicture}[ -> , >=stealth, shorten >=2pt , line width=0.5pt, node distance =2cm, scale=0.9]

\node [circle, draw] (one) at (-2 , 0) {1};
\node [circle, draw] (two) at (2 , 0) {2};
\node [circle, draw, fill, inner sep=0.03cm] (dot1) at (-0.6 , 1.5) {};
\node [circle, draw, fill, inner sep=0.03cm] (dot2) at (0.6 , -1.5) {};
\node [circle, draw, fill, inner sep=0.03cm] (dot3) at (-3.5 , 0) {};

\path (one) edge[-] [bend left] node [above] {$a=2, R=0$~~~~~~~~~~~~~~~~~~} (dot1) ;
\path (dot1) edge[dashed] [bend left] node [right] {$1-\frac{1}{n}$} (one);
\path (dot1) edge[dashed] [bend left] node [above] {$\frac{1}{n}$} (two);
\path (two) edge[-] [bend left] node [right] {$a=2, R=0$} (dot2) ;
\path (dot2) edge[dashed] [bend left] node [left] {$1-\frac{1}{n}$} (two);
\path (dot2) edge[dashed] [bend left] node [below] {$\frac{1}{n}$} (one);
\path (two) edge [loop right, looseness=15] node [right] {$a=1, R =\frac{1}{2}$}  (two) ;
\path (one) edge[-] [bend left] node [below] {$a=1, R = \frac{1}{2}+\frac{1}{n}$~~~~~~~~~~~} (dot3);
\path (dot3) edge[dashed] [bend left, out=120, in=90, looseness=2] node [above] {$\frac{1}{n}$} (two);
\path (dot3) edge[dashed] [bend left] node [above] {~~$1-\frac{1}{n}$} (one);
\node[below,font=\huge\bfseries] at (current bounding box.south) {$P$};
\end{tikzpicture}}
\resizebox{0.49\textwidth}{!}{
\begin{tikzpicture}[ -> , >=stealth, shorten >=2pt , line width=0.5pt, node distance =2cm, scale=0.9]

\node [circle, draw] (one) at (-2 , 0) {1};
\node [circle, draw] (two) at (2 , 0) {2};
\node [circle, draw, fill, inner sep=0.03cm] (dot1) at (-0.6 , 1.5) {};
\node [circle, draw, fill, inner sep=0.03cm] (dot2) at (0.6 , -1.5) {};
\node [circle, draw, fill, inner sep=0.03cm] (dot3) at (-3.5 , 0) {};

\path (one) edge[-] [bend left] node [above] {$a=2, R=0$~~~~~~~~~~~~~~~~~~} (dot1) ;
\path (dot1) edge[dashed] [bend left] node [right] {$1-\frac{1}{n}$} (one);
\path (dot1) edge[dashed] [bend left] node [above] {$\frac{1}{n}$} (two);
\path (two) edge[-] [bend left] node [right] {$a=2, R=0$} (dot2) ;
\path (dot2) edge[dashed] [bend left] node [left] {$1-\frac{1}{n}$} (two);
\path (dot2) edge[dashed] [bend left] node [below] {$\frac{1}{n}$} (one);
\path (two) edge [loop right, looseness=15] node [right] {$a=1, R =\frac{1}{2}$}  (two) ;
\path (one) edge[-] [bend left] node [below] {$a=1, R = \frac{1}{2}+\frac{1}{n}$~~~~~~~~~~~} (dot3);
\path (dot3) edge[->] [bend left] (one);
\node[below,font=\huge\bfseries] at (current bounding box.south) {$\Phat$};
\end{tikzpicture}}
\caption{A true MDP $P$ and an MDP $\Phat$ which has constant probability of being sampled from $P$ when $n$ samples are drawn from each state-action pair. Dashed lines are used to indicate all possible stochastic next-state transitions after taking a given action, with each dashed line being annotated with the probability of the particular next-state transition. They differ only in state-action pair $(s,a) = (1,1)$, for which $P(2\mid 1,1) = \frac{1}{n}$ but $\Phat(2 \mid 1,1) = 0$.}
\label{fig:plugin_failure_1}
\end{figure}
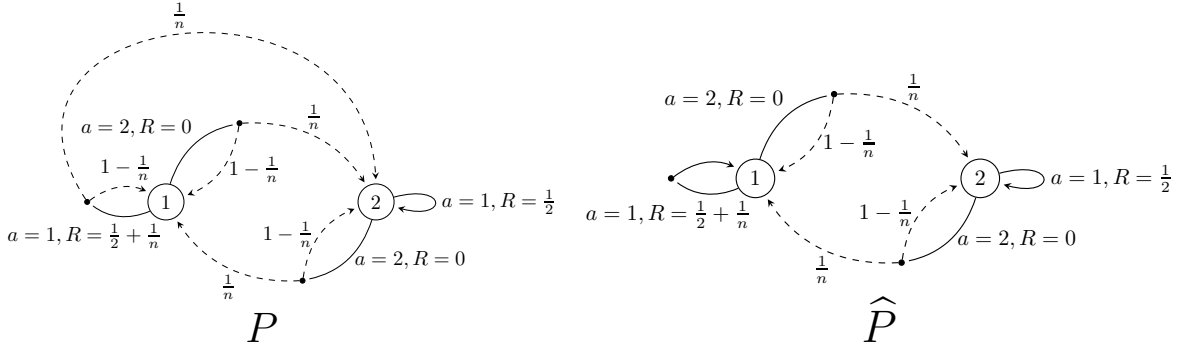

Since there are two states we encode each Markovian deterministic policy as a tuple $(\pi(1), \pi(2))$. First we compute properties of $P$, starting with the gains of the Markovian deterministic policies. It is easy to see that both policies which take action $1$ in state $2$ will stay in state $2$ once reaching it. Also it is easy to see that both policies which take action $2$ in state $2$ will converge to a limiting distribution which is uniform over the two states. Therefore we will have that
\begin{align*}
    \rho^{(1,1)} = \begin{bmatrix}
        \frac{1}{2} \\ \frac{1}{2}
    \end{bmatrix},
    \rho^{(2,1)} = \begin{bmatrix}
        \frac{1}{2} \\ \frac{1}{2}
    \end{bmatrix},
    \rho^{(1,2)} = \begin{bmatrix}
        \frac{1}{4} + \frac{1}{2n} \\ \frac{1}{4} + \frac{1}{2n}
    \end{bmatrix},
    \rho^{(2,2)} = \begin{bmatrix}
        0 \\ 0
    \end{bmatrix}.
\end{align*}
Thus policies $(1,1)$ and $(2,1)$ are both gain-optimal and the optimal gain is $\rho^\star = \frac{1}{2}\one$. Comparing between these two policies it is clear that $(1,1)$ is the only bias-optimal policy (and thus the only Blackwell-optimal policy), since these two gain-optimal policies induce the same distribution over states (the Markov chains $P_{(1,1)}$ and $P_{(2,1)}$ are equal), and thus the only difference is the reward obtained in state $1$, which is larger if action $1$ is taken. 
Now we compute $h^\star$. The Markov chain $P_{ (1,1)}$ is eventually absorbed in state $2$, so (since we must then have $P^\infty_{(1,1)} = \one [0, 1]$ and we have $P^\infty_{(1,1)} h^\star = 0$) it is immediate that $h^\star(2) = 0$.
Using the first row of the equality $\rho^\star + h^\star = r_{(1,1)} + P_{(1,1)}  h^\star$, we have
\begin{align*}
    & \frac{1}{2} + h^\star(1) = \frac{1}{2}+\frac{1}{n} + \left(1-\frac{1}{n} \right) h^\star(1) + \frac{1}{n} h^\star(2) =  \frac{1}{2}+\frac{1}{n} + \left(1-\frac{1}{n} \right) h^\star(1) \\
    \iff & \frac{1}{n}h^\star(1) = \frac{1}{n} \\
    \iff &h^\star(1) = 1.
\end{align*}
Thus $\spannorm{h^\star} = 1$.

It is immediate that the diameter of this MDP is equal to $D = n$, since the expected hitting time of state $2$ from state $1$ (from any policy) is the expected value of a $\text{Geom}(\frac{1}{D})$ RV (with range $\{1, 2, \dots\}$), which is $D$, and likewise for reaching state $1$ from state $2$, only action $2$ leads to this state, and it again has expected hitting time of $n = D$. To calculate $\tmix$ for $P$, since either action taken in state $1$ has the same next-state distribution, we simply need to compute the mixing times of the Markov chains with transition matrices
\begin{align*}
    \begin{bmatrix}
        1-\frac{1}{n} & \frac{1}{n} \\
        \frac{1}{n} & 1 - \frac{1}{n}
    \end{bmatrix},
    \begin{bmatrix}
        1-\frac{1}{n} & \frac{1}{n} \\
        0 & 1
    \end{bmatrix}.
\end{align*}
The first of these matrices is symmetric and (as mentioned before, or is trivial to check) has stationary distribution $[\frac{1}{2}, \frac{1}{2}]$ and is irreducible, and thus its mixing time is bounded by a constant times its relaxation time \citep[Theorem 12.4]{levin_markov_2017}. This matrix has trace $2 - \frac{2}{n}$ and determinant $(1-\frac{1}{n})^2 - \frac{1}{n^2} = 1 - \frac{2}{n}$, so it eigenvalues are $1$ and $1-\frac{2}{n}$. Therefore this matrix has absolute spectral gap $\frac{2}{n}$ and thus relaxation time $\frac{n}{2}$, so it has a mixing time which is $\Theta(n)$. For the second of the matrices in the above display, since the stationary distribution is simply a distribution with all its mass on state $2$, the mixing time is simply the smallest integer $t$ such that $\P_1(\eta_2 \leq t) \geq \frac{1}{2}$ where $\eta_2$ is the hitting time of state $2$ (and $\P_1$ is the probability distribution induced by this Markov chain starting in state $1$). Since $\eta_2 \sim \text{Geom}(\frac{1}{n})$ (where this is the Geometric random variable taking values in the range $\{1,2,\dots\}$), setting $t = n \ln 2 $ we obtain
\[
\P_1(\eta_2 \leq n \ln 2) = 1 - \left(\left(1 -\frac{1}{n} \right)^{n}\right)^{\ln 2} \geq 1 - \left(\frac{1}{e}\right)^{\ln 2} = 1 - \frac{1}{2} = \frac{1}{2}
\]
where we used the fact that $n \mapsto (1-\frac{1}{n})^n$ is an increasing function which approaches $1/e$. Thus the Markov chain associated with this second matrix also has mixing time $\Theta(n)$, so we have $\tmix = \Theta(n)$.

Later we will need the fact that $\infnorm{V_\gamma^\star - V_\gamma^{(1,2)}} \geq \frac{n^2}{5}$, which we will check now by showing that $V_\gamma^{(1,2)}(2) \leq V_\gamma^{\star}(2) - \frac{n^2}{5}$. First we compute $V_\gamma^{(1,2)}$. For convenience we will abbreviate this as $V$. By the Bellman evaluation equations we have
\begin{align*}
    V(1) &= \frac{1}{2} + \frac{1}{n} + \gamma \left(1-\frac{1}{n} \right) V(1) + \frac{\gamma}{n} V(2) \\
    &= \frac{1}{2} + \frac{1}{n} + \left(1-\frac{1}{n^2} \right) \left(1-\frac{1}{n} \right) V(1) + \frac{1-\frac{1}{n^2}}{n} V(2) \\
    &= \frac{1}{2} + \frac{1}{n} + \left(1 -\frac{1}{n}-\frac{1}{n^2} + \frac{1}{n^3} \right) V(1) + \frac{n^2 - 1}{n^3} V(2) \\
    &= \frac{1}{2} + \frac{1}{n} + \left(1 -\frac{n^2 + n - 1}{n^3} \right) V(1) + \frac{n^2 - 1}{n^3} V(2)
\end{align*}
which implies
\begin{align}
    V(1) = \frac{n^3}{n^2 + n - 1} \left(\frac{1}{2}+\frac{1}{n} + \frac{n^2 - 1}{n^3}V(2) \right). \label{eq:plugin_hard_ex_1}
\end{align}
Writing the other evaluation equation, we have
\begin{align*}
    & V(2) = \gamma \left(1 - \frac{1}{n} \right)V(2) + \frac{\gamma}{n} V(1) \\
    \iff & \left(1 - \left(1 - \frac{1}{n^2} - \frac{1}{n} + \frac{1}{n^3} \right) \right) V(2) =  \left(1 - \gamma \left(1 - \frac{1}{n} \right) \right) V(2) = \frac{\gamma}{n} V(1) = \frac{1 - \frac{1}{n^2}}{n} V(1) \\
    \iff & \frac{n^2 + n - 1}{n^3} V(2) = \frac{n^2 - 1}{n^3} V(1).
\end{align*}
Substituting~\eqref{eq:plugin_hard_ex_1} into the above display,
\begin{align*}
    &\frac{n^2 + n - 1}{n^3} V(2) = \frac{n^2 - 1}{n^3} \frac{n^3}{n^2 + n - 1} \left(\frac{1}{2}+\frac{1}{n} + \frac{n^2 - 1}{n^3}V(2) \right) \\
    \iff & \left(\frac{n^2 + n - 1}{n^3} - \frac{n^2 - 1}{n^2 + n - 1} \frac{n^2 - 1}{n^3}\right) V(2) = \frac{n^2 - 1}{n^2 + n - 1} \left( \frac{1}{2} + \frac{1}{n}\right) \\
    \iff & \left(\frac{(n^2 + n - 1)^2 - (n^2 - 1)^2}{(n^2 + n - 1)n^3}\right) V(2) = \frac{n^2 - 1}{n^2 + n - 1} \left( \frac{1}{2} + \frac{1}{n}\right) \\
    \iff & \left(\frac{(n^2 + n - 1)^2 - (n^2 - 1)^2}{n^3}\right) V(2) = (n^2 - 1) \left( \frac{1}{2} + \frac{1}{n}\right) \\
    \iff & \left(\frac{n(2n^2 + n - 2)}{n^3}\right) V(2) = (n^2 - 1) \left( \frac{1}{2} + \frac{1}{n}\right) \\
    \iff & \frac{2n^2 + n - 2}{n^2 - 1}V(2) = n^2 \left( \frac{1}{2} + \frac{1}{n}\right).
\end{align*}
Therefore
\[
2 V(2) < \frac{2n^2 + n - 2}{n^2 - 1}V(2) = n^2 \left( \frac{1}{2} + \frac{1}{n}\right) < n^2 \left( \frac{1}{2} + \frac{1}{10}\right)
\]
(using that $n \geq 10$) so $V(2) < n^2\frac{1}{2}\left( \frac{1}{2} + \frac{1}{10}\right) = n^2\frac{3}{10}$.

Now we can easily observe $V^\star_\gamma(2) \geq \frac{1}{1-\gamma}\frac{1}{2} = n^2 \frac{1}{2}$ since there exists a policy which takes action 1 in that state and is absorbed there, collecting reward $\frac{1}{2}$ at all times (in fact this action is optimal but we don't need to check this). Therefore we have that
\[
    V^\star_\gamma(2) - V^{(1,2)}_\gamma(2) \geq n^2\frac{1}{2} - n^2\frac{3}{10} = \frac{n^2}{5},
\]
which implies that $\infnorm{V^\star_\gamma - V^{(1,2)}_\gamma} \geq \frac{n^2}{5}$.

Now we check that the probability of $\Phat$ being equal to the instance displayed above is at least $\frac{1}{25}$. There are 4 state-action pairs which are sampled independently so we can compute the probability for each state-action pair separately. There are only two states, so we can encode $\Phat$ with the values of a random variable $N(s,a)$ which for each $s,a$ counts how many transitions to state $1$ are observed. We have $N(1,1) \sim \text{Binom}(n, 1-\frac{1}{n})$, $N(1,2) \sim \text{Binom}(n, 1-\frac{1}{n})$, $N(2,1) \sim \text{Binom}(n, 0)$, and $N(2,2) \sim \text{Binom}(n, \frac{1}{n})$. With this definition of $N$, we get the $\Phat$ displayed in Figure \ref{fig:plugin_failure_1} if we have $N(1,1) = n$, $N(1,2) = n-1$, $N(2,1) = 0$, $N(2,2) = 1$. By independence and the Binomial pmf we have
\begin{align*}
    &\P\left(N(1,1) = n, N(1,2) = n-1, N(2,1) = 0, N(2,2) = 1 \right) \\
    &= \P(N(1,1) = n) \P(N(1,2) = n-1) \P(N(2,1) = 0) \P(N(2,2) = 1) \\
    &= \left(1-\frac{1}{n} \right)^n \cdot n \frac{1}{n} \left(1-\frac{1}{n} \right)^{n-1} \cdot 1 \cdot n \frac{1}{n} \left(1-\frac{1}{n} \right)^{n-1} \\
    &= \left(1-\frac{1}{n} \right)^{3n-2} \\
    & \geq \left(\left(1-\frac{1}{n} \right)^n\right)^3
\end{align*}
and it is a standard fact that this final expression is increasing in $n$, so we can lower bound it by plugging in the lowest value $n=10$ for which we obtain $\left(\frac{9}{10} \right)^{30} > 0.04 = \frac{1}{25}$. (As $n \to \infty$ this approaches $\frac{1}{e^3}$.)

From here we operate on this event that $\Phat$ is equal to the instance shown in Figure \ref{fig:plugin_failure_1}. It is easy to see that $\Phat$ is communicating, since in both states action $2$ has positive probability of leading to either state. 

First we compute the Blackwell optimal policy $\pihstar$ of $\Phat$. It is easy to see that
\begin{align*}
    \rhohat^{(1,1)} = \begin{bmatrix}
        \frac{1}{2} + \frac{1}{n} \\ \frac{1}{2}
    \end{bmatrix},
    \rhohat^{(1,2)} = \begin{bmatrix}
        \frac{1}{2} + \frac{1}{n} \\ \frac{1}{2} + \frac{1}{n}
    \end{bmatrix},  
    \rhohat^{(2,1)} = \begin{bmatrix}
        \frac{1}{2} \\ \frac{1}{2}
    \end{bmatrix},
     \rhohat^{(2,2)} = \begin{bmatrix}
        0 \\ 0
    \end{bmatrix}.
\end{align*}
Therefore the only Blackwell optimal policy is the only gain-optimal policy, $(1,2)$. 
As we have already checked, this policy has suboptimality (in the true $P$)
\[
\infnorm{\rho^{(1,2)} - \rho^\star} = \frac{1}{2} - \left(\frac{1}{4} + \frac{1}{2n}\right) \geq \frac{1}{4} - \frac{1}{20} = \frac{1}{5}.
\]

Next we compute the discounted optimal policy for effective horizon $\frac{1}{1-\gamma} = n^2$. It is obvious that the optimal action in state $1$ will be action $1$, so we will compute and compare the value functions $\Vhat^{(1,1)}$ and $\Vhat^{(1,2)}$. It is easy to see that
\begin{align}
    \Vhat^{(1,1)} = \frac{1}{1-\gamma}\begin{bmatrix}
    \frac{1}{2} + \frac{1}{n } \\ \frac{1}{2} 
\end{bmatrix} \label{eq:1,1policy_valfn}
\end{align}
since both states are absorbing under this policy. Now we compute $\Vhat^{(1,2)}$. First, since state $1$ is absorbing, it is immediate that $\Vhat^{(1,2)}(1) = \frac{1}{1-\gamma} \left(\frac{1}{2} +\frac{1}{n} \right) = n^2 \left(\frac{1}{2} +\frac{1}{n} \right)$. From the Bellman evaluation equation for state $2$ we have that
\begin{align*}
    & \Vhat^{(1,2)}(2) = 0 + \gamma \frac{1}{n} \Vhat^{(1,2)}(1) + \gamma \left(1-\frac{1}{n}\right)   \Vhat^{(1,2)}(2) = \gamma n \left(\frac{1}{2} +\frac{1}{n} \right) + \gamma \left(1-\frac{1}{n}\right)   \Vhat^{(1,2)}(2) \\
    \iff & \left(1 - \gamma \left(1-\frac{1}{n}\right) \right) \Vhat^{(1,2)}(2) = \gamma n \left(\frac{1}{2} +\frac{1}{n} \right) \\
    \iff & \left(1 -  \left(1-\frac{1}{n^2}\right) \left(1-\frac{1}{n}\right) \right) \Vhat^{(1,2)}(2) =  \left(1-\frac{1}{n^2}\right) n \left(\frac{1}{2} +\frac{1}{n} \right) \\
    \iff & \frac{n^2 + n - 1}{n^3} \Vhat^{(1,2)}(2) =  \left(1-\frac{1}{n^2}\right) n \left(\frac{1}{2} +\frac{1}{n} \right) = \frac{n}{2} + 1 - \frac{1}{2n} - \frac{1}{n^2}  = \frac{\frac{n^3}{2} + n^2 - \frac{n}{2} - 1}{n^2}\\
     \iff &  \Vhat^{(1,2)}(2)  = \frac{n^3}{n^2 + n - 1} \frac{\frac{n^3}{2} + n^2 - \frac{n}{2} - 1}{n^2} = \frac{\frac{n^4}{2}+ n^3 - \frac{n^2}{2} - n}{n^2 + n - 1} > \frac{\frac{n^4}{2}+ \frac{n^3}{2} - \frac{n^2}{2}}{n^2 + n - 1} = \frac{n^2}{2}
\end{align*}
where the final strict inequality requires $\frac{n^3}{2}-n > 0$, which holds for all $n \geq 2$. Thus we have shown that $\Vhat^{(1,2)}(2) > \frac{n^2}{2} = \frac{1}{1-\gamma}\frac{1}{2} = \Vhat^{(1,1)}(2)$, so the optimal policy for the DMDP with horizon $\frac{1}{1-\gamma}$ is $(1,2)$.
As we have previously checked, this policy has suboptimality in $P$ at least $\frac{n^2}{5}$.

Now we compute $\spannorm{\hhat^\star} = \spannorm{\hhat^{\pihstar}} = \spannorm{\hhat^{(1,2)}}$. Policy $(1,2)$ induces the Markov chain transition matrix
\begin{align*}
    \Phat_{(1,2)} = \begin{bmatrix}
        1 & 0 \\
    \frac{1}{n} & 1-\frac{1}{n}
    \end{bmatrix}
\end{align*}
and this Markov chain has a stationary distribution which is a point mass on state $1$ (since state $1$ is absorbing and state $2$ has positive probability of reaching state $1$ in one step). Therefore from the fact that $P_{\pi}^\infty h^\pi = 0$ for any policy $\pi$ and any MDP, we must have that $\hhat^{(1,2)}(1) = 0$. Now to calculate $\hhat^{(1,2)}(2)$, the Poisson equation $\rhohat^{(1,2)} + \hhat^{(1,2)} = r_{(1,2)} + \Phat_{(1,2)} \hhat^{(1,2)}(2)$ gives that
\begin{align*}
    \hhat^{(1,2)}(2) = 0 - \left(\frac{1}{2} + \frac{1}{n}\right) + \frac{1}{n}\hhat^{(1,2)}(1) +\left(1 - \frac{1}{n}\right)\hhat^{(1,2)}(2) = - \left(\frac{1}{2} + \frac{1}{n}\right)  +\left(1 - \frac{1}{n}\right)\hhat^{(1,2)}(2)
\end{align*}
(by looking at the second row of this system of equations). Solving yields that $\hhat^{(1,2)}(2) = -n(\frac{1}{2} + \frac{1}{n}) = -\frac{n}{2} - 1$. Therefore $\spannorm{\hhat^{(1,2)}} = \hhat^{(1,2)}(1) - \hhat^{(1,2)}(2) = \frac{n}{2}+1$.

Next we show that $\empbod \leq 1 - \frac{1}{n^3}$, or equivalently that $\frac{1}{1-\empbod} \leq n^3$. Now we let $\gammared = 1-\frac{1}{n^3}$. From the computation~\eqref{eq:1,1policy_valfn}, we have that
\begin{align*}
    \Vhat_{\gammared}^{(1,1)} = n^3\begin{bmatrix}
    \frac{1}{2} + \frac{1}{n } \\ \frac{1}{2} 
\end{bmatrix},
\end{align*}
and also since state $1$ is absorbing under $\Phat_{(1,2)}$ it is obvious that $\Vhat^{(1,2)}_{\gammared}(1) = n^3 \left( \frac{1}{2} + \frac{1}{n }\right)$. 

Thus by the definition of $\empbod$~\eqref{eq:bias_optimal_discount_cond}, it suffices to show that $\Vhat_{\gammared}^{\star} = \Vhat_{\gammared}^{(1,2)}$ and that
\begin{align}
    \infnorm{\Vhat_{\gammared}^{\star} - \frac{1}{1-\gammared} \rhohat^{(1,2)} -\hhat^{(1,2)} } = \infnorm{\Vhat_{\gammared}^{\star} - n^3 \left(\frac{1}{n} + \frac{1}{2} \right)\one  - \begin{bmatrix}
        0 \\ -\frac{n}{2} - 1
    \end{bmatrix} } \leq \frac{1}{n}. \label{eq:example_empbod_cond_1}
\end{align}
To show $\Vhat_{\gammared}^{\star} = \Vhat_{\gammared}^{(1,2)}$, it only remains to show that $\Vhat_{\gammared}^{(1,2)}(2) > \Vhat_{\gammared}^{(1,1)}(2) = \frac{n^3}{2}$ (since, similarly to before, it is obvious that the optimal action in state $1$ is action $1$). To check~\eqref{eq:example_empbod_cond_1}, assuming we have shown that $\Vhat_{\gammared}^{\star} = \Vhat_{\gammared}^{(1,2)}$, it would remain to show that
\begin{align}
     -\frac{1}{n} \leq \Vhat_{\gammared}^{(1,2)}(2) - n^3 \left(\frac{1}{n} + \frac{1}{2} \right) \leq \frac{1}{n} \label{eq:example_empbod_cond_2}
\end{align}
(since we know that $\Vhat^{(1,2)}_{\gammared}(1) = n^3 \left( \frac{1}{2} + \frac{1}{n }\right)$).
But then~\eqref{eq:example_empbod_cond_2} would imply that $\Vhat_{\gammared}^{(1,2)}(2) > \Vhat_{\gammared}^{(1,1)}(2) = \frac{n^3}{2}$ (and thus that $\Vhat_{\gammared}^{\star} = \Vhat_{\gammared}^{(1,2)}$), so we actually can conclude by simply checking~\eqref{eq:example_empbod_cond_2}. From the (discounted) Bellman equation and the fact that $\Vhat_{\gammared}^{(1,2)}(1) = n^3 \left( \frac{1}{2} + \frac{1}{n }\right)$, we have that
\begin{align*}
    & \Vhat_{\gammared}^{(1,2)}(2) = 0 + \left(1-\frac{1}{n^3}\right) \left(\frac{1}{n}n^3 \left( \frac{1}{2} + \frac{1}{n }\right) + \left(1-\frac{1}{n}\right) \Vhat_{\gammared}^{(1,2)}(2)\right) \\
    \iff & \left(1 - \left(1-\frac{1}{n^3}\right) \left(1-\frac{1}{n}\right)\right) \Vhat_{\gammared}^{(1,2)}(2) = \left(1-\frac{1}{n^3}\right) \frac{1}{n} n^3 \left( \frac{1}{2} + \frac{1}{n }\right) \\
    \iff & \frac{n^3 + n - 1}{n^4} \Vhat_{\gammared}^{(1,2)}(2) = \frac{\frac{n^6}{2} + n^5 - \frac{n^3}{2} - n^2}{n^4} \\
    \iff & \Vhat_{\gammared}^{(1,2)}(2) = \frac{\frac{n^6}{2} + n^5 - \frac{n^3}{2} - n^2}{n^3 + n - 1}.
\end{align*}
Now to check~\eqref{eq:example_empbod_cond_2} we need to show that this expression for $\Vhat_{\gammared}^{(1,2)}(2)$ satisfies
\begin{align}
    \frac{\frac{n^6}{2} + n^5 - \frac{n^3}{2} - n^2}{n^3 + n - 1} & \leq n^3 \left(\frac{1}{2} + \frac{1}{n} \right) - \frac{n}{2} - 1 + \frac{1}{n}= \frac{n^3}{2} + n^2 - \frac{n}{2} - 1 + \frac{1}{n}\label{eq:example_empbod_cond_3}
\end{align}
and
\begin{align}
    \frac{\frac{n^6}{2} + n^5 - \frac{n^3}{2} - n^2}{n^3 + n - 1} & \geq n^3 \left(\frac{1}{2} + \frac{1}{n} \right) - \frac{n}{2} - 1 -  \frac{1}{n} = \frac{n^3}{2} + n^2 - \frac{n}{2} - 1 -  \frac{1}{n}.\label{eq:example_empbod_cond_4}
\end{align}

Multiplying both sides of~\eqref{eq:example_empbod_cond_4} by $n^3 + n - 1$ and canceling common terms, we obtain the statement $0 \geq \frac{-3}{2}n^2 - \frac{n}{2} + \frac{1}{n}$, which is clearly true whenever $n \geq 1$. Multiplying both sides of~\eqref{eq:example_empbod_cond_3} by $n^3 + n - 1$ and again canceling common terms, we obtain the statement $0 \leq \frac{n^2}{2} - \frac{n}{2} + 2 - \frac{1}{n}$ which is also true whenever $n \geq 1$. Therefore we have shown that $\empbod \leq 1 - \frac{1}{n^3}$.

Now it remains to check the final two statements of the theorem. Fix a constant $C > 0$. Then since the term $\sqrt{\frac{ \log\left( n \right)}{n}}$ goes to $0$ as $n \to \infty$, we can choose $n $ sufficiently large so that $C\sqrt{\frac{ \log\left( n \right)}{n}} < \frac{1}{5}$. Then considering the instance $P$ constructed with the parameter $n$, since it has $\spannorm{h^\star} = 1$, we have
\[
C\sqrt{\frac{ \log\left( n \right)}{n}} = C \sqrt{\frac{\spannorm{h^\star} \log\left(\spannorm{h^\star} n \right)}{n}} < \frac{1}{5}.
\]
As we have argued, there is probability at least $\frac{1}{25}$ that $\Phat$ is sampled from $P$, and under this event, we have both $\infnorm{\rho^\star - \rho^{\pihstar}} \geq  \frac{1}{5}$ and $\infnorm{V_\gamma^\star - V_\gamma^{\pihstar}} \geq \frac{n^2}{5} = \frac{1}{1-\gamma}\frac{1}{5}$ (where we choose $\gamma = 1 - \frac{1}{n^2}$). Therefore the statements
\begin{gather*}
    \P \left( \rho^{\pihstar} \geq \rho^\star - C \sqrt{\frac{\spannorm{h^\star} \log\left(\spannorm{h^\star} n \right)}{n}} \right) > 1 - \frac{1}{25} \\
    \P \left( V_\gamma^{\pihstar_\gamma} \geq V_\gamma^{\star} - C \frac{1}{1-\gamma}\sqrt{\frac{\spannorm{h^\star} \log\left(\spannorm{h^\star} n \right)}{n}} \right) > 1 - \frac{1}{25}
\end{gather*}
are both false.
\end{proof}

\end{document}